\documentclass[12pt,english]{article}
\usepackage[T1]{fontenc}
\usepackage[latin9]{inputenc}
\usepackage{geometry}
\usepackage{babel}
\usepackage{verbatim}
\usepackage{float}
\usepackage{booktabs}
\usepackage{amsmath}
\usepackage{amsthm}
\usepackage{amssymb}
\usepackage{setspace}
\usepackage[unicode=true,pdfusetitle,
 bookmarks=true,bookmarksnumbered=true,bookmarksopen=false,
 breaklinks=false,pdfborder={0 0 1},backref=false,colorlinks=false]
 {hyperref}

\makeatletter

\newcommand*\LyXZeroWidthSpace{\hspace{0pt}}
\providecommand{\tabularnewline}{\\}
\floatstyle{ruled}
\newfloat{algorithm}{tbp}{loa}
\providecommand{\algorithmname}{Algorithm}
\floatname{algorithm}{\protect\algorithmname}

\numberwithin{equation}{section}
\theoremstyle{plain}
\newtheorem{thm}{\protect\theoremname}[section]
\theoremstyle{definition}
\newtheorem{example}{\protect\examplename}[section]
\theoremstyle{definition}
\newtheorem{defn}{\protect\definitionname}[section]
\theoremstyle{plain}
\newtheorem{prop}{\protect\propositionname}[section]
\theoremstyle{plain}
\newtheorem{lem}{\protect\lemmaname}[section]

\usepackage{xfrac}

\makeatother

\providecommand{\definitionname}{Definition}
\providecommand{\examplename}{Example}
\providecommand{\lemmaname}{Lemma}
\providecommand{\propositionname}{Proposition}
\providecommand{\theoremname}{Theorem}

\begin{document}
\begin{titlepage}
\title{Imprecise Multi-Armed Bandits}
\author{\textbf{Author:} Vanessa Kosoy\\
\textbf{Advisor:} Shai Shalev-Shwartz}
\maketitle
\begin{center}
A thesis submitted for the degree of
\par\end{center}

\begin{center}
Master of Science
\par\end{center}

\begin{center}
at the
\par\end{center}

\begin{center}
Hebrew University of Jerusalem
\par\end{center}

\begin{center}
The Rachel and Selim Benin School of Computer Science
\par\end{center}

\end{titlepage}
\begin{abstract}
We introduce a novel multi-armed bandit framework, where each arm
is associated with a fixed unknown \emph{credal set} over the space
of outcomes (which can be richer than just the reward). The arm-to-credal-set
correspondence comes from a known class of hypotheses. We then define
a notion of regret corresponding to the \emph{lower prevision} defined
by these credal sets. Equivalently, the setting can be regarded as
a two-player \emph{zero-sum game}, where, on each round, the agent
chooses an arm and the adversary chooses the distribution over outcomes
from a set of options associated with this arm. The regret is defined
with respect to the value of game. For certain natural hypothesis
classes, loosely analgous to stochastic linear bandits (which are
a special case of the resulting setting), we propose an algorithm
and prove a corresponding upper bound on regret. We also prove lower
bounds on regret for particular special cases.

\pagebreak{}

\tableofcontents{}

\pagebreak{}
\end{abstract}

\section{\label{sec:Notation}Notation}

\global\long\def\SC#1#2{\left\{  #1\,\middle|\,#2\right\}  }%

\global\long\def\E#1#2{\mathrm{E}_{#1}\left[#2\right]}%

\global\long\def\A{\mathcal{A}}%

\global\long\def\B{\mathcal{B}}%

\global\long\def\N{\mathbb{N}}%

\global\long\def\R{\mathbb{R}}%

\global\long\def\D{\mathcal{D}}%

\global\long\def\H{\mathcal{H}}%

\global\long\def\U{\mathcal{U}}%

\global\long\def\T#1{#1^{\mathrm{t}}}%

\global\long\def\DTV#1#2{\mathrm{d}_{\mathrm{TV}}\left(#1,#2\right)}%

\global\long\def\DKL#1#2{D_{\text{KL}}\left(#1\middle|\!\middle|#2\right)}%

$\N$ denotes the set of natural numbers, which we define to include
$0$. For $n\in\N$, $[n]$ stands for $\SC{m\in\N}{m<n}$

$\R$ denotes the set of real numbers. $\R_{+}$ stands for $\SC{x\in\R}{x\geq0}$.
$\R_{>0}$ stands for $\SC{x\in\R}{x>0}$.

For sets $A$ and $B$, $A\subseteq B$ means that $A$ is a subset
of $B$ and $A\subset B$ means that $A$ is a proper subset of $B$.
$2^{B}$ stands for power set of $B$.

The notation $\boldsymbol{1}_{A}$ will mean $1$ if condition $A$
is true and $0$ if condition $A$ is false. By abuse of notation,
$\boldsymbol{1}_{A}$ can also denote the characteristic function
of $A$. That is, given a set $B$ and some $A\subseteq B$, we have
$\boldsymbol{1}_{A}:B\rightarrow\{0,1\}$, where $\boldsymbol{1}_{A}(x)=1$
if and only if $x\in A$. ($B$ is assumed to be understood from context.)

For $n\in\N$, $v\in\R^{n}$ and $0\leq k<n$, $v_{k}$ denotes the
$k$-th component of $v$ (where the indices start from $0$). $\left\Vert v\right\Vert _{p}$
stands for the $\ell_{p}$ norm of $v$, i.e.

\[
\left\Vert v\right\Vert _{p}:=\sqrt[p]{\sum_{k\in[n]}v_{k}^{p}}
\]

We regard elements of $\R^{n}$ as column vectors. For $v\in\R^{n}$,
$\T v$ denotes the row vector which is the transpose of $v$. For
$M$ a matrix, $\T M$ denotes the transpose of $M$.

For a finite set $A$, $\R^{A}$ denotes the vector space of functions
from $A$ to $\R$. We use vector notation for such functions, i.e.
if $v\in\R^{A}$ and $a\in A$ then the value of the function $v$
at $a$ is denoted $v_{a}$.

For $n\in\N$, $I_{n}$ denotes the $n\times n$ identity matrix.

For $\mathcal{X}$ a vector space, $\mathcal{X}^{\star}$ denotes
the dual space.

For $\mathcal{X}$ and $\mathcal{Y}$ vector spaces, $\mathcal{X\oplus\mathcal{Y}}$
denotes their direct sum. For $x\in\mathcal{X}$ and $\alpha\in\mathcal{Y}^{\star}$,
$x\otimes\alpha$ denotes the linear operator from $\mathcal{Y}$
to $\mathcal{X}$ given by

\[
\left(x\otimes\alpha\right)(y):=\alpha(y)x
\]

For $A:\mathcal{X}\rightarrow\mathcal{Y}$ a linear operator, $\ker A\subseteq\mathcal{X}$
will denote the kernel of $A$.

For $\mathcal{X}$ an inner product space and $\U$ a subspace, $\U^{\bot}$
denotes the orthogonal complement of $\U$.

For $\mathfrak{A}$ an affine space, $\overrightarrow{\mathfrak{A}}$
denotes its tangent vector space.

For $X$ a topological space, $\Delta X$ denotes the space of Borel
probability measures on $X$. When $X$ is a finite set (equipped
with the discrete topology), we can regard $\Delta X$ as the subset
of $\R^{X}$ consisting of $v\in\R^{X}$ s.t. for all $x\in X$, $v_{x}\geq0$
and also $\sum_{x}v_{x}=1$.

For $X,Y$ topological spaces and $f:X\rightarrow Y$ a measurable
mapping, $f_{*}:\Delta X\rightarrow\Delta Y$ stands for pushforward
by $f$.

For $X$ a topological space and $\zeta,\xi\in\Delta X$, the notation
$\DTV{\zeta}{\xi}$ stands for the total variation distance between
$\zeta$ and $\xi$. That is,

\[
\DTV{\zeta}{\xi}:=\frac{1}{2}\int_{X}\left|\zeta\left(dx\right)-\xi\left(dx\right)\right|
\]

In the same context, the notation $\DKL{\zeta}{\xi}$ stands for the
Kullback-Leibler divergence of $\zeta$ from $\xi$. That is,

\[
\DKL{\zeta}{\xi}:=\int_{X}\zeta(dx)\ln\frac{\zeta(dx)}{\xi(dx)}
\]

For $\Sigma$ a set (an ``alphabet''), $\Sigma^{*}$ stands for
$\sqcup_{n=0}^{\infty}\Sigma^{n}$ (finite strings over $\Sigma$).
$\Sigma^{\omega}$ denotes infinite sequences (strings) over $\Sigma$.
For $x$ in $\Sigma^{*}$, $|x|\in\N$ stands for the length of $x$
(i.e. $x\in\Sigma^{|x|}$). For $x$ in $\Sigma^{*}$ or $\Sigma^{\omega}$
and any $n\in\N$, $x_{n}$ stands for the $n$-th symbol in the string,
where we start counting from 0 (i.e. $x=x_{0}x_{1}\ldots$) For any
$m,l\in\N$, $x_{:m}$ stands for $x_{0}x_{1}\ldots x_{m-1}$ and
$x_{l:m}$ stands for $x_{l}x_{l+1}\ldots x_{m-1}$. For $y\in\Sigma^{*}$,
$y\sqsubseteq x$ means that $y$ is a prefix of $x$.

\section{Background}

\subsection{Stochastic Finite Multi-Armed Bandits}

\global\long\def\Pol{\varphi}%

\global\long\def\Reg#1#2#3{\mathrm{ERg}_{#1}\left(#2;#3\right)}%

\global\long\def\Argmax#1{\underset{#1}{\mathrm{argmax}}}%

\global\long\def\Argmin#1{\underset{#1}{\mathrm{argmin}}}%

\global\long\def\RI{[-1,+1]}%

\global\long\def\EH#1#2{\mathrm{E}_{#1}\left(#2\right)}%

\global\long\def\OH#1{\mathrm{E}_{#1}^{*}}%

The theory of stochastic multi-armed bandits is one of the simplest
models of sequential decision making. There is a finite or infinite
sequence of rounds, and on each round an agent chooses an ``arm''
(action) out of some set $\A$ (which we assume to be finite for now)
and observes a reward in $\R$. For simplicity, we'll assume that
rewards take values in some finite interval $[r_{\min},r_{\max}]$\footnote{The theory is of stochastic bandits is more general than this: for
example, most results hold even if the distribution of rewards is
unbounded but subgaussian. However, we will only consider the bounded
case.}. Moreover, we will assume this interval to be $\RI$. This loses
no generality, since we can always redefine rewards according to

\[
r':=\frac{2r-\left(r_{\min}+r_{\max}\right)}{r_{\max}-r_{\min}}
\]

There is some function $H^{*}:\A\rightarrow\Delta\RI$ s.t. whenever
the agent selects arm $x\in\A$, the reward is drawn from the distribution
$H^{*}(x)$, independently of anything that happened before. However,
$H^{*}$ is unknown to the agent. We call $H^{*}$ the ``true environment''.

An \emph{agent policy} is a (Borel measurable) mapping $\Pol:(\A\times\RI)^{*}\rightarrow\A$.
The meaning of $\Pol(h)$ is, the action the agent takes given a past
history $h$ of actions and rewards. An agent policy $\varphi$ together
with an environment $H:\A\rightarrow\Delta\RI$ define a distribution
$\Pol H\in\Delta(\A\times\RI)^{\omega}$ in a natural way. That is,
sampling $xr\sim\Pol H$ is performed recursively according to

\[
\begin{cases}
x_{n} & =\Pol\left(x_{0}r_{0}\ldots x_{n-1}r_{n-1}\right)\\
r_{n} & \sim H\left(x_{n}\right)
\end{cases}
\]

Given $H:\A\rightarrow\Delta\RI$ and $x\in\A$, we will use the shorthand
notation

\[
\EH Hx:=\E{r\sim H(x)}r
\]

That is, $\EH Hx$ is the expected reward of arm $x$ in environment
$H$.

We also denote

\[
\OH H:=\max_{x\in\A}\EH Hx
\]

That is, $\OH H$ is the expected reward of the best arm.

The goal of the agent is maximizing its expected sum of rewards over
time, despite its uncertainty about $H^{*}$. To formalize this, we
introduce the following performance measure. Given a policy $\Pol$,
an environment $H:\A\rightarrow\Delta\RI$ and a time horizon $N\in\N$,
the \emph{expected regret} is

\[
\Reg H{\Pol}N:=N\cdot\OH H-\sum_{n=0}^{N-1}\E{xr\sim\Pol H}{r_{n}}
\]

That is, the expected regret is the difference between the expected
total reward that \emph{would} be obtainable in environment $H$ \emph{if}
it was known to the agent, with the expected total reward actually
obtained by $\Pol$ in environment $H$. 

The central problem of stochastic bandit theory is finding algorithms
that guarantee upper bounds on expected regret. Many of these algorithms
are variations on the theme of UCB (see Algorithm \ref{alg:UCB}).

\begin{algorithm}
\caption{\label{alg:UCB}Upper Confidence Bound (UCB)}

\begin{onehalfspace}
\textbf{Input} $N$ (a positive integer)

\textbf{for} every $x\in\A$:

\qquad{}select arm $x$ and receive reward $r$

\qquad{}$r_{x}\leftarrow r$

\qquad{}$t_{x}\leftarrow1$

\textbf{end for}

\textbf{for} every following round:

\qquad{}$x^{*}\leftarrow\Argmax{x\in\A}\left(\frac{r_{x}}{t_{x}}+2\sqrt{\frac{\ln N}{t_{x}}}\right)$

\qquad{}select arm $x^{*}$ and receive reward $r$

\qquad{}$r_{x^{*}}\leftarrow r_{x^{*}}+r$

\qquad{}$t_{x^{*}}\leftarrow t_{x^{*}}+1$

\textbf{end for}
\end{onehalfspace}
\end{algorithm}

Essentially, UCB calculates an \emph{optimistic estimate} of $\E{r\sim H(x)}r$
for each $x\in\A$, and selects the arm for which this estimate is
the highest. Intuitively, the reason we want the estimate to be optimistic
is: if it turns out to be accurate then we are in the best-case scenario,
whereas if it turns out to be inaccurate then we obtained new information
that allows us to substantially shrink that arm's ``confidence interval''
(which cannot happen too frequently, since we can only shrink it that
much).

Denote $\Pol_{\text{UCB}}^{N}$ the policy implemented by Algorithm
\ref{alg:UCB}. We have the following (see Theorem 7.2 in \cite{lattimore2020bandit}\footnote{The assumption there is that the reward distribution is 1-subgaussian.
Our setting is a special case since we assume the reward to be bounded
in $\RI$.}):
\begin{thm}
\label{thm:reg-fin-a}Let $N$ be a positive integer. Then, for any
$H:\A\rightarrow\Delta\RI$

\[
\Reg H{\Pol_{\text{UCB}}^{N}}N\leq8\sqrt{|\A|N\ln N}+3|\A|
\]
\end{thm}
In fact, a slightly different algorithm admits a regret bound of $O(\sqrt{|\A|N})$,
without the $\ln N$ factor (see Theorem 9.1 in \cite{lattimore2020bandit}).
This is the best possible asymptotic, as follows from the following
(see Theorem 3.4 in \cite{DBLP:journals/ftml/BubeckC12}):
\begin{thm}
For any policy $\varphi$ and positive integer $N$, there exists
$H:\A\rightarrow\Delta\{-1,+1\}$ s.t.

\[
\Reg H{\Pol}N\geq\frac{1}{10}\sqrt{|\A|N}
\]
\end{thm}
An even stronger asymptotic bound can be proved, if we allow the coefficient
to depend on $H$. 

Given an environment $H$, we define the associated \emph{gap} $g(H)$
to be

\[
g(H):=\OH H-\max_{x\in\A:\EH Hx<\OH H}\EH Hx
\]

That is, $g(H)$ is the difference between the expected reward of
the best arm and the expected reward of the \emph{next-best} arm\footnote{In principle, it's possible that all arms have the same expected reward,
in which case the gap is undefined. However, in that case expected
regret vanishes identically for all policies, so it's not especially
interesting.}. This allows formulating the following bound (see Theorem 7.1 in
\cite{lattimore2020bandit}).
\begin{thm}
\label{thm:gap-reg-fin-a}Let $N$ be a positive integer and $H:\A\rightarrow\Delta\RI$
s.t. $g(H)$ is well-defined. Then,

\[
\Reg H{\Pol_{\text{UCB}}^{N}}N\leq\frac{16|\A|\ln N}{g(H)}+3|\A|
\]
\end{thm}

\subsection{Stochastic Linear Multi-Armed Bandits}

\global\long\def\AP{\mathfrak{P}}%

\global\long\def\SUBSTACK#1#2{\substack{#1\\
 #2 
}
 }%

One major limitation of Theorems \ref{thm:reg-fin-a} and \ref{thm:gap-reg-fin-a}
is that these bounds scale with $|\A|$. In many applications, we
want to consider situations where the number of arms is very large
or even infinite (the arms form a continuous space). Deriving a useful
bound in those cases requires assuming some kind of regularity in
how the reward distribution depends on the arm. One of the simplest
possible assumptions of this type is, assuming the arms can be embedded
into some feature space s.t. the expected reward is a linear function
of the features.

Formally, we fix $D\in\N$ and let $\A$ be a compact subset of $\R^{D}$(which
is allowed to be infinite). As before, on every round the agent selects
an arm $x\in\A$ and observes a reward in $\RI$. We assume that there
exists some $\theta^{*}\in\R^{D}$ s.t. whenever arm $x$ is selected,
the expectation of the resulting reward is $\T x\theta^{*}$. The
exact shape of the distribution is allowed to depend on the entire
history of arms and rewards in an arbitrary way\footnote{This is different from the previous subsection, where the full shape
of the distribution was fixed per arm. However, this difference is
not essential: we could have made a similar generalization in the
finite $\A$ case, which would cost us merely worse coefficients in
the regret bounds. It only requires using the Azuma-Hoeffding inequality
instead of the Hoeffding inequality in the proof.}. Analogically to $H^{*}$ in the previous section, $\theta^{*}$
is unknown to the agent. Notice, however, that because the reward
range is $\RI$, $\theta^{*}$ has lie in $\A^{\circ}$, the absolute
polar of $\A$. That is,

\[
\theta^{*}\in\SC{\theta\in\R^{D}}{\forall x\in\A:\T x\theta\in\RI}
\]

Given $\theta\in\R^{D},$ a \emph{nature policy} compatible with $\theta$
is a mapping (technically, a Markov kernel) $\nu:\left(\A\times\RI\right)^{*}\times\A\rightarrow\Delta\RI$
s.t. for any $hx\in\left(\A\times\RI\right)^{*}\times\A$, it holds
that $\E{r\sim\nu(hx)}r=\T x\theta$. We denote by $\AP_{\theta}$
the set of all nature policies compatible with $\theta$. An agent
policy $\Pol$ together with a nature policy $\nu$ define a distribution
$\Pol\nu\in\Delta(\A\times\RI)^{\omega}$. Sampling $xr\sim\Pol\nu$
is performed recursively according to

\[
\begin{cases}
x_{n} & =\Pol\left(x_{0}r_{0}\ldots x_{n-1}r_{n-1}\right)\\
r_{n} & \sim\nu\left(x_{0}r_{0}\ldots x_{n-1}r_{n-1}x_{n}\right)
\end{cases}
\]

We denote

\[
\mathrm{E}_{\mathcal{A},\theta}^{*}:=\max_{x\in\A}\T x\theta
\]

That is, $\mathrm{E}_{\mathcal{A},\theta}^{*}$ is the expected reward
of the best arm, assuming $\theta^{*}=\theta$.

Given an agent policy $\Pol$, a feature vector $\theta\in\A^{\circ}$,
a nature policy $\nu\in\AP_{\theta}$ and a time horizon $N\in\N$,
the expected regret is defined to be

\[
\Reg{\nu}{\Pol}N:=N\cdot\mathrm{E}_{\mathcal{A},\theta}^{*}-\sum_{n=0}^{N-1}\E{xr\sim\Pol\nu}{r_{n}}
\]

From now on, we will assume that $\A$ spans all of $\R^{D}$. This
assumption loses no generality, since otherwise we can just move down
to the subspace spanned by $\A$. 

In this setting, we can use an algorithm from the UCB family that's
been called ``Confidence Ball'' in \cite{DBLP:conf/colt/DaniHK08}
(see Algorithm \ref{alg:CB}\footnote{The calculation of $B^{*}$ is described in \cite{DBLP:conf/colt/DaniHK08}
as ``find a barycentric spanner''. The relation between the barycentric
spanner and maximizing the determinant appears in \cite{DBLP:conf/stoc/AwerbuchK04}
in the proof of Proposition 2.2.}). Here, $\mathrm{M}_{\A}$ denotes the set of $D\times D$ matrices
whose columns lie in $\A$.

\begin{algorithm}
\caption{\label{alg:CB}Confidence Ball}

\begin{onehalfspace}
\textbf{Input} $N$ (a positive integer)

\label{line:alg_CB_maxdet}$B^{*}\leftarrow\Argmax{B\in\mathrm{M}_{\A}}\left|\det B\right|$

$X\leftarrow B\T B$

$\eta\leftarrow\boldsymbol{0}\in\R^{D}$

\textbf{for $n$ from $1$ to $\infty$}:

\qquad{}$\beta\leftarrow\max\left(128D\ln n\ln Nn^{2},\left(\frac{8}{3}\ln Nn^{2}\right)^{2}\right)$

\label{line:alg_CB_C}\qquad{}$\mathcal{C}\leftarrow\SC{\theta\in\R^{D}}{\T{\left(\theta-X^{-1}\eta\right)}X\left(\theta-X^{-1}\eta\right)\leq\beta}$

\qquad{}$x^{*}\leftarrow\Argmax{x\in\A}\:\underset{\theta\in\mathcal{C}}{\max}\:\T x\theta$

\qquad{}select arm $x^{*}$ and receive reward $r$

\qquad{}$X\leftarrow X+x^{*}\T{\left(x^{*}\right)}$

\qquad{}$\eta\leftarrow\eta+rx^{*}$

\textbf{end for}
\end{onehalfspace}
\end{algorithm}

We should think of the variable $\mathcal{C}$ as a \emph{confidence
set} i.e. some set chosen to be simultaneously narrow enough and contain
$\theta^{*}$ with high probability. In this algorithm, it is an ellipsoid
centered on the estimate $X^{-1}\eta$. The latter is essentially
a least squares estimate of $\theta^{*}$, except for the ``regularization''
term $B\T B$ in $X$. The latter term can be thought of as a ``prior''
coming from the knowledge that $\theta^{*}\in\A^{\circ}$. Line 5.3
then implements the same optimistic selection principle as in the
original UCB algorithm.

Denote $\Pol_{\text{CB}}^{N}$ the policy implmented by Algorithm
\ref{alg:CB}. We have (see Theorem 2 in \cite{DBLP:conf/colt/DaniHK08}\footnote{Theorem \ref{thm:reg-lin} is an immediate corollary of their Theorem
2 in the case $\delta=\frac{1}{T}$.}):
\begin{thm}
[Dani-Hayes-Kakade]\label{thm:reg-lin}Let $N$ be a positive integer.
Then, for all $\theta\in\A^{\circ}$ and $\nu\in\AP_{\theta}$

\[
\Reg{\nu}{\Pol_{\text{CB}}^{N}}N\leq D\sqrt{3072N\left(\ln N\right)^{3}}+2
\]
\end{thm}
Up to the factor $(\ln N)^{3}$, this is the best possible asymptotic,
since \cite{DBLP:conf/colt/DaniHK08} also establish a lower bound
of $\Omega(D\sqrt{N})$ in their Theorem 3. For more details, see
Theorem \ref{thm:DHK}.

Some choices of $\A$ admit a stronger asymptotic upper bound. 

Let $\mathrm{Ex}(\A)$ stand for the set of extreme points of $\A$.
Given $\A\subseteq\R^{D}$ compact and $\theta\in\R^{D}$, we define
the associated \emph{gap} $g(\A,\theta)$ to be

\[
g(\A,\theta):=\mathrm{E}_{\mathcal{A},\theta}^{*}-\sup_{x\in\mathrm{Ex}(\A):\T x\theta<\mathrm{E}_{\mathcal{A},\theta}^{*}}\T x\theta
\]

Notice that $g(\A,\theta)>0$ when the convex hull of $\A$ is a polytope
and $\theta\ne\boldsymbol{0}$. On the other hand, if the convex hull
of $\A$ has a smooth boundary then $g(\A,\theta)=0$.

We have (see Theorem 1 in \cite{DBLP:conf/colt/DaniHK08}),
\begin{thm}
[Dani-Hayes-Kakade]Let $N$ be a positive integer. Then, for all
$\theta\in\A^{\circ}$ s.t. $g(\A,\theta)>0$ and $\nu\in\AP_{\theta}$

\[
\Reg{\nu}{\Pol_{\text{CB}}^{N}}N\leq\frac{3072D^{2}\left(\ln N\right)^{3}}{g(\A,\theta)}+2
\]
\end{thm}

\subsection{\label{subsec:advers}Adversarial Linear Multi-Armed Bandits}

In both previous subsections, we made the assumption that the expected
reward on each round only depends on the arm selected on that round.
This is a very strong assumption that doesn't hold in many real-world
examples of decision making. In \emph{adversarial multi-armed bandits}
we relax this assumption, and allow the expected reward to depend
on the entire history in an arbitrary (possibly even adversarial)
way. In order to avoid no-free-lunch theorems, the agent pursues a
relatively ``modest'' goal: select arms such that there is no \emph{fixed}
arm which is consistently better in the short-term counterfactuals
in which it \emph{would} be selected on a given round. (Of course,
in the stochastic case, this ``modest'' goal is already equivalent
to approaching the maximal possible expected reward.) 

In the finite variant (which we won't describe in detail), this setting
leads to a regret bound of $\Theta(\sqrt{|\A|N})$, like in stochastic
bandits (see \cite{DBLP:conf/colt/AudibertB09})\footnote{But, unlike stochastic bandits, there is no logarithmic regret bound
analogous to Theorem \ref{thm:gap-reg-fin-a}.}. In adversarial \emph{linear} multi-armed bandits, there can be an
infinite number of arms, but require that, on each round, the reward
is a linear functional of the arm (but the functional itself still
depends on the entire history in an arbitrary way).

Formally, we fix $D\in\N$ and let $\A$ be a compact convex subset
of $\R^{D}$ of positive volume. On every round $n\in\N$, the agent
selects an arm $x_{n}\in\A$, possibly using randomization. Also on
every round, nature selects some $y_{n}\in\A^{\circ}$ (which can
depend on what happened in previous rounds, but not on $x_{n}$).
The agent doesn't observe $y_{n}$ but it does observe the reward
$r_{n}=\T{x_{n}}y_{n}\in\RI$. 

This time it's important to consider stochastic agent policies\footnote{We can allow stochastic policies for stochastic multi-armed bandits
as well, but it's easy to see that in that setting it doesn't affect
the optimal regret.}. That is, an agent policy is now a mapping (Markov kernel) $\Pol:(\A\times\RI)^{*}\rightarrow\Delta\A$.
The meaning of $\Pol(h)$ is, the \emph{probability distribution}
from which the agent samples an action, given a past history $h$
of actions and rewards. A nature policy is a (Borel measurable) mapping
$\nu:(\A\times\A^{\circ})^{*}\rightarrow\A^{\circ}$. An agent policy
$\Pol$ together with a nature policy $\nu$ define a distribution
$\Delta(\A\times\A^{\circ})^{\omega}$. Sampling $xy\sim\Pol\nu$
is performed recursively according to

\[
\begin{cases}
x_{n} & \sim\Pol\left(x_{0},\T{x_{0}}y_{0}\ldots x_{n-1},\T{x_{n-1}}y_{n-1}\right)\\
y_{n} & =\nu\left(x_{0},y_{0}\ldots x_{n-1},y_{n-1}\right)
\end{cases}
\]

Given an agent policy $\Pol$, a nature policy $\nu$ and a time horizon
$N\in\N$, the expected regret\footnote{Sometimes this is called ``pseudoregret'', whereas ``regret''
is reserved for the quantity resulting from putting the $\max$ inside
the expected value. See \cite{DBLP:journals/ftml/BubeckC12} for more
details and discussion.} is defined to be

\[
\Reg{\nu}{\varphi}N:=\max_{x'\in\A}\sum_{n=0}^{N-1}\E{xy\sim\Pol\nu}{\T{\left(x'\right)}y_{n}}-\sum_{n=0}^{N-1}\E{xy\sim\Pol\nu}{\T{x_{n}}y_{n}}
\]

We will not spell out any algorithm for this setting, since it requires
algorithms of different nature than UCB, and the latter is the basis
of our own algorithm explained in section \ref{sec:algo}. However,
we will state a known regret bound (see Theorem 27.3 in \cite{lattimore2020bandit}):
\begin{thm}
\label{thm:reg-adv}Let $N$ be a positive integer. Then, there exists
an agent policy $\Pol^{N}$ s.t. for all nature policies $\nu$

\[
\Reg{\nu}{\Pol^{N}}N\leq2D\sqrt{3N\left(\max\left(\ln\frac{2N}{D},0\right)+1\right)}
\]
\end{thm}

\subsection{\label{subsec:zerosum}Zero-Sum Two-Player Games with Bandit Feedback}

\global\long\def\PM{P}%

\global\long\def\ME#1#2#3{\mathrm{ME}_{#1}\left[#2\middle|#3\right]}%

\global\long\def\Amax{\operatorname*{argmax}}%

A \emph{finite zero-sum two-player game in normal form} is defined
by specifying a finite set $\B_{1}$ (the \emph{pure stage strategies}
of player 1), a finite set $\B_{2}$ (the pure stage strategies of
player 2) and $\PM:\B_{1}\times\B_{2}\rightarrow\RI$, which is the
\emph{payoff matrix} of player 1. The payoff matrix of player 2 is
then $-\PM$ (hence ``zero-sum''). The intended meaning is, both
players select a strategy simulatenously and then player 1 receives
reward (payoff) $\PM(a,b)$ and player 2 receives reward $-\PM(a,b)$.
Here, $a\in\B_{1}$ is the strategy selected by player 1 and $b\in\B_{2}$
is the strategy selected by player 2. The goal of each player is maximizing
its own expected reward.

It is often assumed that each player has the ability to randomize
its decision, independently of the other player. Then, the full space
of strategies becomes $\A_{1}:=\Delta\B_{1}$ for player 1 (the \emph{mixed
stage strategies} of player 1) and $\A_{2}:=\Delta\B_{2}$ for player
2 (the mixed stage strategies of player 2). Given $x\in\A_{1}$ and
$y\in\A_{2}$, player 1 receives expected reward

\[
\E{\SUBSTACK{a\sim x}{b\sim y}}{\PM\left(a,b\right)}
\]

Regarding $\PM$ as a matrix (i.e. a vector in $\R^{\B_{1}\times\B_{2}}$),
$x$ as a vector in $\R^{\B_{1}}$ and $y$ as a vector in $\R^{\B_{2}}$,
this can be rewritten as

\[
\T x\PM y
\]

By the minimax theorem (see \cite{karlin2017game}, Theorem 2.3.1),
we have

\begin{equation}
\max_{x\in\A_{1}}\min_{y\in\A_{2}}\T x\PM y=\min_{y\in\A_{2}}\max_{x\in\A_{1}}\T x\PM y\label{eq:minimax}
\end{equation}

The quantity above is known as the \emph{value} of the game for player
1. It is the best expected reward player 1 can guarantee without knowing
player 2's strategy in advance (left hand side of (\ref{eq:minimax})),
and also the best expected reward player 1 can guaratee when playing
against a known, but adversarially selected, player 2 strategy (right
hand side of (\ref{eq:minimax})).

Now we consider a setting in which the game is played repeated for
$N$ rounds, and player 1 (who we will also call ``the agent'')
doesn't know the payoff matrix. On each round, the agent selects a
mixed strategy $x\in\A_{1}$ and observes (i) the pure strategy $a\in\B_{1}$
sampled from $x$ (ii) player 2's pure strategy $b\in\B_{2}$ and
(iii) the reward, sampled from the distribution $H^{*}(a,b)$, where
$H^{*}:\B_{1}\times\B_{2}\rightarrow\Delta\RI$ is some (unknown)
function. The role of the payoff matrix is played by

\[
P^{*}:=\E{r\sim H^{*}(a,b)}r
\]

A \emph{player 1 policy} is a mapping $\Pol:(\B_{1}\times\B_{2}\times\RI)^{*}\rightarrow\A_{1}$.
A \emph{player 2 policy} is a mapping $\nu:(\B_{1}\times\B_{2}\times\RI)^{*}\rightarrow\A_{2}$.
We denote the set of player 2 policies by $\AP$. Given a player 1
policy $\Pol$, a player 2 policy $\nu$ and a function $H:\A_{1}\times\A_{2}\rightarrow\Delta\RI$
(the ``environment''), we get a distribution $\Pol\nu H\in\Delta(\B_{1}\times\B_{2}\times\RI)^{\omega}$.
Sampling $abr$ is performed recursively according to

\[
\begin{cases}
a_{n} & \sim\Pol\left(a_{0}b_{0}r_{0}\ldots a_{n-1}b_{n-1}r_{n-1}\right)\\
b_{n} & \sim\nu\left(a_{0}b_{0}r_{0}\ldots a_{n-1}b_{n-1}r_{n-1}\right)\\
r_{n} & \sim H\left(a_{n},b_{n}\right)
\end{cases}
\]

Given an environment $H$, we denote

\[
\mathrm{ME}_{H}^{*}:=\max_{x\in\A_{1}}\min_{b\in\B_{2}}\E{\SUBSTACK{a\sim x}{r\sim H(a,b)}}r
\]

That is, $\mathrm{ME}_{H}^{*}$ is the value of the game corresponding
to $H$ for player 1.

Given a player 1 policy $\Pol$, an environment $H$ and a time horizon
$N\in\N$, expected regret is defined by

\[
\Reg H{\Pol}N:=N\cdot\mathrm{ME}_{H}^{*}-\min_{\nu\in\AP}\sum_{n=0}^{N-1}\E{abr\sim\Pol\nu H}{r_{n}}
\]

That is, the expected regret is the difference between the expected
total reward that would be obtainable in environment $H$ if it was
known to player 1, with the expected total reward actually obtained
by $\Pol$ in environment $H$, \emph{assuming player 2 knows both
$\Pol$ and $H$ and responds optimally}.

In \cite{DBLP:conf/uai/ODonoghueLO21}, the Algorithm \ref{alg:GUCB}
is proposed for this setting. It has a similar form to Algorithm \ref{alg:UCB}:
instead of assigning an optimistic reward estimate to each arm, we
assign an optimistic reward estimate $P_{ab}$ to each $(a,b)\in\B_{1}\times\B_{2}$,
and then play the optimal strategy $x^{*}$ corresponding to the optimistic
payoff matrix $P$.

Denote $\Pol_{\text{GUCB}}^{N}$ the policy implemented by this algorithm.
We have the following (see Theorem 1 in \cite{DBLP:conf/uai/ODonoghueLO21}):

\begin{algorithm}
\begin{onehalfspace}
\caption{\label{alg:GUCB}Game UCB}

\textbf{Input} $N$ (a positive integer)

$R\leftarrow\boldsymbol{0}$ (of type $\R^{\B_{1}\times\B_{2}}$)

$T\leftarrow\boldsymbol{0}$ (of type $\N^{\B_{1}\times\B_{2}}$)

\textbf{for} every round:

\qquad{}\textbf{for} $a\in\B_{1}$ and $b\in\B_{2}$:

\qquad{}\qquad{}$P_{ab}\leftarrow\frac{R_{ab}}{T_{ab}}+\sqrt{\frac{2\ln\left(2\left|\B_{1}\right|\cdot\left|\B_{2}\right|N^{2}\right)}{\max\left(T_{ab},1\right)}}$

\qquad{}\textbf{end for}

\qquad{}$x^{*}\leftarrow\Argmax{x\in\A_{1}}\:\underset{y\in\A_{2}}{\min}\:\T xPy$

\qquad{}sample $a\in\B_{1}$ from $x^{*}$

\qquad{}play strategy $a$

\qquad{}observe player 2 strategy $b\in\B_{2}$ and receive reward
$r$

\qquad{}$T_{ab}\leftarrow T_{ab}+1$

\qquad{}$R_{ab}\leftarrow R_{ab}+r$

\textbf{end for}

\end{onehalfspace}
\end{algorithm}

\begin{thm}
[O'Donoghue-Lattimore-Osband]Let $N$ be a positive integer. Then,
for any $H:\B_{1}\times\B_{2}\rightarrow\Delta\RI$

\[
\Reg H{\Pol_{\text{GUCB}}^{N}}N\leq4\sqrt{\left|\B_{1}\right|\cdot\left|\B_{2}\right|N\ln\left(2\left|\B_{1}\right|\cdot\left|\B_{2}\right|N^{2}\right)}+2
\]
\end{thm}
However, using an adversarial bandit algorithm, we can get a regret
bound of $O(\sqrt{\left|\B_{1}\right|N})$: substantially better since
it doesn't depend on $\left|\B_{2}\right|$ at all. In \cite{DBLP:conf/uai/ODonoghueLO21},
it is argued that Algorithm \ref{alg:GUCB} is nevertheless empirically
better than Exp3 (a standard adversarial bandit algorithm) in many
situations. But, they show no solid theoretical result which explains
this observation. In the present work, we show that, in the positive
gap case (for an appropriate definition of ``gap''), regret in this
setting can be \emph{logarithmic} in $N$, i.e. far better in terms
of dependence on $N$ than Exp3 (see Example \ref{ex:zerosum-gap})\footnote{We show this for our own algorithm IUCB, not for Game UCB.}.

\subsection{Imprecise Probability}

Consider an agent that must make a decision under conditions of uncertainty.
Formally, let $\B$ be some space of possible \emph{states}, s.t.
the true state is pertinent to the decision but unknown. Let $\A$
be the space of possible decisions that the agent can make. Finally,
let $r:\A\times\B\rightarrow\R$ represent the agent's reward as a
function of the state and the decision (this function is assumed to
be known to the agent).

One classical approach to this problem is \emph{Bayesian} decison
theory\footnote{See e.g. \cite{peterson2017introduction}, chapter 8.}.
In this approach, we assume that the agent's uncertainty can be represented
as a distribution $\zeta\in\Delta\B$ (the \emph{prior}). The optimal
decision $a^{*}\in\A$ has to maximize the expected reward:

\[
a^{*}:=\Amax_{a\in\A}\E{b\sim\zeta}{r(a,b)}
\]

One of the objections to the Bayesian approach comes from empirical
studies where people systematically behave in a way incompatible with
Bayesianism. A classical example is the \emph{Ellsberg paradox}\footnote{See \cite{gilboa_2009}, chapter 12.}.
In this scenario, there are two urns with red and black balls. Urn
I contains 50 red balls and 50 black balls. Urn II contains 100 balls,
but it is unknown how many it has of each color. A ball is drawn at
random from each of the urns. The subject is then asked to assign
monetary value to each of:
\begin{itemize}
\item Bet A that pays \$100 if the ball drawn from urn I is red.
\item Bet B that pays \$100 if the ball drawn from urn I is black.
\item Bet C that pays \$100 if the ball drawn from urn II is red.
\item Bet D that pays \$100 if the ball drawn from urn II is black.
\end{itemize}
In this experiment, subjects usually assign the same value to bet
A and bet B, and the same value to bet C and bet D (which makes perfect
sense, because of the symmetry), but a lower value to bet C than to
bet A. Bayesian decision theory cannot explain this behavior: the
symmetry between C and D implies a prior probability of $50\%$ that
the ball drawn from urn II is red, which is the same as for urn I.
\begin{example}
\label{ex:ells}We can formalize bets A and C in the Ellsberg paradox
by taking $\A=\{a_{\text{acc}},a_{\text{rej}}\}$, $\B=\{b_{\text{red}},b_{\text{black}}\}$
and $r$ as in Table \ref{tab:ells}. Here, $a_{\text{acc}}$ means
``accept bet'', $a_{\text{rej}}$ means ``reject bet'', $b_{\text{red}}$
means ``a red ball is drawn'', $b_{\text{black}}$ means ``a black
ball is drawn'' and $p\in(0,100)$ is the price of taking the bet.
The reward is assumed to be linear in money for simplicity. Given
prior $\zeta\in\Delta\B$, it's easy to see that $a_{\text{acc}}$
is a Bayes-optimal decision if and only if $100\zeta_{\text{red}}\geq p$,
where $\zeta_{\text{red}}$ is the probability that $\zeta$ assigns
to $b_{\text{red}}$. Therefore, the value of the bet is $V=100\zeta_{\text{red}}$
(i.e. this is the minimal price at which buying it is profitable).

Bets B and D are formalized analogously, with the conclusion that
$a_{\text{acc}}$ is Bayes-optimal if and only if $100(1-\zeta_{\text{red}})\geq p$,
and hence the value of the bet is $V=100(1-\zeta_{\text{red}})$.
Therefore, the values of the different bets satisfy $V_{A}+V_{B}=100$
and $V_{C}+V_{D}=100$ (assuming only that each urn has a particular
prior $\zeta$, which may vary between the urns). This is inconsistent
with the empirical observation that $\min(V_{A},V_{B})>\max(V_{C},V_{D})$,
for any choice of priors (i.e. even without invoking symmetry between
red and black).
\begin{table}
\caption{\label{tab:ells}Reward in the Ellsberg paradox}

\centering{}%
\begin{tabular}{ccc}
\toprule 
$a\in\A$ & $b\in\B$ & $r(a,b)$\tabularnewline
\midrule
\midrule 
$a_{\text{acc}}$ & $b_{\text{red}}$ & $100-p$\tabularnewline
\midrule 
$a_{\text{acc}}$ & $b_{\text{black}}$ & $-p$\tabularnewline
\midrule 
$a_{\text{rej}}$ & $b_{\text{red}}$ & $0$\tabularnewline
\midrule 
$a_{\text{rej}}$ & $b_{\text{black}}$ & $0$\tabularnewline
\bottomrule
\end{tabular}
\end{table}
\end{example}
There is a more general decision theory in which, instead of describing
the agent's uncertainty as a distribution, we describe it as a \emph{credal
set}: a closed convex subset $\kappa$ of $\Delta\B$\footnote{See \cite{troffaes2014lower}, chapter 2. Another name for this object
is the ``non-unique prior'', see e.g. the classical paper \cite{GILBOA1989141}.}. The optimal decision is then defined according the the following
rule:

\[
a^{*}:=\Amax_{a\in\A}\min_{\nu\in\kappa}\E{b\sim\nu}{r(a,b)}
\]

\begin{example}
In the setting of Example \ref{ex:ells}, a credal set can be described
by a closed interval $[q_{\min},q_{\max}]$, where $q_{\min}$ and
$q_{\max}$ are the minimal and maximal probability of $b_{\text{red}}$
respectively. For bet C, $a_{\text{acc}}$ is now an optimal decision
if an only if $100q_{\text{min}}\geq p$, and hence $V_{C}=100q_{\text{min}}$.
For bet D, $a_{\text{acc}}$ is an optimal decision if an only if
$100(1-q_{\text{\ensuremath{\max}}})\geq p$, and hence $V_{D}=100(1-q_{\text{max}})$.
In particular, if there is some $t\in(0,\frac{1}{2}]$ s.t. $q_{\text{min}}=\frac{1}{2}-t$
and $q_{\max}=\frac{1}{2}+t$, then $V_{C}=V_{D}=100(\frac{1}{2}-t)<50$.
This preserves the symmetry between red and black. On the other hand,
for urn I it's reasonable to use a Bayesian prior of $\frac{1}{2}$
(which can be interpreted as the one point interval $[\frac{1}{2},\frac{1}{2}]$),
yielding $V_{A}=V_{B}=50$.
\end{example}
In order to understand why $\kappa$ is required to be closed and
convex, consider some \emph{arbitrary $\kappa'\subseteq\Delta\B$}.
Let $\kappa$ be the closed convex hull of $\kappa'$. Then, it's
easy to see that for any $a\in\A$

\[
\min_{\nu\in\kappa}\E{b\sim\nu}{r(a,b)}=\min_{\nu\in\kappa'}\E{b\sim\nu}{r(a,b)}
\]

Therefore, $\kappa'$ and $\kappa$ are equivalent for all practical
purposes, so we might as well consider only closed convex sets to
begin with.

Often, \emph{stochastic} decisions are also allowed, making the decision
rule

\[
x^{*}:=\Amax_{x\in\Delta\A}\min_{\nu\in\kappa}\E{\SUBSTACK{a\sim x}{b\sim\nu}}{r(a,b)}
\]

In the Bayesian case this wouldn't increase generality, since there
is always a deterministic decision which attains the maximum. But,
for credal sets, a stochastic decision can be strictly better (in
terms of the worst-case expected value of reward) than any deterministic
decision. On the other hand, the stochastic theory can be regarded
as a \emph{special case} of the deterministic theory, in which the
space of decisions is taken to be $\Delta\A$. 

Finally, notice that the decision theory with credal sets has the
same mathematical structure as a zero-sum two-player game. But, the
conceptual interpretation is different: in a game, the opponent is
an actual agent or algorithm, whereas here, the ``opponent'' is
just a mathematical representation of risk-aversion with respect to
the ``Knightian'' (unquantifiable) part of the uncertainty.

\section{\label{sec:Motivation}Motivation}

As we discussed in subsection \ref{subsec:advers}, stochastic bandits
rely on assumptions that are often unrealistic. Adversarial bandits
partially remedy that, but still require that the counterfactual reward
has regular dependence on the arm. (For example, adversarial linear
bandits require that this dependence is linear.) We suggest making
do with a weaker assumption, namely that the \emph{credal set of possible
outcome}\footnote{We say ``outcome'' and not ``reward'' because it turns out that,
in this new setting, it is important to have auxiliary feedback in
addition to the reward itself. See discussion in subsection \ref{subsec:lin-imprec}.}\emph{ distributions} has regular dependence on the arm. In other
words, the set of responses available to the adversary changes with
the arm in a regular manner, but, within the constraints of this set,
the adversary's strategy can depend on the arm (more or less) arbitrarily. 

The price we have to pay is remaining content with a weaker notion
of regret, where we compare the expected reward to the maximin expected
reward rather than to the best constant-arm response to the aversary's
plays. This weaker notion of regret is very similar to the notion
we had in subsection \ref{subsec:zerosum}. At the same time, this
notion of regret can be regarded as the credal set analogue of regret
in stochastic bandits.

Here is an example of a situation where our framework is applicable
and classical approaches are not:
\begin{example}
\label{ex:traffic}The algorithm needs to control the system of traffic
lights in a city. On each day, the durations of the various states
are set for each intersection. The resulting traffic is then measured,
and its overall ``quality'' (e.g. a combination of travel time and
number of accidents) is calculated. The traffic depends on hard to
predict conditions that vary from day to day (e.g. events held around
the city) and might be influenced by decisions on previous days (via
changing driver habits). In particular it is not IID and cannot be
modeled as a stochastic bandit. Moreover, a small change in the decision
might trigger a large change in the outcome (e.g. due to chaotic dynamics),
implying it is not effective to model the problem as an adversarial
bandit. On the other hand, we assume that finding a constant setting
of traffic light durations that produces maximal expected quality
under worst-case conditions is sufficient.
\end{example}

To apply our method to Example \ref{ex:traffic}, we could, for instance,
start with analytical modeling and/or simulations of traffic and arrive
at some model which depends on unknown parameters, some of which don't
substantially change over time whereas others can change unpredictably.
We then treat the former as the ``hypothesis'' (unknown parameters
on which the credal sets depend) and the latter as the adversarial
choice of a distribution from within the credal set. For a more concrete
demonstration, we give the following (extremely simplistic) example:
\begin{example}
\label{ex:ABCDE}Suppose that our city has 3 roads: one from A (a
residential neighborhood) to B (the beach), a second from A to C (the
cinema) and a third from D to E. The DE road intersects the AB road
at a single intersection and the AC road at a single intersection,
so there are two intersections in total. Suppose also that nobody
needs to travel e.g. from A to D, and therefore we can ignore the
turns and each intersection has only two states. For instance, the
ABDE intersection has a state in which the AB drivers have a green
light and the DE drivers have a red light, and a state in which the
AB drivers have a red light and the DE drivers have a green light. 

The space of arms can then be parametrizes by the duration of each
light at each intersection, i.e. it's

\[
\A:=[\tau_{\min},\tau_{\max}]_{Bg}\times[\tau_{\min},\tau_{\max}]_{Br}\times[\tau_{\min},\tau_{\max}]_{Cg}\times[\tau_{\min},\tau_{\max}]_{Cr}
\]
Here $0<\tau_{\min}<\tau_{\max}$ are the minimal and maximal durations
a light is allowed to be, and the subscripts indicate which light:
e.g. the subscript $Bg$ means there is green light for the AB drivers
whereas $Cr$ means there is red light for the AC drivers. 

The outcome is parameterized by the number of trips made along each
road, which we model as a continuous number normalized to $[0,1]$
since presumably these numbers are large. That is, the space of outcomes
is

\[
\D:=[0,1]_{AB}\times[0,1]_{AC}\times[0,1]_{DE}
\]

The reward is minus the total waiting time of drivers on red lights,
which we model as

\[
r\left(x,y\right):=-\frac{1}{2}\left(\frac{x_{Br}^{2}y_{AB}+x_{Bg}^{2}y_{DE}}{x_{Bg}+x_{Br}}+\frac{x_{Cr}^{2}y_{AC}+x_{Cg}^{2}y_{DE}}{x_{Cg}+x_{Cr}}\right)
\]

This equation is derived by observing that each driver has a probability
to encounter a red light equal to the fraction of time the light is
red, and an average waiting time of half the duration of the red light
(we assume there are no traffic jams or queues at intersections). 

Now we need a model of driver behavior. Let's assume that the expected
number of trips from D to E is an unknown constant $\theta_{DE}\in[0,1]$.
As to trips from A, we assume the residents require a constant amount
of leisure in expectation, but can sometimes prefer the beach to the
cinema or vice versa for hard to predict reasons (e.g. the weather,
the movies currently available and the traffic light settings themselves
which we assume to be known to them in advance of the trip). Hence
the sum of trips from A to B and trips from A to C is in expectation
equal to some unknown constant $\theta_{A}\in[0,1]$. Thus, the hypothesis
space is $\H:=[0,1]_{DE}\times[0,1]_{A}$ and the credal sets are

\[
\kappa_{\theta}:=\SC{\zeta\in\Delta\D}{\E{y\sim\zeta}{y_{DE}}=\theta_{DE}\text{, }\E{y\sim\zeta}{y_{AB}+y_{AC}}=\theta_{A}}
\]

The unpredicatable fluctuations of traffic between AB and AC means
stochastic bandits are inapplicable. Moreover, the number of arms
is infinite and the reward does not depend on the arm in a regular
manner (e.g. it's possible that drivers abruptly change behavior when
some light duration crosses a threshold: we make no assumptions about
it), and hence adversarial bandits are also inapplicable. On the other
hand, our main result (Theorem \ref{thm:main}) produces a regret
bound of $\tilde{O}(\sqrt{N})$ in this example, where $N$ is the
time horizon (and regret is defined w.r.t. the best reward achievable
for worst-case admissibe driver behavior).
\end{example}
For another example, consider the following:
\begin{example}
\label{ex:med}A patient suffers from a condition with three primary
symptoms: call them A, B and C. Every month, the doctor prescribes
the patient treatment, by specifying the dosage of each of $n$ different
medicaments, between 0 to some maximal amount. The space of arms is
hence $\A:=[0,1]^{n}$. In each month, the patient can either experience
or not experience each of the 3 symptoms. The space of outcomes is
hence $\B:=2^{\{A,B,C\}}$.

Studies showed that the probability that a patient experiencing symptom
A after receiving treatment $x$ is given by $f^{A}(x)+\theta^{A}g^{A}(x)$,
where $f^{A},g^{A}:\A\rightarrow\mathbb{R}$ are fixed functions whereas
$\theta^{A}\in[0,1]$ depends on the individual patient. Also, the
probability of a patient experiencing symptom C conditional on each
of the events (1) neither symptom A nor symptom B is present (2) symptom
A is present but not symptom B (3) symptom B is present but not symptom
A (4) both symptom A and symptom B are present, is given by $f_{i}^{C}(x)+\theta_{i}^{C}g_{i}^{C}(x)$,
where $i\in\{1,2,3,4\}$ is the event, \textbf{$f_{i}^{C},g_{i}^{C}:\A\rightarrow\mathbb{R}$}
are fixed functions and $\theta_{i}^{C}\in[0,1]$ depend on the individual
patient. On the other hand, symptom B is difficult to predict and
might depend both on the individual patient and on the patient's history
in a complicated, poorly understood way.

The goal is maximizing the expected value of some function $r:\A\times\B\rightarrow\mathbb{R}$
that takes into account the severity of the various symptoms, and
also the cost and long-term risks of the different medicaments, under
worst-case assumptions about the unpredictable symtom B. Once again,
stochastic bandits are inapplicable due to the nature of symptom B
and adversarial bandits are inapplicable because the number of arms
is infinite (and there is no good enough guarantee of smooth dependency).
On the other hand, we still get a regret bound of $\tilde{O}(\sqrt{N})$
by Theorem \ref{thm:main}, and the constant coefficient can be estimated
using Propositions \ref{prop:chain-s} and \ref{prop:chain-r}.
\end{example}
Hopefully, future work will generalize this framework from bandits
to reinforcement learning (i.e. allow a persistent observable state
that is preserved over time), which would allow extending the possible
range of applications much further.

Our framework can also be regarded as a repeated two-player zero-sym
game, in which the arms are the agent's stage strategies, whereas
a stage strategy of ``nature'' (the adversary) is an assignment
to each arm of a distribution from that arm's respective credal set.
Importantly, the agent doesn't know which credal set is associated
with each arm (the same way as in stochastic bandits, it doesn't know
which distribution is associated with each arm). The regret of an
algorithm is then defined by comparing the minimal expected reward
this algorithm guarantees with the minimal expected reward that could
be guaranteed if the credal sets were known in advance.

Indeed, the setting of subsection \ref{subsec:zerosum} is a special
case of our framework, as we will see in Example \ref{ex:zerosum}.
On the other our hand, our framework is much more general, since the
space of stage strategies doesn't have to be of the form $\Delta\B_{1}$
for some finite set $\B_{1}$. To put it differently, we allow for
an infinite space of ``pure''\footnote{In Example \ref{ex:zerosum}, this space $\A$ actually corresponds
to the space $\Delta\B_{1}$ of \emph{mixed} stage stragies in subsection
\ref{subsec:zerosum}.} stage strategies for player 1, but we do assume player 2 can see
player 1's choice before making their own choice, on each given round.

Naturally, proving useful upper bounds on regret requires either bounding
the number of arms, or assuming some structure on hypotheses (arm-to-credal-set
assignments). We choose the latter path, since the former can be addressed
using adversarial bandits, and the latter leads to a richer theory.
Notice that, as opposed to adversarial bandits, the structure doesn't
imply the reward depends on the arm in some regular way, since the
nature's policy can introduce irregularaties. 

Specifically, we consider hypotheses classes with ``linear'' structure
that generalizes the classical theory of linear bandits. In particular,
this captures Example \ref{ex:ABCDE}. Here are a couple of more mathematical
examples (the fully general framework is explained in the next section): 
\begin{example}
\label{ex:rot}Let $\A:=[0,2\pi]$ be the space of arms and $\B:=\{0,1,2\}$
the set of outcomes. The set $\Delta\B$ of probability distributions
on $\B$ can be naturally regarded as a subset of $\R^{3}$, namely 

\[
\Delta\B=\SC{y\in\R^{3}}{y_{0}+y_{1}+y_{2}=1,\:y_{i}\geq0}
\]

On each round, the agent selects $x\in\A$ and observes an outcome
$b\in\B$. This outcome is associated with the reward $r(x,b)$ where
$r:\A\times\B\rightarrow\R$ is a function continuous in the first
argument. The hypothesis space is parameterized by $\theta\in[0,2\pi]$.
The credal set $\U_{x\theta}\subseteq\Delta\B$ is given by

\[
\U_{x\theta}:=\SC{y\in\Delta\B}{\T{\left(\cos(x+\theta)u+\sin(x+\theta)v\right)}y=0}
\]

where $u,v\in\R^{3}$ form an orthonormal basis together with $\frac{1}{\sqrt{3}}\left[\begin{array}{c}
1\\
1\\
1
\end{array}\right]$. That is, each credal set is the intersection of the triangle $\Delta\B$
with a straight line passing through the triangle's center (the uniform
distribution) at an angle of $x+\theta$.

Suppose the reward function is given by $r(x,0)=0$, $r(x,1)=0$,
$r(x,2)=1$. Then the optimal arm is s.t. the line is parallel to
the side $01$ of the triangle, with a corresponding minimal expected
reward of $\sfrac{1}{3}$.
\end{example}
\begin{example}
\label{ex:prob}Let the set of arms $\A$ be the group of Euclidean
isometries of the square $[0,1]\times[0,1]$, $|\A|=8$. Let $\B:=\{0,1,2,3\}$
be the set of outcomes. The hypothesis space is $\H:=[0,1]\times[0,1]$.
For any $x\in\A$ and $\theta\in\H$, denote $(p^{x},q^{x}):=x\theta$
(the result of applying the isometry $x$ to the point $\theta$).
The credal set $\U_{x\theta}\subseteq\Delta\B$ is the given by

\[
\U_{x\theta}:=\SC{y\in\Delta\B}{y_{0}+y_{1}=\frac{1}{2}p^{x},\:\frac{y_{2}}{y_{2}+y_{3}}=q^{x}}
\]

That is, the probability of $\{0,1\}$ is constrained to be $\frac{1}{2}p^{x}$
and the probability of $2$ conditional on $\{2,3\}$ is constrained
to be $q^{x}$. In this case, $\U_{x\theta}$ is a straight line segment
inside the tetrahedron $\Delta\B$, and the family $\{\U_{x\theta}\}_{x\in\A,\theta\in\H}$
corresponds to a 2-dimensional submanifold of the 4-dimensional manifold
of lines in a 3-dimensional affine space.
\end{example}
\begin{example}
\label{ex:hyperplane}Consider $n\geq2$, $m\ge1$ and let the space
of arms $\A$ be some compact subset of the space of full rank $n\times(m+1)$
matrices. The space of outcomes $\D$ is the ball of radius 1 in $m$-dimensional
Euclidean space. The hypothesis space $\H$ is some compact subset
of $\R^{n}\setminus0$, and we assume that for any $X\in\A$ and $\theta\in\H$,
there is some $y\in\D$ s.t. $\T{\theta}X\left[\begin{array}{c}
y\\
1
\end{array}\right]=0$. The reward function $r:\A\times\D\rightarrow\R$ is continuous in
the first variable, and convex and 1-Lipschitz in the second variable.
The credal set $\kappa_{X,\theta}\subseteq\Delta\D$ is given by

\[
\kappa_{X,\theta}:=\SC{\zeta\in\Delta\D}{\T{\theta}X\left[\begin{array}{c}
\E{y\sim\zeta}y\\
1
\end{array}\right]=0}
\]

That is, the expectation value of the distribution on $\D$ is constrained
to lie on a certain hyperplane determined by $\theta$ and $X$.
\end{example}

\subsection*{Related Work}

Our framework can be naturally viewed as a generalization of stochastic
linear bandits. For the latter, a regret bound optimal up to logarithmic
factors was established in \cite{DBLP:conf/colt/DaniHK08}. Moreover,
we directly use one of the theorems in \cite{DBLP:conf/colt/DaniHK08}
in section \ref{sec:Lower-Regret-Bounds}.

It is also possible to view our framework as a zero-sum two-player
game. Zero-sum two-player games with bandit feedback were studied
in \cite{DBLP:conf/uai/ODonoghueLO21}. Indeed, their setting can
be represented as a special case of our setting, including the notion
of regret they chose to analyze (see Example \ref{ex:zerosum}). In
\cite{DBLP:conf/uai/ODonoghueLO21}, each player has a finite set
of arms (pure strategies in the stage game) and the regret bounds
they show for their proposed algorithms are not better than the corresponding
regret bound by Exp3 (however, they show that empirically Exp3 performs
worse). On the other hand, our setting is much more general (in particular
it allows for an infinite set of arms) and our regret bounds are not
obtained by any previously proposed algorithm (as far as we know).
Even for finite games, we can sometimes produce a regret bound of
the form $O(\ln N)$ (see Example \ref{ex:zerosum-gap}) as opposed
to the $\tilde{O}(\sqrt{N})$ bound in \cite{DBLP:conf/uai/ODonoghueLO21}.

Another work that studies learning zero-sum two-player games is \cite{pmlr-v139-tian21b}.
Their notion of regret is also closely related to the present work.
In one sense, the setting of \cite{pmlr-v139-tian21b} is significantly
more general than our setting, because they study (episodic) \emph{stochastic}
games: i.e. the game has a state, so it's more closely analogous to
reinforcement learning than to multi-armed bandits. On the other hand,
their regret bound scales with $|\A|^{\frac{1}{3}}$, where $\A$
is the set of actions (in particular, it has to be finite) and with
$N^{\frac{2}{3}}$, where $N$ is the number of episodes, in contrast
to our bound (Theorem \ref{thm:main}) which allows an infinite set
of arms and scales with $N^{\frac{1}{2}}$ (where $N$ is the time
horizon). Therefore, these results are incomparable. 

In \cite{DBLP:conf/alt/ChenZG22a}, a regret bound is proved for an
algorithm that learns zero-sum two-player stochastic games which is
roughly of the form $\tilde{O}(dN^{\frac{1}{2}})$, where $d$ is
the dimension of a certain linear feature space. However, they assume
that \emph{both} players are controlled by the algorithm (which is
then required to converge to the Nash equilibrium of the game), whereas
in our framework only one player is controlled by the algorithm, and
the other is truly adversarial.

We also see a philosophical analogy between our approach and the notion
of ``semirandom models'' in complexity theory (see \cite{DBLP:books/cu/20/Feige20}).
Semirandom models interpolate between worst-case and average-complexity,
by assuming the instances of a computational problem are drawn from
a distribution which can be controlled in some limited, prescribed
way by an adversary. Similarly, in our framework, the \emph{outcomes}
are sampled from a distribution which is controlled in some limited,
prescribed way by an adversary. Semirandom models are concerned with
computational complexity, while we are concerned with sample complexity
(although the computational complexity of learning in this setting
should also be studied.) In both cases the ``adversary'' is primarily
viewed as a metaphor for properties of the real-world source of the
data that are difficult or impossible to model explicitly (rather
than an actual agent.)

\rule[0.5ex]{1\columnwidth}{1pt}

The structure of the thesis is as follows. Section \ref{sec:Setting}
formally defines the setting and the notion of regret that we are
going to analyze. Section \ref{sec:algo} defines the IUCB algorithm.
Section \ref{sec:Upper-Bounds} proves upper bound on expected regret
for IUCB. Section \ref{sec:Lower-Regret-Bounds} proves lower bounds
on expected regret for particular special cases, which imply the necessity
of particular parameters in any upper bound. Section \ref{sec:Special-Cases}
studies the behavior of the parameters appearing in the regret bounds
in some natural special cases. Section \ref{sec:Summary} summarizes
the results and proposes directions for future work.

\section{\label{sec:Setting}Setting}

\global\long\def\Y{\mathcal{Y}}%

\global\long\def\G#1{\mathrm{Gr}_{#1}^{+}(\Y)}%

\global\long\def\Du#1{#1^{\star}}%

\subsection{General Imprecise Bandits}

We start by formally defining the framework. Fix $D_{X}\in\mathbb{N}$
and let $\A$ be a compact subset of $\R^{D_{X}}$ that represents
arms\footnote{More generally, $\A$ can be an arbitrary compact Polish space.}.
Consider also a finite-dimensional vector space $\Y$, and fix a linear
functional $\mu\in\Du{\Y}\setminus0$ and a compact convex set $\D\subset\mu^{-1}(1)\subset\Y$
which represents outcomes\footnote{As opposed to classical multi-armed bandits, we will always work with
some space of outcomes rather than just a reward scalar.}. $\Y$ and $\mu$ are purely technical devices for a convenient representation
of $\D$. We assume the affine hull of $\D$ equals $\mu^{-1}(1)$.
We also fix a continuous reward function $r:\A\times\D\rightarrow\R$.
\begin{example}
\label{ex:simplex}An interesting special case is, when the outcomes
that can be actually observed are some finite set $\B$ without any
special structure. To represent this, we take $\Y:=\R^{\B}$ and define
$\mu$ and $\D$ by

\[
\mu(y):=\sum_{a\in\B}y_{a}
\]

\[
\D:=\Delta\B=\SC{y\in\R^{\B}}{\sum_{a\in\B}y_{a}=1,\,\forall a\in\B:y_{a}\geq0}
\]

This is a valid representation since there is the canonical embedding
$\iota:\B\rightarrow\Delta\B$ given by

\[
\iota(a)_{b}:=\begin{cases}
1 & \text{if }a=b\\
0 & \text{if }a\ne b
\end{cases}
\]

Given a reward function $r_{0}:\A\times\B\rightarrow\mathbb{R}$,
we can extend it to $r:\A\times\D\rightarrow\mathbb{R}$ as

\[
r(x,y):=\sum_{a\in\B}y_{a}r_{0}(x,a)
\]
\end{example}
A common way to represent beliefs in imprecise probability theory
is \emph{credal sets.} A credal set over $\D$ is a closed\footnote{``Closed'' in the sense of the weak topology on probability measures.}
convex subset of $\Delta\D$. We denote the space of credal sets by
$\Box\D$. A \emph{hypothesis} is a mapping $H:\mathcal{A}\rightarrow\Box\D$.
The meaning of this mapping is, whenever the agent pulls arm $x\in\A$,
the outcome is drawn from some distribution in the set $H(x)$. The
precise distribution is left unspecified: it can vary arbitrarily
over time and as a function of the previous history of arms and outcomes,
treated as adversarial. Given $x\in\A$ and $H$ as above, the associated
lower prevision (minimal expected reward) is

\[
\ME Hrx:=\min_{\zeta\in H(x)}\E{y\sim\zeta}{r(x,y)}
\]

The optimal arm for $H$ is then

\[
x_{H}^{*}:=\Amax_{x\in\A}\ME Hrx
\]

If there are multiple optimal arms, the ambiguity can be resolved
arbitrarily, except in particular cases we will point out.

An \emph{agent policy} is a mapping\footnote{Technically, it has to be a Borel measurable mapping.}
$\Pol:(\A\times\D)^{*}\rightarrow\A$. A \emph{nature policy} compatible
with $H$ is a mapping\footnote{More precisely, a Markov kernel.}
$\nu:(\A\times\D)^{*}\times\A\rightarrow\Delta\D$ s.t. for any $hx\in(\A\times\D)^{*}\times\A$,
$\nu(hx)\in H(x)$. We denote by $\AP_{H}$ the set of all nature
polcies compatible with $H$. An agent policy $\Pol$ together with
a nature policy $\nu$ define a distribution $\Pol\nu\in\Delta((\A\times\D)^{\omega})$
in the natural way. That is, sampling $xy\sim\Pol\nu$ is performed
recursively according to

\[
\begin{cases}
x_{n} & =\Pol\left(x_{0}y_{0}\ldots x_{n-1}y_{n-1}\right)\\
y_{n} & \sim\nu\left(x_{0}y_{0}\ldots x_{n-1}y_{n-1}x_{n}\right)
\end{cases}
\]

Given a policy $\Pol$, a hypothesis $H$ and a time horizon $N\in\mathbb{N}$,
the \emph{expected regret} is

\[
\Reg H{\Pol}N:=N\cdot\ME Hr{x_{H}^{*}}-\min_{\nu\in\AP_{K}}\sum_{n=0}^{N-1}\E{xy\sim\Pol\nu}{r\left(x_{n},y_{n}\right)}
\]

\global\long\def\HC{\mathcal{\H C}}%

Our goal is finding algorithms for which upper bounds on expected
regret can be guaranteed for any hypothesis out of some class $\HC$.
Notice that this is consistent with extending the outcome space, like
in Example \ref{ex:simplex}, because that strictly increases the
set of possible nature policies, and hence any regret bound for the
extended version also applies to the original version.

\subsection{Partial Conditional Bandits}

Since the number of arms is unbounded, guaranteeing a meaningful regret
bound requires assumptions on $\HC$ and $r$. Before stating the
assumptions we make in full generality, we exhibit a particular family
of examples, which will serve to illustrate various points in the
rest of the thesis. Specifically, we will consider a setting where
the outcome set consists of sequences of fixed length. A distribution
over such a set can be specified by choosing the probability distribution
of every element in the sequence conditional on its predecessors.
We then require some of those conditional probability distributions
to be fixed per arm (i.e. behaving like a stochastic bandits) while
letting other conditional distributions to vary arbitrarily. Moreover,
the dependency on the arm is restricted to a known ``linear'' finite-dimensional
family of possible functions. Details follow.

\global\long\def\CB#1{\mathcal{G}_{#1}}%

\global\long\def\PB#1{\mathcal{\mathcal{G}}_{-}^{#1}}%

\global\long\def\NB#1{\mathcal{G}_{+}^{#1}}%

\global\long\def\Stoc{\mathcal{S}}%

\global\long\def\GS{\Stoc^{\sharp}}%

\global\long\def\Z{\mathcal{Z}}%

Fix $D_{X}\in\mathbb{N}$, $\A$ a compact subset of $\R^{D_{X}}$
as before (arms), an integer $n\geq1$ and a family of finite sets
$\left\{ \CB i\right\} _{i<n}$. Let the set of outcomes be

\[
\B:=\prod_{i<n}\CB i
\]

We think of such an outcome as a temporal sequence (first an element
of $\CB 0$ is selected, then an element of $\CB 1$ and so on). Denote

\[
\PB *:=\bigcup_{i<n}\prod_{j<i}\CB j
\]

Fix some $\mathcal{S}\subseteq\PB *$. These will correspond to prefixes
s.t. the conditional probability of the next element in the sequence
is fixed. 

Fix also some family of finite-dimensional vector spaces $\left\{ \Z_{a}\right\} _{a\in\Stoc}$
equipped with linear functionals $\left\{ \psi_{a}\in\Z_{a}^{\star}\right\} _{a\in\Stoc}$
(each $(\Z_{a},\psi_{a})$ pair is just a convenient representation
of the affine space $\psi_{a}^{-1}(1)$). The bandit has an unknown
vector of parameters

\[
\theta^{*}\in\bigoplus_{a\in\Stoc}\psi_{a}^{-1}(1)
\]

Here, the notation $\bigoplus_{a\in\Stoc}\psi_{a}^{-1}(1)$ refers
to the affine subspace of $\bigoplus_{a\in\mathcal{S}}\Z_{a}$ canonically
isomorphic to $\prod_{a\in\mathcal{S}}\psi_{a}^{-1}(1)$.

We also have some \emph{known} continuous functions

\[
\left\{ f_{a}:\A\times\Z_{a}\rightarrow\R^{\CB{|a|}}\right\} _{a\in\Stoc}
\]

These are assumed to be linear in the second argument. The role of
$f$ is specifying how the conditional probabilities depend on the
arm and the unknown parameters. This is needed to obtain a low-dimensional
family of hypotheses, because a priori the number of free parameters
is $|\A|\cdot|\CB{|a|}|$ per each $a\in\Stoc$, which can be much
larger than $\dim\Z_{a}-1$ or even infinite.

We require that for any $x\in\A$ and $a\in\Stoc$
\begin{itemize}
\item For any $c\in\CB{|a|}$, $f_{a}(x,\theta_{a}^{*})_{c}\geq0$ (enforcing
that the conditional probabilities are non-negative).
\item For any $z\in\Z_{a}$, $\sum_{c\in\CB{|a|}}f_{a}(x,z)_{c}=\psi_{a}(z)$
(enforcing that the conditional probabilities sum to 1).
\end{itemize}
Every time an arm $x\in\A$ is selected, the outcome is drawn from
any distribution $y\in\Delta\B$ s.t. for any $a\in\mathcal{S}$ and
$c\in\CB{|a|}$

\[
\Pr_{b\sim y}\left[b_{|a|}=c\:\middle|\:a\sqsubseteq b\right]=f_{a}\left(x,\theta_{a}^{*}\right)_{c}
\]

Here, the notation $a\sqsubseteq b$ means that $a$ is a prefix of
$b$. That is, we require that, conditional on the prefix $a$, the
probability that the next sequence element is $c$ equals to $f_{a}\left(x,\theta_{a}^{*}\right)_{c}$.

Finally, the reward of the round is determined according to some function
$r_{0}:\A\times\B\rightarrow[-1,+1]$, continuous in the first variable.

Here is how this bandit is a special case of the general framework
of the previous subsection:
\begin{example}
\label{ex:pcb}Let $\Y$, $\mu$, $\D$ and $r$ be defined as in
Example \ref{ex:simplex}. For any $x\in\A$, $a\in\Stoc$ and $c\in\CB{|a|}$,
define $\phi_{x,a,c}\in\Y^{\star}$ by

\[
\phi_{x,a,c}(y):=\sum_{b\in\B:ac\sqsubseteq b}y_{b}-f_{a}\left(x,\theta_{a}^{*}\right)_{c}\sum_{b\in\B:a\sqsubseteq b}y_{b}
\]

Notice that for $y\in\Delta\B$, $\phi_{x,a,c}(y)=0$ if and only
$y$ has the required conditional probabilities for an outcome distribution
resulting from selecting arm $x$.

Finally, define $H:\A\rightarrow\Box\Delta\B$ by

\[
H(x):=\SC{\zeta\in\Delta\Delta\B}{\forall a\in\Stoc,c\in\CB{|a|}:\E{y\sim\zeta}{\phi_{x,a,c}(y)}=0}
\]
\end{example}
Notice that Example \ref{ex:med} is a special case of Example \ref{ex:pcb}.
Zero-sum games with bandit feedback is another special case:
\begin{example}
\label{ex:zerosum}Consider the setting of subsection \ref{subsec:zerosum}.
Let $\A:=\A_{1}$, $n:=3$, $\CB 0:=\B_{2}$, $\CB 1:=\B_{1}$, $\CB 2:=\{-1,+1\}$,
$\mathcal{S}:=\B_{2}\sqcup\left(\B_{2}\times\B_{1}\right)$. For any
$b\in\B_{2}$, let $\Z_{b}:=\R$, $\psi_{b}:=1$ and $\theta_{b}^{*}:=1$.
Define $f_{b}:\A_{1}\times\R\rightarrow\R^{\B_{1}}$ by $f_{b}(x,z):=zx$.
For any $(b,a)\in\B_{2}\times\B_{1}$, let $\Z_{ba}:=\R^{2}$, $\psi_{ba}:=[\begin{array}{cc}
1 & 1\end{array}]$ and

\[
\theta_{ba}^{*}:=\left[\begin{array}{cc}
\frac{1-P_{ab}^{*}}{2} & \frac{1+P_{ab}^{*}}{2}\end{array}\right]
\]

Define $f_{ba}:\A_{1}\times\R^{2}\rightarrow\R^{2}$ by $f_{ba}(x,z):=z$
(we implicitly identify $\R^{\{-1,+1\}}$ with $\R^{2}$). Finally,
define $r_{0}:\A_{1}\times\B_{2}\times\B_{1}\times\{-1,+1\}\rightarrow\RI$
by $r_{0}(x,b,a,\sigma):=\sigma$.
\end{example}
Strictly speaking, if we interpret Example \ref{ex:zerosum} as a
partial conditional bandit, it requires that the distributions $H^{*}(a,b)$
of the game are Bernoulli (i.e. supported on $\{-1,+1\}$). However,
the general case still fits within the imprecise bandits framework.
Indeed, consider an outcome $(a,b,r)\in\B_{1}\times\B_{2}\times\RI$
of a particular round of the game, where $r$ is the reward. Then,
we can interpret it as corresponding to the $y\in\D=\Delta(\B_{2}\times\B_{1}\times\{-1,+1\})$
whose only non-zero components are

\begin{align*}
y_{a,b,-1} & :=\frac{1-r}{2}\\
y_{a,b,+1} & :=\frac{1+r}{2}
\end{align*}

With this interpretation, the game is a special case of Example \ref{ex:pcb},
for any $H^{*}$.

\subsection{\label{subsec:lin-imprec}Linear Imprecise Bandits}

Now, we spell out the exact assumptions we will use in the rest of
the thesis. We will make 4 such assumptions.

\newtheorem{ass}{Assumption}

\begin{ass}\label{ass:expect}

For any $H\in\HC$ and $x\in\A$, there is some closed convex subset
$K(x)^{+}$ of $\D$ s.t.

\[
H(x)=\SC{\zeta\in\Delta\D}{\E{y\sim\zeta}y\in K(x)^{+}}
\]

\end{ass}

That is, we require that our credal set is defined entirely in terms
of the \emph{expected value} of the distribution.
\begin{example}
\label{ex:pcb-expect}In Example \ref{ex:pcb}, Assumption \ref{ass:expect}
holds since we can take

\[
K(x)^{+}:=\SC{y\in\Delta\B}{\forall a\in\Stoc,c\in\CB{|a|}:\phi_{x,a,c}(y)=0}
\]
\end{example}
\begin{example}
In the setting of Example \ref{ex:simplex}, consider any mapping
$H:\A\rightarrow\Box\B$ representing a hypothesis. Then, we can transform
it into a mapping $\hat{H}:\A\rightarrow\Box\D$ satisfying assumption
Assumption \ref{ass:expect} by taking

\[
\hat{H}(x):=\SC{\zeta\in\Delta\D}{\E{y\sim\zeta}y\in H(x)}
\]

This transformation is compatible with the intended meanings of $H$
and $\hat{H}$: if the agent selects arm $x\in\A$ and the resulting
distribution of outcomes is some $y\in H(x)$, then $\iota_{*}y\in\hat{H}(x)$
since

\[
\E{a\sim y}{\iota(a)}=y
\]
\end{example}
Assumption \ref{ass:expect} is a natural starting point for generalizing
the classical theory of stochastic bandits. In stochastic bandits,
the central assumption is that the distribution of reward for every
arm is fixed over time. However, if we only assumed that the \emph{expected
values} of those distributions are fixed, we could derive only mildly
weaker regret bounds: instead of a IID process we get a martingale,
and instead of the Hoeffding inequality we can use the Azuma-Hoeffding
inequality. This is essentially the special case of our setting where
$K(x)^{+}$ consists of a single point.

In addition, Assumption \ref{ass:expect} is less restrictive than
it superficially seems, because additional moments of the distribution
can be represented by embedding the outcome space into some higher-dimensional
space. For example:
\begin{example}
\label{ex:moment}Consider some $n\geq1$. Let $\Y:=\R^{n+1}$, $\mu(y):=y_{0}$
and $\D$ be the convex hull of the moment curve $\mathfrak{M}$:

\[
\mathcal{\mathfrak{M}}:=\SC{\left[\begin{array}{c}
1\\
t\\
t^{2}\\
\vdots\\
t^{n}
\end{array}\right]}{t\in\RI}
\]

The only actual feedback is a reward scalar $r\in\RI$, but we represent
it as the point $y(r)\in\D$ given by $y_{k}(r):=r^{k}$. Then, Assumption
\ref{ass:expect} admits any credal set defined in terms of the \emph{first
$n$ moments} of the reward distribution.
\end{example}
\begin{ass}\label{ass:affine}

For any $H\in\HC$ and $x\in\A$, $K(x)^{+}$ is closed under affine
(and not just convex) linear combinations. Equivalently, there exists
$K(x)$ a linear subspace of $\Y$ s.t.

\[
K(x)^{+}=K(x)\cap\D
\]

\end{ass}
\begin{example}
In Example \ref{ex:pcb}, Assumption \ref{ass:affine} holds since
we can take

\begin{equation}
K(x):=\SC{y\in\R^{\B}}{\forall a\in\Stoc,c\in\CB{|a|}:\phi_{x,a,c}(y)=0}\label{eq:pcb-k}
\end{equation}
\end{example}
The only motivation for this is that it makes deriving an upper regret
bound much easier. Indeed, some simple bandits without this assumption
admit strong \emph{negative} results (i.e. lower bounds): see discussion
after Theorem \ref{thm:lower-s}.

Superificially, Assumption \ref{ass:affine} rules out the case in
which the reward scalar is the only feedback, except for the special
case of ordinary stochastic bandits. That's because if $\D$ is 1-dimensional
then $K(x)^{+}$ has to be either a single point or all of $\D$.
Indeed, simple examples of imprecise bandits in which each hypothesis
is an \emph{inequality} on the expected reward turn out to often have
poor regret bounds for reasons similar to Theorem \ref{thm:lower-s}\footnote{We will not spell out the argument, but that is an obstacle that the
author encountered on early attempts to find regret bounds for imprecise
bandits.}. However, our framework admits non-trivial examples with such feedback
by embedding the reward into a higher-dimensional space, as in Example
\ref{ex:moment}.

\begin{ass}\label{ass:r}

Let $\Y$ be equipped with the norm whose unit ball is the absolute
convex hull of $\D$. Then, $r$ is convex and 1-Lipschitz (w.r.t.
the norm on $\Y$) in the second argument.

\end{ass}
\begin{example}
In the setting of Example \ref{ex:simplex}, $r$ is affine in the
second argument and in particular convex. The norm on $\Y$ is $\ell_{1}$
(see Lemma \ref{lem:l1}). If the range of $r_{0}$ is $[-1,+1]$,
then it's easy to see $r$ is 1-Lipschitz in the second argument and
hence Assumption \ref{ass:r} holds. (Otherwise, $r_{0}$ can be renormalized
to have range in $\RI$, at the cost of a multiplicative cost in any
regret bound.) In particular, it applies to Example \ref{ex:pcb}.
\end{example}
Note that Assumptions \ref{ass:expect} and \ref{ass:r} imply

\[
\ME Hrx=\min_{y\in K(x)^{+}}r(x,y)
\]

In general, Lipschitz is a fairly mild condition, but convexity is
not. However, given a reward function $r:\A\times\D\rightarrow\R$
which doesn't satisfy the convexity condition, we can construct the
convex reward function

\[
\tilde{r}(x,y):=\min_{\zeta\in\Delta\D:\E{y'\sim\zeta}{y'}=y}\E{y'\sim\zeta}{r(x,y')}
\]

It is easy to see that $\tilde{r}\leq r$, and that for any $x\in\A$
and hypothesis $H$ that satisfies Assumption \ref{ass:expect},

\[
\ME H{\tilde{r}}x=\ME Hrx
\]

Hence, the expected regret for $r$ is upper bounded by the expected
regret for $\tilde{r}$, and any upper bound on the latter carries
over to the former. The caveat is that $\tilde{r}$ might fail to
be Lipschitz even if $r$ is\footnote{\label{fn:felix}However, if $\D$ is a polytope and $r$ is 1-Lipschitz,
then $\tilde{r}$ is $c$-Lipschitz, for some $c$ that depends only
on $\D$. The proof of this fact was suggested to the author by Felix
Harder, and is not included in the present work.}. Even if $\tilde{r}$ is Lipschitz, its Lipschitz constant might
be larger than the Lipschitz constant of $r$. This means that, even
if $r$ is 1-Lipschtiz and $\tilde{r}$ is Lipschitz, applying the
results of this work might require rescaling $\tilde{r}$ by a constant
multiplicative factor, which would then appear as an extra penalty
in the regret bound (through the parameter $C$: see section \ref{sec:Upper-Bounds}).
This caveat is why we cannot state the exact same results without
the convexity condition.

Assumptions \ref{ass:expect} and \ref{ass:affine} are entirely about
the individual credal sets. However, we also need to make an assumption
about the relationship between different hypotheses (it is necessary
even in classical theory). A common type of assumption in learning
theory is requiring the hypothesis class is ``low-dimensional''
in some sense. In some cases, it takes the form of requiring an embedding
of this class into a low-dimensional linear space in an appropriate
way (in particular, this is the case for linear bandit theory). This
is exactly the type of assumption we will me make.

\global\long\def\W{\mathcal{W}}%

Let $\Z$ be a finite-dimensional vector space, which will serve to
parameterize our hypothesis class $\HC$, and $\H\subseteq\Z$ a compact
set which will be the representation of $\HC$ within $\Z$. Let $\W$
be another finite-dimensional vector space, which will serve to parametrize
the linear conditions that a credal set satisfying Assumptions \ref{ass:expect}
and \ref{ass:affine} imposes on the expected value of the outcome
distribution. That is, the $K(x)$ from Assumption \ref{ass:affine}
will be the kernel of some linear operator from $\Y$ to $\W$. Finally,
let $F:\A\times\Z\times\Y\rightarrow\W$ be a mapping continuous in
the first argument and bilinear in the second and third arguments,
which serves to enable the interpretation of vectors in $\Z$ as hypotheses
(see below). For any $x\in\A$ and $z\in\Z$, denote $F_{xz}:\Y\rightarrow\W$
the linear operator defined by $F_{xz}y:=F(x,z,y)$. For any $x\in\A$
and $\theta\in\H$, denote $K_{\theta}(x):=\ker F_{x\theta}$.

For any $\U\subseteq\Y$ a linear subspace, denote $\U^{+}:=\U\cap\D$
and define $\kappa(\U)\in\Box\D$ by

\[
\kappa(\U)=\SC{\zeta\in\Delta\D}{\E{y\sim\zeta}y\in\U^{+}}
\]

\begin{ass}\label{ass:lin}

For any $x\in\A$ and $\theta\in\H$, $K_{\theta}(x)\cap\D\ne\varnothing$
and $F_{x\theta}$ is onto. Moreover,

\[
\HC=\SC{H:\A\rightarrow\Box\D}{\exists\theta\in\H\forall x\in\A:H(x)=\kappa\left(K_{\theta}(x)\right)}
\]

\end{ass}
\begin{example}
\label{ex:pcb-lin}To see that Assumption \ref{ass:lin} holds in
Example \ref{ex:pcb}, we take

\[
\Z:=\bigoplus_{a\in\Stoc}\Z_{a}
\]

\[
\H:=\SC{\theta\in\bigoplus_{a\in\Stoc}\psi_{a}^{-1}(1)}{\forall x\in\A,a\in\Stoc,c\in\CB{|a|}:f_{a}(x,\theta_{a})_{c}\geq0}
\]

The latter is compact if we assume that there is no $a\in\A$ and
$z\in\ker\psi_{a}$ s.t. for all $x\in\A$ and $c\in\CB{|a|}$

\[
f_{a}(x,z)_{c}=0
\]

This loses no generality because we can always replace $\Z_{a}$ by
its quotient by the subspace of all $z$ satisfying this identity.

Denote

\[
\GS:=\SC{ac\in\prod_{i\leq|a|}\CB i}{a\in\Stoc,c\in\CB{|a|}}
\]

Also, for each $i<n$, denote

\[
\NB i:=\prod_{i<j<n}\CB j
\]

Then, we take

\[
\W:=\SC{w\in\R^{\GS}}{\forall a\in\Stoc:\sum_{c\in\CB{|a|}}w_{ac}=0}
\]

\[
F(x,z,y)_{ac}:=\psi_{a}(z_{a})\sum_{e\in\NB{|a|}}y_{ace}-f_{a}\left(x,z_{a}\right)_{c}\sum_{g\in\NB{|a|-1}}y_{ag}
\]

It is easy to see that the induced $K_{\theta}(x)$ is the same as
in equation (\ref{eq:pcb-k}).
\end{example}
Since $F_{x\theta}$ is onto, it follows that $\dim K_{\theta}(x)=\dim\Y-\dim\W$
and in particular it doesn't depend on $x$ or $\theta$. We will
denote this number $d$.

Assumption \ref{ass:lin} implies that any $\theta\in\H$ defines
some $H\in\HC$, and conversly, any $H\in\HC$ is induced by some
$\theta\in\H$. Given $\theta\in\H$ and its corresponding $H\in\HC$,
we will use the shorthand notations $\mathrm{ME_{\theta}:=\mathrm{ME}}_{H}$
and $\mathrm{ERg}_{\theta}:=\mathrm{\mathrm{ERg}}_{H}$.

We finish the section by showing that the classical framework of stochastic
linear bandits is also a special case of linear imprecise bandits.
\begin{example}
\label{ex:linear-bandits}Consider some $D\geq1$, let $\A\subset\R^{D}$
and assume the linear span of $\A$ is all of $\R^{D}$. The only
observable outcome is the reward which can be assumed to lie in $[-1,+1]$,
so we take $\Y:=\R^{2}$, $\mu=[\begin{array}{cc}
1 & 0\end{array}]$, and

\[
\D:=\SC{\left[\begin{array}{c}
1\\
t
\end{array}\right]\in\R^{2}}{t\in[-1,+1]}
\]

Let $\Z:=\R^{D}$ and $\H$ given by

\[
\H:=\SC{\theta\in\R^{D}}{\forall x\in\A:\T x\theta\in[-1,+1]}
\]

The reward function is

\[
r\left(x,\left[\begin{array}{c}
1\\
t
\end{array}\right]\right):=t
\]

Finally, let $\W:=\R$ and set $F:\A\times\R^{D}\times\R^{2}\rightarrow\R$
to be

\[
F\left(x,z,\left[\begin{array}{c}
y_{0}\\
y_{1}
\end{array}\right]\right):=(\T xz)y_{0}-y_{1}
\]

It is easy to see we got the stochastic linear bandit with arm set
$\A$.
\end{example}
In the rest of the thesis, we use Assumptions 1-4 implicitly.

\section{\label{sec:algo}IUCB Algorithm}

\global\long\def\XZ{\bar{\Z}}%

\subsection{Norms}

We propose an algorithm of UCB type which we call ``imprecise UCB''
(IUCB). The algorithm depends on a parameter $\eta\in\R_{+}$, the
optimal choice of which will be discussed in section \ref{sec:Upper-Bounds}.

In order to describe the algorithm, we will need norms on the spaces
$\W$ and a certain extension of $\Z$. We remind that the norm on
$\Y$ is the unique norm s.t. the absolute convex hull of $\D$ is
the unit ball. The norm on $\W$ is defined by

\[
\lVert w\rVert:=\max_{\SUBSTACK{x\in\A}{\theta\in\H}}\min_{y\in\Y:F_{x\theta}y=w}\lVert y\rVert
\]

\global\long\def\Nul{\mathcal{N}}%

\global\long\def\Vxx#1#2{\left[\begin{array}{c}
#1\\
#2
\end{array}\right]}%

Define the subspace $\Nul\subseteq\Z\oplus\W$ by

\[
\Nul:=\SC{\Vxx zw\in\Z\oplus\W}{\forall x\in\A,y\in\Y:F(x,z,y)+\mu(y)w=0}
\]

\global\long\def\F{\bar{F}}%

Denote $\XZ:=(\Z\oplus\W)/\Nul$. Define $\F:\A\times\XZ\times\Y\rightarrow\W$
by

\[
\F\left(x,\Vxx zw+\Nul,y\right):=F(x,z,y)+\mu(y)w
\]

As before, $\F_{xz}$ is a linear operator from $\Y$ to $\W$.

Finally, the norm on $\XZ$ is defined by

\global\long\def\Norm#1{\left\Vert #1\right\Vert }%

\global\long\def\Nxx#1#2{\left\Vert \begin{array}{c}
#1\\
#2
\end{array}\right\Vert }%

\[
\lVert z\rVert:=\max_{x\in\A}\Norm{\bar{F}_{xz}}
\]

Here, the norm on the right hand side the operator norm.

\global\long\def\V#1#2{\mathcal{V}\left(#1,#2\right)}%

Notice that $\W\cap\Nul=\{0\}$, and therefore there is a natural
embedding of $\W$ into $\XZ$. Moreover, the norm on $\W$ is compatible
with the norm on $\XZ$ and this embedding, because $\Norm{w\otimes\mu}=\Norm w\cdot\Norm{\mu}=\Norm w$.
We will also assume, without loss of generality, that $\Z\cap\Nul=\{0\}$,
because otherwise we can just replace $\Z$ with $\Z/(\Z\cap\Nul)$
and $\H$ with its image in $\Z/(\Z\cap\Nul)$, without changing anything.
Thereby, both $\Z$ and $\W$ will be viewed as \emph{subspaces} of
$\XZ$.

Extending $\Z$ to $\XZ$ is merely a device for defining the confidence
set. Because, the constraint on the true hypothesis $\theta^{*}\in\H$
that can be inferred from observing a certain average outcome $\bar{y}$
after selecting a certain arm $x$ multiple times, is most easily
formulated as distance from a subspace in $\XZ$ (see the definition
of $\V x{\bar{y}}$ below). And, $\Nul$ is just the ``null'' subspace
that we quotient out because adding a vector in $\Nul$ has no effect
on $\F$.

\subsection{Algorithm}

\global\long\def\C{\mathcal{C}}%

IUCB (Algorithm \ref{alg:IUCB}) works by maintaining a confidence
set $\C\subseteq\H$ and applying the principle of optimism in the
face of uncertainty. Initially, we set $\C_{0}:=\H$. On cycle $k\in\N$
of the algorithm, we choose the optimistic hypothesis $\theta_{k}$: 

\[
\theta_{k}:=\Amax_{\theta\in\C_{k}}\max_{x\in\A}\ME{\theta}rx
\]

For any $x\in\A$ and $y\in\Y$, we will denote by $\F_{x}^{y}:\XZ\rightarrow\W$
the linear operator given by $\F_{x}^{y}z:=\F(x,z,y)$. We also denote
$\V xy:=\ker\F_{x}^{y}$. The algorithm selects arm $x_{\theta_{k}}^{*}$
\emph{repeatedly} $\tau_{k}$ times, where $\tau_{k}$ is chosen to
be minimal s.t.

\begin{equation}
\sqrt{\tau_{k}}\cdot\max_{\theta\in\C_{k}}\min_{z\in\V{x_{\theta_{k}}^{*}}{\bar{y}_{k}}}\Norm{\theta-z}\geq2\left(D_{Z}+1\right)\eta\label{eq:stop}
\end{equation}

Here, $D_{Z}:=\dim\Z$, $T_{k}:=\sum_{i=0}^{k-1}\tau_{k}$, $\bar{y}_{k}:=\frac{1}{\tau_{k}}\sum_{n=T_{k}}^{T_{k+1}-1}y_{n}$,
where $y_{n}$ is the outcome of round $n$, and $\eta\in\R_{+}$
is a parameter of the algorithm. That is, $T_{k}$ is the total number
of rounds of the first $k$ cycles, and $\bar{y}_{k}$ is the average
outcome over cycle $k$. 

In other words, we keep selecting arm $x_{\theta_{k}}^{*}$ until
condition (\ref{eq:stop}) is met.

The idea behind equation (\ref{eq:stop}) is, we select arm $x_{\theta_{k}}^{*}$
sufficiently many times so that the confidence set can be substantially
narrowed down (by a factor of $\Omega(\eta D_{Z})$) in directions
traverse to $\V{x_{\theta_{k}}^{*}}{\bar{y}_{k}}$. We will use this
to argue that such a narrowing down can only occure a small number
of times ($\tilde{O}(D_{Z}^{2})$) before the confidence set becomes
thin (of size $O(D_{Z}^{2}N^{-\frac{1}{2}})$) in those directions,
and as a result the accrued expected regret becomes small.

After $\tau_{k}$ rounds, cycle $k$ is complete and the confidence
set is updated according to

\[
\C_{k+1}:=\SC{\theta\in\C_{k}}{\min_{z\in\V{x_{\theta_{k}}^{*}}{\bar{y}_{k}}}\Norm{\theta-z}\leq\frac{\eta}{\sqrt{\tau_{k}}}}
\]

Then, a new optimistic hypothesis $\theta_{k+1}$ is selected and
we switch to selecting the arm $x_{\theta_{k+1}}^{*}$, et cetera.

\begin{algorithm}
\begin{onehalfspace}
\caption{\label{alg:IUCB}Imprecise UCB}

\textbf{Input} $\eta\in\R_{+}$

$\C\leftarrow\H$

\textbf{for} $k$ \textbf{from} $1$ \textbf{to} $\infty$:

\qquad{}$\theta^{*}\leftarrow\Amax_{\theta\in\C}\max_{x\in\A}\ME{\theta}rx$

\qquad{}$\tau\leftarrow0$

\qquad{}$\Sigma y\leftarrow\boldsymbol{0}\in\Y$

\qquad{}\textbf{do}

\qquad{}\qquad{}select arm $x_{\theta^{*}}^{*}$ and observe outcome
$y$

\qquad{}\qquad{}$\tau\leftarrow\tau+1$

\qquad{}\qquad{}$\Sigma y\leftarrow\Sigma y+y$

\qquad{}\qquad{}$\bar{y}\leftarrow\frac{\Sigma y}{\tau}$

\qquad{}\textbf{while $\sqrt{\tau}\cdot\max_{\theta\in\C}\min_{z\in\V{x_{\theta^{*}}^{*}}{\bar{y}}}\Norm{\theta-z}<2\left(D_{Z}+1\right)\eta$}

\qquad{}$\C\leftarrow\SC{\theta\in\C}{\min_{z\in\V{x_{\theta^{*}}^{*}}{\bar{y}}}\Norm{\theta-z}\leq\frac{\eta}{\sqrt{\tau}}}$

\textbf{end for}
\end{onehalfspace}
\end{algorithm}

\section{\label{sec:Upper-Bounds}Upper Regret Bounds for IUCB}

\subsection{Parameters}

In addition to the dimensions of $\Z$, $\Y$ and $\W$, three parameters
characterizing the hypothesis class are needed in order to formulate
our regret bound. 

The first is just

\[
R:=\max_{\theta\in\H}\Norm{\theta}
\]

Notice that there is some redundancy in our description of a given
bandit. First, there is redundancy in the choice of $\H$: given any
$\theta\in\Z$ and $\chi\in\R\setminus0$, $K_{\theta}=K_{\chi\theta}$.
Hence, we get an equivalent bandit for any way of rescaling different
hypotheses by different scalars. Second, there is redundancy in the
choice of $F$: we can multiply it by any continuous function of $x$
without affecting $K_{\theta}$ for any $\theta$. While these redefinitions
have no effect on the predictions different hypotheses make, and in
particular no effect on regret, they \emph{can} affect $R$. In order
to get the tightest regret upper bounds from our results, we need
to choose the scaling with minimal $R$. See Proposition \ref{prop:hyperplane-rn}
for an example of using this.

To define the second parameter we need the following:

\global\long\def\Af{\mathfrak{A}}%

\global\long\def\Vec{\overset{\rightarrow}{\Af}}%

\global\long\def\SubAf{\mathfrak{B}}%

\begin{defn}
Let $\Af$ be a finite-dimensional affine space. Assume that the associated
vector space $\Vec$ is equipped with a norm. Let $\D\subseteq\Af$
be a closed convex set and $\SubAf\subseteq\Af$ an affine subspace
s.t. $\SubAf\cap\D\ne\varnothing$ and $\SubAf\not\subseteq\D$. We
define the \emph{sine} of $\SubAf$ relative to $\D$ to be

\[
\sin(\SubAf,\D):=\inf_{p\in\SubAf\setminus\D}\frac{\min_{q\in\D}\Norm{p-q}}{\min_{q\in\SubAf\cap\D}\Norm{p-q}}
\]
\end{defn}
The motivation for the name is the following (see Appendix \ref{sec:s}
for the proof):

\global\long\def\SubC{\mathfrak{C}}%

\global\long\def\X{\mathcal{X}}%

\global\long\def\VecB{\overset{\rightarrow}{\SubAf}}%

\global\long\def\VecC{\overset{\rightarrow}{\SubC}}%

\begin{prop}
\label{prop:sin}Let $\Af$ be a finite-dimensional affine space and
$\SubAf,\SubC\subseteq\Af$ subspaces s.t. $\SubAf\cap\SubC\ne\varnothing$
and $\SubAf\not\subseteq\SubC$. Suppose $\Vec$ is an inner product
space. Denote $\X:=\VecB\cap(\VecB\cap\VecC)^{\bot}$ and $\Y:=\VecC\cap(\VecB\cap\VecC)^{\bot}$.
Then,

\[
\sin(\SubAf,\SubC)=\min_{\SUBSTACK{x\in\X:\Norm x=1}{y\in\Y:\Norm y=1}}\sqrt{1-(x\cdot y)^{2}}
\]
\end{prop}
In other words, $\sin(\SubAf,\SubC)$ is the sine of the first non-vanishing
principal angle\footnote{See e.g. \cite{van1996subspace} section 1.4.3.}
between $\SubAf$ and $\SubC$.

Notice that in general it is possible that $\sin(\SubAf,\D)=0$, but
when $\D$ is a polytope it is impossible (see Proposition \ref{prop:polytope-sine}).

It is possible to bound $\sin\left(\SubAf,\D\right)$ in terms of
the sine of $\SubAf$ with the supporting hyperplanes of $\D$ (see
Appendix \ref{sec:s} for the proof):

\global\long\def\SH#1{\mathrm{SH}_{#1}\mathcal{\D}}%

\global\long\def\DD#1{\partial_{#1}\D}%

\begin{prop}
\label{prop:sin-sh}Let $\Af$ be a finite-dimensional affine space.
Assume that the associated vector space $\Vec$ is equipped with a
norm. Let $\D\subseteq\Af$ be a closed convex set and $\SubAf\subseteq\Af$
an affine subspace s.t. $\SubAf\cap\D\ne\varnothing$. Assume that
the affine hull of $\D$ is $\Af$. Let $\DD{\SubAf}$ denote the
bounary of $\SubAf\cap\D$ relative to $\SubAf$. For each $q\in\DD{\SubAf}$,
let $\SH q$ denote the set of supporting hyperplanes of $\D$ passing
through $q$. Then,

\[
\inf_{q\in\DD{\SubAf}}\inf_{\pi\in\SH q}\sin\left(\SubAf,\pi\right)\leq\sin(\SubAf,\D)\leq\inf_{q\in\DD{\SubAf}}\sup_{\pi\in\SH q}\sin\left(\SubAf,\pi\right)
\]

Here, $\sin\left(\SubAf,\pi\right)$ is defined to be $0$ whenever
$\SubAf\subseteq\pi$.
\end{prop}
Notice that when $\D$ is nonsingular (i.e. every boundary point has
a unique supporting hyperplane), the lower bound and the upper bound
coincide.

Returning to our setting, for any linear subspace $\U\subseteq\Y$,
we denote $\U^{\flat}:=\U\cap\mu^{-1}(1)$. Then, the second parameter
is

\[
S:=\min_{\SUBSTACK{x\in\A}{\theta\in\H}}\sin\left(K_{\theta}(x)^{\flat},\D\right)
\]

In section \ref{sec:Special-Cases}, we will study the properties
of $R$ and $S$ in some special cases. 
\begin{example}
In Example \ref{ex:pcb-lin}, we have $R\leq4n$ (by Proposition \ref{prop:chain-r}
and Proposition \ref{prop:point-r}) and $S=1$ (by Proposition \ref{prop:chain-s}).
\end{example}
Finally, the third parameter is just the size of the range of the
reward function:

\global\long\def\RR{C}%

\[
\RR:=\max r-\min r
\]

Notice that if $r$ doesn't depend on the first argument, $C\leq2$
because $r$ is 1-Lipschitz.

\subsection{General Case}

\global\long\def\IUCB{\Pol_{\text{\ensuremath{\mathrm{IUCB}}}}}%

We can now state our main result. Let $D_{W}:=\dim\W$. Denote $\IUCB^{\eta}$
the policy implemented by the IUCB algorithm.
\begin{thm}
\label{thm:main}Let $N$ be a positive integer and fix any $\delta>0$.
Denote

\[
\gamma:=\frac{1}{\ln\left(1-\frac{1}{e^{2}}\right)^{-1}}
\]

Then, for all $\theta\in\H$,

\begin{align*}
\Reg{\theta}{\IUCB^{\eta}}N & \leq & 8\eta\left(S^{-1}+1\right)D_{Z}(D_{Z}+1)\sqrt{\gamma\ln\frac{D_{Z}R}{\delta}\cdot N}\\
 &  & +CD_{W}N^{2}(N+1)\exp\left(-\Omega(1)\cdot\frac{\eta^{2}}{R^{2}D_{W}^{\frac{5}{3}}}\right)\\
 &  & +\gamma CD_{Z}^{2}\ln\frac{D_{Z}R}{\delta}\\
 &  & +\left(S^{-1}+1\right)D_{Z}\left(36D_{Z}+8\right)N\delta
\end{align*}

In particular, we can set

\[
\eta:=\Theta(1)\cdot RD_{W}^{\frac{5}{6}}\sqrt{\ln\left(CD_{W}N\right)}
\]

\[
\delta:=\frac{1}{\sqrt{N}}
\]

And then,

\[
\Reg{\theta}{\IUCB^{\eta}}N=\tilde{O}\left(RS^{-1}D_{Z}^{2}D_{W}^{\frac{5}{6}}\sqrt{N}+D_{Z}^{2}C\right)
\]
\end{thm}
Here and elsewhere, we use the notation $\tilde{O}(f):=O(f\cdot\text{poly}(\ln f))$.
The proof is given in Appendix \ref{sec:Proof-of-Main}.

At first glance, it may appear strange that the bound in Theorem \ref{thm:main}
doesn't depend on the size or dimension of $\A$. However, notice
that we only ``care'' about arms of the form $x_{\theta}^{*}$ for
some $\theta\in\H$. Indeed, other arms don't affect the definition
of regret and are never selected in IUCB. Hence, effectively $\A$
can be assumed to be at most $D_{Z}$-dimensional.

In the case when $\D$ is a simplex we can get another bound which
is often better:

\global\long\def\Abs#1{\left|#1\right|}%

\begin{thm}
\label{thm:simplex}Let $\B$, $\Y$, $\mu$ and $\D$ be as in Example
\ref{ex:simplex}\footnote{$r$ doesn't have to be as in Example \ref{ex:simplex} for this theorem,
as long as Assumption \ref{ass:r} is satisfied.}. Let $N$ be a positive integer and fix any $\delta>0$. Then, for
all $\theta\in\H$,

\begin{align*}
\Reg{\theta}{\IUCB^{\eta}}N & \leq & 8\eta\left(S^{-1}+1\right)D_{Z}(D_{Z}+1)\sqrt{\gamma\ln\frac{D_{Z}R}{\delta}\cdot N}\\
 &  & +\frac{1}{2}C\left(\frac{2e\Abs{\B}}{D_{W}+1}\right)^{D_{W}+1}N^{2}(N+1)\exp\left(-\frac{\eta^{2}}{2R^{2}}\right)\\
 &  & +\gamma CD_{Z}^{2}\ln\frac{D_{Z}R}{\delta}\\
 &  & +\left(S^{-1}+1\right)D_{Z}\left(36D_{Z}+8\right)N\delta
\end{align*}

In particular, we can set

\[
\eta:=R\sqrt{2\left(D_{W}+1\right)\ln\left(\frac{CN^{3}\Abs{\B}}{D_{W}+1}\right)}
\]

\[
\delta:=\frac{1}{\sqrt{N}}
\]

And then,

\[
\Reg{\theta}{\IUCB^{\eta}}N=\tilde{O}\left(RS^{-1}D_{Z}^{2}\sqrt{D_{W}\ln\frac{\Abs{\B}}{D_{W}}\cdot N}+D_{Z}^{2}C\right)
\]
\end{thm}
The proof is given in Appendix \ref{sec:simplex}.
\begin{example}
In Example \ref{ex:pcb-lin}, $R\leq4n$, $S=1$, $D_{W}\leq\Abs{\Stoc^{\sharp}}$,
$\Abs{\B}=\prod_{i<n}\Abs{\CB i}$ and $C\leq2$. Therefore, Theorem
\ref{thm:simplex} implies

\[
\Reg{\theta}{\IUCB}N=\tilde{O}\left(nD_{Z}^{2}\sqrt{\Abs{\Stoc^{\sharp}}\sum_{i<n}\ln\Abs{\CB i}\cdot N}\right)
\]
\end{example}

\subsection{Positive Gap Case}

Now, we go on to an upper bound \emph{logarithmic} in $N$, which
is applicable when the hypothesis class has a positive \emph{gap}.
This is analogous to the classical theory of stochastic linear bandits,
except that the definition of ``gap'' for our framework is not entirely
obvious. Specifically, it is as follows:

\global\long\def\DY#1#2{d_{Y}\!\!\left(#1,#2\right)}%

\begin{defn}
\label{def:gap}The \emph{gap} is the maximal number $g\in\R_{+}$
for which the following holds. Consider any $\theta,\theta'\in\H$
satisfying the following conditions:
\begin{enumerate}
\item $\ME{\theta'}r{x_{\theta'}^{*}}\geq\ME{\theta}r{x_{\theta}^{*}}$
\item $\ME{\theta}r{x_{\theta'}^{*}}<\ME{\theta}r{x_{\theta}^{*}}$
\end{enumerate}
Then,

\[
\DY{K_{\theta}\left(x_{\theta'}^{*}\right)^{+}}{K_{\theta'}\left(x_{\theta'}^{*}\right)^{+}}\geq g
\]
\end{defn}
Here, $d_{Y}$ stands for minimal distance w.r.t. the norm on $\Y$.
Notice that the gap depends, in principle, on the precise way we assign
optimal arms to hypotheses (when there is more than one optimal arm),
so for this purpose we will assume $x_{\theta}^{*}$ stands for some
specific measurable mapping from $\H$ to $\A$ that we chose.

Recall that in classical theory, the ``gap'' is the difference between
the reward of the best arm and the reward of the next best arm. The
relation between that notion and Definition \ref{def:gap} is made
clearer by the following (see Appendix \ref{sec:gap} for the proof):
\begin{prop}
\label{prop:gap}Assume $\theta\in\H$ and $g_{\theta}\in\R_{+}$
are s.t. for any $x\in\A$, one of the following is true:
\begin{enumerate}
\item $\ME{\theta}rx=\ME{\theta}r{x_{\theta}^{*}}$
\item $\max_{y\in K_{\theta}\left(x\right)^{+}}r(x,y)\leq\ME{\theta}r{x_{\theta}^{*}}-g_{\theta}$
\end{enumerate}
Then, for any $\theta'\in\H$, if conditions 1 and 2 of Definition
\ref{def:gap} hold, then

\[
\DY{K_{\theta}\left(x_{\theta'}^{*}\right)^{+}}{K_{\theta'}\left(x_{\theta'}^{*}\right)^{+}}\geq g_{\theta}
\]
\end{prop}
Notice that ``gap'' in Proposition \ref{prop:gap} is between the
minimum reward of the optimal arm and the \emph{maximum} reward of
an unoptimal arm.

The following is the upper bound enabled by a positive gap:
\begin{thm}
\label{thm:gap}Assume the gap is $g>0$. Let $N$ be a positive integer.
Then, for all $\theta\in\H$

\begin{align*}
\Reg{\theta}{\IUCB^{\eta}}N\leq & \gamma D_{Z}^{2}\left(\frac{256S^{-1}\left(S^{-1}+1\right)(D_{Z}+1)^{2}\eta^{2}}{g}+C\right)\ln\frac{144S^{-1}D_{Z}^{3}R}{g}\\
 & +CD_{W}N^{2}(N+1)\exp\left(-\Omega(1)\cdot\frac{\eta^{2}}{R^{2}D_{W}^{\frac{5}{3}}}\right)
\end{align*}

In particular, we can set

\[
\eta:=\Theta(1)\cdot RD_{W}^{\frac{5}{6}}\sqrt{\ln\left(CD_{W}N\right)}
\]

And then,

\[
\Reg{\theta}{\IUCB^{\eta}}N=\tilde{O}\left(\frac{R^{2}S^{-2}D_{W}^{\frac{5}{3}}D_{Z}^{4}}{g}\cdot\ln N+D_{Z}^{2}C\right)
\]
\end{thm}
The proof is given in Appendix \ref{sec:gap}. Notice that we can
achieve optimal asymptotics using the same value of $\eta$ for both
Theorem \ref{thm:main} and Theorem \ref{thm:gap}.

In order to apply Theorem \ref{thm:gap} to Example \ref{ex:zerosum},
we will need the following:
\begin{prop}
\label{prop:zerosum-gap}In the setting of Example \ref{ex:zerosum},
consider some $P,P'\in\RI^{\B_{1}\times\B_{2}}$. Let $\Y$, $\mu$
and $\D$ be as in Example \ref{ex:pcb} and $\Z$, $\W$ and $F$
as in Example \ref{ex:pcb-lin}. Let $\theta,\theta'\in\Z$ correspond
to the payoff matrices $P,P'$ in the manner described in Example
\ref{ex:zerosum}. Consider any $x\in\A_{1}$. Then

\[
\DY{K_{\theta}(x)^{+}}{K_{\theta'}(x)^{+}}\geq\frac{1}{2}\min_{b\in\B_{2}}\E{a\sim x}{\Abs{P_{ab}-P'_{ab}}}
\]
\end{prop}
The proof is given in Appendix \ref{sec:gap}.
\begin{example}
\label{ex:zerosum-gap}In the setting of Example \ref{ex:zerosum},
consider some $\mathcal{M}\subseteq\RI^{\B_{1}\times\B_{2}}$. Let
$\Y$, $\mu$, $\D$ and $r$ be as in Example \ref{ex:pcb} and $\Z$,
$\W$ and $F$ as in Example \ref{ex:pcb-lin}, but define $\H$ by

\[
\H:=\SC{\bigoplus_{b\in\B_{2}}1\oplus\bigoplus_{\SUBSTACK{b\in\B_{2}}{a\in\B_{2}}}\left[\begin{array}{cc}
\frac{1-P_{ab}}{2} & \frac{1+P_{ab}}{2}\end{array}\right]}{\PM\in\mathcal{M}}
\]
For any $P\in\mathcal{M}$, denote $x_{P}^{*}:=\Argmax{x\in\A_{1}}\:\underset{y\in\A_{2}}{\min}\:\T xPy$.
Fix $\tilde{g}>0$. Assume that for any $P,P'\in\mathcal{M}$ satisfying
\begin{enumerate}
\item $\max_{x\in\A_{1}}\min_{y\in\A_{2}}\T xP'y\geq\max_{x\in\A_{1}}\min_{y\in\A_{2}}\T xPy$
\item $\min_{y\in\A_{2}}\T{(x_{P'}^{*})}Py<\max_{x\in\A_{1}}\min_{y\in\A_{2}}\T xPy$
\end{enumerate}
it holds that

\[
\frac{1}{2}\min_{b\in\B_{2}}\E{a\sim x}{\Abs{P_{ab}-P'_{ab}}}\geq\tilde{g}
\]

Then, by Proposition \ref{prop:zerosum-gap}, the gap satisfies $g\geq\tilde{g}$.
Moreover, $R\leq12$, $S=1$, $D_{W}=O(\Abs{\B_{1}}\cdot\Abs{\B_{2}})$,
$D_{Z}=\Abs{\B_{2}}+2\Abs{\B_{1}}\cdot\Abs{\B_{2}}$ and $C\leq2$.
Therefore, by Theorem \ref{thm:gap}

\[
\Reg{\theta}{\IUCB^{\eta}}N=\tilde{O}\left(\frac{\left(\Abs{\B_{1}}\cdot\Abs{\B_{2}}\right)^{\frac{17}{3}}}{\tilde{g}}\cdot\ln N\right)
\]
\end{example}

\section{\label{sec:Lower-Regret-Bounds}Lower Regret Bounds}

\subsection{Scaling with $D_{Z}$}

Since stochastic linear bandits are a special case of our framework,
lower bounds for the former are also lower bounds for the latter.
In particular, we have the following theorem adapted from \cite{DBLP:conf/colt/DaniHK08}
(it appears there as ``Theorem 3''):
\begin{thm}
[Dani-Hayes-Kakade]\label{thm:DHK}Consider the setting of Example
\ref{ex:linear-bandits}, with $D=2n$ for some $n\geq1$ and

\[
\A:=\SC{x\in\R^{2n}}{\forall i<n:x_{2i}^{2}+x_{2i+1}^{2}=1}
\]

Let $\xi\in\Delta\H$ be the uniform distribution\footnote{The torus is a Cartesian product of $n$ circles, and the uniform
distribution on the torus is just the product of the uniform distributions
on the circles.} over the torus $\frac{1}{n}\A\subseteq\H$. Then, for \emph{any}
agent policy $\Pol$ and time horizon $N\geq\max(\frac{n^{2}}{16},256)$,

\[
\E{\theta\sim\xi}{\Reg{\theta}{\Pol}N}\geq\frac{1}{2304}n\sqrt{N}
\]
\end{thm}
In order to appreciate the significance of this lower bound, we need
to evaluate the parameters of section \ref{sec:Upper-Bounds}. For
the dimensions, $D_{Z}=2D$, $D_{W}=1$. Since $d=1$, we have $S=1$.
And, obviously $C=2$. The remaining parameter is $R$ for which we
have (see Appendix \ref{sec:r} for the proof):
\begin{prop}
\label{prop:DHK-r}In the setting of Theorem \ref{thm:DHK}, $R=2$.
\end{prop}
It follows that any upper bound on regret expressed in terms of the
same parameters as in Theorem \ref{thm:main}, must be $\Omega(D_{Z}\sqrt{N})$
(since, all the other parameters are bounded in this family of examples).

\subsection{Scaling with $S$}

Next, we demonstrate a novel lower bound, which establishes the necessity
of the parameter $S$.

\global\long\def\Vxxx#1#2#3{\left[\begin{array}{c}
#1\\
#2\\
#3
\end{array}\right]}%

\begin{thm}
\label{thm:lower-s}Let $D\geq4$, $\alpha\in(0,\frac{1}{4}]$, $\Y:=\R^{D+2}$,
$\Z:=\R^{D+2}$, $\W:=\R$ and

\[
\A:=\SC{x\in\R^{D}}{\Norm x_{2}=1}
\]

\[
\mu(y):=y_{0}+y_{1}
\]

\[
\D:=\SC{y\in\R^{D+2}}{y_{0}+y_{1}=1\text{, }y_{0}\geq0\text{ and }y_{1}^{2}\geq\sum_{i=2}^{D+1}y_{i}^{2}}
\]

\[
\H:=\SC{\Vxxx{1-\alpha}{-\alpha}{2\alpha u}\in\R^{D+2}}{u\in\R^{D}\text{, }\Norm u_{2}=1}
\]

\[
r\left(x,y\right):=\sum_{i<D}x_{i}y_{i+2}
\]

\[
F\left(x,z,y\right):=\T zy
\]

Let $\xi\in\Delta\H$ be the uniform distribution\footnote{Notice that $\H$ is a sphere.}.
Then, for \emph{any} agent policy $\Pol$ and time horizon $N>0$,

\[
\E{\theta\sim\xi}{\Reg{\theta}{\Pol}N}\geq\frac{2}{\pi^{2}}\left(1-\frac{1}{\sqrt{\pi\left(2D-1\right)}}\right)\left(\frac{1-\alpha}{1+2\alpha}\right)^{N}N^{\frac{D-3}{D-1}}-\frac{1}{2}
\]

In particular, for $\alpha\ll1$ we get

\[
\E{\theta\sim\xi}{\Reg{\theta}{\Pol}N}=\Omega\left(e^{-3\alpha N}N^{\frac{D-3}{D-1}}\right)
\]
\end{thm}
The proof is given in Appendix \ref{sec:lower-s}. The parameters
in this example are $D_{Z}=D+2$, $D_{W}=1$, $C=2$ and (see Appendix
\ref{sec:lower-s} for the proof):
\begin{prop}
\label{prop:lower-s-params}In the setting of Theorem \ref{thm:lower-s}:

\[
R=1
\]

And, for any fixed $D\geq4$,

\[
\liminf_{\alpha\rightarrow0}\frac{S(\alpha)}{\alpha}>0
\]
\end{prop}
Consider the asymptotics in Theorem \ref{thm:lower-s} when $\alpha=\frac{1}{N}$.
It follows that, in any upper bound on regret expressed in terms of
the same parameters as in Theorem \ref{thm:main}, if it scales as
$f(S^{-1})\sqrt{N}$ with $S$ and $N$, for some function $f$, then
$f(t)=\omega(t^{\beta})$ for any $\beta<\frac{1}{2}$. 

Another moral from Theorem \ref{thm:lower-s} is that we cannot easily
drop Assumption \ref{ass:affine} (i.e. replace the ``affine'' sets
$\U^{+}$ with arbitrary convex subsets of $\D$ in our specification
of the credal sets). Indeed, in the limit $\alpha\rightarrow0$, the
setting becomes a family of \emph{halfspaces} (inside the ball $\SC{y\in\D}{y_{1}=1}$,
which is the base of the cone $\D$). The theorem shows that, even
though this family is low-dimensional in a very natural sense, it
is impossible to get a power law regret bound for it with an exponent
that doesn't depend on $D$.

\subsection{Scaling with $R$}

Finally, we demonstrate another lower bound that establishes the necessity
of the parameter $R$.

\global\long\def\Rxx#1#2{\left[\begin{array}{cc}
#1 & #2\end{array}\right]}%

\global\long\def\Mxx#1#2#3#4{\left[\begin{array}{cc}
#1 & #2\\
#3 & #4
\end{array}\right]}%

\begin{thm}
\label{thm:lower-r}Let $\lambda\in(0,\infty)$, $\alpha\in(\frac{3}{8}\pi,\frac{1}{2}\pi)$,
$\Y:=\R^{3}$, $\Z:=\R^{2}$, $\W:=\R$ and

\[
\A:=\SC{x\in\R^{2}}{\Norm x_{2}=1\text{, }x_{1}>0\text{ and }\Abs{x_{0}}\leq\cos\alpha}
\]

\[
\mu(y):=y_{2}
\]

\[
\D:=\SC{y\in\R^{3}}{y_{2}=1\text{ and }y_{0}^{2}+y_{1}^{2}\leq1}
\]

\[
\H:=\SC{\theta\in\R^{2}}{\Norm{\theta}_{2}=1\text{, }\text{\ensuremath{\theta_{0}>0}\text{ and} }\Abs{\theta_{1}}\leq\cos\alpha}
\]

\[
r(x,y):=-y_{0}
\]

\[
F(x,z,y):=\T z\left(I_{2}+\lambda x\T x\right)\Vxx{y_{0}}{y_{1}}
\]

Let $\xi\in\Delta\H$ be the uniform distribution. Then, for \emph{any
}agent policy $\Pol$ and time horizon $N>0$,

\begin{align*}
\E{\theta\sim\xi}{\Reg{\theta}{\Pol}N} & \geq & \frac{1}{2}\left(\cos\frac{3}{8}\pi-\cos\alpha\right)\times\\
 &  & \min\left(\left(\frac{1}{2}-\frac{\alpha}{\pi}\right)\left(\lambda+1\right)\tan\frac{1}{5\sqrt{N}}-1,N\right)
\end{align*}

\[
\]

In particular, for $\lambda\gg\sqrt{N}$,

\[
\E{\theta\sim\xi}{\Reg{\theta}{\Pol}N}=\Omega\left(\min\left(\frac{\lambda}{\sqrt{N}},N\right)\right)
\]
\end{thm}
The proof is given in Appendix \ref{sec:lower-r}. The parameters
in this example are $D_{Z}=2$, $D_{W}=1$, $C=2$ and (see Appendix
\ref{sec:lower-r} for proof):
\begin{prop}
\label{prop:lower-r-params}In the setting of Theorem \ref{thm:lower-r}:

\[
R\leq\lambda+1
\]

\[
S=1
\]
\end{prop}
Consider the asymptotics in Theorem \ref{thm:lower-r} when $\lambda=N^{\frac{3}{2}}$.
In this case, the expected regret has to be $\Omega(N)$. It follows
that, in any upper bound on regret expressed in terms of the same
parameters as in Theorem \ref{thm:main}, if it scales as $g(R)\sqrt{N}$
with $R$ and $N$, for some function $g$, then $g(N^{\frac{3}{2}})=\Omega(\sqrt{N}$).
That is, $g(R)=\Omega(R^{\frac{1}{3}})$.

\section{\label{sec:Special-Cases}Special Cases}

\subsection{Simple Bounds on $R$}

\global\long\def\AO#1{A_{#1}}%

Let $\W:=\R$. In this case, $K_{\theta}(x)^{\flat}$ is a hyperplane
inside $\mu^{-1}(1)$ for any $\theta\in\H$ and $x\in\A$. For any
$x\in\A$, we can define the linear operator $\AO x:\Z\rightarrow\Y^{\star}$
given by

\[
\left(\AO xz\right)(y):=F(x,z,y)
\]

\global\long\def\NZ#1{\lVert#1\rVert_{0}}%

Assume that for all $x\in\A$, $\AO x$ is invertible (in particular,
$\dim\Y=D_{Z}$). Let $\NZ{\cdot}$ be any norm on $\Z$ (which can
be entirely different from the norm $\Norm{\cdot}$ which we defined
in section \ref{sec:algo}). Together with the usual norm on $\Y$,
this allows defining the operator norms $\NZ{\AO x}$ and $\NZ{\AO x^{-1}}$.
We have,
\begin{prop}
\label{prop:hyperplane-r}In the setting above,

\[
R(\H,F)\leq\frac{\max_{\theta\in\H}\Norm{\theta}_{0}}{\min_{\theta\in\H}\Norm{\theta}_{0}}\left(\max_{x\in\A}\NZ{\AO x}\right)\left(\max_{a\in\A}\NZ{\AO x^{-1}}\right)
\]
\end{prop}
Here, we made the dependence of $R$ on $\H$ and $F$ explicit in
the notation.

For any continuous function $\chi:\H\rightarrow\R\setminus0$, we
define

\[
\H^{\chi}:=\SC{\chi(\theta)\theta}{\theta\in\H}
\]

Also, for any continuous function $\lambda:\A\rightarrow\R\setminus0$,
we define $F^{\lambda}:\A\times\Z\times\Y\rightarrow\W$ by

\[
F^{\lambda}(x,z,y):=\lambda(x)F(x,z,y)
\]

As we remarked in section \ref{sec:Upper-Bounds}, the bandit defined
by $\H^{\chi}$ and $F^{\lambda}$ is equivalent to that defined by
$\H$ and $F$. However, it might have different $R$. We have
\begin{prop}
\label{prop:hyperplane-rn}In the same setting, there exist $\chi:\H\rightarrow\R\setminus0$
and $\lambda:\A\rightarrow\R\setminus0$ s.t.

\[
R\left(\H^{\chi},F^{\lambda}\right)\leq\max_{x\in\A}\left(\NZ{\AO x}\cdot\NZ{\AO x^{-1}}\right)
\]
\end{prop}
This bound is an improvement on Proposition \ref{prop:hyperplane-r}.
Notice that the expression $\NZ{\AO x}\cdot\NZ{\AO x^{-1}}$ is the
\emph{condition number} of $\AO x$.

Now, we consider the case where $\W=\ker\mu$, and we have $\psi\in\Z^{\star}$
and $f:\A\times\Z\rightarrow\Y$ continuous and linear in the second
argument s.t.
\begin{itemize}
\item For any $x\in\A$ and $z\in\Z$, $\mu(f(x,z))=\psi(z)$.
\item $\H\subseteq\psi^{-1}(1)$
\item $f(\A\times\H)\subseteq\D$
\item For any $x\in\A$, $z\in\Z$ and $y\in\Y$, $F(x,z,y)=\psi(z)y-\mu(y)f(x,z)$.
\end{itemize}
Then, for any $x\in\A$ and $\theta\in\H$, $K_{\theta}(x)^{+}=\{f(x,\theta)\}$.
Essentially, this datum is just an arm-dependent affine mapping from
$\H$ to $\D$ that depends continuously on the arm. We have,
\begin{prop}
\label{prop:point-r}In the setting above,

\[
R(\H,F)\leq2
\]
\end{prop}
See Appendix \ref{sec:r} for the proofs of the Propositions in this
subsection.

\subsection{\label{subsec:s}Simple Bounds on $S$}

For any $d\geq1$, denote 

\[
\G d:=\SC{\U\subseteq\Y\text{ a linear subspace}}{\ensuremath{\dim\U=d},\,\ensuremath{\U\cap\D\ne\varnothing}}
\]

$\G d$ is a subset of the Grassmannian $\mathrm{\mathrm{Gr}}_{d}(\Y)$.
Notice that $\G 1$ is canonically isomorphic to $\D$.

The following is a bound on $S$ for the case where $\D$ is a simplex.
\begin{prop}
\label{prop:simplex-full-s}Let $\B$ be a finite set, $\Y:=\R^{\B}$
and

\[
\mu(y):=\sum_{i\in\B}y_{i}
\]

\[
\D:=\Delta\B
\]

Consider any $d\geq1$ and $\U\in\G d$. Then,

\[
\sin\left(\U^{\flat},\D\right)\geq\max_{y\in\U^{+}}\min_{i\in\B}y_{i}
\]
\end{prop}
As we saw in Example \ref{ex:simplex}, the simplex is a natural special
case. We will use the perpsective of Example \ref{ex:simplex} throughout
this section when discussing the simplex, even though the formal propositions
are applicable more generally (i.e. they are still useful when non-vertex
outcomes are allowed).

Proposition \ref{prop:simplex-full-s} allows us to lower bound the
sine as long as the credal set has an outcome distribution with full
support. More generally, we have the following
\begin{prop}
\label{prop:simplex-s}Let $\B$, $\Y$, $\mu$, $\D$ and $\U$ be
as in Proposition \ref{prop:simplex-full-s}. Suppose that $\mathcal{\mathcal{E\subseteq\B}}$
is s.t. $\U\subseteq\R^{\mathcal{E}}$. Then,

\[
\sin\left(\U^{\flat},\D\right)\geq\max_{y\in\U^{+}}\min_{i\in\mathcal{E}}y_{i}
\]
\end{prop}
Here, $\R^{\mathcal{E}}$ is regarded as a subspace of $\R^{\B}$
by setting all the coordinates in $\B\setminus\mathcal{E}$ to 0.

This allows us to produce some lower bound for any given $\U$. Indeed,
if $y_{1,2}\in\U\cap\Delta\B$, the support of $y_{1}$ is $\mathcal{E}_{1}$
and the support of $y_{2}$ is $\mathcal{E}_{2}$, then the support
of $\frac{1}{2}(y_{1}+y_{2})$ is $\mathcal{E}_{1}\cup\mathcal{E}_{2}$.
Hence for any $\U$ there is some $\mathcal{E}$ s.t. all distributions
in $\U$ are supported on it and some distribution in $\U$ has the
entire $\mathcal{E}$ as its support.

Now, we consider the case where $\D$ is the unit ball.
\begin{prop}
\label{prop:ball-s}Let $n\in\N$, $\Y=\R^{n+1}$ and

\[
\mu(y):=y_{n}
\]

\[
\D:=\SC{\Vxx y1\in\R^{n+1}}{\Norm y_{2}\leq1}
\]

Consider some $d\geq1$ and $\U\in\G d$. Then,

\[
\sin\left(\U^{\flat},\D\right)\geq\sqrt{1-\min\SC{\Norm y_{2}^{2}}{\Vxx y1\in\U^{\flat}}}
\]
\end{prop}
This lower bound vanishes when $\U^{\flat}$ approaches a tangent
to the ball: indeed, for a tangent, the sine is zero.

See Appendix \ref{sec:s} for the proofs of the Propositions in this
subsection.

\subsection{Probability Systems}

\global\long\def\PSet{\mathcal{F}}%

Let's once again examine the case where $\D$ is a simplex and consider
a credal set defined by fixing the probabilities of a family of events
on $\B$. 
\begin{prop}
\label{prop:sys-probs}Let $\B$ be a finite set, $\Y:=\R^{\B}$ and
$\D:=\Delta\B$. Let $\PSet$ be a non-empty finite set and $f:\B\rightarrow2^{\PSet}$
a surjection. Fix $p\in[0,1]^{\PSet}$ and define the linear subspace
$\U$ of $\Y$ by

\[
\U:=\SC{y\in\Y}{\forall i\in\PSet:\sum_{a\in\B:i\in f(a)}y_{a}=p_{i}\sum_{a\in\B}y_{a}}
\]

Then,

\[
\sin\left(\U^{\flat},\D\right)\geq\frac{1}{\Abs{\PSet}}
\]

In particular,
\end{prop}
\[
\sin\left(\U^{\flat},\D\right)\geq\frac{1}{\log\Abs{\B}}
\]

That is, each $i\in\PSet$ corresponds to an event $i\in f(a)$ whose
probability we require to be $p_{i}$. The condition that $f$ is
surjective means that we require these events to be ``logically independent''
(i.e. it's possible for every one of them to happen or not, regardless
of whether the others happened or not).

For $\Abs{\PSet}=1$, we get $\sin\left(\U^{\flat},\D\right)=1$.
The same holds if replace absolute probability by conditional probability:

\global\long\def\Econd{\mathcal{E}}%

\global\long\def\Event{\mathcal{E}'}%

\begin{prop}
\label{prop:cond-s}Let $\B$ be a finite set, $\Y:=\R^{\B}$ and
$\D:=\Delta\B$. Consider some $\Econd,\Event\subseteq\B$, fix $p\in[0,1]$
and define the linear subspace $\U$ of $\Y$ by

\[
\U:=\SC{y\in\Y}{\sum_{a\in\Econd\cap\Event}y_{a}=p\sum_{a\in\Econd}y_{a}}
\]

Then,

\[
\sin\left(\U^{\flat},\D\right)=1
\]
\end{prop}
In the above, the subspace $\U$ can be interpreted as fixing the
probability of event $\Event$ conditional on event $\Econd$ to the
value $p$. Now we'll consider a \emph{system} of conditional probabilities,
but require a particular structure.

\global\long\def\SuppU#1#2{\mathcal{E}_{#1}\left(#2\right)}%

\global\long\def\Ker{\mathcal{\U}}%

\global\long\def\Gany#1{\mathrm{Gr}^{+}\left(#1\right)}%

\begin{prop}
\label{prop:chain-s}Fix an integer $n\geq1$, and let $\B:=\prod_{i<n}\CB i$
for some family of finite sets $\left\{ \CB i\right\} _{i<n}$. Let
$\Y:=\R^{\B}$, $\D:=\Delta\B$. For each $i<n$, denote

\[
\PB i:=\prod_{j<i}\CB j
\]

\[
\NB i:=\prod_{i<j<n}\CB j
\]

Let $\Ker_{i}:\PB i\rightarrow\Gany{\R^{\CB i}}$, where dropping
the subscript $d$ means we consider subspaces of all dimensions.
Given $y\in\Y$, $i<n$ and $a\in\PB i$, let $\SuppU ia\subseteq\CB i$
be the minimal set s.t. $\U_{i}(a)\subseteq\R^{\SuppU ia}$ and define
$y^{a}\in\R^{\CB i}$ by

\[
y_{b}^{a}:=\sum_{c\in\NB i}y_{abc}
\]

Define the subspace $\U$ of $\Y$ by

\begin{align*}
\U:=\{y\in\Y\mid\forall i<n,a\in\PB i: & \:y^{a}\in\U_{i}(a)\text{ and }\\
 & \:\forall b\in\CB i\setminus\SuppU ia,c\in\NB i:y_{abc}=0\}
\end{align*}

Then,

\[
\sin\left(\U^{\flat},\Delta\B\right)\geq\min_{\SUBSTACK{i<n}{a\in\PB i}}\sin\left(\U_{i}(a)^{\flat},\Delta\CB i\right)
\]
\end{prop}
That is, $\U_{i}(a)$ represents the set of admissible conditional
probability distributions for the $i$-th component of the outcome
given the condition $a$ for the previous components. Notice that
we define $\U$ s.t. any vector inside it vanishes outside the ``support''
of the $\U_{i}(a)$: this happens automatically for points in $\Delta\B$
but outside we could have cancellation between negative and positive
components. 

So, we can lower bound the sine of the ``composite'' subspace $\U$
by the minimum of the sines of the subspaces governing each conditional
probability. Now, we'll show a similar bound for $R$.

In the setting of Proposition \ref{prop:chain-s}, suppose that we
are given, for each $i<n$ and $a\in\PB i$, finite dimensional vector
spaces $\Z_{a}$ and $\W_{a}$, a compact set $\H_{a}\subseteq\Z_{a}$
and $F_{a}:\A\times\Z_{a}\times\R^{\CB i}\rightarrow\W_{a}$, a mapping
continuous in the first argument and bilinear in the second and third
arguments. Assume that for any $x\in\A$ and $\theta\in\H_{a}$, $\ker F_{ax\theta}\cap\Delta\CB i\ne\varnothing$
and $F_{ax\theta}$ is onto. In other words, each $a$ has its own
bandit. Let $\Z:=\bigoplus_{i,a}\Z_{a}$, $\W:=\bigoplus_{i,a}\W_{a}$
and define $F$ by

\[
F\left(x,z,y\right)_{a}:=F_{a}\left(x,z_{a},y^{a}\right)
\]

Define also $\H:=\prod_{i,a}\H_{a}$. We have,
\begin{prop}
\label{prop:chain-r}In the setting above

\[
R(\H,F)\leq2n\max_{i,a}R(\H_{a},F_{a})
\]
\end{prop}
See Appendix \ref{sec:prob-sys} for the proofs of the Propositions
in this subsection.

\section{\label{sec:Summary}Summary}

In this work, we proposed a new type of multi-armed bandit setting,
importantly different from both stochastic and aversarial bandits,
by replacing the probability distribution of a stochastic bandit by
an imprecise belief (credal set). We studied a particular form of
this setting, analogical to linear bandits. For this form, we proved
an upper bound on regret of the usual form $\tilde{O}(\sqrt{N})$,
where the coefficient includes the parameters $S$ and $R$ in addition
to the more intuitive $D_{W}$, $D_{Z}$ and $C$. We then studied
$S$ and $R$ and showed they satisfy reasonable upper bounds in many
natural cases. Finally, we demonstrated lower bounds that show that
the parameteres $S$, $R$ and $D_{Z}$ are in some sense necessary.

The author feels that we only scratched the surface of learing theory
with imprecise probability. Here are a few possible directions for
future research:
\begin{itemize}
\item The upper bound we proved is in expectation and non-anytime. Obvious
next steps are extending it to anytime (i.e. the correct choice of
parameters in the algorithm not depending on $N$), and high confidence
(i.e., showing that regret is low with high-probability rather than
just in expectation).
\item We don't know whether the parameter $D_{W}$ appearing in the upper
bound is necessary. We can try to either prove a matching lower bound,
or find an upper bound without $D_{W}$. 
\item For the special case of Example \ref{ex:linear-bandits}, Theorem
\ref{thm:main} yields the regret bounds $\tilde{O}(D^{2}\sqrt{N})$.
However, we know from \cite{DBLP:conf/colt/DaniHK08} that a regret
bound of $\tilde{O}(D\sqrt{N})$ is possible. It would be interesting
to close this gap in a natural way. More generally, there is still
significant room for tightening the gap between the lower bounds and
the upper bound, in terms of the exponents of the parameters.
\item It might be interesting to demonstrate lower bounds for the positive
gap setting. 
\item Our analysis ignored computational complexity. It seems important
to understand when IUCB or some other low regret algorithm can be
implemented efficiently.
\item The setting we study requires the hypothesis space to be naturally
embedded in the vector space $\Z$ and scales with its dimension $D_{Z}$.
It would be interesting to generalize this to regret bounds that scale
with a \emph{non-linear} dimension parameter, similarly to e.g. eluder
dimension (see \cite{Russo2013}) for stochastic bandits.
\item Another possible direction is extending this setting from bandits
to reinforcement learning. This seems very natural, since we can imagine
an ``MDP'' transition kernel that takes values in credal sets instead
of probability distributions. Indeed, the results of \cite{pmlr-v139-tian21b}
can already be interpeted in this way: but, they scale with the cardinalities
of states and actions, whereas it would be better to have bounds in
terms of dimension parameters.
\end{itemize}

\subsection*{Acknowledgements}

This work was supported by the Machine Intelligence Research Institute
in Berkeley, California, by Effective Ventures Foundation USA and
by the Hebrew University of Jerusalem scholarship for outstanding
students.

The author wishes to thank her advisor, Shai Shalev-Shwartz, for providing
important advice during this work, and in particular emphasizing the
need for examples. The author thanks Eran Malach for reading drafts
and making useful suggestions for the imporvement of Sections \ref{sec:Motivation}
and \ref{sec:Setting}. The author thanks Felix Harder for solving
a related technical question (see footnote \ref{fn:felix}). Finally,
the author thanks Alexander Appel and her spouse Marcus Ogren for
reading drafts, finding errors and making useful suggestions.

\bibliographystyle{plain}
\bibliography{Bandit_Algorithms,Introduction_to_Imprecise_Probabilities,Subspace_Identification_for_Linear_Syste,Giannopoulos,Traces_and_Emergence_of_Nonlinear_Progra,DaniHK08,Robbins,Cover_and_Thomas,Kershaw83,Convex_Polytopes,Ok,pmlr-v139-tian21b,ChenZG22a,ODonoghueLO21,Feige20,NIPS-2013-eluder-dimension-and-the-sample-complexity-of-optimistic-exploration-Bibtex,BubeckC12,AwerbuchK04,AudibertB09,Game_Theory_Alive,An_Introduction_to_Decision_Theory,CambridgeCore_Citation_13Dec2023,Lower_Previsions,Gilboa_89}

\appendix

\section{\label{sec:Proof-of-Main}Main Upper Bound}

The outline of the proof is as follows:

\global\long\def\Regc#1{\mathrm{\mathrm{Reg}}_{#1}}%

\global\long\def\CC{\bar{\C}}%

\global\long\def\DZ#1#2{d_{Z}\!\left(#1,#2\right)}%

\begin{itemize}
\item For any $x\in\A$, $y\in\D$ and $\theta\in\H$, $d_{Z}\left(\theta,\V xy\right)\leq R\cdot d_{Y}(y,K_{\theta}(x)^{\flat})$,
where $d_{Z}$ stands for minimal distance w.r.t. the norm on $\XZ$.
\item Let $\theta^{*}\in\H$ be the true hypothesis. For every cycle $k$
and $\delta>0$, the probability that $d_{Y}(\bar{y}_{k},K_{\theta^{*}}(x_{\theta_{k}}^{*})^{\flat})\geq\delta$
is upper bounded by $D_{W}\exp(-\Omega(1)\cdot D_{W}^{-\frac{5}{3}}\tau_{k}\delta^{2})$.
This is proved with the help of the theorem from \cite{10.1007/978-3-0348-9090-8_7}
and Azuma's inequality.
\item As a consequence, $\theta^{*}$ always remains inside the confidence
set $\C_{k}$ with high probability.
\item The regret $\Regc k$ accrued on cycle $k$ is upper bounded by 
\[
\frac{1}{2}\tau_{k}\DY{\bar{y}_{k}}{K_{\theta_{k}}(x_{\theta_{k}}^{*})^{+}}
\]
Intuitively, if $\bar{y}_{k}$ is near $K_{\theta_{k}}(x_{\theta_{k}}^{*})^{+}$
then this is consistent with $\theta_{k}$ being close to $\theta^{*}$,
in which case we would suffer little regret.
\item For any $y\in\D$, $\theta\in\H$ and $x\in\A$, $d_{Y}(y,K_{\theta}(x)^{\flat})\leq O(1)\cdot d_{Z}\left(\theta,\V xy\right)$
and hence $d(y,K_{\theta}(x)^{+})\leq O(1)\cdot S^{-1}d\left(\theta,\V xy\right)$.
In particular, for $\theta\in\C_{k}$ we get the upper bound
\[
O(1)\cdot S^{-1}\max_{\theta'\in\C_{k}}\DZ{\theta'}{\V xy}
\]
\item Let $\CC_{k}$ be the convex hull of $\C_{k}.$ Equation (\ref{eq:stop})
guarantees that each cycle reduces the width of $\CC_{k}$ along some
dimension by a factor of $2(D_{Z}+1)$. This implies that the volume
of $\CC_{k}$ is reduced by a factor of $\Omega(1)$.
\item A $D$-dimensional convex body of volume at most $(D^{-1}\epsilon)^{D}\cdot v_{D}$
must have width at most $2\epsilon$ along some dimension, where $v_{D}$
is the volume of the unit ball. This is a consequence of John's ellipsoid
theorem.
\item It follows that after enough cycles, $\CC_{k}$ becomes ``thin''
along some dimension. This allows us to effectively go down to a proper
subspace of $\Z$. The same reasoning applies in the lower dimension,
leading to another reduction, et cetera, which can repeat no more
than $D_{Z}$ times. As a result, we can bound the number of cycles
with substantial $\Regc k$. In these cycles, $\sum_{k}\Regc k$ can
be bounded using the inequalities above.
\end{itemize}
Now, let's dive into the details.

First, we want to bound distances in $\Z$ in terms of distances of
compatible outcomes in $\Y$.
\begin{lem}
\label{lem:dz-leq-dy}For any $x\in\A$, $y\in\D$ and $\theta\in\H$,

\[
\DZ{\theta}{\V xy}\leq R\cdot\DY y{K_{\theta}(x)^{\flat}}
\]
\end{lem}
\begin{proof}
Denote $w:=F(x,\theta,y)$. Then,

\begin{align*}
\F\left(x,\theta-w,y\right) & =F(x,\theta,y)-\mu(y)w\\
 & =w-w\\
 & =0
\end{align*}

Therefore, $\theta-w\in\V xy$ and $d_{Z}\left(\theta,\V xy\right)\leq\Norm w$.

Let $y'\in K_{\theta}(x)^{\flat}$ be s.t. $\Norm{y-y'}=d_{Y}(y,K_{\theta}(x)^{\flat})$.
Then, $F_{x\theta}y'=0$ and hence $w=F_{x\theta}y=F_{x\theta}(y-y')$.
We get, 

\begin{align*}
\DZ{\theta}{\V xy} & \leq\Norm w\\
 & =\Norm{F_{x\theta}\left(y-y'\right)}\\
 & \leq\Norm{F_{x\theta}}\cdot\Norm{y-y'}\\
 & \leq\Norm{\theta}\cdot\Norm{y-y'}\\
 & \leq R\cdot\DY y{K_{\theta}(x)^{\flat}}
\end{align*}
\end{proof}
The following is a simple utility lemma:
\begin{lem}
\label{lem:normal}For any $\U\in\G d$ and $y\in\D$,

\[
d_{Y}(y,\U^{\flat})\leq4d_{Y}(y,\U)
\]
\end{lem}
\begin{proof}
Denote $a:=d_{Y}(y,\U)$. Let $y^{*}\in\U$ be s.t. $\Norm{y-y^{*}}=a$.
Notice that $\Norm y\leq1$ and hence

\begin{align*}
\Norm{y^{*}} & \leq\Norm y+a\\
 & \leq1+a
\end{align*}

Moreover, since $\Norm{\mu}=1$ and $\mu(y)=1$

\begin{align*}
\Abs{\mu\left(y^{*}\right)-1} & =\Abs{\mu(y^{*})-\mu(y)}\\
 & \leq\Norm{\mu}\cdot\Norm{y^{*}-y}\\
 & =a
\end{align*}

We consider 2 cases.

\LyXZeroWidthSpace{}

\textbf{Case 1:} $\mu(y^{*})\geq\frac{1}{2}$. Then,

\[
\Abs{1-\frac{1}{\mu\left(y^{*}\right)}}\leq1
\]

Also,

\begin{align*}
\Abs{1-\frac{1}{\mu\left(y^{*}\right)}} & =\frac{\Abs{\mu\left(y^{*}\right)-1}}{\Abs{\mu\left(y^{*}\right)}}\\
 & \leq\frac{a}{\left(\frac{1}{2}\right)}\\
 & =2a
\end{align*}

Define $y^{\flat}\in\U^{\flat}$ by

\[
y^{\flat}:=\frac{y^{*}}{\mu\left(y^{*}\right)}
\]

We get

\begin{align*}
\DY y{\U^{\flat}} & \leq\Norm{y-y^{\flat}}\\
 & \leq\Norm{y-y^{*}}+\Norm{y^{*}-y^{\flat}}\\
 & =a+\Abs{1-\frac{1}{\mu\left(y^{*}\right)}}\cdot\Norm{y^{*}}\\
 & \leq a+\min\left(1,2a\right)\cdot(1+a)\\
 & \leq a+3a\\
 & =4a
\end{align*}

\textbf{Case 2:} $\mu(y^{*})<\text{\ensuremath{\frac{1}{2}}}$. Then,
$a>\frac{1}{2}$ and

\begin{align*}
\DY y{\U^{\flat}} & \leq\Norm y+\DY{\boldsymbol{0}}{\U^{\flat}}\\
 & \leq1+1\\
 & <4a
\end{align*}
\end{proof}
The next lemma uses a theorem from \cite{10.1007/978-3-0348-9090-8_7}:
\begin{thm}
[Giannopoulos]\emph{\label{thm:giannopoulos}}Let $\X$ be a finite
dimensional normed space. Denote $D:=\dim\X$. Then, the multiplicative
Banach-Mazur distance from $\X$ to $\R^{D}$ equipped with the $\ell_{1}$
norm\footnote{In \cite{10.1007/978-3-0348-9090-8_7}, the theorem is stated for
$\ell_{\infty}$, but it is the same by duality.} is at most $O(1)\cdot D^{\frac{5}{6}}$. That is, there exists a
linear isomorphism $G:\R^{D}\xrightarrow{\sim}\X$ s.t. for all $x\in\R^{D}$

\[
\Norm x_{1}\leq\Norm{Gx}\leq O(1)\cdot D^{\frac{5}{6}}\Norm x_{1}
\]
\end{thm}
We now derive a concentration bound that shows the mean outcome over
a sequence of rounds with constant arm $x$ is close to $K_{\theta^{*}}(x)^{\flat}$
with high probability.
\begin{lem}
\label{lem:concentration}Suppose that arm $x\in\A$ is selected $\tau\in\N$
times in a row. Denote the average outcome by $\bar{y}\in\D$ and
let $\delta>0$. Then,

\begin{equation}
\Pr\left[\DY{\bar{y}}{K_{\theta^{*}}(x)^{\flat}}\geq\delta\right]\leq2D_{W}\exp\left(-\Omega(1)\cdot\frac{\tau\delta^{2}}{D_{W}^{\frac{5}{3}}}\right)\label{eq:concentration}
\end{equation}
\end{lem}
\begin{proof}
Let $\U:=K_{\theta^{*}}(x)$, $\X:=\Du{(\Y/\U)}$. $\X$ is a subspace
of $\Du{\Y}$, and therefore has a natural norm. Observe that $\dim\X=D_{W}$,
and apply Theorem \ref{thm:giannopoulos} to get $G:\R^{D_{W}}\rightarrow\X$.
Let $T$ be the time our sequence of rounds begins, so that for each
$n<\tau$, $y_{T+n}$ is the outcome of the $n$-th round in the sequence.
For each $i<D_{W}$, let $e_{i}\in\R^{D_{W}}$ be the $i$-th canonical
basis vector and $\alpha_{i}:=Ge_{i}$. Then, $\alpha_{i}(y_{T+n})$
is a random variable whose range is contained in an interval of length

\[
\Abs{\max_{y\in\D}\alpha_{i}(y)-\min_{y\in\D}\alpha_{i}(y)}\leq2\Norm{\alpha_{i}}\leq O(1)\cdot D_{W}^{\frac{5}{6}}
\]

By the Azuma-Hoeffding inequality,

\[
\Pr\left[\Abs{\alpha_{i}(\bar{y})}\geq\delta\right]\leq2\exp\left(-\Omega(1)\cdot\frac{\tau\delta^{2}}{D_{W}^{\frac{5}{3}}}\right)
\]

Applying a union bound, we get

\[
\Pr\left[\max_{i}\Abs{\alpha_{i}(\bar{y})}\geq\delta\right]\leq2D_{W}\exp\left(-\Omega(1)\cdot\frac{\tau\delta^{2}}{D_{W}^{\frac{5}{3}}}\right)
\]

Observe that

\begin{align*}
\DY{\bar{y}}{\U} & =\max_{\beta\in\X:\Norm{\beta}\leq1}\beta(\bar{y})\\
 & =\max_{v\in\R^{D_{W}}:\Norm{Gv}\leq1}Gv(\bar{y})\\
 & \leq\max_{v\in\R^{D_{W}}:\Norm v_{1}\leq1}Gv(\bar{y})\\
 & =\max_{v\in\R^{D_{W}}:\Norm v_{1}\leq1}\sum_{i<D_{W}}v_{i}\alpha_{i}(\bar{y})\\
 & =\max_{i<D_{W}}\Abs{\alpha_{i}(\bar{y})}
\end{align*}

Hence,

\[
\Pr\left[\DY{\bar{y}}{\U}\geq\delta\right]\leq2D_{W}\exp\left(-\Omega(1)\cdot\frac{\tau\delta^{2}}{D_{W}^{\frac{5}{3}}}\right)
\]

Since $\bar{y}\in\D$, Lemma \ref{lem:normal} implies that $d_{Y}(\bar{y},\U^{\flat})\leq4d_{Y}(\bar{y},\U)$.
Redefining $\delta$, we get (\ref{eq:concentration}).
\end{proof}
From now on, we'll use the notation $\Pr[A;B]$ to mean ``probability
of the event $A$ and $B$'' and $\E{}{f;B}$ to mean ``the expected
value of $f\cdot\boldsymbol{1}_{B}$, where $\boldsymbol{1}_{B}$
is defined to be $1$ when event $B$ happens and $0$ when event
$B$ doesn't happen''.

The next lemma shows that, given a concentration bound such as in
Lemma \ref{lem:concentration}, the true hypothesis remains in the
confidence set with high probability. We state it for an arbitrary
bound that scales with $\tau\delta^{2}$ because it will be used again
in Appendix \ref{sec:simplex}.
\begin{lem}
\label{lem:generic-confidence}Assume that a given bandit is s.t.,
in the setting of Lemma \ref{lem:concentration}, it always holds
that

\[
\Pr\left[\DY{\bar{y}}{K_{\theta^{*}}(x)^{\flat}}\geq\delta\right]\leq f\left(\tau\delta^{2}\right)
\]

Here, $f:\R\rightarrow\R$ is some continuous function. 

Assume agent policy $\IUCB^{\eta}$ and \emph{any} nature policy.
Then, for any $N\in\N$

\begin{align*}
\Pr\left[\exists k:T_{k}<N\text{ and }\theta^{*}\not\in\C_{k}\right] & \leq\frac{1}{2}N(N+1)f\left(\frac{\eta^{2}}{R^{2}}\right)
\end{align*}
\end{lem}
\begin{proof}
For any $i,j\in\N$ s.t. $i<j$, denote $\bar{y}_{ij}:=\frac{1}{j-i}\sum_{n=i}^{j-1}y_{n}$.
Denote $A_{ij}$ the event $\forall i\leq n<j:x_{n}=x_{i}$. By the
assumption,

\global\long\def\CP#1#2{\Pr\left[#1\middle|#2\right]}%

\[
\Pr\left[\DY{\bar{y}_{ij}}{K_{\theta^{*}}\left(x_{i}\right)^{\flat}}\geq\delta;A_{ij}\right]\leq f\left((j-i)\delta^{2}\right)
\]

Using Lemma \ref{lem:dz-leq-dy}, we get

\[
\Pr\left[\frac{1}{R}\cdot\DZ{\theta^{*}}{\V{x_{i}}{\bar{y}_{ij}}}\geq\delta;A_{ij}\right]\leq f\left((j-i)\delta^{2}\right)
\]

Substituting $\delta:=\frac{1}{R}\cdot\frac{\eta}{\sqrt{j-i}}$, this
becomes

\[
\Pr\left[\DZ{\theta^{*}}{\V{x_{i}}{\bar{y}_{ij}}}\geq\frac{\eta}{\sqrt{j-i}};A_{ij}\right]\leq f\left(\frac{\eta^{2}}{R^{2}}\right)
\]

Denote $B_{ij}$ the event inside the probability operator on the
left hand side. Taking a union bound, we get

\begin{align*}
\Pr\left[\exists0\leq i<j\leq N:B_{ij}\right] & \leq\frac{N(N+1)}{2}\cdot f\left(\frac{\eta^{2}}{R^{2}}\right)
\end{align*}

From the definition of $\C_{k}$, it is clear that if $\forall0\leq i<j\leq N:\text{not-}B_{ij}$,
then $\forall k:\text{if }T_{k}<N\text{ then }\theta^{*}\in\C_{k}$.
\end{proof}
Putting Lemma \ref{lem:concentration} and Lemma \ref{lem:generic-confidence}
together immediately gives us: 
\begin{lem}
\label{lem:confidence}Assume agent policy $\IUCB^{\eta}$ and \emph{any}
nature policy. Then, for any $N\in\N$

\begin{align*}
\Pr\left[\exists k:T_{k}<N\text{ and }\theta^{*}\not\in\C_{k}\right] & \leq D_{W}N(N+1)\exp\left(-\Omega(1)\cdot\frac{\eta^{2}}{R^{2}D_{W}^{\frac{5}{3}}}\right)
\end{align*}
\end{lem}
Now we show that the regret of a sequence of rounds with an optimistically
selected arm can be bounded in terms of the minimal distance of the
mean outcome from the possible expected outcomes of the optimistic
hypothesis. We denote $r^{*}:=\ME{\theta^{*}}r{x_{\theta^{*}}^{*}}$.
\begin{lem}
\label{lem:regret}Suppose $\theta\in\H$ is s.t. 

\begin{equation}
\ME{\theta}r{x_{\theta}^{*}}\geq r^{*}\label{eq:optimism}
\end{equation}

Assume that the action $x_{\theta}^{*}\in\A$ is selected $\tau\in\N$
times in a row, starting from round $T\in\N$. Denote the average
outcome by $\bar{y}$. Then,

\global\long\def\CE#1#2{\mathrm{E}\left[#1\middle|#2\right]}%

\begin{equation}
\tau r^{*}-\sum_{n<\tau}r\left(x_{\theta}^{*},y_{T+n}\right)\leq\tau\DY{\bar{y}}{K_{\theta}\left(x_{\theta}^{*}\right)^{+}}\label{eq:reg-leq-dist}
\end{equation}
\end{lem}
\begin{proof}
By convexity of $r$,

\[
\frac{1}{\tau}\sum_{n<\tau}r\left(x_{\theta}^{*},y_{T+n}\right)\geq r\left(x_{\theta}^{*},\bar{y}\right)
\]

Since $r$ is 1-Lipschitz, this implies

\begin{align*}
\frac{1}{\tau}\sum_{n<\tau}r\left(x_{\theta}^{*},y_{T+n}\right) & \geq\min_{y\in K_{\theta}\left(x_{\theta}^{*}\right)^{+}}r\left(x_{\theta}^{*},y\right)-\DY{\bar{y}}{K_{\theta}\left(x_{\theta}^{*}\right)^{+}}\\
 & =\ME{\theta}r{x_{\theta}^{*}}-\DY{\bar{y}}{K_{\theta}\left(x_{\theta}^{*}\right)^{+}}
\end{align*}

Using (\ref{eq:optimism}) and rearranging, we get (\ref{eq:reg-leq-dist}).
\end{proof}
The following is a simple geometric observation through which ``sines''
make it into the regret bound.
\begin{lem}
\label{lem:sin}Let $\Af$ be a finite-dimensional affine space. Assume
that the associated vector space $\Vec$ is equipped with a norm.
Let $\D\subseteq\Af$ be a closed convex set and $\SubAf\subseteq\Af$
an affine subspace. Assume $\sin\left(\SubAf,\D\right)>0$ and consider
any $y\in\D$. Then,

\begin{equation}
d\left(y,\SubAf\cap\D\right)\leq\left(\frac{1}{\sin\left(\SubAf,\D\right)}+1\right)d\left(y,\SubAf\right)\label{eq:lem-sin}
\end{equation}
\end{lem}
\begin{proof}
Let $y'\in\SubAf$ be s.t. $\Norm{y-y'}=d\left(y,\SubAf\right)$.
If $y'\in\D$ then $d(y,\SubAf\cap\D)=d(y,\SubAf)$ and (\ref{eq:lem-sin})
is true. Now, suppose $y'\not\in\D$. Then, by definition of $\sin\left(\SubAf,\D\right)$,

\[
\frac{d\left(y',\D\right)}{d\left(y',\SubAf\cap\D\right)}\geq\sin\left(\SubAf,\D\right)
\]

We get

\begin{align*}
d\left(y,\SubAf\cap\D\right) & \leq\Norm{y-y'}+d\left(y',\SubAf\cap\D\right)\\
 & \leq\Norm{y-y'}+\frac{d\left(y',\D\right)}{\sin\left(\SubAf,\D\right)}\\
 & \leq\Norm{y-y'}+\frac{\Norm{y'-y}}{\sin\left(\SubAf,\D\right)}\\
 & =\left(1+\frac{1}{\sin\left(\SubAf,\D\right)}\right)d\left(y,\SubAf\right)
\end{align*}
\end{proof}
Now we derive a sort of converse to Lemma \ref{lem:dz-leq-dy}.
\begin{lem}
\label{lem:dy-leq-dz}For any $y\in\D$, $\theta\in\H$ and $x\in\A$

\begin{equation}
\DY y{K_{\theta}(x)^{\flat}}\leq4d_{Z}\left(\theta,\V xy\right)\label{eq:dy-leq-dz}
\end{equation}
\end{lem}
\begin{proof}
Let $z\in\V xy$ be s.t. $\Norm{\theta-z}=\DZ{\theta}{\V xy}$. Denote
$w:=F(x,\theta,y)$. Since $\F(x,z,y)=0$, we have $w=\F(x,\theta-z,y)$.
Therefore

\begin{align*}
\Norm w & =\Norm{\F_{x,\theta-z}y}\\
 & \leq\Norm{\F_{x,\theta-z}}\cdot\Norm y\\
 & \leq\Norm{\theta-z}\cdot1\\
 & =\DZ{\theta}{\V xy}
\end{align*}

By definition of the norm on $\W$, there exists $\delta y\in\Y$
s.t. $F_{x\theta}\delta y=w$ and $\Norm{\delta y}\leq\Norm w$. Denote
$y':=y-\delta y$. We have

\begin{align*}
F_{x\theta}y' & =F_{x\theta}y-F_{x\theta}\delta y\\
 & =w-w\\
 & =0
\end{align*}

Hence, $y'\in K_{\theta}(x)$ and therefore

\begin{align*}
\DY y{K_{\theta}(x)} & \leq\Norm{y-y'}\\
 & \leq\Norm w\\
 & \leq\DZ{\theta}{\V xy}
\end{align*}

Applying Lemma \ref{lem:normal}, we get (\ref{eq:dy-leq-dz}). 
\end{proof}
The next lemma shows that when a $D$-dimensional convex set is ``sliced''
in a manner which reduces its width along some dimension by a factor
of $\Omega(D)$, its volume must become reduced by a factor of $\Omega(1)$.
This is the key to exploiting the finite dimension of the hypothesis
space in the regret bound. The idea of the proof is reducing the problem
to the special case of a cone sliced parallel to its base, in which
case the conclusion follows from symmetry.
\begin{lem}
\label{lem:vol}Let $D\geq1$, $\X$ a $D$-dimensional vector space,
$\alpha\in\Du{\X}$ and $\C\subseteq\X$ a compact convex set s.t.

\[
\max_{z\in\C}\Abs{\alpha(z)}=1
\]

Define $\C'\subseteq\X$ by

\[
\C':=\SC{z\in\C}{\Abs{\alpha(z)}\leq\frac{1}{D+1}}
\]

Then,

\begin{equation}
vol(\C')\leq\left(1-\frac{1}{e^{2}}\right)\cdot vol(\C)\label{eq:vol}
\end{equation}
\end{lem}
Here, $vol$ stands for volume. Notice that volume is only defined
up to a scalar, but since it appears on both sides of the inequality,
it doesn't matter.
\begin{proof}
[Proof of Lemma \ref{lem:vol}]Assume w.l.o.g. that there is $v\in\C$
s.t. $\alpha(v)=+1$, and fix some $v$ like that. Define $\B_{+},\B_{-}\subseteq\X$
by

\begin{align*}
\B_{+} & :=\SC{z\in\C}{\alpha(z)=+\frac{1}{D+1}}\\
\B_{-} & :=\SC{\text{\ensuremath{\left(1+\frac{2}{D}\right)}}z-\frac{2}{D}v}{z\in\B_{+}}
\end{align*}

For any $z\in\B_{-}$, we have

\begin{align*}
\alpha(z) & =\alpha\left(\text{\ensuremath{\left(1+\frac{2}{D}\right)}}z'-\frac{2}{D}v\right)\\
 & =\text{\ensuremath{\left(1+\frac{2}{D}\right)}}\alpha\left(z'\right)-\frac{2}{D}\alpha(v)\\
 & =\text{\ensuremath{\left(1+\frac{2}{D}\right)}}\frac{1}{D+1}-\frac{2}{D}\\
 & =\frac{1}{D+1}\left(1+\frac{2}{D}-\frac{2\left(D+1\right)}{D}\right)\\
 & =\frac{1}{D+1}\cdot\frac{D+2-2\left(D+1\right)}{D}\\
 & =\frac{1}{D+1}\cdot\frac{-D}{D}\\
 & =-\frac{1}{D+1}
\end{align*}

Here, $z'$ is some point in $\B_{+}$.

Let $\C_{\pm}$ be the convex hulls of $\B_{\pm}$ and $v$. It is
easy to see that $\C_{+}\setminus\B_{+}\subseteq\C\setminus\C'$ and
$\C'\setminus\B_{+}\subseteq\C_{-}\setminus\C_{+}$. Moreover, $\C_{-}$
can be obtained from $\C_{+}$ by a homothety with center $v$ and
ratio $1+\frac{2}{D}$, and hence

\begin{align*}
vol\left(\C_{-}\right) & =\left(1+\frac{2}{D}\right)^{D}vol\left(\C_{+}\right)\\
 & \leq e^{2}\cdot vol\left(\C_{+}\right)
\end{align*}

It follows that

\begin{align*}
vol(\C') & \leq vol(\C_{-})-vol(\C_{+})\\
 & \leq e^{2}vol(\C_{+})-vol(\C_{+})\\
 & =\left(e^{2}-1\right)vol(\C_{+})\\
 & \leq\left(e^{2}-1\right)\left(vol(\C)-vol(\C')\right)
\end{align*}

Rearranging, we get (\ref{eq:vol}).
\end{proof}
The following allows us to lower bound the volume of a convex body
in terms of its \emph{minimal width}. Combined with Lemma \ref{lem:vol},
it will allow us to ensure the confidence set effectively becomes
lower dimensional over time.
\begin{lem}
\label{lem:vol-and-wid}Let $D\geq1$ and $\C\subseteq\R^{D}$ a compact
convex set. Define $w\in[0,\infty)$ by

\[
w:=\min_{\alpha\in\R^{D}:\Norm{\alpha}_{2}=1}\left(\max_{z\in\C}\alpha^{t}z-\min_{z\in\C}\alpha^{t}z\right)
\]

Then,

\[
vol(\C)\geq\left(\frac{w}{2D}\right)^{D}v_{D}
\]
\end{lem}
Here, $v_{D}$ is the volume of the unit ball in $\R^{D}$.

\global\long\def\El{\mathcal{E}}%

\begin{proof}
[Proof of Lemma \ref{lem:vol-and-wid}]By John's ellipsoid theorem\footnote{See \cite{giorgi2013traces}, p. 213, Theorem III},
there is an ellipsoid\footnote{By an "ellipsoid" we mean the convex body, rather than the surface. That is, an ellipsoid is defined to be the image of a closed ball under an affine transformation.}
$\El$ in $\R^{D}$ with the following property. Denote by $\El'$
the ellipsoid obtained from $\El$ by a homothety with center at the
center of $\El$ and ratio $\frac{1}{D}$. Then, $\El'\subseteq\C\subseteq\El$.
Each axis of $\El$ must be $\geq w$, hence each axis of $\El'$
must be $\geq\frac{w}{D}$, and therefore

\begin{align*}
vol(\C) & \geq vol(\El')\\
 & \geq\left(\frac{w}{2D}\right)^{D}v_{D}
\end{align*}
\end{proof}
Lemmas \ref{lem:vol} and \ref{lem:vol-and-wid} will be used to show
that the confidence set effectively becomes lower dimensional, in
the sense that we can bound its thickness along multiple dimensions.
However, the ``slicing'' performed by IUCB still happens in the
original $D_{Z}$-dimensional space rather than the effective subspace.
In order to translate the slicing factor from the ambient space to
the subspace, we will use the following.
\begin{lem}
\label{lem:thin}Let $D\geq1$, $E\geq0$, $\delta\in(0,\infty)$,
$\C\subseteq\R^{D}\times[-\delta,+\delta]^{E}$ a compact set and
$\alpha\in\R^{D+E}$ s.t. $\Norm{\alpha}_{2}\leq1$. Define $\rho_{0}\in(0,\infty)$
by

\[
\rho_{0}:=\max_{z\in\C}\Abs{\T{\alpha}z}
\]

Define $\C'\subseteq\R^{D+E}$ by

\[
\C':=\SC{z\in\C}{\Abs{\T{\alpha}z}\leq\frac{1}{2(D+1)}\cdot\rho_{0}}
\]

Let $P:\R^{D+E}\rightarrow\R^{D}$ be the orhogonal projection, and
define $\rho_{1},\rho_{1}'\in(0,\infty)$ by

\begin{align*}
\rho_{1} & :=\max_{z\in\C}\Abs{\T{\alpha}Pz}\\
\rho_{1}' & :=\max_{z\in\C'}\Abs{\T{\alpha}Pz}
\end{align*}

Assume $\rho_{0}\geq(4D+5)\sqrt{E}\cdot\delta$. Then,

\[
\rho_{1}'\leq\frac{1}{D+1}\cdot\rho_{1}
\]
\end{lem}
\begin{proof}
For any $z\in\R^{D}\times[-\delta,+\delta]^{E}$, we have

\[
\T{\alpha}z-\sqrt{E}\cdot\delta\leq\T{\alpha}Pz\leq\T{\alpha}z+\sqrt{E}\cdot\delta
\]

It follows that

\begin{align*}
\rho_{1} & \geq\rho_{0}-\sqrt{E}\cdot\delta\\
 & \geq(4D+5)\sqrt{E}\cdot\delta-\sqrt{E}\cdot\delta\\
 & =4(D+1)\sqrt{E}\cdot\delta
\end{align*}

Moreover,

\begin{align*}
\rho_{1}' & \leq\frac{1}{2(D+1)}\cdot\rho_{0}+\sqrt{E}\cdot\delta\\
 & \leq\frac{1}{2(D+1)}\cdot\left(\rho_{1}+\sqrt{E}\cdot\delta\right)+\sqrt{E}\cdot\delta\\
 & =\frac{1}{2(D+1)}\cdot\rho_{1}+\left(1+\frac{1}{2(D+1)}\right)\sqrt{E}\cdot\delta\\
 & \leq\frac{1}{2(D+1)}\cdot\rho_{1}+\left(1+\frac{1}{2(D+1)}\right)\frac{1}{4(D+1)}\cdot\rho_{1}\\
 & \leq\frac{1}{2(D+1)}\cdot\rho_{1}+2\cdot\frac{1}{4(D+1)}\cdot\rho_{1}\\
 & =\frac{1}{D+1}\cdot\rho_{1}
\end{align*}
\end{proof}
We also need to bound the regret once the confidence set effectively
becomes 0-dimensional, i.e. is contained within a small hypercube.
For this we use the following.
\begin{lem}
\label{lem:regret-by-distance}Suppose $\theta,\theta^{*}\in\H$ are
s.t.

\begin{equation}
\ME{\theta}r{x_{\theta}^{*}}\geq r^{*}\label{eq:optimism-1}
\end{equation}
Let $y$ be a $\D$-valued random variable whose distribution lies
in $\kappa(K_{\theta^{*}}(x_{\theta}^{*}))$. Then,

\begin{equation}
r^{*}-\E{}{r\left(x_{\theta}^{*},y\right)}\leq4\left(S^{-1}+1\right)\Norm{\theta^{*}-\theta}\label{eq:regret-by-distance}
\end{equation}
\end{lem}
\begin{proof}
Denote $y^{*}:=\E{}y$. By (\ref{eq:optimism-1}) and convexity of
$r$,

\[
r^{*}-\E{}{r\left(x_{\theta}^{*},y\right)}\leq\ME{\theta}r{x_{\theta}^{*}}-r\left(x_{\theta}^{*},y^{*}\right)
\]

Since $r$ is $1$-Lipschitz, this implies

\[
r^{*}-\E{}{r\left(x_{\theta}^{*},y\right)}\leq\DY{y^{*}}{K_{\theta}\left(x_{\theta}^{*}\right)^{+}}
\]

Using Lemma \ref{lem:sin}, we get

\[
r^{*}-\E{}{r\left(x_{\theta}^{*},y\right)}\leq\left(S^{-1}+1\right)\DY{y^{*}}{K_{\theta}\left(x_{\theta}^{*}\right)^{\flat}}
\]

Applying Lemma \ref{lem:dy-leq-dz},

\[
r^{*}-\E{}{r\left(x_{\theta}^{*},y\right)}\leq4\left(S^{-1}+1\right)\DZ{\theta}{\V{x_{\theta}^{*}}{y^{*}}}
\]

It remains to observe that $y^{*}\in K_{\theta^{*}}(x_{\theta}^{*})$
and therefore $F(x_{\theta}^{*},\theta^{*},y^{*})=0$ and $\theta^{*}\in\V{x_{\theta}^{*}}{y^{*}}$.
\end{proof}
We can now start the regret analysis. We consider the probability
space resulting from agent policy $\IUCB^{\eta}$ and some fixed nature
policy. First, we establish that regret can be expressed a sum over
cycle contributions, and as long as $\theta^{*}$ is in the confidence
set, cycle contributions can be bounded using a certain geometric
construction in $\Z$.

\global\long\def\RgS#1{\mathrm{Rg}\left[#1\right]}%

\global\long\def\Yk#1#2{\bar{y}_{#1}^{#2}}%

\global\long\def\Tau#1#2{\tau_{#1}^{#2}}%

\global\long\def\EA#1{T_{#1+1}\leq N}%

We will need some notation. For any $n\in\N$, we denote $[n]:=\SC{m\in\N}{m<n}$.
For any $N\in\N$ and $X\subseteq[N]$, define the $\R$-valued random
variable $\RgS X$ by

\[
\RgS X:=\sum_{k\in X}\sum_{n=T_{k}}^{\min\left(T_{k+1},N\right)-1}\left(r^{*}-r\left(x_{n},y_{n}\right)\right)
\]

That is, $\RgS X$ is the contribution of the cycles in $X$ to regret.
For any $a\in\{0,1\}$ and $k\in\N$, define the $\D$-valued random
variable $\Yk ka$, the $\R_{+}$-valued random variable $\rho_{k}^{a}$
and $\N$-valued random variable $\Tau ka$ by

\[
\Tau ka:=\min\left(T_{k+1}-a,N\right)-\min\left(T_{k},N\right)
\]

\[
\Yk ka:=\bar{y}_{T_{k},\min(T_{k+1}-a,N)}
\]

\[
\rho_{k}^{a}:=\max_{z\in\C_{k}}\DZ z{\V{x_{\theta_{k}}^{*}}{\Yk ka}}
\]

That is, $\Tau ka$ is the duration of the intersection of the $k$-th
cycle minus its last $a$ rounds with the time interval of the first
$N$ rounds; $\Yk ka$ is the average of all outcomes in the $k$-th
cycle except for the last $a$, and also excluding outcomes beyond
the time horizon $N$; and $\rho_{k}^{a}$ is the ``radius'' of
the confidence set w.r.t. $\V{x_{\theta_{k}}^{*}}{\Yk ka}$.
\begin{lem}
\label{lem:tau-rho}Let $X$ be an $[N]$-valued random variable s.t.
for all $k\in X$, $\theta^{*}\in\C_{k}$. Then,

\[
\RgS X\leq\sum_{k\in X}\min_{a\in\{0,1\}}\left(4\left(S^{-1}+1\right)\Tau ka\rho_{k}^{a}+C\cdot\boldsymbol{1}_{a=1}\right)
\]
\end{lem}
\begin{proof}
Consider any sequence $\{a_{k}\in\{0,1\}\}_{k<N}$. We have,

\begin{align*}
\RgS X=\sum_{k\in X} & \Bigg(\sum_{n=T_{k}}^{\min\left(T_{k+1}-a,N\right)-1}\left(r^{*}-r\left(x_{n},y_{n}\right)\right)\\
 & +\left(r^{*}-r\left(x_{T_{k+1}-1},y_{T_{k+1}-1}\right)\right)\boldsymbol{1}_{a_{k}=1}\boldsymbol{1}_{\EA k}\Bigg)
\end{align*}

Hence,

\begin{align*}
\RgS X & = & \sum_{k\in X}\min_{a\in\{0,1\}} & \Bigg(\sum_{n=T_{k}}^{\min\left(T_{k+1}-a,N\right)-1}\left(r^{*}-r\left(x_{n},y_{n}\right)\right)\\
 &  &  & +\left(r^{*}-r\left(x_{T_{k+1}-1},y_{T_{k+1}-1}\right)\right)\boldsymbol{1}_{a=1}\boldsymbol{1}_{\EA k}\Bigg)\\
 & \leq & \sum_{k\in X}\min_{a\in\{0,1\}} & \left(\sum_{n=\min\left(T_{k},N\right)}^{\min\left(T_{k+1}-a,N\right)-1}\left(r^{*}-r\left(x_{n},y_{n}\right)\right)+\RR\cdot\boldsymbol{1}_{a=1}\boldsymbol{1}_{\EA k}\right)\\
 & = & \sum_{k\in X}\min_{a\in\{0,1\}} & \left(\Tau kar^{*}-\sum_{m<\Tau ka}r\left(x_{T_{k}+m},y_{T_{k}+m}\right)+\RR\cdot\boldsymbol{1}_{a=1}\boldsymbol{1}_{\EA k}\right)
\end{align*}

For any $k\in X$, $\theta^{*}\in\C_{k}$, and since $\theta_{k}$
are selected optimistically, this implies

\[
\ME{\theta_{k}}r{x_{\theta_{k}}^{*}}\geq r^{*}
\]

Hence, we can use Lemma \ref{lem:regret} to conclude

\begin{align*}
\RgS X\leq\sum_{k\in X}\min_{a\in\{0,1\}}\left(\Tau ka\DY{\Yk ka}{K_{\theta_{k}}\left(x_{\theta_{k}}^{*}\right)^{+}}+\RR\cdot\boldsymbol{1}_{\EA k}\boldsymbol{1}_{a=1}\right)
\end{align*}

By Lemma \ref{lem:sin}, we have

\begin{align*}
\DY{\Yk ka}{K_{\theta_{k}}\left(x_{\theta_{k}}^{*}\right)}^{+} & \leq\left(S^{-1}+1\right)\DY{\Yk ka}{K_{\theta_{k}}\left(x_{\theta_{k}}^{*}\right)}^{\flat}
\end{align*}

We get

\[
\RgS X\leq\sum_{k\in X}\min_{a\in\{0,1\}}\left(\left(S^{-1}+1\right)\Tau ka\DY{\Yk ka}{K_{\theta_{k}}\left(x_{\theta_{k}}^{*}\right)^{\flat}}+\RR\cdot\boldsymbol{1}_{\EA k}\boldsymbol{1}_{a=1}\right)
\]

Applying Lemma \ref{lem:dy-leq-dz}, this becomes

\[
\RgS X\leq\sum_{k\in X}\min_{a\in\{0,1\}}\left(4\left(S^{-1}+1\right)\Tau ka\DZ{\theta_{k}}{\V{x_{\theta_{k}}^{*}}{\Yk ka}}+C\cdot\boldsymbol{1}_{\EA k}\boldsymbol{1}_{a=1}\right)
\]

Since $\theta_{k}\in\C_{k}$, we get

\[
\RgS X\leq\sum_{k\in X}\min_{a\in\{0,1\}}\left(4\left(S^{-1}+1\right)\Tau ka\rho_{k}^{a}+C\cdot\boldsymbol{1}_{a=1}\right)
\]
\end{proof}
Next, we split all cycles $k$ in which $\theta^{*}\in\C_{k}$ into
3 sets, two of which admit certain geometric constraints and the 3rd
has a bounded number of cycles.

\global\long\def\BD{\mathrm{A}}%

\global\long\def\AD{\mathrm{B}}%

\global\long\def\ED{\Gamma}%

\begin{lem}
\label{lem:ABE}Let $\delta>0$ and $N\in\N$. Define the $2^{[N]}$-valued
random variables $\BD,\AD,\ED$ by

\[
\BD:=\SC{k\in[N]}{T_{k}<N\text{ and }\max_{z,z'\in\C_{k}}\Norm{z-z'}\leq2D_{Z}\delta}
\]

\[
\AD:=\SC{k\in[N]\setminus\BD}{T_{k}<N\text{ and }\rho_{k}^{0}\leq9D_{Z}^{2}\delta}
\]

\[
\ED:=\SC{k\in[N]\setminus\left(\BD\cup\AD\right)}{T_{k}<N}
\]

Then,

\[
\Abs{\ED}\leq\gamma D_{Z}^{2}\ln\frac{D_{Z}R}{\delta}
\]
\end{lem}
\begin{proof}
Using John's ellipsoid theorem for balanced convex sets\footnote{See \cite{giorgi2013traces}, p. 214.},
we can find an inner product on $\Z$ s.t. for any $z\in\Z$

\begin{equation}
\Norm z_{2}\leq\Norm z\leq\sqrt{D_{Z}}\cdot\Norm z_{2}\label{eq:inner-product-1}
\end{equation}

Here, $\Norm z_{2}$ stands for the norm of $z$ w.r.t. the inner
product.

Fix some $\delta>0$. We can find some $0\leq D\leq D_{Z}$, a non-decreasing
sequence $\{J_{i}\in\N\}_{i\leq D}$ and an orthonormal set $\{e_{i}\in\Z\}_{i<D}$
s.t.
\begin{itemize}
\item For any $0\leq i\leq D$, $0\leq j<i$ and any $\theta,\theta'\in\C_{J_{i}}$,
$|e_{j}\cdot(\theta-\theta')|\leq2\delta$.
\item For any $0\leq i<D$, either (a) $J_{i}=J_{i+1}$ or (b) for any unit
vector $v\in\Z$, if for all $j<i$, $e_{j}\cdot v=0$, then there
are some $\theta,\theta'\in\C_{J_{i}}$ s.t. $|v\cdot(\theta-\theta')|>2\delta$.
\item For any $k\in\N$ and unit vector $v\in\Z$, if for all $i<D$, $e_{i}\cdot v=0$,
then there are some $\theta,\theta'\in\C_{k}$ s.t. $|v\cdot(\theta-\theta')|>2\delta$.
\end{itemize}
That is, $\{e_{i}\}$ is constructed by adding vectors iteratively.
Once the width of $\C_{k}$ becomes less than $2\delta$ along any
direction orthogonal to the vectors already in the set, we add this
direction to the set.

Consider any $0\leq i\leq D$ and $k\geq J_{i}$ s.t. $T_{k}<N$.
Choose $\theta^{k}\in\C_{k}$ s.t.

\[
\DZ{\theta^{k}}{\V{x_{\theta_{k}}^{*}}{\bar{y}_{k}}}=\rho_{k}^{0}
\]

(Assuming $\C_{k}\ne\varnothing$, we'll get back to this below.)
By the supporting hyperplane theorem, there are $\alpha_{k}\in\XZ^{\star}$
and $c_{k}\in\R$ s.t.
\begin{itemize}
\item $\Norm{\alpha_{k}}=1$
\item For all $z\in\V{x_{\theta_{k}}^{*}}{\bar{y}_{k}}$, $\alpha_{k}(z)=c_{k}$.
\item $\alpha_{k}(\theta^{k})=c_{k}+\rho_{k}^{0}$
\end{itemize}
For any $\theta\in\C_{k+1}$, we have

\begin{align*}
\Abs{\alpha_{k}(\theta)-c_{k}} & \leq\frac{\eta}{\sqrt{\tau_{k}}}\\
 & \leq\frac{\rho_{k}^{0}}{2\left(D_{Z}+1\right)}
\end{align*}

Here, the first inequality follows from the definition of $\C_{k+1}$
and the fact that $\Norm{\alpha_{k}}=1$, and the second inequality
follows from (\ref{eq:stop}).

Let $\beta_{k}:=\alpha_{k}|_{\Z}$. For all $\theta\in\C_{k}$:

\begin{align*}
\Abs{\beta_{k}(\theta)-c_{k}} & =\Abs{\alpha_{k}(\theta)-c_{k}}\\
 & \leq\Norm{\alpha_{k}}\cdot\DZ{\theta}{\V{x_{\theta_{k}}^{*}}{\bar{y}_{k}}}\\
 & \leq\rho_{k}^{0}
\end{align*}

Moreover,

\begin{align*}
\Abs{\beta_{k}\left(\theta^{k}\right)-c_{k}} & =\Abs{\alpha_{k}\left(\theta^{k}\right)-c_{k}}\\
 & =\rho_{k}^{0}
\end{align*}

And, for all $\theta\in\C_{k+1}$:

\begin{align*}
\Abs{\beta_{k}(\theta)-c_{k}} & =\Abs{\alpha_{k}(\theta)-c_{k}}\\
 & \leq\frac{\rho_{k}^{0}}{2\left(D_{Z}+1\right)}
\end{align*}

Observe that

\begin{align*}
\Norm{\beta_{k}}_{2} & \leq\sqrt{D_{Z}}\cdot\Norm{\beta_{k}}\\
 & \leq\sqrt{D_{Z}}\cdot\Norm{\alpha_{k}}\\
 & =\sqrt{D_{Z}}
\end{align*}

Let $P_{i}:\Z\rightarrow\Z$ be the orthogonal projection to the orthogonal
complement of $\{e_{j}\}_{j<i}$ in $\Z$. We can apply Lemma \ref{lem:thin}
to a translation of $\C_{k}$ and $\frac{1}{\sqrt{D_{Z}}}\beta_{k}$
to conclude that \emph{one} of the following two inequalities must
be true (and, we can drop the assumption $\C_{k}\ne\varnothing$,
because the 1st inequality obviously still holds):

\begin{align*}
\rho_{k}^{0} & <\sqrt{D_{Z}}\cdot\left(4\left(D_{Z}-i\right)+5\right)\sqrt{i}\cdot\delta\\
 & \leq9D_{Z}^{2}\delta
\end{align*}

\begin{equation}
\exists c_{k}'\in\R:\max_{\theta\in\C_{k+1}}\Abs{\beta_{k}\left(P_{i}\theta\right)-c_{k}'}\leq\frac{1}{D_{Z}-i+1}\cdot\max_{\theta\in\C_{k}}\Abs{\beta_{k}\left(P_{i}\theta\right)-c_{k}'}\label{eq:beta-reduction-1}
\end{equation}

\global\long\def\I{\mathcal{\mathcal{I}}}%

For any $k\in\N$, denote by $\bar{\C}_{k}$ the convex hull of $\C_{k}$.
For any $0\leq i<\min(D_{Z},D+1)$, denote by $\I_{i}$ the set of
all $J_{i}\leq k<J_{i+1}$ (for $i=D$, just $k\geq J_{i}$) for which
(\ref{eq:beta-reduction-1}) holds. For any $k\in\I_{i}$, we can
apply Lemma \ref{lem:vol} to get

\[
vol_{i}\left(P_{i}\bar{\C}_{k+1}\right)\leq\left(1-\frac{1}{e^{2}}\right)\cdot vol_{i}\left(P_{i}\bar{\C}_{k}\right)
\]

Here, $vol_{i}$ stands for $(D_{Z}-i)$-dimensional volume inside
the orthogonal complement of $\{e_{j}\}_{j<i}$. For every index in
$\I_{i}$, the volume of $P_{i}\bar{\C}$ is reduced by a factor of
$\Omega(1)$. This volume starts out as $R^{D_{Z}-i}v_{D_{Z}-i}$
at most (since $\H$ is contained in a ball of radius $R$) and, by
Lemma \ref{lem:vol-and-wid}, cannot go below $\left(\frac{\delta}{D_{Z}-i}\right)^{D_{Z}-i}v_{D_{Z}-i}$:
once the volume is below that, the minimal width is below $2\delta$
which contradicts the condition $k<J_{\text{i}+1}$ (or the definition
of $D$ in the case $i=D$). We get that,

\begin{align*}
\Abs{\I_{i}} & \leq\frac{1}{\ln\left(1-\frac{1}{e^{2}}\right)^{-1}}\text{\ensuremath{\cdot\ln\frac{R^{D_{Z}-i}v_{D_{Z}-i}}{\left(\frac{\delta}{D_{Z}-i}\right)^{D_{Z}-i}v_{D_{Z}-i}}}}\\
 & =\gamma\left(D_{Z}-i\right)\ln\frac{(D_{Z}-i)R}{\delta}\\
 & \leq\gamma D_{Z}\ln\frac{D_{Z}R}{\delta}
\end{align*}

We now see that every $k$ with $T_{k}<N$ satisfies one of 3 conditions:
\begin{enumerate}
\item $\rho_{k}^{0}\leq9D_{Z}^{2}\delta$
\item $k\in\I_{i}$ for some $0\leq i<\min(D_{Z},D+1)$
\item $k\geq J_{D_{Z}}$ (this can only happen for $D=D_{Z}$)
\end{enumerate}
Consider some $k$ s.t. $T_{k}<N$ and $\theta^{*}\in\C_{k}$. If
condition 3 holds, then $\C_{k}$ is contained in a $D_{Z}$-cube
with side $2\delta$ and hence $k\in\BD$ (notice there's an extra
factor of $\sqrt{D_{Z}}$ in the diameter coming from (\ref{eq:inner-product-1})).
If condition 1 holds, then $k\in\BD\cup\AD$. Therefore, if $k\in\ED$
then condition 2 holds. We get

\begin{align*}
\Abs{\ED}\leq & \sum_{i=0}^{\min(D_{Z},D-1)}\Abs{\I_{i}}\\
\leq & \gamma D_{Z}^{2}\ln\frac{D_{Z}R}{\delta}
\end{align*}

\global\long\def\UR{\mathrm{Rg}}%

\global\long\def\BDF{\mathrm{A}_{0}}%

\global\long\def\ADF{\mathrm{B}_{0}}%

\global\long\def\EDF{\Gamma_{0}}%

\global\long\def\FD{\Delta}%

\global\long\def\FC{\Delta^{c}}%
\end{proof}
Let $\FD:=\SC{k\in[N]}{\theta^{*}\in\C_{k}}$. Denote $\BDF:=\BD\cap\FD$,
$\ADF:=\AD\cap\FD$, $\EDF:=\ED\cap\FD$ and

\[
\FC:=\SC{k\in[N]\setminus\FD}{T_{k}<N}
\]

Let's bound the contribution to regret coming from $\BDF$, $\ADF$
and $\EDF$.
\begin{lem}
\label{lem:regret-abc}
\begin{align*}
\E{}{\RgS{\BDF}+\RgS{\ADF}+\RgS{\EDF}} & \leq & \left(S^{-1}+1\right)D_{Z}\left(36D_{Z}+8\right)N\delta\\
 &  & +8\eta\left(S^{-1}+1\right)(D_{Z}+1)\sqrt{\gamma D_{Z}^{2}\ln\frac{D_{Z}R}{\delta}\cdot N}\\
 &  & +\gamma CD_{Z}^{2}\ln\frac{D_{Z}R}{\delta}
\end{align*}
\end{lem}
\begin{proof}
We will bound the exepectation of each term on the left hand side. 

The expected value in (\ref{eq:regret-by-distance}) of Lemma \ref{lem:regret-by-distance}
can be interpreted as the expected value of a the contribution of
a particular round to $\RgS{\BDF}$ \emph{conditional} on the past,
since the event ``$k\in\BDF$, $T_{k}\leq n<T_{k+1}$'' depends
only on the past of the $n$-th round. Hence,

\begin{align*}
\E{}{\RgS{\BDF}} & \leq N\cdot4\left(S^{-1}+1\right)\cdot2D_{Z}\delta\\
 & =8\left(S^{-1}+1\right)ND_{Z}\delta
\end{align*}

Here, we use the fact that $\BDF$ only includes cycles in which $\theta^{*}\in\C_{k}$,
because that's needed for condition (\ref{eq:optimism-1}) and also
to bound $\Norm{\theta^{*}-\theta_{k}}$ by the diameter of $\C_{k}$.

For $\ADF$, Lemma \ref{lem:tau-rho} implies (taking $a=0$)

\begin{align*}
\RgS{\ADF} & \leq\sum_{k<N}4\left(S^{-1}+1\right)\tau_{k}^{0}\cdot9D_{Z}^{2}\delta\\
 & =36\left(S^{-1}+1\right)ND_{Z}^{2}\delta
\end{align*}

For $\EDF$, Lemma \ref{lem:tau-rho} implies (taking $a=1$)

\begin{align*}
\RgS{\EDF} & \leq\sum_{k\in\EDF}\left(4\left(S^{-1}+1\right)\Tau k1\rho_{k}^{1}+C\right)\\
 & \leq4\left(S^{-1}+1\right)\sum_{k\in\EDF}\Tau k1\rho_{k}^{1}+C\Abs{\EDF}
\end{align*}

We know that stopping condition (\ref{eq:stop}) doesn't hold before
the last round of a cycle, for any $k\in\N$. Hence,

\[
\sqrt{\Tau k1}\cdot\rho_{k}^{1}<2(D_{Z}+1)\eta
\]

Dividing both sides by $\sqrt{\Tau k1}$,

\[
\rho_{k}^{1}<\frac{2(D_{Z}+1)\eta}{\sqrt{\Tau k1}}
\]

Using this, we can rewrite the second term on the left hand side of
the bound for $\RgS{\EDF}$:

\begin{align*}
\RgS{\EDF} & \leq8\left(S^{-1}+1\right)(D_{Z}+1)\eta\sum_{k\in\EDF}\sqrt{\Tau k1}+C\Abs{\EDF}
\end{align*}

Since the sum of the $\Tau k1$ cannot exceed $N$, the sum of their
square roots cannot exceed $\sqrt{\Abs{\EDF}N}$. We get

\[
\RgS{\EDF}\leq8\left(S^{-1}+1\right)(D_{Z}+1)\eta\sqrt{\Abs{\EDF}N}+C\Abs{\EDF}
\]

Using Lemma \ref{lem:ABE}, we can bound $|\EDF|\leq|\ED|$ and get

\begin{align*}
\RgS{\EDF} & \leq8\left(S^{-1}+1\right)(D_{Z}+1)\eta\sqrt{\gamma D_{Z}^{2}\ln\frac{D_{Z}R}{\delta}\cdot N}+\gamma CD_{Z}^{2}\ln\frac{D_{Z}R}{\delta}
\end{align*}

We can now assemble all the contributions together:

\begin{align*}
\E{}{\RgS{\BDF}+\RgS{\ADF}+\RgS{\EDF}} & \leq & 8\left(S^{-1}+1\right)ND_{Z}\delta\\
 &  & +36\left(S^{-1}+1\right)ND_{Z}^{2}\delta\\
 &  & +8\left(S^{-1}+1\right)(D_{Z}+1)\eta\sqrt{\gamma D_{Z}^{2}\ln\frac{D_{Z}R}{\delta}\cdot N}\\
 &  & +\gamma CD_{Z}^{2}\ln\frac{D_{Z}R}{\delta}\\
 & = & \left(S^{-1}+1\right)D_{Z}\left(36D_{Z}+8\right)N\delta\\
 &  & +8\eta\left(S^{-1}+1\right)(D_{Z}+1)\sqrt{\gamma D_{Z}^{2}\ln\frac{D_{Z}R}{\delta}\cdot N}\\
 &  & +\gamma CD_{Z}^{2}\ln\frac{D_{Z}R}{\delta}
\end{align*}
\end{proof}
To complete the proof of the main theorem, it remains only to bound
the contribution of $\FC$. Denote $\UR:=\UR[[N]]$ (the total regret).
\begin{proof}
[Proof of Theorem \ref{thm:main}]

We have

\[
\UR=\RgS{\BDF}+\RgS{\ADF}+\RgS{\EDF}+\RgS{\FC}
\]

By Lemma \ref{lem:confidence},

\[
\Pr\left[\FC\ne\varnothing\right]\leq D_{W}N(N+1)\exp\left(-\Omega(1)\cdot\frac{\eta^{2}}{R^{2}D_{W}^{\frac{5}{3}}}\right)
\]

Since $\RgS{\varnothing}=0$ and $\RgS{\FC}\leq CN$, this implies

\begin{align*}
\E{}{\RgS{\FC}} & \leq D_{W}N(N+1)\exp\left(-\Omega(1)\cdot\frac{\eta^{2}}{R^{2}D_{W}^{\frac{5}{3}}}\right)\cdot CN\\
 & =CD_{W}N^{2}(N+1)\exp\left(-\Omega(1)\cdot\frac{\eta^{2}}{R^{2}D_{W}^{\frac{5}{3}}}\right)
\end{align*}

Combining this with Lemma \ref{lem:regret-abc}, we get,

\begin{align*}
\E{}{\UR} & \leq & \left(S^{-1}+1\right)D_{Z}\left(36D_{Z}+8\right)N\delta\\
 &  & +8\eta\left(S^{-1}+1\right)(D_{Z}+1)\sqrt{\gamma D_{Z}^{2}\ln\frac{D_{Z}R}{\delta}\cdot N}\\
 &  & +\gamma CD_{Z}^{2}\ln\frac{D_{Z}R}{\delta}\\
 &  & +CD_{W}N^{2}(N+1)\exp\left(-\Omega(1)\cdot\frac{\eta^{2}}{R^{2}D_{W}^{\frac{5}{3}}}\right)
\end{align*}

Since this holds for any nature policy, this is a bound on $\Reg{\theta^{*}}{\IUCB^{\eta}}N$.
\end{proof}

\section{\label{sec:gap}Upper Bound with Gap}

\subsection{Gap in Special Cases}
\begin{proof}
[Proof of Proposition \ref{prop:gap}]By the assumption on $\theta'$,
$\ME{\theta}r{x_{\theta'}^{*}}<\ME{\theta}r{x_{\theta}^{*}}$. Therefore,
condition 1 is false for $x=x_{\theta'}^{*}$, and hence condition
2 must be true. Consider any $y\in K_{\theta}(x_{\theta'}^{*})$.
We have

\begin{align*}
r\left(x_{\theta'}^{*},y\right) & \leq\ME{\theta}r{x_{\theta}^{*}}-g_{\theta}\\
 & \leq\ME{\theta'}r{x_{\theta'}^{*}}-g_{\theta}
\end{align*}

Here, the first inequality follows from condition 2, and the second
from the assumption on $\theta'$. Since $r$ is 1-Lipschitz, this
implies

\[
\DY y{K_{\theta'}\left(x_{\theta'}^{*}\right)}\geq g_{\theta}
\]
\end{proof}
\rule[0.5ex]{1\columnwidth}{1pt}

Now, we start working towards the proof of Proposition \ref{prop:zerosum-gap}.
First, we'll need two simple technical lemmas. The first one is a
lower bound on the $\ell_{1}$ distance between a probability distribution
and another probability distribution scaled down by a scalar.
\begin{lem}
\label{lem:l1-scalar}Let $B$ be a finite set, $t\in[0,1]$ and $u,v\in\Delta B$.
Then,

\[
\Norm{tu-v}_{1}\geq\frac{1}{2}\Norm{u-v}_{1}
\]
\end{lem}
\begin{proof}
Define $A\subset B$ by

\[
A:=\SC{b\in B}{u_{b}<v_{b}}
\]

Then,

\begin{align*}
\Norm{u-v}_{1} & =\sum_{b\in B}\Abs{u_{b}-v_{b}}\\
 & =\sum_{b\in A}\Abs{u_{b}-v_{b}}+\sum_{b\in B\setminus A}\Abs{u_{b}-v_{b}}\\
 & =\sum_{b\in A}\left(v_{b}-u_{b}\right)+\sum_{b\in B\setminus A}\left(u_{b}-v_{b}\right)
\end{align*}

Moreover,

\begin{align*}
\sum_{b\in A}\left(v_{b}-u_{b}\right)-\sum_{b\in B\setminus A}\left(u_{b}-v_{b}\right) & =\sum_{b\in A}v_{b}-\sum_{b\in A}u_{b}-\sum_{b\in B\setminus A}u_{b}+\sum_{b\in B\setminus A}v_{b}\\
 & =\sum_{b\in B}v_{b}-\sum_{b\in B}u_{b}\\
 & =1-1\\
 & =0
\end{align*}

Therefore,

\[
\sum_{b\in A}\left(v_{b}-u_{b}\right)=\sum_{b\in B\setminus A}\left(u_{b}-v_{b}\right)
\]

Combining this with the expression we got for $\Norm{u-v}_{1}$, we
conclude

\begin{align*}
\Norm{u-v}_{1} & =2\sum_{b\in A}\left(v_{b}-u_{b}\right)
\end{align*}

It follows that

\begin{align*}
\Norm{tu-v}_{1} & =\sum_{b\in B}\Abs{tu_{b}-v_{b}}\\
 & \geq\sum_{b\in A}\Abs{tu_{b}-v_{b}}\\
 & =\sum_{b\in A}\left(v_{b}-tu_{b}\right)\\
 & \geq\sum_{b\in A}\left(v_{b}-u_{b}\right)\\
 & =\frac{1}{2}\Norm{u-v}_{1}
\end{align*}
\end{proof}
We now extend the previous lemma to allow for a scalar in front of
both distributions.
\begin{lem}
\label{lem:l1-biscalar}Let $B$ be a finite set, $s,t\in[0,1]$ and
$u,v\in\Delta B$. Then,

\[
\Norm{su-tv}_{1}\geq\frac{1}{2}\max\left(s,t\right)\Norm{u-v}_{1}
\]
\end{lem}
\begin{proof}
Assume w.l.o.g. that $s\leq t$. If $t=0$ then $s=0$ and the claim
is true. If $t>0$, we have

\begin{align*}
\Norm{su-tv}_{1} & =t\Norm{\frac{s}{t}u-v}_{1}
\end{align*}

Applying Lemma \ref{lem:l1-scalar} to the right hand side, we get

\begin{align*}
\Norm{su-tv}_{1} & \geq t\cdot\frac{1}{2}\Norm{u-v}_{1}\\
 & =\frac{1}{2}\max\left(s,t\right)\Norm{u-v}_{1}
\end{align*}
\end{proof}
The following is a lower bound on the $\ell_{1}$ distance between
two credal sets, when each of them is specified by fixing the conditional
distribution of one variable relatively to another variable.
\begin{lem}
\label{lem:dist-sdp}Let $A,B$ be finite sets and $Q,\tilde{Q}:A\rightarrow\Delta B$
mappings. Define $\kappa,\tilde{\kappa}\subseteq\Delta(A\times B)$
by

\begin{align*}
\kappa & =\SC{y\in\Delta(A\times B)}{\forall a\in A,b\in B:y_{ab}=Q(a)_{b}\sum_{b'\in B}y_{ab'}}\\
\tilde{\kappa} & =\SC{y\in\Delta(A\times B)}{\forall a\in A,b\in B:y_{ab}=\tilde{Q}(a)_{b}\sum_{b'\in B}y_{ab'}}
\end{align*}

Then,

\[
\min\SC{\Norm{y-\tilde{y}}_{1}}{y\in\kappa,\,\tilde{y}\in\tilde{\kappa}}\geq\frac{1}{2}\min_{a\in\A}\Norm{Q(a)-\tilde{Q}(a)}_{1}
\]
\end{lem}
\begin{proof}
Consider any $y\in\kappa$, $\tilde{y}\in\tilde{\kappa}$. Define
$u,\tilde{u}\in\Delta A$ by

\begin{align*}
u_{a} & :=\sum_{b'\in B}y_{ab'}\\
\tilde{u}_{a} & :=\sum_{b'\in B}\tilde{y}_{ab'}
\end{align*}

We have

\begin{align*}
\Norm{y-\tilde{y}}_{1} & =\sum_{a\in A}\sum_{b\in B}\Abs{Q(a)_{b}u_{a}-\tilde{Q}(a)_{b}\tilde{u}_{a}}\\
 & =\sum_{a\in A}\Norm{u_{a}Q(a)-\tilde{u}_{a}\tilde{Q}(a)}_{1}\\
 & \geq\sum_{a\in A}\frac{1}{2}\max\left(u_{a},\tilde{u}_{a}\right)\Norm{Q(a)-\tilde{Q}(a)}_{1}\\
 & \geq\frac{1}{2}\sum_{a\in A}u_{a}\Norm{Q(a)-\tilde{Q}(a)}_{1}\\
 & \geq\frac{1}{2}\sum_{a\in A}u_{a}\min_{a'\in A}\Norm{Q(a')-\tilde{Q}(a')}_{1}\\
 & =\frac{1}{2}\min_{a'\in A}\Norm{Q(a')-\tilde{Q}(a')}_{1}
\end{align*}

Here, we used Lemma \ref{lem:l1-biscalar} on the 3rd line.
\end{proof}
Notice that, in the setting of Lemma \ref{lem:dist-sdp}, it is obvious
that

\[
\min\SC{\Norm{y-\tilde{y}}_{1}}{y\in\kappa,\,\tilde{y}\in\tilde{\kappa}}\leq\min_{a\in\A}\Norm{Q(a)-\tilde{Q}(a)}_{1}
\]

Because, for any $a_{0}\in A$, we can take $y_{ab}:=\boldsymbol{1}_{a=a_{0}}Q(a)_{b}$
and likewise for $\tilde{y}$. Hence, that bound is tight up to a
multiplicative constant.

In order to apply Lemma \ref{lem:dist-sdp} to the proof of Proposition
\ref{prop:zerosum-gap}, we need to show that the credal sets figuring
in the latter are of the form required for the former. The following
lemma establishes it.
\begin{lem}
\label{lem:zerosum-sdp}In the setting of Example \ref{ex:zerosum}
and Example \ref{ex:pcb-lin}, let $\theta\in\Z$ be s.t. for any
$b\in\B_{2}$, $\theta_{b}=1$ and for any $(b,a)\in\B_{2}\times\B_{1}$,
$\theta_{ba}\in\Delta\left\{ -1,+1\right\} $. Recall that $\B=\B_{2}\times\B_{1}\times\left\{ -1,+1\right\} $.
Then, for any $x\in\A_{1}$,

\[
K_{\theta}(x)^{+}=\SC{y\in\Delta\B}{\forall b,a,\sigma:y_{ba\sigma}=\theta_{ba\sigma}x_{a}\sum_{\SUBSTACK{a'\in\B_{1}}{\sigma'\in\left\{ -1,+1\right\} }}y_{ba'\sigma'}}
\]
\end{lem}
\begin{proof}
In Example \ref{ex:pcb-expect}, we saw that for any $x\in\A_{1}$
and $y\in\Delta\B$, we have that $y\in K_{\theta}(x)^{+}$ if and
only if

\[
\forall a\in\Stoc,c\in\CB{|a|}:\phi_{x,a,c}(y)=0
\]

Looking at the definitions of $\Stoc$ and $\CB{}$ in Example \ref{ex:zerosum},
this condition can be rewritten as the pair of conditions

\[
\begin{cases}
\forall b\in\B_{2},a\in\B_{1}: & \phi_{x,b,a}(y)=0\\
\forall b\in\B_{2},a\in\B_{1},\sigma\in\left\{ -1,+1\right\} : & \phi_{x,ba,\sigma}(y)=0
\end{cases}
\]

Unpacking the definition of $\phi$ from Example \ref{ex:pcb}, we
can rewrite $\phi_{x,b,a}(y)=0$ as

\[
\sum_{\sigma\in\left\{ -1,+1\right\} }y_{ba\sigma}-f_{b}\left(x,\theta_{b}\right)_{a}\sum_{\SUBSTACK{a'\in\B_{1}}{\sigma\in\left\{ -1,+1\right\} }}y_{ba'\sigma}=0
\]

Equivalently,

\[
\sum_{\sigma\in\left\{ -1,+1\right\} }y_{ba\sigma}=f_{b}\left(x,\theta_{b}\right)_{a}\sum_{\SUBSTACK{a'\in\B_{1}}{\sigma\in\left\{ -1,+1\right\} }}y_{ba'\sigma}
\]

Using the definition of $f_{b}$ in Example \ref{ex:zerosum}, this
becomes

\[
\sum_{\sigma\in\left\{ -1,+1\right\} }y_{ba\sigma}=\theta_{b}x_{a}\sum_{\SUBSTACK{a'\in\B_{1}}{\sigma\in\left\{ -1,+1\right\} }}y_{ba'\sigma}
\]

Since $\theta_{b}=1$, we get

\begin{equation}
\sum_{\sigma\in\left\{ -1,+1\right\} }y_{ba\sigma}=x_{a}\sum_{\SUBSTACK{a'\in\B_{1}}{\sigma\in\left\{ -1,+1\right\} }}y_{ba'\sigma}\label{eq:lem_pcb-sdp_x}
\end{equation}

Similarly, we can rewrite $\phi_{x,ba,\sigma}(y)=0$ (the second condition)
as

\[
y_{ba\sigma}=f_{ba}\left(x,\theta_{ba}\right)_{\sigma}\sum_{\sigma'\in\left\{ -1,+1\right\} }y_{ba\sigma'}
\]

Using the definition of $f_{ba}$ in Example \ref{ex:zerosum}, this
becomes

\begin{equation}
y_{ba\sigma}=\theta_{ba\sigma}\sum_{\sigma'\in\left\{ -1,+1\right\} }y_{ba\sigma'}\label{eq:lem_pcb-sdp_th}
\end{equation}

It remains to show that the conjuction of equations (\ref{eq:lem_pcb-sdp_x})
and (\ref{eq:lem_pcb-sdp_th}) is equivalent to

\begin{equation}
y_{ba\sigma}=\theta_{ba\sigma}x_{a}\sum_{\SUBSTACK{a'\in\B_{1}}{\sigma'\in\left\{ -1,+1\right\} }}y_{ba'\sigma'}\label{eq:lem_pcb-sdp_th-x}
\end{equation}

Assuming equations (\ref{eq:lem_pcb-sdp_x}) and (\ref{eq:lem_pcb-sdp_th}),
we have

\begin{align*}
y_{ba\sigma} & =\theta_{ba\sigma}\sum_{\sigma'\in\left\{ -1,+1\right\} }y_{ba\sigma'}\\
 & =\theta_{ba\sigma}x_{a}\sum_{\SUBSTACK{a'\in\B_{1}}{\sigma'\in\left\{ -1,+1\right\} }}y_{ba'\sigma'}
\end{align*}

Here, we used equation (\ref{eq:lem_pcb-sdp_th}) on the 1st line
and equation (\ref{eq:lem_pcb-sdp_x}) on the second line.

Conversly, assuming equation (\ref{eq:lem_pcb-sdp_th-x}), we can
sum both sides over $\sigma$, yielding equation (\ref{eq:lem_pcb-sdp_x}):

\begin{align*}
\sum_{\sigma\in\left\{ -1,+1\right\} }y_{ba\sigma} & =\sum_{\sigma\in\left\{ -1,+1\right\} }\theta_{ba\sigma}x_{a}\sum_{\SUBSTACK{a'\in\B_{1}}{\sigma'\in\left\{ -1,+1\right\} }}y_{ba'\sigma'}\\
 & =x_{a}\sum_{\SUBSTACK{a'\in\B_{1}}{\sigma'\in\left\{ -1,+1\right\} }}y_{ba'\sigma'}
\end{align*}

Here, we used the assumption $\theta_{ba}\in\Delta\left\{ -1,+1\right\} $. 

Moreover, applying equation (\ref{eq:lem_pcb-sdp_x}) to equation
(\ref{eq:lem_pcb-sdp_th-x}), we get equation (\ref{eq:lem_pcb-sdp_th}):

\begin{align*}
y_{ba\sigma} & =\theta_{ba\sigma}\left(x_{a}\sum_{\SUBSTACK{a'\in\B_{1}}{\sigma'\in\left\{ -1,+1\right\} }}y_{ba'\sigma'}\right)\\
 & =\theta_{ba\sigma}\sum_{\sigma\in\left\{ -1,+1\right\} }y_{ba\sigma}
\end{align*}
\end{proof}
It is now straightforward to prove of Proposition \ref{prop:zerosum-gap}
by putting together Lemma \ref{lem:dist-sdp} and Lemma \ref{lem:zerosum-sdp}.
\begin{proof}
[Proof of Proposition \ref{prop:zerosum-gap}]Fix $x\in\A_{1}$. Define
$Q,\tilde{Q}:\B_{2}\rightarrow\Delta\left(\B_{1}\times\left\{ -1,+1\right\} \right)$
by

\begin{align*}
Q(b)_{a\sigma} & :=\theta_{ba\sigma}x_{a}\\
\tilde{Q}(b)_{a\sigma} & :=\theta'_{ba\sigma}x_{a}
\end{align*}

By Lemma \ref{lem:zerosum-sdp}, we have

\begin{align*}
K_{\theta}(x)^{+} & =\SC{y\in\Delta\B}{\forall b,a,\sigma:y_{ba\sigma}=Q(b)_{a\sigma}\sum_{\SUBSTACK{a'\in\B_{1}}{\sigma'\in\left\{ -1,+1\right\} }}y_{ba'\sigma'}}\\
K_{\theta'}(x)^{+} & =\SC{y\in\Delta\B}{\forall b,a,\sigma:y_{ba\sigma}=\tilde{Q}(b)_{a\sigma}\sum_{\SUBSTACK{a'\in\B_{1}}{\sigma'\in\left\{ -1,+1\right\} }}y_{ba'\sigma'}}
\end{align*}

By Lemma \ref{lem:dist-sdp} and Lemma \ref{lem:l1}, it follows that

\begin{align*}
\DY{K_{\theta}(x)^{+}}{K_{\theta'}(x)^{+}} & \geq\frac{1}{2}\min_{b\in\B_{2}} & \Norm{Q(b)-\tilde{Q}(b)}_{1}\\
 & =\frac{1}{2}\min_{b\in\B_{2}} & \sum_{\SUBSTACK{a\in\B_{1}}{\sigma\in\left\{ -1,+1\right\} }}\Abs{Q(b)_{a\sigma}-\tilde{Q}(b)_{a\sigma}}\\
 & =\frac{1}{2}\min_{b\in\B_{2}} & \sum_{\SUBSTACK{a\in\B_{1}}{\sigma\in\left\{ -1,+1\right\} }}\Abs{\theta_{ba\sigma}x_{a}-\theta'_{ba\sigma}x_{a}}\\
 & =\frac{1}{2}\min_{b\in\B_{2}} & \sum_{a\in\B_{1}}x_{a}\sum_{\sigma\in\left\{ -1,+1\right\} }\Abs{\theta_{ba\sigma}-\theta'_{ba\sigma}}\\
 & =\frac{1}{2}\min_{b\in\B_{2}} & \sum_{a\in\B_{1}}x_{a}\left(\Abs{\theta_{ba,-1}-\theta'_{ba,-1}}+\Abs{\theta_{ba,+1}-\theta'_{ba,+1}}\right)\\
 & =\frac{1}{2}\min_{b\in\B_{2}} & \sum_{a\in\B_{1}}x_{a}\bigg(\Abs{\frac{1-P_{ab}}{2}-\frac{1-P_{ab}'}{2}}+\\
 &  & \Abs{\frac{1+P_{ab}}{2}-\frac{1+P_{ab}'}{2}}\bigg)\\
 & =\frac{1}{2}\min_{b\in\B_{2}} & \sum_{a\in\B_{1}}x_{a}\left(\Abs{\frac{-P_{ab}+P_{ab}'}{2}}+\Abs{\frac{P_{ab}-P_{ab}'}{2}}\right)\\
 & =\frac{1}{2}\min_{b\in\B_{2}} & \sum_{a\in\B_{1}}x_{a}\Abs{P_{ab}-P_{ab}'}
\end{align*}
\end{proof}

\subsection{Proof of the Upper Bound}

The regret analysis in Theorem \ref{thm:gap} is similar to that in
Theorem \ref{thm:main} with two important differences: we use the
specific value $\delta\approx\frac{Sg}{144D_{Z}^{2}}$, and we show
that cycles $k$ with either small $\rho_{k}$ or small diameter of
$\C_{k}$ must select optimal arms and therefore don't accrue regret.
The following lemma handles the case of small $\rho_{k}$:
\begin{lem}
\label{lem:gap}Let $\theta,\theta^{*}\in\H$ and $\bar{y}\in\D$
be s.t. the following conditions hold:
\begin{enumerate}
\item $\ME{\theta}r{x_{\theta}^{*}}\geq r^{*}$
\item $\DZ{\theta}{\V{x_{\theta}^{*}}{\bar{y}}}<\frac{1}{16}Sg$
\item $\DZ{\theta^{*}}{\V{x_{\theta}^{*}}{\bar{y}}}<\frac{1}{16}Sg$
\end{enumerate}
Then,

\begin{equation}
\ME{\theta^{*}}r{x_{\theta}^{*}}=r^{*}\label{eq:lem-gap}
\end{equation}
\end{lem}
\begin{proof}
Using Lemma \ref{lem:sin} and Lemma \ref{lem:dy-leq-dz}, we have,
for any $\theta'\in\{\theta,\theta^{*}\}$:

\begin{align*}
\DY{\bar{y}}{K_{\theta'}\left(x_{\theta}^{*}\right)^{+}} & \leq\left(S^{-1}+1\right)\DY{\bar{y}}{K_{\theta'}\left(x_{\theta}^{*}\right)^{\flat}}\\
 & \leq4\left(S^{-1}+1\right)\DZ{\theta}{\V{x_{\theta}^{*}}{\bar{y}}}\\
 & <4\left(S^{-1}+1\right)\cdot\frac{1}{16}Sg\\
 & \leq\frac{1}{2}g
\end{align*}

Here we used the obvious fact that $S\leq1$. By the triangle inequality,

\[
\DY{K_{\theta}\left(x_{\theta}^{*}\right)^{+}}{K_{\theta^{*}}\left(x_{\theta}^{*}\right)^{+}}<g
\]

By the definition of $g$, condition 1 and this inequality imply (\ref{eq:lem-gap}).
\end{proof}
Now, we handle the case of small diameter. In particular, thanks to
the following we won't need Lemma \ref{lem:regret-by-distance}.
\begin{lem}
\label{lem:gap-hypercube}Let $\theta,\theta^{*}\in\H$ and $\bar{y}\in\D$
be s.t. the following conditions hold:
\begin{enumerate}
\item $\ME{\theta}r{x_{\theta}^{*}}\geq r^{*}$
\item $\Norm{\theta^{*}-\theta}<\frac{1}{8}Sg$
\end{enumerate}
Then,

\begin{equation}
\ME{\theta^{*}}r{x_{\theta}^{*}}=r^{*}\label{eq:lem-gap-hypercube}
\end{equation}
\end{lem}
\begin{proof}
Take any $y\in K_{\theta}\left(x_{\theta}^{*}\right)^{+}$. Using
Lemma \ref{lem:sin} and Lemma \ref{lem:dy-leq-dz} and the fact that
$\theta\in\V{x_{\theta}^{*}}y$,

\begin{align*}
\DY{K_{\theta}\left(x_{\theta}^{*}\right)^{+}}{K_{\theta^{*}}\left(x_{\theta}^{*}\right)^{+}} & \leq\DY y{K_{\theta^{*}}\left(x_{\theta}^{*}\right)^{+}}\\
 & \leq\left(S^{-1}+1\right)\DY y{K_{\theta^{*}}\left(x_{\theta}^{*}\right)^{\flat}}\\
 & \leq4\left(S^{-1}+1\right)\DZ{\theta^{*}}{\V{x_{\theta}^{*}}y}\\
 & \leq4\left(S^{-1}+1\right)\Norm{\theta^{*}-\theta}\\
 & <4\left(S^{-1}+1\right)\cdot\frac{1}{8}Sg\\
 & \leq g
\end{align*}

By the definition of $g$, condition 1 and this inequality imply (\ref{eq:lem-gap-hypercube}).
\end{proof}
We can now go on to prove the theorem.

\global\long\def\CycOpt{\mathrm{A}_{1}}%

\global\long\def\CycDim{\Gamma_{1}}%

\begin{proof}
[Proof of Theorem \ref{thm:gap}]Consider any $\epsilon>0$, take
$\delta:=\frac{Sg}{(144+\epsilon)D_{Z}^{2}}$, and denote

\[
\CycOpt:=\SC{k\in\FD}{\ME{\theta^{*}}r{x_{\theta_{k}}^{*}}=r^{*}}
\]

Since $2D_{Z}\delta<\frac{1}{8}Sg$, Lemma \ref{lem:gap-hypercube}
implies that $\BDF\subseteq\CycOpt$. Since $9D_{Z}^{2}\delta<\frac{1}{16}Sg$,
Lemma \ref{lem:gap} implies that $\ADF\subseteq\CycOpt$. Denote
$\CycDim:=\EDF\setminus\CycOpt$. We get

\[
\UR=\RgS{\CycOpt}+\RgS{\CycDim}+\RgS{\FC}
\]

We will now bound the expectation of each term.

Since $\CycOpt$ consists of cycles in which an optimal arm is selected,
and the event ``$k\in\CycOpt$ and $T_{k}\leq n<T_{k+1}$'' depends
only on the past of round $n$, we have

\[
\E{}{\RgS{\CycOpt}}=0
\]

For any $k\in\EDF$, if $\Tau k1>0$ and $\rho_{k}^{1}\leq9D_{Z}^{2}\delta<\frac{1}{16}Sg$,
then by Lemma \ref{lem:gap}, $k\in\CycOpt$ and hence $k\not\in\CycDim$.
Hence, for every $k\in\CycDim$, either $\Tau k1=0$ or $\rho_{k}^{1}>9D_{Z}^{2}\delta$.
Recall that, 

\[
\sqrt{\Tau k1}\cdot\rho_{k}^{1}<2(D_{Z}+1)\eta
\]

It follows that,

\[
\Tau k1<\frac{4(D_{Z}+1)^{2}\eta^{2}}{\left(\rho_{k}^{1}\right)^{2}}
\]

By Lemma \ref{lem:tau-rho}, we get

\begin{align*}
\RgS{\CycDim} & \leq\sum_{k\in\CycDim}\left(4\left(S^{-1}+1\right)\Tau k1\rho_{k}^{1}+C\right)\\
 & \leq\sum_{k\in\CycDim}\left(\frac{16\left(S^{-1}+1\right)(D_{Z}+1)^{2}\eta^{2}}{\left(\rho_{k}^{1}\right)^{2}}\cdot\rho_{k}^{1}+C\right)\\
 & =\sum_{k\in\CycDim}\left(\frac{16\left(S^{-1}+1\right)(D_{Z}+1)^{2}\eta^{2}}{\rho_{k}^{1}}+C\right)\\
 & \leq\Abs{\CycDim}\left(\frac{16\left(S^{-1}+1\right)(D_{Z}+1)^{2}\eta^{2}}{9D_{Z}^{2}\delta}+C\right)\\
 & \leq\Abs{\CycDim}\left(\frac{\left(256+\frac{16}{9}\epsilon\right)S^{-1}\left(S^{-1}+1\right)(D_{Z}+1)^{2}\eta^{2}}{g}+C\right)
\end{align*}

Using Lemma \ref{lem:ABE}, it follows that,

\begin{align*}
\RgS{\CycDim}\leq & \gamma D_{Z}^{2}\ln\frac{D_{Z}R}{\delta}\cdot\left(\frac{\left(256+\frac{16}{9}\epsilon\right)S^{-1}\left(S^{-1}+1\right)(D_{Z}+1)^{2}\eta^{2}}{g}+C\right)\\
\leq & \gamma D_{Z}^{2}\left(\frac{\left(256+\frac{16}{9}\epsilon\right)S^{-1}\left(S^{-1}+1\right)(D_{Z}+1)^{2}\eta^{2}}{g}+C\right)\times\\
 & \ln\frac{(144+\epsilon)S^{-1}D_{Z}^{3}R}{g}
\end{align*}

As before, Lemma \ref{lem:confidence} implies that

\[
\E{}{\RgS{\FC}}\leq D_{W}N(N+1)\exp\left(-\Omega(1)\cdot\frac{\eta^{2}}{R^{2}D_{W}^{\frac{5}{3}}}\right)\cdot CN
\]

We can now assemble the contributions together and take $\epsilon\rightarrow0$:

\begin{align*}
\E{}{\UR}\leq & \gamma D_{Z}^{2}\left(\frac{256S^{-1}\left(S^{-1}+1\right)(D_{Z}+1)^{2}\eta^{2}}{g}+C\right)\ln\frac{144S^{-1}D_{Z}^{3}R}{g}\\
 & +CD_{W}N^{2}(N+1)\exp\left(-\Omega(1)\cdot\frac{\eta^{2}}{R^{2}D_{W}^{\frac{5}{3}}}\right)
\end{align*}
\end{proof}

\section{\label{sec:lower-s}Lower Bound for Scaling with $S$}

\subsection{Proof of the Lower Bound}

The proof of Theorem \ref{thm:lower-s} works by constructing a nature
policy which produces constant regret per round while revealing little
about the true hypothesis $\theta^{*}$, with high probability. When
the optimal arm $x^{*}$ is selected, the worst-case outcome has the
form

\[
y=\Vxxx 01{-\frac{1}{2}x^{*}}
\]

The reward is $r(x^{*},y)=-\frac{1}{2}$. Our nature policy instead
responds to every arm $x$ with the outcome

\[
y=\Vxxx 01{-\left(\frac{1}{2}+\delta\right)x}
\]

The key observation is, it is always possible to enforce a high probability
of this outcome unless $x$ is very close to the optimal arm $x^{*}$.
This results in a reward of $r(x,y)=-\frac{1}{2}-\delta$ and hence
a regret of $\delta$, which persists until either a different outcome
results randomly (which only happens with probability $\frac{1-\alpha}{1+3\alpha}$
per round) or $x$ is close to $x^{*}$. But, since this outcome provides
little information about $\theta^{*}$, the agent can only get close
to $x^{*}$ by exhaustive search. Each step in this search covers
a spherical ball around $x$, and hence the required number of steps
is approximately the ratio of the volume of the sphere to the volume
of a single spherical ball. Importantly, this ratio scales exponentially
with $D$.

We will use the following inequality concerning the gamma function,
proved in \cite{Kershaw83}:
\begin{thm}
[Kershaw]\label{thm:kershaw}For any $t>0$ and $s\in(0,1)$:

\[
\frac{\Gamma\left(t+1\right)}{\Gamma\left(t+s\right)}>\left(t+\frac{s}{2}\right)^{1-s}
\]
\end{thm}
The following geometric lemma is used to lower bound the number of
steps in the abovementioned ``search''.
\begin{lem}
\label{lem:ball}For any $D\geq1$, let $v_{D}$ be the volume of
the $D$-dimensional unit ball and $a_{D}$ be the volume of the $D$-dimensional
unit sphere\footnote{That is, the \emph{intrinsic} dimension of the sphere is $D$. In
the usual representation, it is a submanifold of $\R^{D+1}$.}. Then,

\[
\frac{v_{D}}{a_{D}}<\frac{1}{\sqrt{\pi\left(2D+1\right)}}
\]
\end{lem}
\begin{proof}
We have

\[
v_{D}=\frac{\pi^{\frac{D}{2}}}{\Gamma\left(\frac{D}{2}+1\right)}
\]

\[
a_{D}=\frac{2\pi^{\frac{D+1}{2}}}{\Gamma\left(\frac{D+1}{2}\right)}
\]

It follows that

\[
\frac{v_{D}}{a_{D}}=\frac{1}{2\sqrt{\pi}}\cdot\frac{\Gamma\left(\frac{D+1}{2}\right)}{\Gamma\left(\frac{D}{2}+1\right)}
\]

Applying Theorem \ref{thm:kershaw} with $t=\frac{D}{2}$ and $s=\frac{1}{2}$,
we get

\begin{align*}
\frac{v_{D}}{a_{D}} & <\frac{1}{2\sqrt{\pi}}\cdot\frac{1}{\sqrt{\frac{D}{2}+\frac{1}{4}}}\\
 & =\frac{1}{\sqrt{\pi\left(2D+1\right)}}
\end{align*}
\end{proof}
We are now ready to give the proof.
\begin{proof}
[Proof of Theorem \ref{thm:lower-s}]Let $\theta^{*}\in\H$ be the
true hypothesis (a random variable with distribution $\xi$). Let
$u^{*}\in\R^{D}$ be s.t.

\[
\theta^{*}=\Vxxx{1-\alpha}{-\alpha}{2\alpha u^{*}}
\]

\global\long\def\NuT#1{y_{#1}^{*}}%

\global\long\def\Nu{\NuT{\delta}}%

\global\long\def\Ybot{y_{\bot}}%

Define $\Ybot\in\D$ by

\[
\Ybot:=\Vxxx 10{\boldsymbol{0}}
\]

Fix $\delta\in(0,\frac{1}{2})$ and define $\NuT 0,\NuT{\delta}:\A\rightarrow\D$
by

\[
\NuT 0(x):=\Vxxx{\alpha\left(1+\T xu^{*}\right)}{1-\alpha\left(1+\T xu^{*}\right)}{-\frac{1}{2}x}
\]

\[
\NuT{\delta}(x):=\Vxxx 01{-\left(\frac{1}{2}+\delta\right)x}
\]

Notice that $\NuT 0(x)\in\D$ because $\alpha\leq\frac{1}{4}$, implying
that $\alpha\left(1+\T xu^{*}\right)\in[0,\frac{1}{2}]$, $1-\alpha\left(1+\T xu^{*}\right)\in[\frac{1}{2},1]$
and hence

\begin{align*}
\Norm{-\frac{1}{2}x}_{2} & =\frac{1}{2}\\
 & \leq1-\alpha\left(1+\T xu^{*}\right)
\end{align*}

Consider the following nature policy $\nu$. Given any history

\[
h:=x_{0}y_{0}\ldots x_{n-1}y_{n-1}x_{n}
\]

we need to define $\nu(h)\in\Delta\D$. There are two cases:
\begin{itemize}
\item For all $0\leq m\leq n$, $\T{(u^{*})}x_{m}\geq-\frac{1}{1+2\delta}$
and $\Ybot\not\in\{y_{0}\ldots y_{n-1}\}$: In this case, $\nu(h)$
is the unique distribution in $\kappa(K_{\theta^{*}}(x_{n}))$ supported
on the two point set $\{\Nu(x_{n}),\Ybot\}\subseteq\D$.
\item Either for some $0\leq m\leq n$, $\T{(u^{*})}x_{m}<-\frac{1}{1+2\delta}$,
or $\Ybot\in\{y_{0}\ldots y_{n-1}\}$: In this case, $\nu(h)$ assigns
probability 1 to $\NuT 0(x_{n})$.
\end{itemize}
In the first case, the distribution is unique because the closed line
segment between $\Nu(x_{n})$ and $\Ybot$ intersects $K_{\theta^{*}}(x_{n})$
at a single point. Indeed, $\Nu(x_{n})$ and $\Ybot$ lie on opposite
sides of the hyperplane $K_{\theta^{*}}(x_{n})$, as evidenced by
the linear functional $F_{x_{n}\theta^{*}}$ taking values with opposite
signs at these points. For $\Nu(x_{n})$, we have: 

\global\long\def\Rxxx#1#2#3{\left[\begin{array}{ccc}
#1 & #2 & #3\end{array}\right]}%

\begin{align*}
F_{x_{n}\theta^{*}}\left(\Nu\left(x_{n}\right)\right) & =\Rxxx{1-\alpha}{-\alpha}{2\alpha\T{\left(u^{*}\right)}}\Vxxx 01{-\left(\frac{1}{2}+\delta\right)x_{n}}\\
 & =-\alpha-2\alpha\left(\frac{1}{2}+\delta\right)\T{\left(u^{*}\right)}x_{n}\\
 & =-\alpha\left(1+2\left(\frac{1}{2}+\delta\right)\T{\left(u^{*}\right)}x_{n}\right)\\
 & =-\alpha\left(1+\left(1+2\delta\right)\T{\left(u^{*}\right)}x_{n}\right)\\
 & \leq-\alpha\left(1+\left(1+2\delta\right)\left(-\frac{1}{1+2\delta}\right)\right)\\
 & =0
\end{align*}

For $\Ybot$, we have:

\begin{align*}
F_{x_{n}\theta^{*}}\left(\Ybot\right) & =\Rxxx{1-\alpha}{-\alpha}{2\alpha\T{\left(u^{*}\right)}}\Vxxx 10{\boldsymbol{0}}\\
 & =1-\alpha\\
 & >0
\end{align*}

Now, let's lower bound the probability that the outcome in the first
case is $\Nu(x_{n})$:

\global\long\def\KE#1#2{\left(#1\middle|#2\right)}%

\begin{align*}
\nu\KE{\Nu\left(x_{n}\right)}h & =\frac{F_{x_{n}\theta^{*}}\left(\Ybot\right)}{F_{x_{n}\theta^{*}}\left(\Ybot\right)-F_{x_{n}\theta^{*}}\left(\Nu\left(x_{n}\right)\right)}\\
 & =\frac{1-\alpha}{1-\alpha+\alpha\left(1+\left(1+2\delta\right)\T{\left(u^{*}\right)}x_{n}\right)}\\
 & =\frac{1-\alpha}{1+\alpha\left(1+2\delta\right)\T{\left(u^{*}\right)}x_{n}}\\
 & \geq\frac{1-\alpha}{1+\alpha\left(1+2\delta\right)}\\
 & >\frac{1-\alpha}{1+2\alpha}
\end{align*}

Here, we used that $\T{(u^{*})}x_{n}\leq1$ (because both $u^{*}$
and $x_{n}$ are unit vectors) and $\delta<\frac{1}{2}$.

In the second case, we have $\nu(x)\in\kappa(K_{\theta^{*}}(x_{n}))$
because

\begin{align*}
\T{\left(\theta^{*}\right)}\NuT 0(x) & =\Rxxx{1-\alpha}{-\alpha}{2\alpha\T{\left(u^{*}\right)}}\Vxxx{\alpha\left(1+\T xu^{*}\right)}{1-\alpha\left(1+\T xu^{*}\right)}{-\frac{1}{2}x}\\
 & =\left(1-\alpha\right)\alpha\left(1+\T xu^{*}\right)-\alpha\left(1-\alpha\left(1+\T xu^{*}\right)\right)-\alpha\T xu^{*}\\
 & =\left((1-\alpha)\alpha+\alpha^{2}-\alpha\right)\T xu^{*}+(1-\alpha)\alpha-\alpha+\alpha^{2}\\
 & =0
\end{align*}

Consider the history $\bar{x}\bar{y}\in(\A\times\D)^{N}$ that would
be generated by the agent policy and the ``illegal'' nature policy
that always produces the outcome $\Nu(x_{n})$ in response to arm
$x$. Let $q$ denote the probability that, for a uniformly random
$u^{*}$, $\T{(u^{*})}\bar{x}_{n}<-\frac{1}{1+2\delta}$ for at least
one $n<N$. The actual history is the same as $\bar{x}\bar{y}$ as
long as $\T{(u^{*})}x\geq-\frac{1}{1+2\delta}$ holds for all previous
actions $x$, and the outcome $\Ybot$ never occured. Therefore, the
probability that the actual history is $\bar{x}\bar{y}$ is at least
$\left(\frac{1-\alpha}{1+2\alpha}\right)^{N}(1-q)$.

Let's estimate $q$. For any fixed $x\in\A$, denote

\global\long\def\Ball{B_{x}}%

\global\long\def\Sph{S}%

\[
\Sph:=\SC{u\in\R^{D}}{\Norm u=1}
\]

\[
\Ball:=\SC{u\in S}{\T{\left(u^{*}\right)}x<-\frac{1}{1+2\delta}}
\]

Since $u^{*}$ is a uniformly random unit vector, we have

\[
q=\frac{vol\left(\bigcup_{n<N}B_{\bar{x}_{n}}\right)}{vol\left(S\right)}
\]

This is ratio is at most $N$ times the ratio corresponding to a \emph{single}
ball. Now, the spherical radius of $\Ball$ is

\[
\rho=\arccos\frac{1}{1+2\delta}
\]

We have

\begin{align*}
\frac{1}{1+2\delta} & =\cos\rho\\
 & \leq1-\frac{4}{\pi^{2}}\rho^{2}
\end{align*}

The last inequality holds for any $\rho\in[0,\frac{\pi}{2}]$, as
can be seen by noticing that (i) at $\rho=0$ these functions are
equal and their first derivatives are equal (ii) at $\rho=\frac{\pi}{2}$
these functions are also equal and (iii) the second derivative of
the cosine changes montonically from $-1$ to $0$ over the interval
while that of the quadratic function stays constant at $-\frac{8}{\pi^{2}}\in(-1,0)$.

Rearranging, we get

\begin{align*}
\rho & \leq\sqrt{\frac{\pi^{2}}{4}\left(1-\frac{1}{1+2\delta}\right)}\\
 & =\frac{\sqrt{2}}{2}\pi\cdot\sqrt{\frac{\delta}{1+2\delta}}\\
 & \leq\frac{\sqrt{2}}{2}\pi\cdot\sqrt{\delta}
\end{align*}

Observe that the volume of $\Ball$ is upper bounded by the volume
of a Euclidean ball of the same radius. We get

\begin{align*}
\frac{vol\left(\Ball\right)}{vol\left(\Sph\right)} & \leq\frac{v_{D-1}\rho^{D-1}}{a_{D-1}}\\
 & \leq\frac{v_{D-1}}{a_{D-1}}\left(\frac{\sqrt{2}}{2}\pi\cdot\sqrt{\delta}\right)^{D-1}\\
 & \leq\frac{1}{\sqrt{\pi(2D-1)}}\left(\frac{\sqrt{2}}{2}\pi\cdot\sqrt{\delta}\right)^{D-1}
\end{align*}

where we used Lemma \ref{lem:ball} on the last line.

We also need to analyze the rewards. Let's start with a lower bound
for the optimal reward. Consider the arm $x^{*}:=-u^{*}$. We have

\begin{align*}
\ME{\theta^{*}}r{x^{*}} & =\min_{y\in\D:\T y\theta^{*}=0}r\left(x^{*},y\right)\\
 & =\min\SC{\T vx^{*}}{\Vxxx{y_{0}}{y_{1}}v\in\D:\Rxxx{y_{0}}{y_{1}}{\T v}\Vxxx{1-\alpha}{-\alpha}{2\alpha u^{*}}=0}\\
 & =-\max\SC{\T vu^{*}}{\Vxxx{y_{0}}{y_{1}}v\in\D:(1-\alpha)y_{0}-\alpha y_{1}+2\alpha\T vu^{*}=0}
\end{align*}

On the right hand side, we're maximizing $\T vu^{*}$ under the constraint
$(1-\alpha)y_{0}-\alpha y_{1}+2\alpha\T vu^{*}=0$. This requires
minimizing $(1-\alpha)y_{0}-\alpha y_{1}$, where $y_{0}$ and $y_{1}$
are subject to the constraints $y_{0},y_{1}\geq0$ and $y_{0}+y_{1}=1$.
Therefore, we can set $y_{0}=0$ and $y_{1}=1$. There is also the
constraint $\Norm v_{2}\leq y_{1}$, but this doesn't change the conclusion
because increasing $y_{1}$ only increases the range of values for
$v$ we can use. We get

\begin{align*}
\ME{\theta^{*}}r{x^{*}} & =-\max\SC{\T vu^{*}}{\Vxxx 01v\in\D:-\alpha+2\alpha\T vu^{*}=0}\\
 & =-\max\SC{\T vu^{*}}{v\in\R^{D}:\Norm v_{2}\leq1\text{ and }\T vu^{*}=\frac{1}{2}}\\
 & =-\frac{1}{2}
\end{align*}

Now let $x\in\A$ be any arm s.t. $\T{(u^{*})}x\geq-\frac{1}{1+2\delta}$.
Suppose that the outcome is $\Nu(x)$. Then, the reward is

\begin{align*}
r\left(x,\Nu(x)\right) & =r\left(x,\Vxxx 01{-\left(\frac{1}{2}+\delta\right)x}\right)\\
 & =\T x\cdot-\left(\frac{1}{2}+\delta\right)x\\
 & =-\left(\frac{1}{2}+\delta\right)\Norm x_{2}^{2}\\
 & =-\left(\frac{1}{2}+\delta\right)
\end{align*}

Now, consider any arm $x\in\A$ and suppose the outcome is $\NuT 0(x)$.
Then, the reward is

\begin{align*}
r\left(x,\NuT 0(x)\right) & =r\left(x,\Vxxx{\alpha\left(1+\T xu^{*}\right)}{1-\alpha\left(1+\T xu^{*}\right)}{-\frac{1}{2}x}\right)\\
 & =\T x\cdot-\frac{1}{2}x\\
 & =-\frac{1}{2}\Norm x_{2}^{2}\\
 & =-\frac{1}{2}
\end{align*}

We see that, as long as $\T{(u^{*})}x\geq-\frac{1}{1+2\delta}$ and
the outcome isn't $\Ybot$, we accrue a regret of $\delta$. If the
outcome is $\Ybot$ then the reward is 0 and we hence accrue a regret
of $-\frac{1}{2}$: but this can only happen once. Once we switch
to the second case in the definition of $\nu$, we stop accruing regret.
But, we can upper bound the probability this will happen during the
first $N$ rounds. Summing up, we get the following lower bound for
expected regret:

\[
\E{}{\UR}\geq\left(\frac{1-\alpha}{1+2\alpha}\right)^{N}\left(1-\frac{N}{\sqrt{\pi(2D-1)}}\left(\frac{\sqrt{2}}{2}\pi\cdot\sqrt{\delta}\right)^{D-1}\right)N\delta-\frac{1}{2}
\]

Set $\delta$ to be

\[
\delta:=\frac{2}{\pi^{2}}\cdot N^{-\frac{2}{D-1}}
\]

Notice that this value is indeed in $(0,\frac{1}{2})$. We get

\begin{align*}
\E{}{\UR} & \geq & \left(\frac{1-\alpha}{1+2\alpha}\right)^{N}\times\\
 &  & \left(1-\frac{N}{\sqrt{\pi(2D-1)}}\left(\frac{\sqrt{2}}{2}\pi\cdot\sqrt{\frac{2}{\pi^{2}}\cdot N^{-\frac{2}{D-1}}}\right)^{D-1}\right)\times\\
 &  & N\cdot\frac{2}{\pi^{2}}\cdot N^{-\frac{2}{D-1}}-\frac{1}{2}\\
 & = & \left(\frac{1-\alpha}{1+2\alpha}\right)^{N}\times\\
 &  & \left(1-\frac{1}{\sqrt{\pi\left(2D-1\right)}}\right)\times\\
 &  & \frac{2}{\pi^{2}}\cdot N^{\frac{D-3}{D-1}}-\frac{1}{2}\\
 & = & \frac{2}{\pi^{2}}\left(1-\frac{1}{\sqrt{\pi\left(2D-1\right)}}\right)\left(\frac{1-\alpha}{1+2\alpha}\right)^{N}N^{\frac{D-3}{D-1}}-\frac{1}{2}
\end{align*}
\end{proof}

\subsection{Calculating the Parameters}
\begin{proof}
[Proof of Proposition \ref{prop:lower-s-params}]Let's start from
calculating $R$. We use Proposition \ref{prop:hyperplane-r}. In
our case, $A_{x}=I_{D+2}$ for all $x\in\A$. Hence, we can choose
$\Norm{\cdot}_{0}$ s.t. $A_{x}$ is norm-preserving. Then, $\Norm{\theta}_{0}$
is the same for all $\theta\in\H$ by the rotational symmetry acting
on the last $D$ coordinates. It follows that

\begin{align*}
R & =\frac{\max_{\theta\in\H}\Norm{\theta}_{0}}{\min_{\theta\in\H}\Norm{\theta}_{0}}\left(\max_{x\in\A}\NZ{\AO x}\right)\left(\max_{a\in\A}\NZ{\AO x^{-1}}\right)\\
 & =1\cdot\left(\max_{x\in\A}1\right)\left(\max_{a\in\A}1\right)\\
 & =1
\end{align*}

Now, let's estimate $S$. Recall that any two norms on a finite-dimensional
space are equivalent (within a multplicative constant from each other).
Since it is sufficient to estimate $S$ up to a multiplicative constant
(that can depend on $D$), we can calculate the sine w.r.t. \emph{any}
fixed norm on $\R^{D+2}$. We will use a Euclidean norm s.t. $\D$
is a right cone of unit height with base of unit radius.

Consider some $u\in\R^{D}$ s.t. $\Norm u_{2}=1$ and let $\U$ be
the linear subspace defined by

\[
\U:=\SC{y\in\Y}{\T y\Vxxx{1-\alpha}{-\alpha}{2\alpha u}=0}
\]

Then, $\U^{\flat}$ is a hyperplane inside $\mu^{-1}(1)$. It is easy
to see that this hyperplane intersects the base of $\D$ at distance
$\frac{1}{2}$ from the center and the altitude of $\D$ at distance
$\alpha$ from the base. We use this to change coordinates. Define
$\tilde{\D}$ and $\sigma$ by

\[
\tilde{\D}:=\SC{y\in\R^{D+1}}{y_{0}\geq0\text{ and }\sum_{i=1}^{D}y_{i}^{2}\leq1-y_{0}}
\]

\[
\sigma:=\SC{y\in\R^{D+1}}{y_{0}+2\alpha y_{1}=\alpha}
\]

By isometry, $\sin(\sigma,\tilde{\D})=\sin(\U^{\flat},\D)$ (where
the left-hand side is defined using using the $\ell_{2}$ norm on
$\R^{D+1}$). We will now use Proposition \ref{prop:sin-sh} to lower
bound $\sin(\sigma,\tilde{\D})$.

We need to take an infimum over $q\in\partial_{\sigma}\tilde{\D}$.
For $q_{0}>0$, the unique element of $\mathrm{SH}_{q}\tilde{D}$
is tangent to lateral surface of the cone $\tilde{\D}$ and forms
an angle with $\sigma$ that approaches $\frac{\pi}{4}$ as $\alpha\rightarrow0$.
For $q_{0}=0$ and $\sum_{1}^{D}q_{i}^{2}<1$, the unique element
of $\mathrm{SH}_{q}\tilde{D}$ is the hyperplane $y_{0}=0$, which
also lies in $\mathrm{SH}_{q'}\tilde{D}$ for any $q'\in\partial_{\sigma}\tilde{\D}$
with $q'_{0}=0$. Therefore, we can assume $q_{0}=0$ and $\sum_{1}^{D}q_{i}^{2}=1$.
All points like this are related by rotational symmetries that preserve
the 0th and 1st coordinates (and therefore fix $\sigma$ and $\tilde{D}$).
Therefore, it is sufficient to consider one of them, for example

\[
q:=\left[\begin{array}{c}
0\\
\\
\frac{1}{2}\\
\\
\frac{\sqrt{3}}{2}\\
\\
0\\
\vdots\\
0
\end{array}\right]
\]

The outer unit normals to the supporting hyperplanes of $\tilde{\D}$
at $q$ are given by

\[
n_{\tilde{\D}}(\beta)=\left[\begin{array}{c}
\sin\beta\\
\\
\frac{1}{2}\cos\beta\\
\\
\frac{\sqrt{3}}{2}\cos\beta\\
\\
0\\
\vdots\\
0
\end{array}\right]
\]

for $\beta\in[-\frac{\pi}{2},+\frac{\pi}{4}]$ (for $\beta=-\frac{\pi}{2}$
we get a hyperplane tangent to the base and for $\beta=+\frac{\pi}{4}$
we get a hyperplane tangent to the lateral surface).

On the other hand, the unit normal to $\sigma$ is

\[
n_{\sigma}=-\frac{1}{\sqrt{1+4\alpha^{2}}}\left[\begin{array}{c}
1\\
\\
2\alpha\\
\\
0\\
\vdots\\
0
\end{array}\right]
\]

Hence, the angle $\eta$ between $\sigma$ and the suppoting hyperplane
with parameter $\beta$ satisfies

\begin{align*}
\cos\eta(\beta) & =\T{n_{\sigma}}n_{\tilde{\D}}(\beta)\\
 & =-\frac{1}{\sqrt{1+4\alpha^{2}}}\left(\alpha\cos\beta+\sin\beta\right)
\end{align*}

Clearly, $\cos\eta(\beta)$ is maximized for $\beta=-\frac{\pi}{2}$.
The minimal value of $\eta$ therefore satisfies

\begin{align*}
\sin\eta^{*} & =\sqrt{1-\left(\frac{1}{\sqrt{1+4\alpha^{2}}}\right)^{2}}\\
 & =\frac{2\alpha}{\sqrt{1+4\alpha^{2}}}
\end{align*}

By Proposition \ref{prop:sin-sh}, we conclude that for $\alpha\ll1$,
it holds for any $u$ as above that

\[
\sin(\U^{\flat},\D)=\frac{2\alpha}{\sqrt{1+4\alpha^{2}}}
\]

It is now obvious that

\[
\liminf_{\alpha\rightarrow0}\frac{S(\alpha)}{\alpha}>0
\]
\end{proof}

\section{\label{sec:lower-r}Lower Bound for Scaling with $R$}

\subsection{Proof of the Lower Bound}

We will need the following upper bound on KL divergence between Bernoulli
distributions:
\begin{lem}
\label{lem:kl}For any $p,q\in(0,1)$, it holds that

\[
\DKL pq\leq\frac{e}{2}\left(\ln\frac{p}{1-p}-\ln\frac{q}{1-q}\right)^{2}
\]

Here, the left hand side is the Kullback-Leibler divergence between
the Bernoulli distributions corresponding to $p$ and $q$.
\end{lem}
\begin{proof}
Define $s,t\in\R$ by

\[
s:=\ln\frac{p}{1-p}
\]

\[
t:=\ln\frac{q}{1-q}
\]

We have

\[
p=\frac{1}{1+e^{-s}}
\]

\[
q=\frac{1}{1+e^{-t}}
\]

\[
1-p=\frac{e^{-s}}{1+e^{-s}}
\]

\[
1-q=\frac{e^{-t}}{1+e^{-t}}
\]

Substituting this into the definition of KL-divergence, we get

\begin{align*}
\DKL pq & =p\ln\frac{p}{q}+(1-p)\ln\frac{1-p}{1-q}\\
 & =\frac{1}{1+e^{-s}}\ln\frac{1+e^{-t}}{1+e^{-s}}+\frac{e^{-s}}{1+e^{-s}}\ln\left(\frac{1+e^{-t}}{1+e^{-s}}\cdot e^{t-s}\right)\\
 & =\frac{1}{1+e^{-s}}\ln\frac{1+e^{-t}}{1+e^{-s}}+\frac{e^{-s}}{1+e^{-s}}\ln\frac{1+e^{-t}}{1+e^{-s}}+\frac{e^{-s}}{1+e^{-s}}\ln e^{t-s}\\
 & =\ln\frac{1+e^{-t}}{1+e^{-s}}+\frac{e^{-s}}{1+e^{-s}}\cdot(t-s)
\end{align*}

Denoting $\delta:=t-s$, this can be rewritten as

\begin{equation}
\DKL pq=\ln\frac{1+e^{-s}e^{-\delta}}{1+e^{-s}}+\frac{e^{-s}}{1+e^{-s}}\cdot\delta\label{eq:kl-logit}
\end{equation}

Applying the inequality $\ln x\leq x-1$ to the first term on the
left hand side, we get

\begin{align*}
\DKL pq & \leq\frac{1+e^{-s}e^{-\delta}}{1+e^{-s}}-1+\frac{e^{-s}}{1+e^{-s}}\cdot\delta\\
 & =\frac{1+e^{-s}e^{-\delta}}{1+e^{-s}}-\frac{1+e^{-s}}{1+e^{-s}}+\frac{e^{-s}}{1+e^{-s}}\cdot\delta\\
 & =\frac{1+e^{-s}e^{-\delta}-1-e^{-s}}{1+e^{-s}}+\frac{e^{-s}}{1+e^{-s}}\cdot\delta\\
 & =\frac{e^{-s}\left(e^{-\delta}-1\right)}{1+e^{-s}}+\frac{e^{-s}}{1+e^{-s}}\cdot\delta\\
 & =\frac{e^{-s}}{1+e^{-s}}\cdot\left(e^{-\delta}+\delta-1\right)\\
 & \leq e^{-\delta}+\delta-1
\end{align*}

Observe that the function $e^{-\delta}+\delta-1$ has vanishing value
and 1st derivative 0 at $\delta=0$. Moreover, its second derivative
is $e^{-\delta}$, which is bounded by $e$ as long as $\delta\geq-1$.
Hence, for $\delta\geq-1$, we have

\begin{align*}
\DKL pq & \leq e^{-\delta}+\delta-1\\
 & \leq\frac{1}{2}\cdot e\cdot\delta^{2}\\
 & =\frac{e}{2}\left(\ln\frac{p}{1-p}-\ln\frac{q}{1-q}\right)^{2}
\end{align*}

It remains to establish the claim for $\delta<-1$. In this case,
equation (\ref{eq:kl-logit}) implies

\begin{align*}
\DKL pq & =\ln\frac{1+e^{-s}e^{-\delta}}{1+e^{-s}}+\frac{e^{-s}}{1+e^{-s}}\cdot\delta\\
 & \leq\ln\frac{1+e^{-s}e^{-\delta}}{1+e^{-s}}\\
 & \leq\ln\frac{e^{-\delta}+e^{-s}e^{-\delta}}{1+e^{-s}}\\
 & =\ln\left(\frac{1+e^{-s}}{1+e^{-s}}\cdot e^{-\delta}\right)\\
 & =\ln e^{-\delta}\\
 & =-\delta\\
 & \leq\delta^{2}\\
 & =\left(\ln\frac{p}{1-p}-\ln\frac{q}{1-q}\right)^{2}
\end{align*}
\end{proof}
Once again, the proof works by constructing a nature policy which
produces constant regret per round while revealing little about the
true hypothesis $\theta^{*}$, with high probability. This nature
policy produces a Bernoulli distribution between two outcomes, s.t.
the straight line segment connecting them intersects the line of admissible
expected outcomes. For large values of $\lambda$, the parameter of
the Bernoulli distribution depends only weakly on $\theta^{*}$, unless
the selected arm is very close to optimal. This allows us to show
that the agent can do no better than ``exhaustive search'' over
arms with granularity $O(\lambda^{-1}\sqrt{N}$). The expected regret
is hence no less than the minimum between $N$ and the number of steps
required for this search, i.e. $\Omega(\lambda N^{-\frac{1}{2}})$.
\begin{proof}
[Proof of Theorem \ref{thm:lower-r}]Let $\theta^{*}\in\H$ be the
true hypothesis. Fix $\psi\in\left(\frac{\pi}{4},\alpha\right)$ s.t.
$\alpha+\psi>\frac{3}{4}\pi$ and $\delta\in(0,1)$ s.t.

\begin{equation}
\delta>\frac{1}{(\lambda+1)\tan\left(\alpha+\psi-\frac{\pi}{2}\right)}\label{eq:lower-r-delta}
\end{equation}

Notice that this is possible since $\alpha+\psi\in(\frac{3}{4}\pi,\pi)$
and therefore $\tan\left(\alpha+\psi-\frac{\pi}{2}\right)>1$.

Consider the following nature policy $\nu$. As long as for every
arm $x\in\A$ selected so far we have $\Abs{\T x\theta^{*}}>\delta$,
the distribution of outcomes is supported on $\{y_{+},y_{-}\}$ where
$y_{\pm}\in\D$ are defined by

\[
y_{+}:=\Vxxx{\cos\psi}{\sin\psi}1
\]

\[
y_{-}:=\Vxxx{\cos\psi}{-\sin\psi}1
\]

\global\long\def\TTheta{\tilde{\theta}}%

Let's verify this is admissible. We know that $\theta^{*}=\sqrt{1-(\T x\theta^{*})^{2}}\cdot\theta_{\bot}+(\T x\theta^{*})x$
for some $\theta_{\bot}\in\R^{2}$ with $\T x\theta_{\bot}=0$. Define
$\TTheta:\A\rightarrow\R^{2}$ by

\[
\TTheta(x):=\left(I_{2}+\lambda x\T x\right)\theta^{*}
\]

We have

\begin{align*}
\TTheta(x) & =\left(I_{2}+\lambda x\T x\right)\left(\sqrt{1-(\T x\theta^{*})^{2}}\cdot\theta_{\bot}+(\T x\theta^{*})x\right)\\
 & =\sqrt{1-(\T x\theta^{*})^{2}}\cdot\theta_{\bot}+(\T x\theta^{*})x+\lambda(\T x\theta^{*})x\\
 & =\sqrt{1-(\T x\theta^{*})^{2}}\cdot\theta_{\bot}+\left(\lambda+1\right)(\T x\theta^{*})x
\end{align*}

Denote

\[
\beta_{0}(x):=-\arctan\frac{x_{0}}{x_{1}}
\]

That is, $\beta_{0}(x)$ is the angle between $x$ and the 1st axis.
Also, let $\beta(x)\in(\beta_{0}(x)-\pi,\beta_{0}(x))$ be the angle
between $\TTheta(x)$ and the 1st axis. In particular, $\beta_{0}(x)-\beta(x)$
is the angle between $x$ and $\TTheta(x)$. We have

\begin{align*}
\tan\left(\beta_{0}(x)-\beta(x)\right) & =\frac{\sqrt{1-(\T x\theta^{*})^{2}}}{\left(\lambda+1\right)\Abs{\T x\theta^{*}}}\\
 & \leq\frac{1}{(\lambda+1)\delta}
\end{align*}

On the other hand, by inequality (\ref{eq:lower-r-delta}),

\[
\alpha+\psi-\frac{\pi}{2}>\arctan\frac{1}{(\lambda+1)\delta}
\]

\begin{align*}
\psi & >\frac{\pi}{2}-\alpha+\arctan\frac{1}{(\lambda+1)\delta}\\
 & \geq\frac{\pi}{2}-\alpha+\beta_{0}(x)-\beta(x)
\end{align*}

Since $\Abs{\beta_{0}(x)}\leq\frac{\pi}{2}-\alpha$, we get $-\psi<\beta(x)\leq\frac{\pi}{2}-\alpha<+\psi$.
Therefore, $F_{x\theta^{*}}^{\flat}$ is a line which has an angle
at most $\psi$ with the 0th axis (since $\TTheta(x)$ is normal to
it) and hence indeed intersects the segment between $y_{-}$ and $y_{+}$.

\global\long\def\Ystar{y^{*}}%

\global\long\def\Sgn{\mathrm{sgn\:}}%

Once on some round we had $\Abs{\T x\theta^{*}}\leq\delta$, the outcome
is $(\Sgn\Ystar(x)_{0})\cdot\Ystar(x)$, where $\Ystar:\A\rightarrow\D$
is defined by

\global\long\def\Mxxxyy#1#2#3#4#5#6{\left[\begin{array}{cc}
#1 & #2\\
#3 & #4\\
#5 & #6
\end{array}\right]}%

\[
\Ystar(x):=\Mxxxyy 0{-1}1000\frac{\TTheta(x)}{\Norm{\TTheta(x)}_{2}}+\Vxxx 001
\]

\global\long\def\Dist#1{\mathfrak{\Pol\nu}^{#1\delta}}%

\global\long\def\DistZ#1{\mathfrak{\Pol\nu}_{0}^{#1\delta}}%

Let $\Dist{\theta^{*}}\in\Delta((\A\times\{y_{+},y_{-}\})^{\leq N})$
be the distribution over histories resulting from following agent
policy $\Pol$ and nature policy $\nu$, assuming the true hypothesis
$\theta^{*}$ and assuming that once $\Abs{\T x\theta^{*}}\leq\delta$
happens on some round, or we reach round $N$, the sequence stops.
On a round on which arm $x\in\A$ was selected and the sequence didn't
stop, the probability of $y_{+}$ is

\[
p(x):=\frac{1}{2}\left(1+\frac{\tan\beta(x)}{\tan\psi}\right)
\]

Let $\DistZ{\theta^{*}}$ be a similar distribution, except that the
probability of $y_{+}$ on a round in which $x\in\A$ was selected
and the sequence didn't stop is

\[
p_{0}(x):=\frac{1}{2}\left(1+\frac{\tan\beta_{0}(x)}{\tan\psi}\right)
\]

Now, let's bound the Kullback-Leibler divergence between $p_{0}$
and $p$. By Lemma \ref{lem:kl}

\begin{align*}
\DKL p{p_{0}} & \leq\frac{e}{2}\left(\ln\frac{p}{1-p}-\ln\frac{p_{0}}{1-p_{0}}\right)^{2}\\
 & =\frac{e}{2}\left(\ln\frac{\frac{1}{2}\left(1+\frac{\tan\beta(x)}{\tan\psi}\right)}{\frac{1}{2}\left(1-\frac{\tan\beta(x)}{\tan\psi}\right)}-\ln\frac{\frac{1}{2}\left(1+\frac{\tan\beta_{0}(x)}{\tan\psi}\right)}{\frac{1}{2}\left(1-\frac{\tan\beta_{0}(x)}{\tan\psi}\right)}\right)^{2}\\
 & =\frac{e}{2}\left(\ln\frac{\tan\psi+\tan\beta(x)}{\tan\psi-\tan\beta(x)}-\ln\frac{\tan\psi+\tan\beta_{0}(x)}{\tan\psi-\tan\beta_{0}(x)}\right)^{2}
\end{align*}

Observe that

\begin{align*}
\frac{\mathrm{d}}{\mathrm{d}t}\ln\left(\tan\psi+\tan t\right) & =\frac{1}{\tan\psi+\tan t}\cdot\frac{1}{\cos^{2}t}\\
 & =\frac{1}{\tan\psi\cos t+\sin t}\cdot\frac{1}{\cos t}\\
 & =\frac{\cos\psi}{\sin\psi\cos t+\cos\psi\sin t}\cdot\frac{1}{\cos t}\\
 & =\frac{\cos\psi}{\sin\left(\psi+t\right)\cos t}
\end{align*}

It follows that

\begin{align*}
\frac{\mathrm{d}}{\mathrm{d}t}\ln\frac{\tan\psi+\tan t}{\tan\psi-\tan t} & =\frac{\mathrm{d}}{\mathrm{d}t}\left(\ln\left(\tan\psi+\tan t\right)-\ln\left(\tan\psi-\tan t\right)\right)\\
 & =\frac{\cos\psi}{\cos t}\left(\frac{1}{\sin\left(\psi+t\right)}+\frac{1}{\sin\left(\psi-t\right)}\right)
\end{align*}

This is a non-negative even function, and in the interval $(-\psi,+\psi)$
it monotonically increases with $\Abs t$. Indeed, this is obviously
true for $\frac{1}{\cos t},$ whereas for the other factor we have

\begin{align*}
\frac{\mathrm{d}}{\mathrm{d}t}\left(\frac{1}{\sin\left(\psi+t\right)}+\frac{1}{\sin\left(\psi-t\right)}\right) & =-\frac{\cos\left(\psi+t\right)}{\sin^{2}\left(\psi+t\right)}+\frac{\cos\left(\psi-t\right)}{\sin^{2}\left(\psi-t\right)}
\end{align*}

The right hand side vanishes at $t=0$ and it's easy to see it's monotonically
increasing for $t\in(-\psi,+\psi)$ by examining each factor of each
term individually (for $t>\frac{\pi}{2}-\psi$, $\sin\left(\psi+t\right)$
starts decreasing but $\cos\left(\psi+t\right)$ changes sign, so
the first term is still increasing). Hence, the function under the
derivative decreases for $t<0$, reaches a minimum at $t=0$ and increases
for $t>0$.

\global\long\def\BM{\beta_{\max}}%

Denote

\[
\BM:=\frac{\pi}{2}-\alpha+\arctan\frac{1}{(\lambda+1)\delta}
\]

For $t\in[-\BM,+\BM]$:

\begin{align*}
\frac{\mathrm{d}}{\mathrm{d}t}\ln\frac{\tan\psi+\tan t}{\tan\psi-\tan t} & =\frac{\cos\psi}{\cos t}\left(\frac{1}{\sin\left(\psi+t\right)}+\frac{1}{\sin\left(\psi-t\right)}\right)\\
 & \leq\frac{\cos\psi}{\cos\BM}\left(\frac{1}{\sin\left(\psi+\BM\right)}+\frac{1}{\sin\left(\psi-\BM\right)}\right)\\
 & \leq\frac{1}{\sin\left(\psi+\BM\right)}+\frac{1}{\sin\left(\psi-\BM\right)}
\end{align*}

We know that $-\beta_{0}(x),-\beta(x)\in[-\BM,+\BM]$. Hence

\begin{align*}
\DKL p{p_{0}} & \leq\frac{e}{2}\left(\ln\frac{\tan\psi+\tan\beta(x)}{\tan\psi-\tan\beta(x)}-\ln\frac{\tan\psi+\tan\beta_{0}(x)}{\tan\psi-\tan\beta_{0}(x)}\right)^{2}\\
 & \leq\frac{e}{2}\left(\left(\frac{1}{\sin\left(\psi+\BM\right)}+\frac{1}{\sin\left(\psi-\BM\right)}\right)\left(\beta_{0}(x)-\beta(x)\right)\right)^{2}\\
 & \leq\frac{e}{2}\left(\left(\frac{1}{\sin\left(\psi+\BM\right)}+\frac{1}{\sin\left(\psi-\BM\right)}\right)\arctan\frac{1}{(\lambda+1)\delta}\right)^{2}
\end{align*}

\global\long\def\DM{D_{\max}}%

Denote the right hand side $\DM$. Iteratively applying the chain
rule for KL-divergence (see e.g. \cite{Cover2006}, Theorem 2.5.3),
we get

\[
\DKL{\Dist{\theta^{*}}}{\DistZ{\theta^{*}}}\leq N\DM
\]

\global\long\def\DistA{\Dist{}}%

\global\long\def\DistZA{\DistZ{}}%

Denote $\DistA:=\E{\theta\sim\xi}{\Dist{\theta}}$ and $\DistZA:=\E{\theta\sim\xi}{\DistZ{\theta}}$.
By convexity of KL-divergence, we have

\[
\DKL{\DistA}{\DistZA}\leq N\DM
\]

By Pinsker's inequality, this implies a bound on the total variation
distance between $\DistA$ and $\DistZA$:

\[
\DTV{\DistA}{\DistZA}\leq\sqrt{\frac{1}{2}N\DM}
\]

Observe that, as long as $\Abs{\T x\theta^{*}}>\delta$ held on all
rounds, the reward is $r(x,y_{+})=r(x,y_{-})=-\cos\psi$. On the other
hand, consider the arm

\[
x^{*}:=\Mxx 0{-1}10\theta^{*}
\]

Notice that $\T{(x^{*})}\theta^{*}=0$. We have,

\begin{align*}
\ME{\theta^{*}}r{x^{*}} & =\min_{y\in F_{x\theta^{*}}}r\left(x^{*},y\right)\\
 & =\min\SC{-y_{0}}{y\in\R^{2}:\Norm y_{2}^{2}\leq1\text{ and }\T y\left(I_{2}+\lambda x^{*}\T{\left(x^{*}\right)}\right)\theta^{*}=0}\\
 & =-\max\SC{y_{0}}{y\in\R^{2}:\Norm y_{2}^{2}\leq1\text{ and }\T y\theta^{*}+\lambda\left(\T{\left(x^{*}\right)}\theta^{*}\right)x^{*}=0}\\
 & =-\max\SC{y_{0}}{y\in\R^{2}:\Norm y_{2}^{2}\leq1\text{ and }\T y\theta^{*}=0}\\
 & =-\Abs{\theta_{1}^{*}}\\
 & \geq-\cos\alpha
\end{align*}

As long as $\Abs{\T x\theta^{*}}>\delta$ held on all rounds, we accrue
regret per round of at least $\cos\psi-\cos\alpha$. Once $\Abs{\T x\theta^{*}}\leq\delta$
held on some round, it's easy to see that regret per round is non-negative.

For histories sampled out of $\DistZA$, we gain no information about
$\theta^{*}$ from the outcomes, other than the fact the condition
$\Abs{\T x\theta^{*}}>\delta$ held so far. Each round eliminates
at most $2\arcsin\delta$ radians from $\H$, whose total angular
size is $2(\frac{\pi}{2}-\alpha)$. Hence, the expected number of
rounds this condition holds, $N_{0}$, satisfies

\begin{alignat*}{1}
N_{0} & \geq\min\left(\frac{1}{2}\cdot\frac{2\left(\frac{\pi}{2}-\alpha\right)}{2\arcsin\delta}-1,N\right)\\
 & =\min\left(\frac{\frac{\pi}{2}-\alpha}{2\arcsin\delta}-1,N\right)
\end{alignat*}

Using the fact that $\arcsin\delta\leq\frac{\pi}{2}\cdot\delta$,
this becomes

\[
N_{0}\geq\min\left(\frac{\frac{\pi}{2}-\alpha}{\pi\delta}-1,N\right)
\]

Using the bound on $\DTV{\DistA}{\DistZA}$, we conclude that

\begin{equation}
\E{}{\UR}\geq\left(1-\sqrt{\frac{1}{2}N\DM}\right)\min\left(\frac{\frac{\pi}{2}-\alpha}{\pi\delta}-1,N\right)\cdot\left(\cos\psi-\cos\alpha\right)\label{eq:reg-delta-psi}
\end{equation}

Now set $\psi:=\frac{3}{8}\pi$ and

\[
\delta:=\frac{1}{\left(\lambda+1\right)\tan\frac{0.2}{\sqrt{N}}}
\]

Observe that
\begin{align*}
\arctan\frac{1}{(\lambda+1)\delta} & =\frac{0.2}{\sqrt{N}}\\
 & \leq0.2\\
 & <\alpha-0.95
\end{align*}

In particular, inequality (\ref{eq:lower-r-delta}) is satisfied because

\begin{align*}
\alpha+\psi-\frac{\pi}{2} & =\alpha-\frac{\pi}{8}\\
 & >\alpha-0.95
\end{align*}

Moreover, we get

\begin{align*}
\BM & <\frac{\pi}{2}-\alpha+\alpha-0.95\\
 & =\frac{\pi}{2}-0.95
\end{align*}

It follows that

\begin{align*}
\DM & =\frac{e}{2}\left(\left(\frac{1}{\sin\left(\frac{3}{8}\pi+\BM\right)}+\frac{1}{\sin\left(\frac{3}{8}\pi-\BM\right)}\right)\arctan\frac{1}{(\lambda+1)\delta}\right)^{2}\\
 & \leq\frac{e}{2}\left(\left(\frac{1}{\sin\left(\frac{3}{8}\pi\right)}+\frac{1}{\sin\left(\frac{3}{8}\pi-\left(\frac{\pi}{2}-0.95\right)\right)}\right)\frac{0.2}{\sqrt{N}}\right)^{2}\\
 & <\frac{1}{2N}
\end{align*}

By inequality (\ref{eq:reg-delta-psi}), we conclude

\begin{align*}
\E{}{\UR} & \geq & \left(\cos\frac{3}{8}\pi-\cos\alpha\right)\left(1-\sqrt{\frac{1}{2}N\DM}\right)\min\left(\frac{\frac{\pi}{2}-\alpha}{\pi\delta}-1,N\right)\\
 & \geq & \left(\cos\frac{3}{8}\pi-\cos\alpha\right)\left(1-\sqrt{\frac{1}{2}N\cdot\frac{1}{2N}}\right)\times\\
 &  & \min\left(\left(\frac{1}{2}-\frac{\alpha}{\pi}\right)\left(\lambda+1\right)\tan\frac{0.2}{\sqrt{N}}-1,N\right)\\
 & = & \frac{1}{2}\left(\cos\frac{3}{8}\pi-\cos\alpha\right)\min\left(\left(\frac{1}{2}-\frac{\alpha}{\pi}\right)\left(\lambda+1\right)\tan\frac{0.2}{\sqrt{N}}-1,N\right)
\end{align*}
\end{proof}

\subsection{Calculating the Parameters}
\begin{proof}
[Proof of Proposition~\ref{prop:lower-r-params}]The absolute convex
hull of $\D$ is equal to the convex hull of $[-1,+1]\cdot\D$, which
is given by

\[
\mathrm{ACH}(\D)=\SC{y\in\R^{3}}{\Abs{y_{2}}\leq1\text{ and }y_{0}^{2}+y_{1}^{2}\leq1}
\]

It follows that the norm on $\Y$ is

\[
\Norm y=\max\left(\Abs{y_{2}},\sqrt{y_{0}^{2}+y_{1}^{2}}\right)
\]

Now, let's examine the standard norm $\Norm{\cdot}_{\W}$ on $\W$.
Since $\W\cong\R$, it is enough to compute $\Norm 1_{\W}$:

\begin{align*}
\Norm 1_{\W} & =\max_{\SUBSTACK{x\in\A}{\theta\in\H}}\min_{y\in\Y:F_{x\theta}y=1}\lVert y\rVert\\
 & =\max_{\SUBSTACK{x\in\A}{\theta\in\H}}\min_{y\in\R^{3}:\T{\theta}\left(I_{2}+\lambda x\T x\right)\Vxx{y_{0}}{y_{1}}=1}\max\left(\Abs{y_{2}},\sqrt{y_{0}^{2}+y_{1}^{2}}\right)\\
 & =\max_{\SUBSTACK{x\in\A}{\theta\in\H}}\min_{y\in\R^{2}:\T{\theta}\left(I_{2}+\lambda x\T x\right)y=1}\Norm y_{2}\\
 & =\max_{\SUBSTACK{x\in\A}{\theta\in\H}}\frac{1}{\Norm{\left(I_{2}+\lambda x\T x\right)\theta}_{2}}
\end{align*}

The last line can be obtained from Lemma \ref{lem:min-dual}.

Notice that $\Norm{\left(I_{2}+\lambda x\T x\right)\theta}_{2}\geq\Norm{\theta}_{2}=1$,
and we get equality if we choose $x$ orthogonal to $\theta$. Hence,

\begin{align*}
\Norm 1_{\W} & =1
\end{align*}

We can now bound $R$:

\begin{align*}
R & =\max_{\theta\in\H}\Norm{\theta}\\
 & =\max_{\theta\in\H}\max_{x\in\A}\max_{y\in\R^{3}:\max\left(\Abs{y_{2}},\sqrt{y_{0}^{2}+y_{1}^{2}}\right)\leq1}\Abs{\T{\theta}\left(I_{2}+\lambda x\T x\right)\Vxx{y_{0}}{y_{1}}}\cdot\Norm 1_{\W}\\
 & =\max_{\theta\in\H}\max_{x\in\A}\max_{y\in\R^{2}:\Norm y_{2}\leq1}\Abs{\T{\theta}\left(I_{2}+\lambda x\T x\right)y}\\
 & =\max_{\theta\in\H}\max_{x\in\A}\Norm{\left(I_{2}+\lambda x\T x\right)\theta}_{2}\\
 & \leq\max_{\theta\in\H}\max_{x\in\A}\left(\Norm{I_{2}\theta}_{2}+\Norm{\lambda x\T x\theta}_{2}\right)\\
 & =\max_{\theta\in\H}\max_{x\in\A}\left(\Norm{\theta}_{2}+\lambda\Abs{\T x\theta}\cdot\Norm x_{2}\right)\\
 & \leq\max_{\theta\in\H}\max_{x\in\A}\left(\Norm{\theta}_{2}+\lambda\Norm x_{2}^{2}\Norm{\theta}_{2}\right)\\
 & =1+\lambda
\end{align*}

Finally, for any $\theta\in\H$ and $x\in\A$, $K_{\theta}(x)^{\flat}$
is a line passing from the center of the disk $\D$ (it's the unique
such line with normal $(I_{2}+\lambda x\T x)\theta$). By Proposition
\ref{prop:ball-s}, we have $\sin(K_{\theta}(x)^{\flat},\D)=1$. Hence,
$S=1$.
\end{proof}

\section{\label{sec:simplex}The Simplex Case}

The idea behind the proof of Theorem \ref{thm:simplex} is replacing
Lemma \ref{lem:concentration} in the derivation of Theorem \ref{thm:main}
by a different bound for the simplex case.

The following lemma implies that when $\D=\Delta\B$, the norm on
$\Y=\R^{\B}$ is the $\ell_{1}$ norm. This fact will be used repeatedly
in this appendix, in Appendix \ref{sec:s} and in Appendix \ref{sec:prob-sys}.
\begin{lem}
\label{lem:l1}For any finite set $\B$, the absolute convex hull
of $\Delta\B$ in $\R^{\B}$ is the unit ball of the $\ell_{1}$ norm.
\end{lem}
\begin{proof}
Let $y$ be in the absolute convex hull of $\Delta\B$. Then, there
are $n\in\N$, $\{c_{i}\in\R\}_{i<n}$ and $\{y_{i}\in\Delta\B\}_{i<n}$
s.t.

\[
\sum_{i=0}^{n-1}\Abs{c_{i}}\leq1
\]

\[
y=\sum_{i=0}^{n-1}c_{i}y_{i}
\]

It follows that

\begin{align*}
\Norm y_{1} & =\Norm{\sum_{i=0}^{n-1}c_{i}y_{i}}_{1}\\
 & \leq\sum_{i=0}^{n-1}\Norm{c_{i}y_{i}}_{1}\\
 & =\sum_{i=0}^{n-1}\Abs{c_{i}}\cdot\Norm{y_{i}}_{1}\\
 & =\sum_{i=0}^{n-1}\Abs{c_{i}}\\
 & \leq1
\end{align*}

Here, we used that $\Norm{y_{i}}=1$ since $y_{i}\in\Delta\B$. 

Conversly, let $y\in\R^{\B}$ be s.t. $\Norm y_{1}\leq1$. Let $y^{-},y^{+}\in\Delta\B$
be s.t. for any $a\in\B$

\[
y_{a}^{-}:=\frac{\min\left(y_{a},0\right)}{\sum_{b\in\B}\min\left(y_{b},0\right)}
\]

\[
y_{a}^{+}:=\frac{\max\left(y_{a},0\right)}{\sum_{b\in\B}\max\left(y_{b},0\right)}
\]

Here, if the denominator vanishes then we choose any element of $\Delta\B$
arbitrarily. We get,

\begin{align*}
y_{a} & =\min\left(y_{a},0\right)+\max\left(y_{a},0\right)\\
 & =\sum_{b\in\B}\min\left(y_{b},0\right)y_{a}^{-}+\sum_{b\in\B}\max\left(y_{b},0\right)y_{a}^{+}
\end{align*}

Hence,

\[
y=\sum_{b\in\B}\min\left(y_{b},0\right)\cdot y^{-}+\sum_{b\in\B}\max\left(y_{b},0\right)\cdot y^{+}
\]

Moreover,

\begin{align*}
\Abs{\sum_{b\in\B}\min\left(y_{b},0\right)}+\Abs{\sum_{b\in\B}\max\left(y_{b},0\right)} & \leq\sum_{b\in\B}\Abs{\min\left(y_{b},0\right)}+\sum_{b\in\B}\Abs{\max\left(y_{b},0\right)}\\
 & =\sum_{b\in\B}\left(\Abs{\min\left(y_{b},0\right)}+\Abs{\max\left(y_{b},0\right)}\right)\\
 & =\sum_{b\in\B}\Abs{y_{b}}\\
 & =\Norm y_{1}\\
 & \leq1
\end{align*}

Therefore, $y$ is in the absolute convex hull of $\Delta\B$.
\end{proof}
The following is an elementary topology fact which will be needed
later.
\begin{lem}
\label{lem:dec-max}Let $X$ be a non-empty compact topological space
and $f,v:X\rightarrow\R$ continuous functions. Denote $f^{*}:=\max f$.
Assume that for any $x\in X$, if $f(x)=f^{*}$ then $v(x)<0$. Then,
there exists $\epsilon>0$ s.t.

\[
\max\left(f+\epsilon v\right)<f^{*}
\]
\end{lem}
\begin{proof}
Define

\[
M:=\SC{x\in X}{f(x)=f^{*}}
\]

Notice that $M$ is a non-empty closed set. Define also,

\[
v_{-}^{*}:=-\max_{x\in M}v(x)
\]

\[
v_{+}^{*}=\max v
\]

\[
U:=\SC{x\in X}{v(x)<-\frac{1}{2}v_{-}^{*}}
\]

If $v_{+}^{*}<0$, we set $\epsilon:=1$ and get

\begin{align*}
\max\left(f+\epsilon v\right) & \leq f^{*}+v_{+}^{*}\\
 & <f^{*}
\end{align*}

If $v_{+}^{*}\geq0$ then $U\ne X$. Since $U$ is open, $X\setminus U$
is closed and we can define

\[
\delta:=\min_{x\in X\setminus U}\left(f^{*}-f(x)\right)
\]

For any $x\in M$, $v(x)\leq-v_{-}^{*}$ and hence $x\in U$\@. Therefore,
for any $x\in X\setminus U$, it holds that $x\not\in M$ and hence
$f(x)<f^{*}$. It follows that $\delta>0$. 

If $v_{+}^{*}=0$, we set $\epsilon:=1$. Then, for any $x\in X\setminus U$
we have

\begin{align*}
f(x)+\epsilon v(x) & \leq f(x)\\
 & \leq f^{*}-\delta
\end{align*}

If $v_{+}^{*}>0$, we set

\[
\epsilon:=\frac{\delta}{2v_{+}^{*}}
\]

Then, for any $x\in X\setminus U$ we have

\begin{align*}
f(x)+\epsilon v(x) & =f(x)+\frac{\delta}{2v_{+}^{*}}\cdot v(x)\\
 & \leq f(x)+\frac{\delta}{2v_{+}^{*}}\cdot v_{+}^{*}\\
 & =f(x)+\frac{\delta}{2}\\
 & \leq f^{*}-\delta+\frac{\delta}{2}\\
 & =f^{*}-\frac{\delta}{2}
\end{align*}

In either of the cases in which $v_{+}^{*}\geq0$, for any $x\in U$
we have

\begin{align*}
f(x)+\epsilon v(x) & <f^{*}+\epsilon\left(-\frac{1}{2}v_{-}^{*}\right)\\
 & =f^{*}-\frac{1}{2}\epsilon v_{-}^{*}
\end{align*}

We conclude

\begin{align*}
\max\left(f+\epsilon v\right) & \leq f^{*}-\frac{1}{2}\cdot\min\left(\delta,\epsilon v_{-}^{*}\right)\\
 & <f^{*}
\end{align*}
\end{proof}
To derive the concentration bound we need, we will need to express
the $\ell_{1}$ distance to an affine subspace of $\mu^{-1}(1)$ as
a maximum over a sufficiently small set of affine functions. This
is achieved as follows.

\global\long\def\Fun{\mathfrak{\mathcal{F}}}%

\begin{lem}
\label{lem:dist-polytope}Let $\B$ be a finite set, $\U\subseteq\R^{\B}$
a linear subspace and $y^{*}\in\R^{\B}$. Define $\mu\in\R^{\B}$
and $\Fun\subseteq\R^{\B}$ by

\[
\forall a\in\B:\mu_{a}=1
\]

\[
\Fun:=\SC{f\in[-1,+1]^{\B}}{\exists c\in\R\forall y\in\U:\T fy=c\T{\mu}y}
\]

Denote $\U^{\flat}:=\U\cap\{y\mid\mu^{t}y=1\}$ and assume $\U^{\flat}\ne\varnothing$.
Assume also that $\T{\mu}y^{*}=1$. For any $f\in\Fun$, let $c_{f}\in\R$
be the unique number s.t.

\[
\forall y\in\U:\T fy=c_{f}\T{\mu}y
\]

Then,

\[
d_{1}\left(y^{*},\U^{\flat}\right)=\max_{f\in\Fun}\left(\T fy^{*}-c_{f}\right)
\]

Here, $d_{1}$ stands for minimal $\ell_{1}$ distance.
\end{lem}
\begin{proof}
Let $\tilde{y}\in\U^{\flat}$ be s.t.

\[
\Norm{y^{*}-\tilde{y}}_{1}=d_{1}\left(y^{*},\U^{\flat}\right)
\]

The $\ell_{\infty}$ norm is dual to the $\ell_{1}$ norm, and therefore

\[
\Norm{y^{*}-\tilde{y}}_{1}=\max_{f\in[-1,+1]^{\B}}\T f\left(y^{*}-\tilde{y}\right)
\]

Define

\[
M:=\SC{f\in[-1,+1]^{\B}}{\T f\left(y^{*}-\tilde{y}\right)=\Norm{y^{*}-\tilde{y}}_{1}}
\]

Let's show that $M\cap\Fun\ne\varnothing$. Assume to the contrary
that $M\cap\Fun=\varnothing$. Define the linear subpace $\U^{\sharp}\subseteq\R^{\B}$
by

\[
\U^{\sharp}:=\SC{f\in\R^{\B}}{\exists c\in\R\forall y\in\U:\T fy=c\T{\mu}y}
\]

Since $M\subseteq[-1,+1]^{\B}$ and $\U^{\sharp}\cap[-1,+1]^{\B}=\Fun$,
we have $M\cap\U^{\sharp}=\varnothing$. Also, $M$ is a compact convex
set. Hence, by the hyperplane separation theorem, there exists $\alpha\in\R^{\B}$
s.t.:
\begin{itemize}
\item For all $f\in\U^{\sharp}$, $\T f\alpha=0$.
\item For all $f\in M$, $\T f\alpha<0$.
\end{itemize}
By Lemma \ref{lem:dec-max}, there exists $\epsilon>0$ s.t.

\[
\max_{f\in[-1,+1]^{\B}}\left(\T f\left(y^{*}-\tilde{y}\right)+\epsilon\T f\alpha\right)<\max_{f\in[-1,+1]^{\B}}\T f\left(y^{*}-\tilde{y}\right)
\]

It follows that

\begin{align*}
\Norm{y^{*}-\left(\tilde{y}-\epsilon\alpha\right)}_{1} & <\Norm{y^{*}-\tilde{y}}_{1}
\end{align*}

Moreover,

\begin{align*}
\U^{\sharp} & =\SC{f\in\R^{\B}}{\exists c\in\R\forall y\in\U:\T{\left(f-c\mu\right)}y=0}\\
 & =\U^{\bot}+\R\mu
\end{align*}

Here, the superscript $\bot$ stands for orthogonal complement. Since
$\alpha\in(\U^{\sharp})^{\bot}$ and $\U^{\bot}\subseteq\U^{\sharp}$
it follows that $\alpha\in\U$. Hence, $\tilde{y}-\epsilon\alpha\in\U$.
Since $\alpha\in(\U^{\sharp})^{\bot}$ and $\mu\in\U^{\sharp}$, we
get $\T{\mu}\alpha=0$ and hence

\begin{align*}
\T{\mu}\left(\tilde{y}-\epsilon\alpha\right) & =\T{\mu}\tilde{y}-\epsilon\T{\mu}\alpha\\
 & =1-0\\
 & =1
\end{align*}

We conclude that $\tilde{y}-\epsilon\alpha\in\U^{\flat}$. This is
a contradiction with the choice of $\tilde{y}$.

Having established that $M\cap\Fun\ne\varnothing$, we get

\begin{align*}
d_{1}\left(y^{*},\U^{\flat}\right) & =\Norm{y^{*}-\tilde{y}}_{1}\\
 & =\max_{f\in[-1,+1]^{\B}}\T f\left(y^{*}-\tilde{y}\right)\\
 & =\max_{f\in[-1,+1]^{\B}}\left(\T fy^{*}-\T f\tilde{y}\right)\\
 & =\max_{f\in\Fun}\left(\T fy^{*}-\T f\tilde{y}\right)\\
 & =\max_{f\in\Fun}\left(\T fy^{*}-c_{f}\T{\mu}\tilde{y}\right)\\
 & =\max_{f\in\Fun}\left(\T fy^{*}-c_{f}\right)
\end{align*}
\end{proof}
The desired concentration bound now follows from the previous lemma
by applying the Azuma-Hoeffding inequality to each $f\in\PSet$ and
using a union bound.
\begin{lem}
\label{lem:simplex-concentration}Suppose that arm $x\in\A$ is selected
$\tau\in\N$ times in a row. Denote the average outcome by $\bar{y}\in\D$
and let $\delta>0$. Then,

\[
\Pr\left[\DY{\bar{y}}{K_{\theta^{*}}(x)^{\flat}}\geq\delta\right]\leq\left(\frac{2e\Abs{\B}}{D_{W}+1}\right)^{D_{W}+1}\exp\left(-\frac{1}{2}\tau\delta^{2}\right)
\]
\end{lem}
\begin{proof}
Let the polytope $\Fun$ be defined as in Lemma \ref{lem:dist-polytope}
for $\U:=K_{\theta^{*}}(x)$. Denote its vertices $\{f_{i}\}_{i<k}$,
where $k\in\N$ is their number. Let $T$ be the time our sequence
of rounds begins, so that for each $n<\tau$, $y_{T+n}$ is the outcome
of the $n$-th round in the sequence. Then, $y_{T+n}$ is an $\R^{\B}$-valued
random variable whose expected value lies in $\U^{+}$. Therefore,
for any $i<k$, $\T{f_{i}}y_{T+n}$ is a random variable whose expected
value is $c_{f}$ (the latter is defined as in Lemma \ref{lem:dist-polytope}).
Moreover,

\begin{align*}
\Abs{\T fy_{T+n}} & \leq\Norm f_{\infty}\Norm{y_{T+n}}_{1}\\
 & =1\cdot1\\
 & =1
\end{align*}

Using the Azuma-Hoeffding inequality, we get

\[
\Pr\left[\T{f_{i}}\bar{y}-c_{f}\geq\delta\right]\leq\exp\left(-\frac{1}{2}\tau\delta^{2}\right)
\]

It follows that

\begin{align*}
\Pr\left[d_{Y}\left(\bar{y},\U^{\flat}\right)\geq\delta\right] & =\Pr\left[d_{1}\left(\bar{y},\U^{\flat}\right)\geq\delta\right]\\
 & =\Pr\left[\max_{f\in\Fun}\left(\T f\bar{y}-c_{f}\right)\geq\delta\right]\\
 & =\Pr\left[\max_{i<k}\left(\T{f_{i}}\bar{y}-c_{f}\right)\geq\delta\right]\\
 & \leq\sum_{i<k}\Pr\left[\T{f_{i}}\bar{y}-c_{f}\geq\delta\right]\\
 & \leq k\exp\left(-\frac{1}{2}\tau\delta^{2}\right)
\end{align*}

Here, we used Lemma \ref{lem:l1} on the 1st line, Lemma \ref{lem:dist-polytope}
on the 2nd line, and a union bound on the 4th line.

Observe that $\Fun$ is the intersection of the linear subspace $\U^{\bot}\oplus\R\mu$
and the hypercube $[-1,+1]^{\B}$. Here, the sum is direct because
$\U\cap\mu^{-1}(1)\ne\varnothing$ and hence $\mu\not\in\U^{\bot}$.
When a generic $m$-dimensional affine subspace intersects a fixed
polytope in $\R^{D}$, each vertex of the intersection lies in a unique
$D-m$-face of the polytope. Hence, the number of vertices of the
intersection is at most the number of $D-m$-faces of the polytope.
Since the number of vertices is a lower semicontinuous function of
the polytope (see subsection 5.3 in \cite{grunbaum2003convex}), the
same bound holds for \emph{any} affine subspace. We conclude that
$k$ is bounded by the number of faces of $[-1,+1]^{\B}$ of dimension\footnote{It is possible to produce a stronger bound using the upper bound theorem
for convex polytopes. However, since the bound in Theorem \ref{thm:simplex}
only depends on the \emph{logarithm} of this number, the improvement
would be mild (no better than reducing the constant coefficient).}

\begin{align*}
\Abs{\B}-\dim\left(\U^{\bot}\oplus\R\mu\right) & =\Abs{\B}-\left(\Abs{\B}-\dim\U+1\right)\\
 & =\dim\U-1\\
 & =\Abs{\B}-D_{W}-1
\end{align*}

We get

\begin{align*}
k & \leq2^{D_{W}+1}\binom{\Abs B}{D_{W}+1}\\
\\
 & \leq2^{D_{W}+1}\left(\frac{e\Abs B}{D_{W}+1}\right)^{D_{W}+1}\\
\\
 & =\left(\frac{2e\Abs B}{D_{W}+1}\right)^{D_{W}+1}
\end{align*}

It follows that

\[
\Pr\left[d_{Y}\left(\bar{y},\U^{\flat}\right)\geq\delta\right]\leq\left(\frac{2e\Abs B}{D_{W}+1}\right)^{D_{W}+1}\exp\left(-\frac{1}{2}\tau\delta^{2}\right)
\]
\end{proof}
Putting Lemma \ref{lem:simplex-concentration} and Lemma \ref{lem:generic-confidence}
together immediately gives us:
\begin{lem}
\label{lem:simplex-confidence}Assume agent policy $\IUCB^{\eta}$
and \emph{any} nature policy. Then, for any $N\in\N$

\begin{align*}
\Pr\left[\exists k:T_{k}<N\text{ and }\theta^{*}\not\in\C_{k}\right] & \leq\frac{1}{2}\left(\frac{2e\Abs{\B}}{D_{W}+1}\right)^{D_{W}+1}N(N+1)\exp\left(-\frac{\eta^{2}}{2R^{2}}\right)
\end{align*}
\end{lem}
We now can complete the proof of Theorem \ref{thm:simplex}.
\begin{proof}
[Proof of Theorem \ref{thm:simplex}]We have

\[
\UR=\RgS{\BDF}+\RgS{\ADF}+\RgS{\EDF}+\RgS{\FC}
\]

By Lemma \ref{lem:simplex-confidence},

\[
\Pr\left[\FC\ne\varnothing\right]\leq\frac{1}{2}\left(\frac{2e\Abs{\B}}{D_{W}+1}\right)^{D_{W}+1}N(N+1)\exp\left(-\frac{\eta^{2}}{2R^{2}}\right)
\]

Since $\RgS{\varnothing}=0$ and $\RgS{\FC}\leq CN$, this implies

\begin{align*}
\E{}{\RgS{\FC}} & \leq\frac{1}{2}\left(\frac{2e\Abs{\B}}{D_{W}+1}\right)^{D_{W}+1}N(N+1)\exp\left(-\frac{\eta^{2}}{2R^{2}}\right)\cdot CN\\
 & =\frac{1}{2}C\left(\frac{2e\Abs{\B}}{D_{W}+1}\right)^{D_{W}+1}N^{2}(N+1)\exp\left(-\frac{\eta^{2}}{2R^{2}}\right)
\end{align*}

Combining this with Lemma \ref{lem:regret-abc}, we get,

\begin{align*}
\E{}{\UR} & \leq & \left(S^{-1}+1\right)D_{Z}\left(36D_{Z}+8\right)N\delta\\
 &  & +8\eta\left(S^{-1}+1\right)(D_{Z}+1)\sqrt{\gamma D_{Z}^{2}\ln\frac{D_{Z}R}{\delta}\cdot N}\\
 &  & +\gamma CD_{Z}^{2}\ln\frac{D_{Z}R}{\delta}\\
 &  & +\frac{1}{2}C\left(\frac{2e\Abs{\B}}{D_{W}+1}\right)^{D_{W}+1}N^{2}(N+1)\exp\left(-\frac{\eta^{2}}{2R^{2}}\right)
\end{align*}

Since this holds for any nature policy, this is a bound on $\Reg{\theta^{*}}{\IUCB^{\eta}}N$.
\end{proof}

\section{\label{sec:r}Some Bounds on $R$}

\subsection{$K(x)$ of Codimension 1}

We start with a simple fact about norms on dual spaces.
\begin{lem}
\label{lem:min-dual}Let $\X$ be a finite-dimensional normed space.
Then, for any $x\in\X$:

\[
\min_{\alpha\in\X^{\star}:\alpha(x)=1}\Norm{\alpha}=\Norm x^{-1}
\]
\end{lem}
\begin{proof}
Let $\alpha\in\X^{\star}$ be any s.t. $\alpha(x)=1$. Then, $\Norm{\alpha}\cdot\Norm x\geq1$,
implying that $\Norm{\alpha}\geq\Norm x^{-1}$ and hence

\[
\min_{\alpha\in\X^{\star}:\alpha(x)=1}\Norm{\alpha}\geq\Norm x^{-1}
\]

On the other hand, consider $\X_{0}:=\R x$. Consider the unique $\alpha_{0}\in\X_{0}^{\star}$
s.t. $\alpha_{0}(x)=1$. Then, $\Norm{\alpha_{0}}=\Norm x^{-1}$.
By the Hanh-Banach theorem, this implies there is $\alpha\in\X^{\star}$
s.t. $\Norm{\alpha}=\Norm x^{-1}$ and $\alpha|_{\X_{0}}=\alpha_{0}$,
the latter implying that $\alpha(x)=1$. Therefore,

\[
\min_{\alpha\in\X^{\star}:\alpha(x)=1}\Norm{\alpha}\leq\Norm x^{-1}
\]
\end{proof}
The norm of an operator bounds how much applying it can increase the
norm of a vector. Similarly, the norm of the inverse operator bounds
how much applying the original operator can \emph{decrease} the norm
of a vector.
\begin{lem}
\label{lem:inv-op-norm}Let $\X$ and $\X'$ be finite-dimensional
normed spaces, $x\in\X$ and $A:\X\rightarrow\X'$ an invertible operator.
Then,

\[
\Norm{Ax}\geq\frac{\Norm x}{\Norm{A^{-1}}}
\]
\end{lem}
\begin{proof}
We have,

\begin{align*}
\Norm x & =\Norm{A^{-1}Ax}\\
 & \leq\Norm{A^{-1}}\cdot\Norm{Ax}
\end{align*}

Dividing both sides of the inequality by $\Norm{A^{-1}}$ we get the
desired result. 
\end{proof}
Proposition \ref{prop:hyperplane-r} now follows from a straightforward
calculation.
\begin{proof}
[Proof of Proposition~\ref{prop:hyperplane-r}]First, let's examine
the standard norm $\Norm{\cdot}_{\W}$ on $\W$. Since $\W\cong\R$,
it is enough to compute $\Norm 1_{\W}$:

\begin{align*}
\Norm 1_{\W} & =\max_{\SUBSTACK{x\in\A}{\theta\in\H}}\min_{y\in\Y:F_{x\theta}y=1}\lVert y\rVert\\
 & =\max_{\SUBSTACK{x\in\A}{\theta\in\H}}\min_{y\in\Y:\left(A_{x}\theta\right)y=1}\lVert y\rVert\\
 & =\max_{\SUBSTACK{x\in\A}{\theta\in\H}}\Norm{A_{x}\theta}^{-1}\\
 & =\frac{1}{\min_{\SUBSTACK{x\in\A}{\theta\in\H}}\Norm{A_{x}\theta}}
\end{align*}

Here, we used the definition of $\AO x$ followed by Lemma \ref{lem:min-dual}
applied to $\X:=\Y^{\star}$.

Second, let's examine the standard norm on $\Z$. We have,

\begin{align*}
\Norm z & =\max_{x\in\A}\Norm{F_{xz}}\\
 & =\max_{x\in\A}\left(\Norm{A_{x}z}\cdot\Norm 1_{\W}\right)\\
 & =\frac{\max_{x\in\A}\Norm{A_{x}z}}{\min_{\SUBSTACK{x\in\A}{\theta\in\H}}\Norm{A_{x}\theta}}
\end{align*}

Here, we observed that the induced norm on the space of operators
from $\Y$ to $\W$ is the same as the induced norm on $\Y^{\star}$
multiplied by the scalar $\Norm 1_{\W}$.

Finally, we get,

\begin{align*}
R(\H,F) & =\max_{\theta\in\H}\Norm{\theta}\\
 & =\frac{\max_{\SUBSTACK{x\in\A}{\theta\in\H}}\Norm{A_{x}\theta}}{\min_{\SUBSTACK{x\in\A}{\theta\in\H}}\Norm{A_{x}\theta}}\\
 & \leq\frac{\max_{\SUBSTACK{x\in\A}{\theta\in\H}}\left(\Norm{\AO x}_{0}\cdot\Norm{\theta}_{0}\right)}{\min_{\SUBSTACK{x\in\A}{\theta\in\H}}\left(\Norm{\AO x^{-1}}_{0}^{-1}\cdot\Norm{\theta}_{0}\right)}\\
 & \leq\frac{\left(\max_{x\in\A}\Norm{\AO x}_{0}\right)\left(\max_{\theta\in\H}\Norm{\theta}_{0}\right)}{\left(\min_{x\in\A}\Norm{\AO x^{-1}}_{0}^{-1}\right)\left(\min_{\theta\in\H}\Norm{\theta}_{0}\right)}\\
 & =\frac{\max_{\theta\in\H}\Norm{\theta}_{0}}{\min_{\theta\in\H}\Norm{\theta}_{0}}\left(\max_{x\in\A}\NZ{\AO x}\right)\left(\max_{a\in\A}\NZ{\AO x^{-1}}\right)
\end{align*}

Here we used Lemma \ref{lem:inv-op-norm} on the third line.
\end{proof}
Proposition \ref{prop:hyperplane-rn} follows easily from Proposition
\ref{prop:hyperplane-r}.
\begin{proof}
[Proof of Proposition~\ref{prop:hyperplane-rn}]Set $\chi$ and $\lambda$
to be

\[
\chi(\theta):=\frac{1}{\Norm{\theta}_{0}}
\]

\[
\lambda(x):=\Norm{\AO x^{-1}}_{0}
\]

Notice that the bandit specified by $\H^{\chi}$ and $F^{\lambda}$
is of a form to which Proposition \ref{prop:hyperplane-r} is applicable,
with the operators $\AO x$ replaced by $\tilde{A}_{x}:=\lambda(x)A_{x}$.
Applying Proposition \ref{prop:hyperplane-r} we get

\begin{align*}
R\left(\H^{\chi},F^{\lambda}\right) & =\frac{\max_{\theta\in\H^{\chi}}\Norm{\theta}_{0}}{\min_{\theta\in\H^{\chi}}\Norm{\theta}_{0}}\left(\max_{x\in\A}\NZ{\tilde{A}_{x}}\right)\left(\max_{a\in\A}\NZ{\tilde{A}_{x}^{-1}}\right)\\
 & =\frac{\max_{\theta\in\H}\chi(\theta)\Norm{\theta}_{0}}{\min_{\theta\in\H}\chi(\theta)\Norm{\theta}_{0}}\left(\max_{x\in\A}\lambda(x)\NZ{\AO x}\right)\left(\max_{a\in\A}\lambda(x)^{-1}\NZ{\AO x^{-1}}\right)\\
 & =1\cdot\left(\max_{x\in\A}\Norm{\AO x^{-1}}_{0}\cdot\NZ{\AO x}\right)\left(\max_{a\in\A}\Norm{\AO x^{-1}}_{0}^{-1}\cdot\NZ{\AO x^{-1}}\right)\\
 & =\max_{x\in\A}\Norm{\AO x^{-1}}_{0}\cdot\NZ{\AO x}
\end{align*}
\end{proof}

\subsection{$K(x)$ of Dimension 1}
\begin{proof}
[Proof of Proposition \ref{prop:point-r}]Since $\W$ is a subspace
of $\Y$, it is equipped with two norms: the standard norm $\Norm{\cdot}_{\W}$
on $\W$ and the norm $\Norm{\cdot}_{\Y}$ induced by the norm on
$\Y$. Observe that for any $x\in\A$, $\theta\in\H$ and $w\in\W$,

\begin{align*}
F_{x\theta}w & =\psi(\theta)w-\mu(w)f(x,\theta)\\
 & =1\cdot w-0\cdot f(x,\theta)\\
 & =w
\end{align*}

It follows that

\begin{align*}
\Norm w_{\W} & =\max_{\SUBSTACK{x\in\A}{\theta\in\H}}\min_{y\in\Y:F_{x\theta}y=w}\lVert y\rVert_{\Y}\\
 & \leq\max_{\SUBSTACK{x\in\A}{\theta\in\H}}\Norm w_{\Y}\\
 & =\Norm w_{\Y}
\end{align*}

We get

\begin{align*}
R(\H,F) & =\max_{\theta\in\H}\Norm{\theta}\\
 & =\max_{\SUBSTACK{x\in\A}{\theta\in\H}}\Norm{F_{x\theta}}\\
 & =\max_{\SUBSTACK{x\in\A}{\theta\in\H}}\max_{y\in\Y:\Norm y_{\Y}=1}\Norm{F_{x\theta}y}_{\W}\\
 & \leq\max_{\SUBSTACK{x\in\A}{\theta\in\H}}\max_{y\in\Y:\Norm y_{\Y}=1}\Norm{F_{x\theta}y}_{\Y}\\
 & =\max_{\SUBSTACK{x\in\A}{\theta\in\H}}\max_{y\in\Y:\Norm y_{\Y}=1}\Norm{\psi(\theta)y-\mu(y)f(x,\theta)}_{\Y}\\
 & =\max_{\SUBSTACK{x\in\A}{\theta\in\H}}\max_{y\in\Y:\Norm y_{\Y}=1}\Norm{y-\mu(y)f(x,\theta)}_{\Y}\\
 & \leq\max_{\SUBSTACK{x\in\A}{\theta\in\H}}\max_{y\in\Y:\Norm y_{\Y}=1}\left(\Norm y_{\Y}+\Abs{\mu(y)}\cdot\Norm{f(x,\theta)}_{\Y}\right)\\
 & =\max_{\SUBSTACK{x\in\A}{\theta\in\H}}\max_{y\in\Y:\Norm y_{\Y}=1}\left(\Norm y_{\Y}+\Abs{\mu(y)}\right)\\
 & \leq\max_{\SUBSTACK{x\in\A}{\theta\in\H}}\max_{y\in\Y:\Norm y_{\Y}=1}\left(\Norm y_{\Y}+\Norm y_{\Y}\right)\\
 & =2
\end{align*}
\end{proof}

\subsection{The Dani-Hayes-Kakade Bandit}

Now, let's turn to Proposition \ref{prop:DHK-r}. First, let's describe
$\H$ for that bandit.
\begin{lem}
\label{lem:DHK-h}In the setting of Proposition \ref{prop:DHK-r},
we have

\[
\H=\SC{\theta\in\R^{2n}}{\sum_{i=0}^{n-1}\sqrt{\theta_{2i}^{2}+\theta_{2i+1}^{2}}\leq1}
\]
\end{lem}
\begin{proof}
Consider any $\theta\in\R^{2n}$ s.t. 

\[
\sum_{i=0}^{n-1}\sqrt{\theta_{2i}^{2}+\theta_{2i+1}^{2}}\leq1\,
\]

For any $x\in\A$, we have

\begin{align*}
\Abs{\T x\theta} & =\Abs{\sum_{k=0}^{2n-1}x_{k}\theta_{k}}\\
 & =\Abs{\sum_{i=0}^{n-1}\left(x_{2i}\theta_{2i}+x_{2i+1}\theta_{2i+1}\right)}\\
 & \leq\sum_{i=0}^{n-1}\Abs{x_{2i}\theta_{2i}+x_{2i+1}\theta_{2i+1}}\\
 & \leq\sum_{i=0}^{n-1}\left(\sqrt{x_{2i}^{2}+x_{2i+1}^{2}}\cdot\sqrt{\theta_{2i}^{2}+\theta_{2i+1}^{2}}\right)\\
 & =\sum_{i=0}^{n-1}\sqrt{\theta_{2i}^{2}+\theta_{2i+1}^{2}}\\
 & \leq1
\end{align*}

Therefore, $\theta\in\H$. 

Conversly, consider any $\theta\in\H$. Let $x\in\R^{2n}$ be s.t.
for any $i<n$,

\[
\Vxx{x_{2i}}{x_{2i+1}}=\frac{1}{\sqrt{\theta_{2i}^{2}+\theta_{2i+1}^{2}}}\Vxx{\theta_{2i}}{\theta_{2i+1}}
\]

Here, if the denominator vanishes, we choose an arbitrary vector of
unit norm.

Clearly, $x\in\A$. We get

\begin{align*}
1 & \geq\Abs{\T x\theta}\\
 & =\Abs{\sum_{i=0}^{n-1}\left(x_{2i}\theta_{2i}+x_{2i+1}\theta_{2i+1}\right)}\\
 & =\Abs{\sum_{i=0}^{n-1}\frac{\theta_{2i}^{2}+\theta_{2i+1}^{2}}{\sqrt{\theta_{2i}^{2}+\theta_{2i+1}^{2}}}}\\
 & =\sum_{i=0}^{n-1}\sqrt{\theta_{2i}^{2}+\theta_{2i+1}^{2}}
\end{align*}

Here, if the denominator vanishes in some term on the 3rd line then
the 4th line still holds because this term vanishes on both lines.
\end{proof}
Now we can calculate $R$:
\begin{proof}
[Proof of Proposition \ref{prop:DHK-r}]It is easy to see that the
absolute convex hull of $\D$ is $[-1,+1]^{2}$, and therefore the
norm on $\Y$ is $\ell_{\infty}$\@. Now, let's examine the standard
norm $\Norm{\cdot}_{\W}$ on $\W$. Since $\W\cong\R$, it is enough
to compute $\Norm 1_{\W}$:

\begin{align*}
\Norm 1_{\W} & =\max_{\SUBSTACK{x\in\A}{\theta\in\H}}\min_{y\in\R^{2}:F_{x\theta}y=1}\lVert y\rVert_{\infty}\\
 & =\max_{\SUBSTACK{x\in\A}{\theta\in\H}}\min\SC{\lVert y\rVert_{\infty}}{y\in\R^{2}:\left(\T x\theta\right)y_{0}-y_{1}=1}\\
 & =\max_{\SUBSTACK{x\in\A}{\theta\in\H}}\Norm{\Vxx{\T x\theta}{-1}}_{1}^{-1}\\
 & =\max_{\SUBSTACK{x\in\A}{\theta\in\H}}\frac{1}{\Abs{\T x\theta}+1}
\end{align*}

Here, we used Lemma \ref{lem:min-dual} on the 3rd line.

Taking $\theta=\boldsymbol{0}$, we get $\Norm 1_{\W}=1$. It follows
that,

\begin{align*}
R & =\max_{\theta\in\H}\Norm{\theta}\\
 & =\max_{\SUBSTACK{x\in\A}{\theta\in\H}}\Norm{F_{x\theta}}\\
 & =\max_{\SUBSTACK{x\in\A}{\theta\in\H}}\max_{y\in\R^{2}:\Norm y_{\infty}=1}F_{x\theta}y\cdot\Norm 1_{\W}\\
 & =\max_{\SUBSTACK{x\in\A}{\theta\in\H}}\max_{y\in\R^{2}:\Norm y_{\infty}=1}\left(\left(\T x\theta\right)y_{0}-y_{1}\right)\\
 & =\max_{\SUBSTACK{x\in\A}{\theta\in\H}}\left(\Abs{\T x\theta}+1\right)
\end{align*}

By definition of $\H$, $\Abs{\T x\theta}\leq1$. Conversly, we can
take $\bar{\theta}:=\frac{1}{n}x$. We then get,

\begin{align*}
\sum_{i=0}^{n-1}\sqrt{\bar{\theta}_{2i}^{2}+\bar{\theta}_{2i+1}^{2}} & =\frac{1}{n}\sum_{i=0}^{n-1}\sqrt{x_{2i}^{2}+x_{2i+1}^{2}}\\
 & =\frac{1}{n}\sum_{i=0}^{n-1}1\\
 & =1
\end{align*}

By Lemma \ref{lem:DHK-h}, this implies $\bar{\theta}\in\H$. Moreover,

\begin{align*}
\T x\bar{\theta} & =\frac{1}{n}\T xx\\
 & =\frac{1}{n}\sum_{k=0}^{2n-1}x_{k}^{2}\\
 & =\frac{1}{n}\sum_{i=0}^{n-1}\left(x_{2i}^{2}+x_{2i+1}^{2}\right)\\
 & =\frac{1}{n}\sum_{i=0}^{n-1}1\\
 & =1
\end{align*}

We conclude that the maximum of $\Abs{\T x\theta}$ is $1$ and hence,
$R=2$.
\end{proof}

\section{\label{sec:s}Some Bounds on $S$}

\subsection{Sine in Euclidean Space}
\begin{proof}
[Proof of Proposition \ref{prop:sin}]Without loss of generality,
assume $\Af$ is a vector space and $\SubAf$, $\SubC$ are linear
subspaces. Let $P:\Af\rightarrow\SubC^{\bot}$ and $Q:\Af\rightarrow(\SubAf\cap\SubC)^{\bot}$
be the orthgonal projections. Notice that $\SubC^{\bot}\subseteq(\SubAf\cap\SubC)^{\bot}$
and hence $P=PQ$. For any $u\in\SubAf\setminus\SubC$, we have

\begin{align*}
\frac{d\left(u,\SubC\right)}{d\left(u,\SubAf\cap\SubC\right)} & =\frac{\Norm{Pu}}{\Norm{Qu}}\\
 & =\frac{\Norm{PQu}}{\Norm{Qu}}\\
 & =\Norm{P\cdot\frac{Qu}{\Norm{Qu}}}
\end{align*}

It is easy to see that $Qu\in\X$. Conversly, for any $x\in\X$ s.t.
$\Norm x=1$, we can take $u:=x$ and then

\begin{align*}
\frac{Qu}{\Norm{Qu}} & =x
\end{align*}

Hence,

\[
\sin\left(\SubAf,\SubC\right)=\min_{x\in\X:\Norm x=1}\Norm{Px}
\]

Moreover, for any $x\in\X$, we have

\[
\Norm x^{2}=\Norm{Px}^{2}+\Norm{\left(I-P\right)x}^{2}
\]

Therefore, if $\Norm x=1$ then

\begin{align*}
\Norm{Px} & =\sqrt{1-\Norm{\left(I-P\right)x}^{2}}
\end{align*}

Assume $x\not\in\SubC^{\bot}$ and define $y^{*}\in\Y$ by

\[
y^{*}=\frac{\left(I-P\right)x}{\Norm{\left(I-P\right)x}}
\]

Notice that $\Norm{y^{*}}=1$. We have,

\begin{align*}
\Norm{\left(I-P\right)x}^{2} & =(I-P)x\cdot(I-P)x\\
 & =x\cdot(I-P)x\\
 & =x\cdot\Norm{(I-P)x}y^{*}\\
 & =\Norm{(I-P)x}\left(x\cdot y^{*}\right)
\end{align*}

Dividing both sides by $\Norm{(I-P)x}$, we get

\[
\Norm{(I-P)x}=x\cdot y^{*}
\]

Conversly, consider any $y\in\Y$ s.t. $\Norm y=1$. We have,

\begin{align*}
\Abs{x\cdot y} & =\Abs{x\cdot\left(I-P\right)y}\\
 & =\Abs{\left(I-P\right)x\cdot y}\\
 & \leq\Norm{(I-P)x}
\end{align*}

Therefore,

\[
\Norm{(I-P)x}=\max_{y\in\Y:\Norm y=1}\Abs{x\cdot y}
\]

This holds even when $x\in\SubC^{\bot},$ because then both sides
vanish. It follows that

\begin{align*}
\Norm{Px} & =\sqrt{1-\left(\max_{y\in\Y:\Norm y=1}\Abs{x\cdot y}\right)^{2}}\\
 & =\min_{y\in\Y:\Norm y=1}\sqrt{1-\left(x\cdot y\right)^{2}}
\end{align*}

Putting everything together, we get

\[
\sin\left(\SubAf,\SubC\right)=\min_{\SUBSTACK{x\in\X:\Norm x=1}{y\in\Y:\Norm y=1}}\sqrt{1-(x\cdot y)^{2}}
\]
\end{proof}

\subsection{Bounding Sine via Supporting Hyperplanes}

Now, we start working towards the proof of Proposition \ref{prop:sin-sh}.
We start with a lemma which can be regarded as a generalization of
the fact that, in Euclidean geometry, the nearest point in a smooth
body to some fixed point will have a tangent hyperplane orthogonal
to the line connecting the two points.

In the following, whenever $\pi$ is a hyperplane, we will denote
by $\pi^{+}$ and $\pi^{-}$ the two closed half-spaces bounded by
$\pi$ (we can arbitrarily decide which is ``positive'' and which
is ``negative'').
\begin{lem}
\label{lem:normal-hyperplane}Let $\X$ be a finite-dimensional normed
vector space, $\D\subseteq\X$ convex closed and $v\in\X\setminus\D$.
Let

\[
u^{*}:=\Argmin{u\in\D}\Norm{v-u}
\]

Then, there exists a hyperplane $\pi\subset\X$ s.t. all of the following
conditions hold:
\begin{itemize}
\item $u^{*}\in\pi$
\item $\D\subseteq\pi^{+}$ (i.e. $\pi\in\SH{u^{*}}$)
\item $v\in\pi^{-}$
\item $d\left(v,\pi\right)=d\left(v,\D\right)$
\end{itemize}
\end{lem}
\begin{proof}
Let $B$ be the open ball with center at $v$ and radius $d\left(v,\D\right)$.
Clearly, $B\cap\D=\varnothing$. Hence, we can apply the hyperplane
separation theorem to $B$ and $\D$ to obtain a hyperplane $\pi\subset\X$
s.t. $B\subseteq\pi^{-}$ and $\D\subseteq\pi^{+}$. Since $u^{*}$
is both in $\D$ and in the closure of $B$, we get $u^{*}\in\pi$.
Moreover, $v\in B$ and hence $v\in\pi^{-}$.

We have,

\begin{align*}
d\left(v,\pi\right) & \leq\Norm{v-u^{*}}\\
 & =d\left(v,\D\right)
\end{align*}

On the other hand, since $B\subseteq\pi^{-}$ and $B$ is open, we
have $B\cap\pi=\varnothing$. This implies

\[
d\left(v,\pi\right)\geq d\left(v,\D\right)
\]

We conclude

\[
d\left(v,\pi\right)=d\left(v,\D\right)
\]
\end{proof}
The following lemma says that, roughly speaking, any affine subspace
that intersects a convex body has to either intersect its interior
or be tangent to it.
\begin{lem}
\label{lem:interior-hyperplane}Let $\X$ be a finite-dimensional
normed vector space, $\D\subseteq\X$ convex closed and $\U_{0}\subseteq\X$
a linear subspace s.t. $\U_{0}\cap\D\ne\varnothing$. Assume that
the affine hull of $\D$ is $\X$. Let $u\in\U_{0}\cap\D$ be s.t.
for any $\pi\in\SH u$, $\U_{0}\not\subseteq\pi$. Then, $\U_{0}$
intersects the interior of $\D$.
\end{lem}
\begin{proof}
Denote by $\D_{0}$ the interior of $\D$. Assume to the contrary
that $\U_{0}\cap\D_{0}=\varnothing$. Then, the hyperplane separation
theorem implies there is a hyperplane $\pi\subset\X$ s.t. $\U_{0}\subseteq\pi^{-}$
and $\D_{0}\subseteq\pi^{+}$. The affine hull of $\D$ is $\X$,
therefore $\D$ is the closure of $\D_{0}$ and hence $\D\subseteq\pi^{+}$.
We get $u\in\pi$ and $\pi\in\SH u$. 

Now, consider any $v\in\U_{0}$. Denote $v':=2u-v$. Then, $u$ is
the midpoint between $v$ and $v'$. Since $v'\in\U_{0}$ we have
$v'\in\pi^{-}$, and since $u\in\pi$, this implies $v\in\pi^{+}$.
Since $v\in\U_{0},$ we also have $v\in\pi^{-}$ and hence $v\in\pi$.
We got $\U_{0}\subseteq\pi$ which is a contradiction.
\end{proof}
We will also use the following result, which appears in \cite{Ok2007}
as ``the Krein-Rutman theorem''\footnote{Not to be confused with the Krein-Rutman theorem about the spectrum
of a cone-preserving compact operator. Also, the statement we give
here is more restricted, since we're only concerned with finite-dimensional
spaces.}. It talks abour extending a linear functional while preserving positivity
on a convex cone, but can be equivalently interpreted as extending
an affine function while preserving positivity on a convex body.

\global\long\def\Cone{\mathcal{P}}%

\begin{thm}
[Krein-Rutman]\label{thm:krein-rutman}Let $\X$ be a finite-dimensional
vector space and $\Cone\subset\X$ a closed convex cone. Let $\U_{0}\subseteq\X$
be a linear subspace that intersects the interior of $\Cone$, and
$\alpha_{0}\in\U_{0}^{\star}$ s.t. for any $u\in\U_{0}\cap\Cone$,
$\alpha_{0}(u)\geq0$. Then, there exists $\alpha\in\X^{\star}$ s.t.
$\alpha|_{\U_{0}}=\alpha_{0}$ and for any $u\in\Cone$, $\alpha(u)\geq0$. 
\end{thm}
Now, we prove the lower bound:

\global\long\def\SHg#1#2{\mathrm{SH}_{#1}#2}%

\begin{proof}
[Proof of the lower bound in Proposition \ref{prop:sin-sh}]Without
loss of generality, assume $\Af$ is a vector space and $\SubAf$
a linear subspace. Consider any $p\in\SubAf\setminus\D$. Define $q_{0}\in\DD{\SubAf}$
by

\[
q_{0}:=\Argmin{q\in\SubAf\cap\D}\Norm{p-q}
\]

By Lemma \ref{lem:normal-hyperplane}, there exists $\pi_{0}\subset\SubAf$,
a hyperplane in $\SubAf$, s.t. $q_{0}\in\pi_{0}$, $\SubAf\cap\D\subseteq\pi_{0}^{+}$,
$p\in\pi_{0}^{-}$ and

\[
d\left(p,\pi_{0}\right)=d\left(p,\SubAf\cap\D\right)
\]

Consider two cases: 

In the first case, there exists $\pi'\in\SH{q_{0}}$ s.t. $\SubAf\subseteq\pi'$.
Then, $\sin\left(\SubAf,\pi'\right)=0$ and the lower bound is vacuously
true (because it vanishes).

\global\long\def\Ash{\Af_{1}}%

\global\long\def\Bsh{\SubAf_{1}}%

In the second case, there is no $\pi'$ as above. Then, Lemma \ref{lem:interior-hyperplane}
implies that $\SubAf$ intersects the interior of $\D$. Denote $\Ash:=\Af\oplus\R$
and $\Bsh:=\SubAf\oplus\R$. Let $\alpha_{0}\in\Bsh^{\star}$ be s.t.

\[
\ker\alpha_{0}=\SC{\Vxx{\lambda u}{\lambda}}{\lambda\in\R\text{, }u\in\pi_{0}}
\]

We choose $\alpha_{0}$'s sign s.t. for any $u\in\pi_{0}^{+}$,

\[
\alpha_{0}\Vxx u1\geq0
\]

In particular, the above holds for any $u\in\SubAf\cap\D$. Define
the closed convex cone $\Cone\subset\Ash$ by

\[
\Cone:=\SC{\Vxx{\lambda u}{\lambda}}{\lambda\in\R_{+}\text{, }u\in\D}
\]

$\SubAf$ intersects the interior of $\D$ and hence $\Bsh$ intersects
the interior of $\Cone$. By Theorem \ref{thm:krein-rutman}, this
implies that there exists $\alpha\in\Ash^{\star}$ s.t. $\alpha|_{\Bsh}=\alpha_{0}$
and for any $v\in\Cone$, $\alpha(v)\geq0$. Define the hyperplane
$\pi\subset\Af$ by

\[
\pi:=\SC{u\in\Af}{\alpha\Vxx u1=0}
\]

We get that $\pi_{0}=\SubAf\cap\pi$, $\D\subseteq\pi^{+}$ and $p\in\pi^{-}$
(we choose the signs of $\pi^{\pm}$ according to the sign of $\alpha$).
In particular, $\pi\in\SH{q_{0}}$.

Define $q_{1}$ by

\[
q_{1}:=\Argmin{q\in\D}\Norm{p-q_{1}}
\]

Then, $q_{1}\in\pi^{+}$ and since $p\in\pi^{-}$ this implies $\pi$
intersects the straight line segment $pq_{1}$ at some point $q_{1}'$.
We get

\begin{align*}
\frac{d\left(p,\D\right)}{d\left(p,\SubAf\cap\D\right)} & =\frac{\Norm{p-q_{1}}}{d\left(p,\pi_{0}\right)}\\
 & \geq\frac{\Norm{p-q_{1}'}}{d\left(p,\SubAf\cap\pi\right)}\\
 & \geq\frac{d\left(p,\pi\right)}{d\left(p,\SubAf\cap\pi\right)}\\
 & \geq\sin\left(\SubAf,\pi\right)
\end{align*}
\end{proof}
The following lemma shows that, in some sense, the tangent hyperplane
to a point on the surface of a body depends continuously on the point.
We will need it in the proof of the upper bound, because we will arrive
at the sine of $\SubAf$ with a supporting hyperplane at a point of
$\D$ \emph{close} to $\DD{\SubAf}$ but not actually there.

\global\long\def\GM{\R_{>0}}%

We will use the notation $\GM:=(0,\infty)$.

\global\long\def\S{\mathcal{S}}%

\global\long\def\Gau{\Gamma}%

\global\long\def\DDa{\partial\D}%

\begin{lem}
\label{lem:gauss}Let $\X$ be a finite-dimensional vector space and
$\D\subseteq\X$ closed. Denote by $\DDa$ the boundary of $\D$.
Denote by $\S$ the quotient of $\X^{\star}\setminus\boldsymbol{0}$
by the multiplicative action of $\GM$ (topologically, it is a sphere).
Define $\Gamma\subseteq\DDa\times\S$ by

\[
\Gau:=\SC{\left(q,\GM\alpha\right)\in\DDa\times\S}{\forall u\in\D:\alpha\left(u-q\right)\geq0}
\]

Then, $\Gau$ is the graph of an upper hemincontinuous multivalued
mapping from $\DDa$ to $\S$.
\end{lem}
\begin{proof}
Define $\Gau'\subseteq\DDa\times\X^{\star}\setminus\boldsymbol{0}$
as follows:

\[
\Gamma':=\SC{(q,\alpha)\in\DDa\times\X^{\star}\setminus\boldsymbol{0}}{\forall u\in\D:\alpha\left(u-q\right)\geq0}
\]

For every $u\in\D$, define $\Gau_{u}\subseteq\DDa\times\X^{\star}$
by

\[
\Gau_{u}:=\SC{(q,\alpha)\in\DDa\times\X^{\star}}{\alpha\left(u-q\right)\geq0}
\]

Clearly, $\Gau_{u}$ is closed. Moreover,

\[
\Gamma'=\bigcap_{u\in\D}\Gamma_{u}\cap\left(\DDa\times\X^{\star}\setminus\boldsymbol{0}\right)
\]

Hence, $\Gamma'$ is closed in $\DDa\times\X^{\star}\setminus\boldsymbol{0}$.
Now, notice that for any $\lambda\in\GM$, $\alpha\left(u-q\right)\geq0$
implies $\lambda\alpha\left(u-q\right)\geq0$. Hence, $\Gamma'$ is
$\GM$-invariant and it's the inverse image of $\Gamma$ under the
canonical mapping from $\DDa\times\X^{\star}\setminus\boldsymbol{0}$
to $\DDa\times\S$. By the definition of the quotient topology, this
means $\Gamma$ is closed. Since $\S$ is compact, the closed graph
theorem for multivalued mappings implies that $\Gau$ is the graph
of an upper hemicontinuous multivalued mapping.
\end{proof}
The following lemma shows that any point on the surface of a convex
body is the nearest point to \emph{some} point outside the body.
\begin{lem}
\label{lem:step-out}Let $\X$ be a finite-dimensional normed vector
space, $\D\subseteq\X$ convex and closed, $q\in\DDa$, $t>0$. Then,
there exists $p\in\X\setminus\D$ s.t.

\[
\Norm{p-q}=d\left(p,\D\right)=t
\]
\end{lem}
\begin{proof}
Choose any $\pi\in\SH q$. We have $\D\subseteq\pi^{+}.$ Let $\pi'$
be the unique hyperplane parallel to $\pi$ at distance $t$ s.t.
$\pi'\subset\pi^{-}$. Take

\[
p:=\Argmin{u\in\pi'}\Norm{u-q}
\]

Then,

\begin{align*}
\Norm{p-q} & =d\left(\pi',\pi\right)\\
 & =t
\end{align*}

Moreover, for any $u\in\D$, the straight line segment $pu$ intersects
$\pi$ at some point $u'$ and hence

\begin{align*}
\Norm{p-u} & \geq\Norm{p-u'}\\
 & \geq d\left(\pi',\pi\right)\\
 & =t
\end{align*}

We conclude

\[
d\left(p,\D\right)=t
\]
\end{proof}
We will need the following elementary fact about geometry of normed
vector spaces.
\begin{lem}
\label{lem:hyperplane-dist}Let $\X$ be a finite-dimensional normed
vector space, $u\in\X$ and $\alpha\in\X^{\star}\setminus\boldsymbol{0}$.
Then,

\[
d\left(u,\ker\alpha\right)=\frac{\Abs{\alpha(u)}}{\Norm{\alpha}}
\]
\end{lem}
\begin{proof}
Let $v\in\X$ be s.t. $\Norm v=1$ and $\alpha(v)=\Norm{\alpha}$
(it exists by definition of the norm on $\X^{\star}$). Define $v'\in\X$
by

\[
v':=\frac{\alpha(u)}{\Norm{\alpha}}\cdot v
\]

Then,

\begin{align*}
\alpha\left(v'-u\right) & =\alpha\left(\frac{\alpha(u)}{\Norm{\alpha}}\cdot v-u\right)\\
 & =\frac{\alpha(u)}{\Norm{\alpha}}\cdot\alpha(v)-\alpha(u)\\
 & =\frac{\alpha(u)}{\Norm{\alpha}}\cdot\Norm{\alpha}-\alpha(u)\\
 & =0
\end{align*}

Hence, $v'-u\in\ker\alpha$. Therefore,

\begin{align*}
d\left(u,\ker\alpha\right) & =d\left(v',\ker\alpha\right)\\
 & =\frac{\alpha(u)}{\Norm{\alpha}}\cdot d\left(v,\ker\alpha\right)
\end{align*}

It remains to show that $d\left(v,\ker\alpha\right)=1$. Since $\boldsymbol{0}\in\ker\alpha$,
we have

\begin{align*}
d\left(v,\ker\alpha\right) & \leq\Norm v\\
 & =1
\end{align*}

On the other hand, for any $w\in\ker\alpha$, we have

\begin{align*}
\Norm{\alpha}\cdot\Norm{v-w} & \geq\alpha(v-w)\\
 & =\alpha(v)-\alpha(w)\\
 & =\Norm{\alpha}-0\\
 & =\Norm{\alpha}
\end{align*}

And hence, $\Norm{v-w}\geq1$. We conclude

\[
d\left(v,\ker\alpha\right)\geq1
\]
\end{proof}
Using the previous fact, we now give a simple characterization of
the sine between an affine subspace and a hyperplane. In particular,
this characterization shows the sine depends on the hyperplane continuously.
Simultaneously, we show that, in this case, the ratio of distances
on the definition of the sine doesn't actually depend on the point.
\begin{lem}
\label{lem:const-ratio}Let $\X$ be a finite-dimensional normed vector
space, $\U_{0}\subseteq\X$ a linear subspace and $\alpha\in\X^{\star}$.
Then, for any $u\in\U_{0}\setminus\ker\alpha$,

\[
\frac{d\left(u,\ker\alpha\right)}{d\left(u,\U_{0}\cap\ker\alpha\right)}=\frac{\Norm{\alpha|_{\U_{0}}}}{\Norm{\alpha}}=\sin\left(\U_{0},\ker\alpha\right)
\]
\end{lem}
\begin{proof}
Observe that 

\[
\U_{0}\cap\ker\alpha=\ker\alpha|_{\U_{0}}
\]

Using Lemma \ref{lem:hyperplane-dist}, we get

\begin{align*}
\frac{d\left(u,\ker\alpha\right)}{d\left(u,\U_{0}\cap\ker\alpha\right)} & =\frac{d\left(u,\ker\alpha\right)}{d\left(u,\ker\alpha|_{\U_{0}}\right)}\\
 & =\frac{\left(\frac{\alpha(u)}{\Norm{\alpha}}\right)}{\left(\frac{\alpha(u)}{\Norm{\alpha|_{\U_{0}}}}\right)}\\
 & =\frac{\Norm{\alpha|_{\U_{0}}}}{\Norm{\alpha}}
\end{align*}

The expression on the right hand side doesn't depend on $u$. Hence,
the left hand side must be equal to its own infimum over $u$, i.e.
to $\sin\left(\U_{0},\pi\right)$.
\end{proof}
Now, we prove the upper bound.

\global\long\def\Nor#1{#1^{\mathrm{n}}}%

Given a finite-dimensional normed vector space $\X$, a closed convex
set $\D\subseteq\X$ s.t. its affine hull is $\X$, $q\in\DDa$ and
$\pi\in\SH q$, we will denote by $\Nor{\pi}$ the unique element
of $\X^{\star}$ s.t. $\Norm{\Nor{\pi}}=1$, $\pi=q+\ker\Nor{\pi}$
and for all $u\in\D$,

\[
\Nor{\pi}(u-q)\geq0
\]

\begin{proof}
[Proof of the upper bound in Proposition \ref{prop:sin-sh}]Without
loss of generality, assume $\Af$ is a vector space and $\SubAf$
a linear subspace. Consider any $q_{0}\in\DD{\SubAf}$ and $\epsilon>0$.
By Lemma \ref{lem:gauss}, there exists $\delta>0$ s.t. for any $q\in\DDa$
and $\pi\in\SH q$, if $\Norm{q_{0}-q}<\delta$ then there exists
$\pi_{0}\in\SH{q_{0}}$ with

\[
\Norm{\Nor{\pi_{0}}-\Nor{\pi}}<\epsilon
\]

Since $q_{0}$ is on the boundary of $\SubAf\cap\D$ inside $\SubAf$,
we can use Lemma \ref{lem:step-out} to choose $p\in\SubAf\setminus\D$
s.t.

\[
d\left(p,\SubAf\cap\D\right)=\Norm{p-q_{0}}<\frac{1}{2}\delta
\]

Define $q_{1}\in\DDa$ by

\[
q_{1}:=\Argmin{q\in\D}\Norm{p-q}
\]

We have

\begin{align*}
\Norm{q_{0}-q_{1}} & \leq\Norm{p-q_{0}}+\Norm{p-q_{1}}\\
 & \leq\Norm{p-q_{0}}+\Norm{p-q_{0}}\\
 & <\delta
\end{align*}

By Lemma \ref{lem:normal-hyperplane}, there exists $\pi_{1}\in\SH{q_{1}}$
s.t. $\D\subseteq\pi_{1}^{+}$, $p\in\pi_{1}^{-}$ and

\[
d\left(p,\pi_{1}\right)=d\left(p,\D\right)
\]

Since $p\in\pi_{1}^{-}$ and $q_{0}\in\pi_{1}^{+}$, $\pi_{1}$ intersects
the straight line segment $pq_{0}$ at some point $q_{0}'$. Since
both $p$ and $q_{0}$ lie in $\SubAf$, so does $q_{0}'$. We get

\begin{align*}
\frac{d\left(p,\D\right)}{d\left(p,\SubAf\cap\D\right)} & =\frac{d\left(p,\pi_{1}\right)}{\Norm{p-q_{0}}}\\
 & \leq\frac{d\left(p,\pi_{1}\right)}{\Norm{p-q_{0}'}}\\
 & \leq\frac{d\left(p,\pi_{1}\right)}{d\left(p,\SubAf\cap\pi_{1}\right)}\\
 & =\sin\left(\SubAf,\pi_{1}\right)
\end{align*}

Here, we used Lemma \ref{lem:const-ratio} on the last line.

Since $\Norm{q_{0}-q_{1}}<\delta$, we know that there exists $\pi_{0}\in\SH{q_{0}}$
with

\[
\Norm{\Nor{\pi_{0}}-\Nor{\pi_{1}}}<\epsilon
\]

Using Lemma \ref{lem:const-ratio} again, we have

\begin{align*}
\Abs{\sin\left(\SubAf,\pi_{0}\right)-\sin\left(\SubAf,\pi_{1}\right)} & =\Abs{\frac{\Norm{\Nor{\pi_{0}}|_{\VecB}}}{\Norm{\Nor{\pi_{0}}}}-\frac{\Norm{\Nor{\pi_{1}}|_{\VecB}}}{\Norm{\Nor{\pi_{1}}}}}\\
 & =\Abs{\Norm{\Nor{\pi_{0}}|_{\VecB}}-\Norm{\Nor{\pi_{1}}|_{\VecB}}}\\
 & \leq\Norm{\Nor{\pi_{0}}|_{\VecB}-\Nor{\pi_{1}}|_{\VecB}}\\
 & \leq\Norm{\Nor{\pi_{0}}-\Nor{\pi_{1}}}\\
 & <\epsilon
\end{align*}

We conclude

\begin{align*}
\frac{d\left(p,\D\right)}{d\left(p,\SubAf\cap\D\right)} & <\sin\left(\SubAf,\pi_{0}\right)+\epsilon\\
 & \leq\sup_{\pi\in\SH{q_{0}}}\sin\left(\SubAf,\pi\right)+\epsilon
\end{align*}

Since $q_{0}\in\DD{\SubAf}$ and $\epsilon>0$ are arbitrary, we get
the desired result.
\end{proof}

\subsection{Sine with Polytope is Positive}
\begin{prop}
\label{prop:polytope-sine}Let $\X$ be a finite-dimensional normed
vector space, $\U_{0}\subseteq\X$ a subspace and $\D\subset\X$ a
compact polytope s.t. $\U_{0}\cap\D\ne\varnothing$. Then,

\[
\sin\left(\U_{0},\D\right)>0
\]
\end{prop}
The proof proceeds in two steps: one step replaces the polytope by
a cone while constraining the infimum in the definion of $\sin$,
the other step uses symmetry to identify the infimum with a minimum
over a compact set. The following is the second step.

\global\long\def\Q{\mathcal{Q}}%

\begin{lem}
\label{lem:corner-sine}Let $\X$ be a finite-dimensional normed vector
space, $\U_{0}\subseteq\X$ a subspace and $\Cone$ a closed convex
cone. Denote

\[
V:=\SC{u\in\U_{0}}{\left\{ u,-u\right\} \subseteq\Cone}
\]

\[
\Q:=\SC{u\in\U_{0}}{d\left(u,\U_{0}\cap\Cone\right)=d\left(u,V\right)}
\]

Then,

\[
\inf_{u\in\Q\setminus\Cone}\frac{d\left(u,\Cone\right)}{d\left(u,\U_{0}\cap\Cone\right)}>0
\]
\end{lem}
\begin{proof}
Notice that $V$ is a linear subspace. Assume w.l.o.g. that $V$ is
a proper subspace of $\U_{0}$ (otherwise, $\Q\setminus\Cone=\U_{0}\setminus\Cone=\varnothing$
and the infinimum is $+\infty$). Define $f:\Q\setminus\Cone\rightarrow\R_{>0}$
by

\[
f(u):=\frac{d\left(u,\Cone\right)}{d\left(u,\U_{0}\cap\Cone\right)}
\]

Let $G$ be the group generated by the additive group of $V$ and
the multiplicative group of $\GM$, acting on $\X$. $\U_{0}$, $\Cone$
and $V$ are $G$-invariant. Therefore, for any $u\in\U_{0}$, $v\in V$
and $\lambda\in\GM$,

\[
d\left(\lambda u+v,\Cone\right)=\lambda d\left(u,\Cone\right)
\]

\[
d\left(\lambda u+v,\U_{0}\cap\Cone\right)=\lambda d\left(u,\U_{0}\cap\Cone\right)
\]

\[
d\left(\lambda u+v,V\right)=\lambda d\left(u,V\right)
\]

It follows that $\Q$ is $G$-invariant, and for any $u\in\Q\setminus\Cone$

\[
f(\lambda u+v)=f(u)
\]

Let's equip $\U_{0}$ with some arbitrary inner product and its associated
$\ell_{2}$ norm. Denote by $\S$ the $\ell_{2}$-unit sphere inside
$V^{\bot}$ and let $P:\U_{0}\rightarrow V^{\bot}$ be the orthogonal
projection. Define $\beta:\U_{0}\setminus V\rightarrow\S$ by

\[
\beta(u):=\frac{Pu}{\Norm{Pu}_{2}}
\]

Notice that, for any $u\in\U_{0}\setminus V$, $\beta(u)$ is in the
$G$-orbit of $u$. Therefore, if $u\in\Q\setminus\Cone$ then $\beta(u)\in\Q\setminus\Cone$
and hence

\begin{align*}
\beta\left(\Q\setminus\Cone\right) & \subseteq\left(\Q\setminus\Cone\right)\cap\S
\end{align*}

Also, for any $u\in\S$, we have $\beta(u)=u$ and therefore

\[
\beta\left(\Q\setminus\Cone\right)=\left(\Q\setminus\Cone\right)\cap\S
\]

Moreover, for any $u\in\Q\cap\S$, we have

\begin{align*}
d\left(u,\U_{0}\cap\Cone\right) & =d\left(u,V\right)\\
 & >0
\end{align*}

Here, the first line is since $u\in\Q$ and second line is since $u\in\S$.
Since $u\in\U_{0}$, we get $u\not\in\Cone$. Hence, $\Q\cap\S\cap\Cone=\varnothing$
and therefore

\[
\beta\left(\Q\setminus\Cone\right)=\Q\cap\S
\]

Finally, if $u\in\Q\setminus\Cone$ then $f(\beta(u))=f(u)$. We conclude

\begin{align*}
\inf_{\Q\setminus\Cone}f & =\inf_{\Q\cap\S}f\\
 & >0
\end{align*}

Here, the final line is because $\Q\cap\S$ is a compact set and $f$
is a continuous positive function on it.
\end{proof}
We will also need the following elementary fact about convex geometry.
\begin{lem}
\label{lem:rel-interior}Let $\X$ be a finite-dimensional normed
vector space, $A\subseteq B\subseteq\X$, $v\in\X$. Assume that $A$
is closed, $B$ is convex and let

\[
u^{*}:=\Argmin{u\in A}\Norm{v-u}
\]

Suppose $u^{*}$ is in the interior of $A$ relatively to B. Then,

\[
d\left(v,B\right)=\Norm{v-u^{*}}
\]
\end{lem}
\begin{proof}
Consider any $u\in B$. The straight line segment $u^{*}u$ is contained
in $B$. By the assumption, there is some $\epsilon\in(0,1)$ s.t.

\[
(1-\epsilon)u^{*}+\epsilon u\in A
\]

We get,

\begin{align*}
\Norm{v-u^{*}} & \leq\Norm{v-\left((1-\epsilon)u^{*}+\epsilon u\right)}\\
 & =\Norm{(1-\epsilon)\left(v-u^{*}\right)+\epsilon\left(v-u\right)}\\
 & \leq\left(1-\epsilon\right)\Norm{v-u^{*}}+\epsilon\Norm{v-u}
\end{align*}

Moving $\left(1-\epsilon\right)\Norm{v-u^{*}}$ to the other side
of the inequality and dividing both sides by $\epsilon$, we conclude

\[
\Norm{v-u^{*}}\leq\Norm{v-u}
\]
\end{proof}
We can now complete the proof:

\global\long\def\Face{\phi}%

\begin{proof}
[Proof of Proposition \ref{prop:polytope-sine}]Let $\{\alpha_{i}\in\X^{\star}\}_{i<n}$
and $\{c_{i}\in\R\}_{i<n}$ be irredundant\footnote{That is, no inequality in this sytem follows from the other inequalities.}
s.t.

\[
\D=\SC{u\in\X}{\forall i<n:\alpha_{i}(u)\geq c_{i}}
\]

Consider any face $\Face\subseteq\D$. Let $\I_{\Face}\subseteq[n]$
be the set of indices corresponding to the facets whose intersection
is $\Face$. In other words,

\[
\phi=\SC{u\in\D}{\forall i\in\I_{\phi}:\alpha_{i}(u)=c_{i}}
\]

Denote

\[
\Cone_{\Face}:=\SC{u\in\X}{\forall i\in\I_{\Face}:\alpha_{i}(u)\geq0}
\]

\[
V_{\Face}:=\SC{u\in\U_{0}}{\forall i\in\I_{\Face}:\alpha_{i}(u)=0}
\]

\[
\Q_{\Face}:=\SC{u\in\U_{0}}{d\left(u,\U_{0}\cap\Cone_{\Face}\right)=d\left(u,V_{\Face}\right)}
\]

\[
S_{\Face}:=\inf_{u\in\Q_{\Face}\setminus\Cone_{\Face}}\frac{d\left(u,\Cone_{\Face}\right)}{d\left(u,\U_{0}\cap\Cone_{\Face}\right)}
\]

By Lemma \ref{lem:corner-sine}, $S_{\Face}>0$. Since there is a
finite number of faces, it is sufficient to show that for any $v\in\U_{0}\setminus\D$,
there is a face $\Face$ s.t.

\[
\frac{d\left(v,\D\right)}{d\left(v,\U_{0}\cap\D\right)}\geq S_{\Face}
\]

Consider some $v\in\U_{0}\setminus\D$. Define $u^{*}$ by

\[
u^{*}:=\Argmin{u\in\U_{0}\cap\D}\Norm{v-u}
\]

Let $\Face$ be the minimal face s.t. $u^{*}\in\Face$. For any $i\in\I_{\Face}$,
we have $\alpha_{i}(u^{*})=c_{i}$. Therefore,

\begin{align*}
\D-u^{*} & =\SC{u\in\X}{\forall i<n:\alpha_{i}\left(u+u^{*}\right)\geq c_{i}}\\
 & =\SC{u\in\X}{\forall i<n:\alpha_{i}\left(u\right)\geq c_{i}-\alpha_{i}\left(u^{*}\right)}\\
 & \subseteq\SC{u\in\X}{\forall i\in\I_{\Face}:\alpha_{i}\left(u\right)\geq0}\\
 & =\Cone_{\Face}
\end{align*}

Since $\phi$ is minimal, for any $i\in[n]\setminus\I_{\Face}$, we
have $\alpha_{i}(u^{*})>c_{i}$. Hence, $\boldsymbol{0}$ is in the
interior of $\D-u^{*}$ relatively to $\Cone_{\Face}$. As a consequence,
it is in the interior of $\U_{0}\cap(\D-u^{*})$ relatively to $\U_{0}\cap\Cone_{\Face}$.
By Lemma \ref{lem:rel-interior}, this implies

\[
d\left(v-u^{*},\U_{0}\cap\Cone_{\Face}\right)=\Norm{v-u^{*}}
\]

In particular, since $\boldsymbol{0}\in V_{\Face}$, we get

\[
d\left(v-u^{*},\U_{0}\cap\Cone_{\Face}\right)=d\left(v-u^{*},V_{\Face}\right)
\]

Hence, $v-u^{*}\in\Q_{\Face}$. Moreover, since $v\not\in\D$, $v\ne u^{*}$
and hence $v-u^{*}\not\in\U_{0}\cap\Cone_{\Face}$. So, $v-u^{*}\in\Q_{\Face}\setminus\Cone_{\Face}$.
Finally, since $\D-u^{*}\subseteq\Cone_{\Face}$, it holds that

\[
d\left(v-u^{*},\D-u^{*}\right)\geq d\left(v-u^{*},\Cone_{\Face}\right)
\]

Combining these facts together, we get

\begin{align*}
\frac{d\left(v,\D\right)}{d\left(v,\U_{0}\cap\D\right)} & =\frac{d\left(v-u^{*},\D-u^{*}\right)}{\Norm{v-u^{*}}}\\
 & \geq\frac{d\left(v-u^{*},\Cone_{\Face}\right)}{d\left(v-u^{*},\U_{0}\cap\Cone_{\Face}\right)}\\
 & \geq S_{\Face}
\end{align*}
\end{proof}

\subsection{Sine with Simplex}

We now proceed to proving the propositions in subsection \ref{subsec:s}.

In order to analyze sines in the case $\D=\Delta\B$, we will need
an expression for the $\ell_{1}$ distance between the simplex $\Delta\B$
and points in its affine hull.
\begin{lem}
\label{lem:simplex-dist}For any finite set $\B$ and $y\in\R^{B}$
s.t. $\sum_{a}y_{a}=1$,

\[
d_{1}\left(y,\Delta\B\right)=\sum_{a\in\B}\Abs{y_{a}}-1
\]

where $d_{1}$ stands for minimal $\ell_{1}$ distance.
\end{lem}
\begin{proof}
Define $\B_{\pm}$ by

\[
\B_{+}:=\SC{a\in\B}{y_{a}\geq0}
\]

\[
\B_{-}:=\SC{a\in\B}{y_{a}<0}
\]

For any $y'\in\Delta\B$, we have

\begin{align*}
\Norm{y-y'}_{1} & =\sum_{a\in\B}\Abs{y_{a}-y'_{a}}\\
 & =\sum_{a\in\B_{+}}\Abs{y_{a}-y'_{a}}+\sum_{a\in\B_{-}}\Abs{y_{a}-y'_{a}}\\
 & \geq\Abs{\sum_{a\in\B_{+}}\left(y_{a}-y'_{a}\right)}+\sum_{a\in\B_{-}}\left(y'_{a}-y_{a}\right)\\
 & =\Abs{\sum_{a\in\B_{+}}y_{a}-\sum_{a\in\B_{+}}y'_{a}}+\sum_{a\in\B_{-}}y'_{a}-\sum_{a\in\B_{-}}y_{a}\\
 & =\sum_{a\in\B_{+}}\Abs{y_{a}}-\sum_{a\in\B_{+}}y'_{a}+\sum_{a\in\B_{-}}y'_{a}+\sum_{a\in\B_{-}}\Abs{y_{a}}\\
 & =\sum_{a\in\B}\Abs{y_{a}}-\sum_{a\in\B_{+}}y'_{a}+\sum_{a\in\B_{-}}y'_{a}\\
 & \geq\sum_{a\in\B}\Abs{y_{a}}-\sum_{a\in\B_{+}}y'_{a}-\sum_{a\in\B_{-}}y'_{a}\\
 & =\sum_{a\in\B}\Abs{y_{a}}-1
\end{align*}

Here, we used the fact that $\sum_{a\in\B_{+}}y_{a}\geq1\geq\sum_{a\in\B_{+}}y'_{a}$
on the 5th line.

Conversly, define $y^{*}\in\Delta\B$ by

\[
y_{a}^{*}:=\frac{\max\left(y_{a},0\right)}{\sum_{b\in\B}\max\left(y_{b},0\right)}
\]

We have,

\begin{align*}
\Norm{y-y^{*}}_{1} & =\sum_{a\in\B}\Abs{y_{a}-\frac{\max\left(y_{a},0\right)}{\sum_{b\in\B}\max\left(y_{b},0\right)}}\\
 & =\sum_{a\in\B_{+}}\Abs{y_{a}-\frac{\max\left(y_{a},0\right)}{\sum_{b\in\B}\max\left(y_{b},0\right)}}+\sum_{a\in\B_{-}}\Abs{y_{a}-\frac{\max\left(y_{a},0\right)}{\sum_{b\in\B}\max\left(y_{b},0\right)}}\\
 & =\sum_{a\in\B_{+}}\Abs{y_{a}-\frac{y_{a}}{\sum_{b\in\B_{+}}y_{b}}}+\sum_{a\in\B_{-}}\Abs{y_{a}}\\
 & =\left(1-\frac{1}{\sum_{b\in\B_{+}}y_{b}}\right)\sum_{a\in\B_{+}}y_{a}+\sum_{a\in\B_{-}}\Abs{y_{a}}\\
 & =\sum_{a\in\B_{+}}y_{a}-1+\sum_{a\in\B_{-}}\Abs{y_{a}}\\
 & =\sum_{a\in\B}\Abs{y_{a}}-1
\end{align*}
\end{proof}
Since Proposition \ref{prop:simplex-full-s} is a special case of
Proposition \ref{prop:simplex-s} (for $\mathcal{E}=\B$), it is sufficient
to prove the latter. The proof proceeds by connecting an arbitrary
point in $y\in\U^{\flat}\setminus\Delta\B$ by a straight line segment
with the point in $\U^{+}$ in which the maximum on the right hand
side is achieved. The point where this segment meets the boundary
of the simplex is used to produce an upper bound on the distance from
$y$ to $\U^{+}$. 
\begin{proof}
[Proof of Proposition \ref{prop:simplex-s}]Define $y^{*}$ by

\[
y^{*}:=\Argmax{y\in\U^{+}}\min_{a\in\mathcal{E}}y_{a}
\]

Consider some $y\in\U^{\flat}\setminus\Delta\B$. Define $a^{*}$
by

\[
a^{*}:=\Argmax{a\in\B:y_{a}<0}\frac{\Abs{y_{a}}}{\Abs{y_{a}}+y_{a}^{*}}
\]

Notice that $a^{*}\in\mathcal{E}$ since for any $a\not\in\mathcal{E}$,
$y_{a}=0$. Define $t^{*}$ by

\begin{align*}
t^{*} & :=\max_{a\in\B:y_{a}<0}\frac{\Abs{y_{a}}}{\Abs{y_{a}}+y_{a}^{*}}\\
 & =\frac{\Abs{y_{a^{*}}}}{\Abs{y_{a^{*}}}+y_{a^{*}}^{*}}
\end{align*}

And, define $y'$ by

\[
y':=t^{*}y^{*}+\left(1-t^{*}\right)y
\]

Finally, define $\B_{\pm}$ by

\[
\B_{+}:=\SC{a\in\B}{y_{a}\geq0}
\]

\[
\B_{-}:=\SC{a\in\B}{y_{a}<0}
\]

Then, for any $a\in\B_{+}$

\begin{align*}
y'_{a} & =t^{*}y_{a}^{*}+\left(1-t^{*}\right)y_{a}\\
 & \geq0
\end{align*}

And, for any $a\in\B_{-}$

\begin{align*}
y'_{a} & =t^{*}y_{a}^{*}+\left(1-t^{*}\right)y_{a}\\
 & =y_{a}+t^{*}\left(y_{a}^{*}-y_{a}\right)\\
 & =y_{a}+t^{*}\left(\Abs{y_{a}}+y_{a}^{*}\right)\\
 & \geq y_{a}+\frac{\Abs{y_{a}}}{\Abs{y_{a}}+y_{a}^{*}}\cdot\left(\Abs{y_{a}}+y_{a}^{*}\right)\\
 & =y_{a}+\Abs{y_{a}}\\
 & =0
\end{align*}

We conclude $y'\in\Delta\B$. Since it's also clear that $y'\in\U^{\flat}$,
we get $y'\in\U^{+}$. Hence,

\begin{align*}
d_{1}\left(y,\U^{+}\right) & \leq\Norm{y'-y}_{1}\\
 & =\Norm{\frac{t^{*}}{1-t^{*}}\cdot\left(y^{*}-y'\right)}_{1}\\
 & =\frac{t^{*}}{1-t^{*}}\cdot\Norm{y^{*}-y'}_{1}\\
 & \leq\frac{t^{*}}{1-t^{*}}\cdot\left(\Norm{y^{*}}_{1}+\Norm{y'}_{1}\right)\\
 & =\frac{2t^{*}}{1-t^{*}}\\
 & =\frac{2\Abs{y_{a^{*}}}}{y_{a^{*}}^{*}}
\end{align*}

Moreover, by Lemma \ref{lem:simplex-dist}

\begin{align*}
d_{1}\left(y,\Delta\B\right) & =\sum_{a\in\B}\Abs{y_{a}}-1\\
 & =\Abs{y_{a^{*}}}+\sum_{a\in\B\setminus a^{*}}\Abs{y_{a}}-1\\
 & \geq\Abs{y_{a^{*}}}+\sum_{a\in\B\setminus a^{*}}y_{a}-1\\
 & =\Abs{y_{a^{*}}}+\left(1-y_{a^{*}}\right)-1\\
 & =2\Abs{y_{a^{*}}}
\end{align*}

Combining the last two inequalites, we get

\begin{align*}
\frac{d_{1}\left(y,\Delta\B\right)}{d_{1}\left(y,\U^{+}\right)} & \geq2\Abs{y_{a^{*}}}\cdot\frac{y_{a^{*}}^{*}}{2\Abs{y_{a^{*}}}}\\
 & =y_{a^{*}}^{*}\\
 & \geq\min_{a\in\mathcal{E}}y_{a}^{*}\\
 & =\max_{y\in\U^{+}}\min_{a\in\mathcal{E}}y_{a}
\end{align*}

By Lemma \ref{lem:l1}, $d_{1}$ is the same thing as $d_{Y}$. And,
since this holds for \emph{any} $y\in\U^{\flat}\setminus\Delta\B$,
we conclude

\[
\sin\left(\U^{\flat},\Delta\B\right)\geq\max_{y\in\U^{+}}\min_{i\in\mathcal{E}}y_{i}
\]
\end{proof}

\subsection{Sine with Ball}

The following elementary inequality will be needed for the proof of
Proposition \ref{prop:ball-s}.
\begin{lem}
\label{lem:sqrt-inequality}The following inequality holds for any
$\alpha\in[0,\frac{\pi}{2}]$ and $t>\sin\alpha$:

\[
\frac{\sqrt{t^{2}+\cos^{2}\alpha}-1}{t-\sin\alpha}\geq\sin\alpha
\]
\end{lem}
\begin{proof}
We have

\begin{align*}
t^{2}-2\left(\sin\alpha\right)t+1-\cos^{2}\alpha & =t^{2}-2\left(\sin\alpha\right)t+\sin^{2}\alpha\\
 & =\left(t-\sin\alpha\right)^{2}\\
 & >0
\end{align*}

Multiplying both sides by $\cos^{2}\alpha$ we get

\begin{align*}
\left(\cos^{2}\alpha\right)t^{2}-2\left(\cos^{2}\alpha\right)\left(\sin\alpha\right)t+\cos^{2}\alpha-\cos^{4}\alpha & \geq0
\end{align*}

Applying the identity $\cos^{2}\alpha=1-\sin^{2}\alpha$ to the first
term on the left hand side, we get

\[
t^{2}-(\sin^{2}\alpha)t^{2}-2\left(\cos^{2}\alpha\right)\left(\sin\alpha\right)t+\cos^{2}\alpha-\cos^{4}\alpha\geq0
\]

Now we move all the negative terms to the right hand side:

\begin{align*}
t^{2}+\cos^{2}\alpha & \geq\left(\sin^{2}\alpha\right)t^{2}+2\left(\cos^{2}\alpha\right)\left(\sin\alpha\right)t+\cos^{4}\alpha\\
 & =\left(\left(\sin\alpha\right)t+\cos^{2}\alpha\right)^{2}
\end{align*}

Taking square root of both sides, we get

\begin{align*}
\sqrt{t^{2}+\cos^{2}\alpha} & \geq\left(\sin\alpha\right)t+\cos^{2}\alpha\\
 & =\left(\sin\alpha\right)t+1-\sin^{2}\alpha
\end{align*}

Moving 1 to the left hand side we get

\[
\sqrt{t^{2}+\cos^{2}\alpha}-1\geq\left(\sin\alpha\right)\left(t-\sin\alpha\right)
\]

Finally, we divide both sides by $t-\sin\alpha$ and conclude that

\[
\frac{\sqrt{t^{2}+\cos^{2}\alpha}-1}{t-\sin\alpha}\geq\sin\alpha
\]
\end{proof}
When $\D$ is a ball, the induced metric on its affine hull is Euclidean.
For any $y\in\U^{\flat}\setminus\D$, we connect $y$ by a straight
line segment to the point of $\U^{\flat}$ nearest to the center of
the ball. The point where this segment meets the sphere is the nearest
point to $y$ inside $\U^{+}$.
\begin{proof}
[Proof of Proposition \ref{prop:ball-s}]Define $\U_{0}$ by

\[
\U_{0}:=\SC{y\in\R^{n}}{\Vxx y1\in\U}
\]

We also denote by $\D_{0}$ the closed unit ball in $\R^{n}$. Define
$y^{*}$ and $\rho$ by

\[
y^{*}:=\Argmin{y\in\U_{0}}\Norm y_{2}
\]

\begin{align*}
\rho & :=\min_{y\in\U_{0}}\Norm y_{2}\\
 & =\Norm{y^{*}}_{2}
\end{align*}

Consider some $y\in\U_{0}$ s.t. $\Norm y_{2}>1$. Define $t$ by

\[
t:=\Norm{y-y^{*}}_{2}
\]

Denoting minimal $\ell_{2}$ distance by $d_{2}$, we have

\begin{align*}
d_{2}\left(y,\D_{0}\right) & =\Norm y_{2}-1\\
 & =\Norm{y^{*}+\left(y-y^{*}\right)}_{2}-1\\
 & =\sqrt{\Norm{y^{*}}_{2}^{2}+\Norm{y-y^{*}}_{2}^{2}+2\T{\left(y-y^{*}\right)}y^{*}}-1\\
 & =\sqrt{\rho^{2}+t^{2}}-1
\end{align*}

Here, we observed that $\T{\left(y-y^{*}\right)}y^{*}=0$ due to the
definition of $y^{*}$.

Let $y'$ be a point on the straight line segment $y^{*}y$ at distance
$\sqrt{1-\rho^{2}}$ from $y^{*}$. Then, $\Norm{y'}=1$, and hence
$y'\in\U_{0}\cap\D_{0}$ and

\begin{align*}
d_{2}\left(y,\U_{0}\cap\D_{0}\right) & \leq\Norm{y-y'}_{2}\\
 & =t-\sqrt{1-\rho^{2}}
\end{align*}

We get,

\begin{align*}
\frac{d_{2}\left(y,\D_{0}\right)}{d_{2}\left(y,\U_{0}\cap\D_{0}\right)} & \geq\frac{\sqrt{t^{2}+\rho^{2}}-1}{t-\sqrt{1-\rho^{2}}}\\
 & \geq\sqrt{1-\rho^{2}}
\end{align*}

Here, we used Lemma \ref{lem:sqrt-inequality} with $\alpha=\arccos\rho$
on the second line.

By rotational symmetry, the metric on $\mu^{-1}(1)$ is Euclidean
up to a scalar. Hence, we can conclude

\[
\sin\left(\U^{\flat},\D\right)\geq\sqrt{1-\rho^{2}}
\]
\end{proof}

\section{\label{sec:prob-sys}Probability Systems (Proofs)}

\subsection{Single Conditional Probability}

We prove Proposition \ref{prop:cond-s} first and Proposition \ref{prop:sys-probs}
second, because the latter proof will use a special case of the former
propostion.

In order to prove Proposition \ref{prop:cond-s}, we will need the
following lemma. For any point $y$ in the affine hull of the simplex
$\Delta\B$, it characterizes some (actually all) points in $\Delta\B$
that are at minimal $\ell_{1}$ distance from $y$.
\begin{lem}
\label{lem:simplex-projection}Let $\B$ be a finite set, $y\in\R^{\B}$
s.t. $\sum_{a}y_{a}=1$ and $\xi\in\Delta\B$. Assume that for any
$a\in\B$:
\begin{itemize}
\item If $y_{a}<0$ then $\xi_{a}=0$.
\item If $y_{a}\geq0$ then $\xi_{a}\leq y_{a}$.
\end{itemize}
Then,

\[
\Norm{y-\xi}_{1}=d_{1}\left(y,\Delta\B\right)
\]
\end{lem}
\begin{proof}
We have,

\begin{align*}
\Norm{y-\xi}_{1} & =\sum_{a\in\B}\Abs{y_{a}-\xi_{a}}\\
 & =\sum_{a\in\B:y_{a}<0}\Abs{y_{a}-\xi_{a}}+\sum_{a\in\B:y_{a}\geq0}\Abs{y_{a}-\xi_{a}}\\
 & =\sum_{a\in\B:y_{a}<0}\Abs{y_{a}}+\sum_{a\in\B:y_{a}\geq0}\left(y_{a}-\xi_{a}\right)\\
 & =\sum_{a\in\B:y_{a}<0}\Abs{y_{a}}+\sum_{a\in\B:y_{a}\geq0}\Abs{y_{a}}-\sum_{a\in\B:y_{a}\geq0}\xi_{a}\\
 & =\sum_{a\in\B}\Abs{y_{a}}-\sum_{a\in\B:y_{a}\geq0}\xi_{a}-\sum_{a\in\B:y_{a}<0}\xi_{a}\\
 & =\sum_{a\in\B}\Abs{y_{a}}-\sum_{a\in\B}\xi_{a}\\
 & =\sum_{a\in\B}\Abs{y_{a}}-1\\
 & =d_{1}\left(y,\Delta\B\right)
\end{align*}

Here, we used Lemma \ref{lem:simplex-dist} on the last line.
\end{proof}
The proof of Proposition \ref{prop:cond-s} works by constructing
a point in $\Delta\B$ that satisfies the conditions of the previous
lemma, while staying inside $\U$.

\global\long\def\Ep{\mathcal{E}_{+}}%

\global\long\def\Em{\mathcal{E}_{-}}%

\global\long\def\Eb{\mathcal{E}_{0}}%

\begin{proof}
[Proof of Proposition \ref{prop:cond-s}]Denote $\mathcal{E}_{+}:=\Econd\cap\Event$,
$\Em:=\Econd\setminus\Event$, $\Eb:=\B\setminus\Econd$. For any
$y\in\R^{\B}$, $y\in\U$ if and only if

\begin{equation}
(1-p)\sum_{a\in\Ep}y_{a}=p\sum_{a\in\Em}y_{a}\label{eq:u-cond}
\end{equation}

Consider any $y\in\U^{\flat}\setminus\Delta\B$. By Lemma \ref{lem:l1},
it's enough to prove that

\[
d_{1}\left(y,\U^{+}\right)=d_{1}\left(y,\Delta\B\right)
\]

We split the proof into 3 cases. The first case is when

\[
\sum_{a\in\Ep}y_{a}\leq0
\]

\[
\sum_{a\in\Em}y_{a}\leq0
\]

We will use the notation $y_{a}^{+}:=\max(y_{a},0)$. Observe that

\begin{align*}
\sum_{b\in\Eb}y_{b}^{+} & \geq\sum_{b\in\Eb}y_{b}\\
 & =\sum_{b\in\B}y_{b}-\sum_{a\in\Econd}y_{a}\\
 & =1-\sum_{a\in\Ep}y_{a}-\sum_{a\in\Em}y_{a}\\
 & \geq1
\end{align*}

Define $\xi\in\Delta\B$ by

\[
\xi_{a}:=\begin{cases}
\frac{y_{a}^{+}}{\sum_{b\in\Eb}y_{b}^{+}} & \text{if }a\in\Eb\\
0 & \text{if }a\in\Econd
\end{cases}
\]

We have $\xi\in\U^{+}$, since both sides of equation (\ref{eq:u-cond})
vanish when substituing $\xi$. Hence, we can apply Lemma \ref{lem:simplex-projection}
to conclude

\begin{align*}
d_{1}\left(y,\U^{+}\right) & \leq\Norm{y-\xi}_{1}\\
 & =d_{1}\left(y,\Delta\B\right)
\end{align*}

The second case is when

\[
\sum_{a\in\Ep}y_{a}\geq0
\]

\[
\sum_{a\in\Em}y_{a}\geq0
\]

\[
\sum_{a\in\Eb}y_{a}\geq0
\]

Define $\xi\in\R^{\B}$ by

\begin{align*}
\xi_{a} & :=\begin{cases}
\frac{\sum_{b\in\Ep}y_{b}}{\sum_{b\in\Ep}y_{b}^{+}}\cdot y_{a}^{+} & \text{if }a\in\Ep\\
\\
\frac{\sum_{b\in\Em}y_{b}}{\sum_{b\in\Em}y_{b}^{+}}\cdot y_{a}^{+} & \text{if }a\in\Em\\
\\
\frac{\sum_{b\in\Eb}y_{b}}{\sum_{b\in\Eb}y_{b}^{+}}\cdot y_{a}^{+} & \text{if }a\in\Eb
\end{cases}
\end{align*}

\global\long\def\Es{\mathcal{E}_{*}}%

Here, if the denominator vanishes then $y_{a}^{+}$ must vanish as
well, as we interpret the entire expression as $0$. 

Let $\Es$ be any of $\Ep$, $\Em$, $\Eb$. If $\sum_{b\in\Es}y_{b}^{+}>0$
then, 

\begin{align*}
\sum_{a\in\Es}\xi_{a} & =\sum_{a\in\Es}\frac{\sum_{b\in\Es}y_{b}}{\sum_{b\in\Es}y_{b}^{+}}\cdot y_{a}^{+}\\
 & =\frac{\sum_{b\in\Es}y_{b}}{\sum_{b\in\Es}y_{b}^{+}}\cdot\sum_{a\in\Es}y_{a}^{+}\\
 & =\sum_{b\in\Es}y_{b}
\end{align*}

On the other hand, if $\sum_{b\in\Es}y_{b}^{+}=0$ then

\begin{align*}
0 & =\sum_{b\in\Es}y_{b}^{+}\\
 & \geq\sum_{b\in\Es}y_{b}\\
 & \geq0
\end{align*}

Here, we used the case conditions on the last line. We conclude that
$\sum_{b\in\Es}y_{b}=0$ and therefore

\begin{align*}
\sum_{a\in\Es}\xi_{a} & =\sum_{a\in\Es}0\\
 & =0\\
 & =\sum_{b\in\Es}y_{b}
\end{align*}

We got that \emph{either} way

\[
\sum_{a\in\Es}\xi_{a}=\sum_{b\in\Es}y_{b}
\]

It follows that,

\begin{align*}
\sum_{a\in\B}\xi_{a} & =\sum_{a\in\Ep}\xi_{a}+\sum_{a\in\Em}\xi_{a}+\sum_{a\in\Eb}\xi_{a}\\
 & =\sum_{b\in\Ep}y_{b}+\sum_{b\in\Em}y_{b}+\sum_{b\in\Eb}y_{b}\\
 & =\sum_{b\in\B}y_{b}\\
 & =1
\end{align*}

Also, the case conditions imply that $\xi_{a}\geq0$. We conclude
that $\xi\in\Delta\B$. 

Now, let's check equation (\ref{eq:u-cond}) for $\xi$:

\begin{align*}
(1-p)\sum_{a\in\Ep}\xi_{a} & =(1-p)\sum_{b\in\Ep}y_{b}\\
 & =p\sum_{b\in\Em}y_{b}\\
 & =p\sum_{a\in\Em}\xi_{a}
\end{align*}

Here, we used equation (\ref{eq:u-cond}) for $y$ on the second line.

Therefore, $\xi\in\U\cap\Delta\B$. Moreover, it's easy to see that
Lemma \ref{lem:simplex-projection} is applicable and hence,

\begin{align*}
d_{1}\left(y,\U^{+}\right) & \leq\Norm{y-\xi}_{1}\\
 & =d_{1}\left(y,\Delta\B\right)
\end{align*}

The third case is when

\[
\sum_{a\in\Ep}y_{a}\geq0
\]

\[
\sum_{a\in\Em}y_{a}\geq0
\]

\[
\sum_{a\in\Eb}y_{a}<0
\]

Define $\xi\in\R^{\B}$ by

\[
\xi_{a}:=\begin{cases}
\frac{p}{\sum_{b\in\Ep}y_{b}^{+}}\cdot y_{a}^{+} & \text{if }a\in\Ep\\
\\
\frac{1-p}{\sum_{b\in\Em}y_{b}^{+}}\cdot y_{a}^{+} & \text{if }a\in\Em\\
\\
0 & \text{if }a\in\Eb
\end{cases}
\]

Here, if the denominator vanishes then $y_{a}^{+}$ must vanish as
well, as we interpret the entire expression as $0$. 

If $\sum_{b\in\Ep}y_{b}^{+}>0$ then,

\begin{align*}
\sum_{a\in\Ep}\xi_{a} & =\sum_{a\in\Ep}\frac{p}{\sum_{b\in\Ep}y_{b}^{+}}\cdot y_{a}^{+}\\
 & =\frac{p}{\sum_{b\in\Ep}y_{b}^{+}}\cdot\sum_{a\in\Ep}y_{a}^{+}\\
 & =p
\end{align*}

On the other hand, if $\sum_{b\in\Ep}y_{b}^{+}=0$ then, like in the
second case, $\sum_{a\in\Ep}y_{a}=0$. By equation (\ref{eq:u-cond})
this implies that either $p=0$ or $\sum_{a\in\Em}y_{a}=0$. However,
the latter would imply $\sum_{a\in\Eb}y_{a}=1$ which contradicts
the case conditions. Therefore, $p=0$ and in particular

\[
\sum_{a\in\Ep}\xi_{a}=p
\]

We got that this identity holds either way. By analogous reasoning,

\[
\sum_{a\in\Em}\xi_{a}=1-p
\]

It follows that,

\begin{align*}
\sum_{a\in\B}\xi_{a} & =\sum_{a\in\Ep}\xi_{a}+\sum_{a\in\Em}\xi_{a}+\sum_{a\in\Eb}\xi_{a}\\
 & =p+1-p\\
 & =1
\end{align*}

We conclude that $\xi\in\Delta\B$.

Now, let's check equation (\ref{eq:u-cond}) for $\xi$:

\begin{align*}
(1-p)\sum_{a\in\Ep}\xi_{a} & =(1-p)p\\
 & =p\sum_{a\in\Em}\xi_{a}
\end{align*}

Therefore, $\xi\in\U^{+}$.

We have,

\begin{align*}
\sum_{a\in\Ep}y_{a}^{+} & \geq\sum_{a\in\Ep}y_{a}\\
 & =p\sum_{a\in\Ep}y_{a}+(1-p)\sum_{a\in\Ep}y_{a}\\
 & =p\sum_{a\in\Ep}y_{a}+p\sum_{a\in\Em}y_{a}\\
 & =p\left(\sum_{a\in\Ep}y_{a}+\sum_{a\in\Em}y_{a}\right)\\
 & =p\left(1-\sum_{a\in\Eb}y_{a}\right)\\
 & \geq p
\end{align*}

Here, we used equation (\ref{eq:u-cond}) on the second line and the
case conditions on the last line. By analogous reasoning,

\[
\sum_{a\in\Em}y_{a}^{+}\geq1-p
\]

These inequalities allow us to apply Lemma \ref{lem:simplex-projection}
and conclude that

\begin{align*}
d_{1}\left(y,\U^{+}\right) & \leq\Norm{y-\xi}_{1}\\
 & =d_{1}\left(y,\Delta\B\right)
\end{align*}
\end{proof}

\subsection{System of Absolute Probabilities}

The following lemma will be needed for the proof of Proposition \ref{prop:sys-probs}.
It shows that, given a probability distribution $\psi$ over the subsets
of a fixed finite set $\PSet$, the probability $\alpha_{i}(\psi)$
that a particular element $i$ is in the set can be modified to any
other value $p$ without affecting the probabilities of other elements
$j\ne i$, in such a manner that the total variation distance between
the new distribution $\psi'$ and the original $\psi$ is equal to
the probability shift $|\alpha_{i}(\psi)-p|$.
\begin{lem}
\label{lem:fix-prob}Let $\PSet$ be a finite set. For every $i\in\PSet$,
define $\alpha_{i}:\Delta2^{\PSet}\rightarrow[0,1]$ by

\[
\alpha_{i}(\psi):=\Pr_{A\sim\psi}\left[i\in A\right]
\]

Consider any $i\in\PSet$, $\psi\in\Delta2^{\PSet}$ and $p\in[0,1]$.
Then, there exists $\psi'\in\Delta2^{\PSet}$ s.t. the following conditions
hold:
\begin{itemize}
\item $\alpha_{i}(\psi')=p$
\item For any $j\in\PSet\setminus i$, $\alpha_{j}(\psi')=\alpha_{j}(\psi)$
\item $\Norm{\psi-\psi'}_{1}=2\Abs{\alpha_{i}(\psi)-p}$
\end{itemize}
\end{lem}
\begin{proof}
We can assume that $p\geq\alpha_{i}(\psi)$ without loss of generality:
otherwise, we can apply to $\psi$ the bijection $2^{\PSet}\rightarrow2^{\PSet}$
defined by taking symmetric difference with $\{i\}$ and change $p$
to $1-p$. Define $f:2^{\PSet}\rightarrow2^{\PSet}$ by

\[
f\left(A\right):=A\cup\{i\}
\]

Set $\psi'$ to

\[
\psi':=\frac{(1-p)\psi+\left(p-\alpha_{i}(\psi)\right)f_{*}\psi}{1-\alpha_{i}(\psi)}
\]

This is a convex combination of $f_{*}\psi$ and $\psi$ and hence
$\psi'\in\Delta2^{\PSet}$. Let's verify the conditions. 

For the first condition, we have

\begin{align*}
\alpha_{i}\left(\psi'\right) & =\frac{(1-p)\alpha_{i}(\psi)+\left(p-\alpha_{i}(\psi)\right)\alpha_{i}(f_{*}\psi)}{1-\alpha_{i}(\psi)}\\
 & =\frac{\alpha_{i}(\psi)-p\alpha_{i}(\psi)+p-\alpha_{i}(\psi)}{1-\alpha_{i}(\psi)}\\
 & =\frac{p\left(1-\alpha_{i}(\psi)\right)}{1-\alpha_{i}(\psi)}\\
 & =p
\end{align*}

Here, we used that $i\in f(A)$ for any $A\in2^{\PSet}$ to conclude
that $\alpha_{i}(f_{*}\psi)=1$.

For the second condition,

\begin{align*}
\alpha_{j}\left(\psi'\right) & =\frac{(1-p)\alpha_{j}(\psi)+\left(p-\alpha_{i}(\psi)\right)\alpha_{j}(f_{*}\psi)}{1-\alpha_{i}(\psi)}\\
 & =\frac{(1-p)\alpha_{j}(\psi)+\left(p-\alpha_{i}(\psi)\right)\alpha_{j}(\psi)}{1-\alpha_{i}(\psi)}\\
 & =\frac{\left(1-p+p-\alpha_{i}(\psi)\right)\alpha_{j}(\psi)}{1-\alpha_{i}(\psi)}\\
 & =\frac{\left(1-\alpha_{i}(\psi)\right)\alpha_{j}(\psi)}{1-\alpha_{i}(\psi)}\\
 & =\alpha_{j}(\psi)
\end{align*}

Here, we used that $j\in f(A)$ if and only if $j\in A$ to conclude
that $\alpha_{j}(f_{*}\psi)=\alpha_{j}(\psi)$.

For the third condition, first observe that

\begin{align*}
\Norm{\psi-f_{*}\psi}_{1} & = & \sum_{A\subseteq\PSet}\Abs{\psi_{A}-f_{*}\psi_{A}}\\
 & = & \sum_{A\subseteq\PSet}\Abs{\psi_{A}-\sum_{B\subseteq\PSet:f(B)=A}\psi_{B}}\\
 & = & \sum_{A\subseteq\PSet:i\in A}\Abs{\psi_{A}-\sum_{B\subseteq\PSet:f(B)=A}\psi_{B}} & +\\
 &  & \sum_{A\subseteq\PSet:i\not\in A}\Abs{\psi_{A}-\sum_{B\subseteq\PSet:f(B)=A}\psi_{B}}\\
 & = & \sum_{A\subseteq\PSet:i\in A}\Abs{\psi_{A}-\left(\psi_{A}+\psi_{A\setminus i}\right)} & +\sum_{A\subseteq\PSet:i\not\in A}\psi_{A}\\
 & = & \sum_{A\subseteq\PSet:i\in A}\psi_{A\setminus i} & +\sum_{A\subseteq\PSet:i\not\in A}\psi_{A}\\
 & = & 2\sum_{A\subseteq\PSet:i\not\in A}\psi_{A}\\
 & = & 2\left(1-\alpha_{i}(\psi)\right)
\end{align*}

It follows that

\begin{align*}
\Norm{\psi-\psi'}_{1} & =\Norm{\psi-\frac{(1-p)\psi+\left(p-\alpha_{i}(\psi)\right)f_{*}\psi}{1-\alpha_{i}(\psi)}}_{1}\\
 & =\Norm{\frac{(1-p)\psi+\left(p-\alpha_{i}(\psi)\right)\psi}{1-\alpha_{i}(\psi)}-\frac{(1-p)\psi+\left(p-\alpha_{i}(\psi)\right)f_{*}\psi}{1-\alpha_{i}(\psi)}}_{1}\\
 & =\Norm{\frac{(1-p)(\psi-\psi)+\left(p-\alpha_{i}(\psi)\right)\left(\psi-f_{*}\psi\right)}{1-\alpha_{i}(\psi)}}_{1}\\
 & =\frac{p-\alpha_{i}(\psi)}{1-\alpha_{i}(\psi)}\cdot\Norm{\psi-f_{*}\psi}_{1}\\
 & =\frac{p-\alpha_{i}(\psi)}{1-\alpha_{i}(\psi)}\cdot2\left(1-\alpha_{i}(\psi)\right)\\
 & =2\left(p-\alpha_{i}(\psi)\right)
\end{align*}
\end{proof}
Another lemma we need for Proposition \ref{prop:sys-probs}: given
two probability distributions on a finite set $B$, if we somehow
``lift'' one of them through a surjection, then we can also lift
the other s.t. the total variation distance between them is preserved.
This is achieved by extracting conditional distributions on the fibers
from the lift, and combining them with the other distribution on $B$.
\begin{lem}
\label{lem:lift-tvd}Let $A,B$ be finite sets, $f:A\rightarrow B$
a surjection, $\xi\in\Delta A$ and $\psi\in\Delta B$. Then, there
exists $\xi'\in\Delta A$ s.t. $f_{*}\xi'=\psi$ and

\[
\Norm{\xi-\xi'}_{1}=\Norm{f_{*}\xi-\psi}_{1}
\]
\end{lem}
\begin{proof}
Set $\xi'\in\R^{A}$ to

\[
\xi'_{i}:=\frac{\psi_{f(i)}\xi_{i}}{\sum_{j\in f^{-1}\left(f(i)\right)}\xi_{j}}
\]

Let's verify that $\xi'\in\Delta A$:

\begin{align*}
\sum_{i\in\A}\xi_{i} & =\sum_{i\in\A}\frac{\psi_{f(i)}\xi_{i}}{\sum_{j\in f^{-1}\left(f(i)\right)}\xi_{j}}\\
 & =\sum_{k\in B}\sum_{i\in f^{-1}(k)}\frac{\psi_{k}\xi_{i}}{\sum_{j\in f^{-1}\left(k\right)}\xi_{j}}\\
 & =\sum_{k\in B}\frac{\psi_{k}}{\sum_{j\in f^{-1}\left(k\right)}\xi_{j}}\cdot\sum_{i\in f^{-1}(k)}\xi_{i}\\
 & =\sum_{k\in B}\psi_{k}\\
 & =1
\end{align*}

Let's verify that $f_{*}\xi'=\psi$:

\begin{align*}
f_{*}\xi'_{k} & =\sum_{i\in f^{-1}(k)}\xi'_{i}\\
 & =\sum_{i\in f^{-1}(k)}\frac{\psi_{k}\xi_{i}}{\sum_{j\in f^{-1}\left(k\right)}\xi_{j}}\\
 & =\frac{\psi_{k}}{\sum_{j\in f^{-1}\left(k\right)}\xi_{j}}\cdot\sum_{i\in f^{-1}(k)}\xi_{i}\\
 & =\psi_{k}
\end{align*}

Finally, we have

\begin{align*}
\Norm{\xi-\xi'}_{1} & =\sum_{i\in A}\Abs{\xi_{i}-\xi'_{i}}\\
 & =\sum_{i\in A}\Abs{\xi_{i}-\frac{\psi_{f(i)}\xi_{i}}{\sum_{j\in f^{-1}\left(f(i)\right)}\xi_{j}}}\\
 & =\sum_{i\in A}\xi_{i}\Abs{1-\frac{\psi_{f(i)}}{\sum_{j\in f^{-1}\left(f(i)\right)}\xi_{j}}}\\
 & =\sum_{k\in B}\sum_{i\in f^{-1}(k)}\xi_{i}\Abs{1-\frac{\psi_{k}}{\sum_{j\in f^{-1}\left(k\right)}\xi_{j}}}\\
 & =\sum_{k\in B}\sum_{i\in f^{-1}(k)}\xi_{i}\cdot\frac{1}{\sum_{j\in f^{-1}\left(k\right)}\xi_{j}}\cdot\Abs{\sum_{j\in f^{-1}\left(k\right)}\xi_{j}-\psi_{k}}\\
 & =\sum_{k\in B}\frac{1}{\sum_{j\in f^{-1}\left(k\right)}\xi_{j}}\cdot\Abs{\sum_{j\in f^{-1}\left(k\right)}\xi_{j}-\psi_{k}}\sum_{i\in f^{-1}(k)}\xi_{i}\\
 & =\sum_{k\in B}\Abs{\sum_{j\in f^{-1}\left(k\right)}\xi_{j}-\psi_{k}}\\
 & =\Norm{f_{*}\xi-\psi}_{1}
\end{align*}
\end{proof}
We are now ready to prove Proposition \ref{prop:sys-probs}. The idea
is as follows. Starting from some $y\in\U^{\flat}\setminus\Delta\B$,
we first use Proposition \ref{prop:cond-s} to ``project'' it to
$\Delta\B$ while preserving \emph{one} of the probabilities in the
system and possibly ``spoiling'' the others. We then ``fix'' each
of the other probabilities sequentially using the two previous lemmas,
until we finally get some $\xi\in\U^{+}$ whose distance from $y$
can be bounded by the distances incurred at each step of this process.
\begin{proof}
[Proof of Proposition \ref{prop:sys-probs}]Assume without loss of
generality that $\PSet=[n]$ for some $n\geq1$ and consider any $y\in\U^{\flat}\setminus\Delta\B$.
Let $\alpha$ be defined as in Lemma \ref{lem:fix-prob}, denote $\delta:=d_{1}(y,\Delta\B$)
and define $\U_{0}$ by

\[
\U_{0}:=\SC{y\in\R^{\B}}{\alpha_{0}\left(f_{*}y\right)=p_{0}}
\]

Notice that $\U\subseteq\U_{0}$ and in particular $y\in\U_{0}$.
By Proposition \ref{prop:cond-s}, $\sin(\U_{0}^{\flat},\Delta\B)=1$.
Hence, there exists $\xi_{1}\in\Delta\B$ s.t.

\[
\alpha_{0}\left(f_{*}\xi_{1}\right)=p_{0}
\]

\[
\Norm{y-\xi_{1}}_{1}=\delta
\]

Denote $\psi_{1}:=f_{*}\xi_{1}$. Observe that for any $0<i<n$,

\begin{align*}
\Abs{\alpha_{i}\left(\psi_{1}\right)-p_{i}} & =\Abs{\alpha_{i}\left(\psi_{1}\right)-\alpha_{i}\left(f_{*}y\right)}\\
 & \leq\frac{1}{2}\Norm{\psi_{1}-f_{*}y}_{1}\\
 & =\frac{1}{2}\Norm{f_{*}\xi_{1}-f_{*}y}_{1}\\
 & \leq\frac{1}{2}\Norm{\xi_{1}-y}_{1}\\
 & =\frac{1}{2}\delta
\end{align*}

Here, we extended the definition of $\alpha_{i}$ to $\R^{2^{\PSet}}$
such as to make it a linear functional.

Applying Lemma \ref{lem:fix-prob} by induction, we construct $\{\psi_{i}\in\Delta2^{\PSet}\}_{2\leq i\leq n}$
s.t.:
\begin{itemize}
\item For any $1\leq i\leq n$ and $0\leq j<i$, $\alpha_{j}(\psi_{i})=p_{j}$.
\item For any $1\leq i\leq n$ and $i\leq j<n$, $\alpha_{j}(\psi_{i})=\alpha_{j}(\psi_{1})$
and in particular $|\alpha_{j}(\psi_{i})-p_{j}|\leq\frac{1}{2}\delta$.
\item For any $1\leq i<n$, $\Norm{\psi_{i}-\psi_{i+1}}_{1}\leq\delta$.
\end{itemize}
By the triangle inequality,

\[
\Norm{\psi_{1}-\psi_{n}}_{1}\leq\left(n-1\right)\delta
\]

Applying Lemma \ref{lem:lift-tvd} we get $\xi\in\Delta\B$ s.t.
\begin{itemize}
\item $f_{*}\xi=\psi_{n}$
\item $\Norm{\xi_{1}-\xi}\leq\left(n-1\right)\delta$
\end{itemize}
By the construction of $\psi_{n}$, the first item above implies that
$\xi\in\U$. We get that

\begin{align*}
d_{1}\left(y,\U^{+}\right) & \leq\Norm{y-\xi}_{1}\\
 & \leq\Norm{y-\xi_{1}}_{1}+\Norm{\xi_{1}-\xi}_{1}\\
 & =\delta+\left(n-1\right)\delta\\
 & =n\delta\\
 & =\Abs{\PSet}d_{1}\left(y,\Delta\B\right)
\end{align*}

By Lemma \ref{lem:l1}, this implies

\[
\sin\left(\U^{\flat},\Delta\B\right)\geq\frac{1}{\Abs{\PSet}}
\]
\end{proof}

\subsection{$S$ for Chain of Conditional Probabilities}

Now, we start working towards the proof of Proposition \ref{prop:chain-s}.
The following shows that if some $a^{*}\in\B$ is in the ``support''
of some $\U\in\Gany{\R^{\B}}$, then it is also in the ``support''
of $\U^{\flat}$.

\global\long\def\SuppUZ{\mathcal{E}}%
 
\begin{lem}
\label{lem:flat-supp}Let $\B$ be a finite set, $\U\subseteq\R^{\B}$
a linear subspace. Assume there is some $y\in\U$ s.t. $\sum_{a}y_{a}=1$.
Let $\SuppUZ\subseteq\B$ be the minimal set s.t. $\U\subseteq\R^{\SuppUZ}$.
Then, for any $a^{*}\in\SuppUZ$, there exists $y^{*}\in\U$ s.t.
$\sum_{a}y_{a}^{*}=1$ and $y_{a^{*}}^{*}\ne0$.
\end{lem}
\begin{proof}
Let $y\in\U$ be s.t. $\sum_{a}y_{a}=1$. In case $y_{a^{*}}\ne0$,
take $y^{*}:=y$. In case $y_{a^{*}}=0$, let $y'\in\U$ be s.t. $y'_{a^{*}}\ne0$
(it exists because $a^{*}\in\SuppUZ$). In case $\sum_{a}y'_{a}\ne0$,
take

\[
y^{*}:=\frac{y'}{\sum_{a}y'_{a}}
\]

In case $\sum_{a}y'_{a}=0$, take $y^{*}:=y+y'$. We get

\begin{align*}
\sum_{a}y_{a}^{*} & =\sum_{a}y_{a}+\sum_{a}y'_{a}\\
 & =1+0\\
 & =1
\end{align*}

and,

\begin{align*}
y_{a^{*}}^{*} & =y_{a^{*}}+y'_{a^{*}}\\
 & =y'_{a^{*}}\\
 & \ne0
\end{align*}
\end{proof}
When analyzing sines related to Proposition \ref{prop:chain-s}, it
will be inconvenient to consider completely arbitrary $y\in\U^{\flat}\setminus\Delta\B$.
Instead, we will restrict $y$ to a certain dense subset of ``non-degenerate''
vectors.
\begin{lem}
\label{lem:chain-generic}Let $\CB 0$, $\CB 1$ be finite sets, $\B:=\CB 0\times\CB 1$,
$\U_{0}\in\Gany{\R^{\CB 0}}$ and $\U_{1}:\CB 0\rightarrow\Gany{\R^{\CB 1}}$.
Let $\SuppUZ\subseteq\CB 0$ be the minimal set s.t. $\U_{0}\subseteq\R^{\SuppUZ}$.
Given $y\in\R^{\B}$, define $y^{0}\in\R^{\CB 0}$ by

\[
y_{a}^{0}:=\sum_{b\in\CB 1}y_{ab}
\]

Given also $a\in\CB 0$, define $y^{a}\in\R^{\CB 1}$ by

\[
y_{b}^{a}:=y_{ab}
\]

Define $\U$ by

\[
\U:=\SC{y\in\R^{\B}}{y^{0}\in\U_{0}\text{, }\forall a\in\SuppUZ:y^{a}\in\U_{1}(a)\text{ and }\forall a\in\CB 0\setminus\SuppUZ:y^{a}=0}
\]

Define also $\U^{*}$ by

\[
\U^{*}:=\SC{y\in\U^{\flat}}{\exists v\in\prod_{a\in\CB 0}\U_{1}(a)^{\flat}\:\forall a\in\CB 0:y^{a}=y_{a}^{0}v(a)}
\]

Then, $\U^{*}$ is dense in $\U^{\flat}$.
\end{lem}
\begin{proof}
Consider any $a_{0}\in\SuppUZ$. By Lemma \ref{lem:flat-supp}, there
is some $\hat{u}\in\U_{0}^{\flat}$ s.t. $\hat{u}_{a_{0}}\ne0$. Choose
any $\hat{v}\in\prod_{a\in\CB 0}\U_{1}(a)^{\flat}$ and define $\hat{y}\in\R^{\B}$
by $\hat{y}_{ab}:=\hat{u}_{a}\hat{v}(a)_{b}$. We have

\begin{align*}
\hat{y}_{a}^{0} & =\sum_{b\in\CB 1}\hat{y}_{ab}\\
 & =\sum_{b\in\CB 1}\hat{u}_{a}\hat{v}(a)_{b}\\
 & =\hat{u}_{a}\sum_{b\in\CB 1}\hat{v}(a)_{b}\\
 & =\hat{u}_{a}
\end{align*}

Hence, $\hat{y}^{0}=\hat{u}\in\U_{0}$. Moreover, for any $a\in\SuppUZ$,
$\hat{y}^{a}=\hat{u}_{a}\hat{v}(a)\in\U_{1}(a)$. And, for any $a\in\CB 0\setminus\SuppUZ$,
$\hat{u}_{a}=0$ and hence $\hat{y}^{a}=0$. We conclude that $\hat{y}\in\U$.
Also,

\begin{align*}
\sum_{a,b}\hat{y}_{ab} & =\sum_{a,b}\hat{u}_{a}\hat{v}(a)_{b}\\
 & =\sum_{a}\hat{u}_{a}\sum_{b}\hat{v}(a)_{b}\\
 & =\sum_{a}\hat{u}_{a}\\
 & =1
\end{align*}

Therefore, $\hat{y}\in\U^{\flat}$. And, $\hat{y}_{a_{0}}^{0}=\hat{u}_{a_{0}}\ne0$.
Since $y_{a_{0}}^{0}=0$ is a linear condition on $y$, it follows
$y_{a_{0}}^{0}\ne0$ for all $y\in\U^{\flat}$ except a set of measure
zero w.r.t. the intrinsic Lebesgue measure of $\U^{\flat}$. Hence,
for almost all (in the same sense) $y\in\U^{\flat}$ it holds that
for any $a\in\SuppUZ$, $y_{a}^{0}\ne0$.

It remains to show that any $y$ with the latter property is in $\U^{*}$.
Indeed, consider some such $y\in\U^{\flat}$. Choose some $v'\in\prod_{a}\U_{1}(a)^{\flat}$
and define $v\in\prod_{a}\R^{\CB 1}$ by

\[
v(a):=\begin{cases}
\frac{y^{a}}{y_{a}^{0}} & \text{if }a\in\SuppUZ\\
v'(a) & \text{if }a\not\in\SuppUZ
\end{cases}
\]

Since for any $a\in\CB 1$, $y^{a}\in\U_{1}(a)$, it follows that
$v\in\prod_{a}\U_{1}(a)^{\flat}$. Finally, let's check that $y^{a}=y_{a}^{0}v(a)$.
For $a\in\SuppUZ$ we have

\begin{align*}
y_{a}^{0}v(a) & =y_{a}^{0}\cdot\frac{y^{a}}{y_{a}^{0}}\\
 & =y^{a}
\end{align*}

For $a\not\in\SuppUZ$, $y_{a}^{0}=0$ since $y^{0}\in\U_{0}$. And,
$y^{a}=0$ since $y\in\U$. Hence $y_{a}^{0}v(a)=0=y^{a}$. 
\end{proof}
The total variation distance between two probability distributions
on a product space can be bounded in terms of the total variation
distances between the associated marginal and conditional distributions.
The following is an extension of this bound to the affine hull of
the simplex of distributions.
\begin{lem}
\label{lem:sdp-distance}Let $\CB 0$, $\CB 1$ be finite sets, $u,u'\in\R^{\CB 0}$,
$v:\CB 0\rightarrow\R^{\CB 1}$, $v':\CB 0\rightarrow\Delta\CB 1$.
Assume that $\sum_{a}u_{a}=1$, $\sum_{a}u'_{a}=1$ and $\sum_{b}v(a)_{b}=1$.
Define $y,y'\in\R^{\CB 0\times\CB 1}$ by $y_{ab}:=u_{a}v(a)_{b}$
and $y'_{ab}:=u'_{a}v'(a)_{b}$. Then,

\[
\Norm{y-y'}_{1}\leq\Norm{u-u'}_{1}+\sum_{a\in\CB 0}\Abs{u_{a}}\cdot\Norm{v(a)-v(a)'}_{1}
\]
\end{lem}
\begin{proof}
We have,

\begin{align*}
\Norm{y-y'}_{1} & =\sum_{a,b}\Abs{y_{ab}-y'_{ab}}\\
 & =\sum_{a,b}\Abs{u_{a}v(a)_{b}-u'_{a}v'(a)_{b}}\\
 & =\sum_{a,b}\Abs{u_{a}v(a)_{b}-u{}_{a}v'(a)_{b}+u_{a}v'(a)_{b}-u'_{a}v'(a)_{b}}\\
 & =\sum_{a,b}\Abs{u_{a}\left(v(a)_{b}-v'(a)_{b}\right)+\left(u_{a}-u'_{a}\right)v'(a)_{b}}\\
 & \leq\sum_{a,b}\Abs{u_{a}\left(v(a)_{b}-v'(a)_{b}\right)}+\sum_{a,b}\Abs{\left(u_{a}-u'_{a}\right)v'(a)_{b}}\\
 & =\sum_{a}\Abs{u_{a}}\sum_{b}\Abs{v(a)_{b}-v'(a)_{b}}+\sum_{a}\Abs{u_{a}-u'_{a}}\sum_{b}v'(a)_{b}\\
 & =\sum_{a}\Abs{u_{a}}\cdot\Norm{v(a)-v(a)'}+\Norm{u-u'}_{1}
\end{align*}
\end{proof}
The following is essentially the special case of Proposition \ref{prop:chain-s}
for $n=2$. The proof proceeds by decomposing $y\in\U^{\flat}\setminus\Delta\B$
into ``marginal'' and ``conditionals'' (using the non-degeneracy
condition afforded by Lemma \ref{lem:chain-generic}) and then by
``projecting'' each of them to the intersection of its respective
subspace with its respective simplex and combining the resulting distributions.
\begin{lem}
\label{lem:short-chain-s}In the setting of Lemma \ref{lem:chain-generic},

\[
\sin\left(\U^{\flat},\Delta\B\right)\geq\min\left(\sin\left(\U_{0}^{\flat},\Delta\CB 0\right),\min_{a\in\CB 0}\sin\left(\U_{1}(a)^{\flat},\Delta\CB 1\right)\right)
\]
\end{lem}
\begin{proof}
Consider any $y\in\U^{*}$. Let $v\in\prod_{a}\U_{1}(a)^{\flat}$
be s.t. $y^{a}=y_{a}^{0}v(a)$. By Lemma \ref{lem:l1}, the relevant
norms for our purpose are $\ell_{1}$. By Lemma \ref{lem:simplex-dist},

\begin{align*}
d_{1}\left(y,\Delta\B\right) & =\sum_{a,b}\Abs{y_{ab}}-1\\
 & =\sum_{a,b}\Abs{y_{a}^{0}v(a)_{b}}-1\\
 & =\sum_{a}\Abs{y_{a}^{0}}\sum_{b}\Abs{v(a)_{b}}-\sum_{a}\Abs{y_{a}^{0}}+\sum_{a}\Abs{y_{a}^{0}}-1\\
 & =\sum_{a}\Abs{y_{a}^{0}}\left(\sum_{b}\Abs{v(a)_{b}}-1\right)+\sum_{a}\Abs{y_{a}^{0}}-1\\
 & =\sum_{a}\Abs{y_{a}^{0}}d_{1}\left(v(a),\Delta\CB 1\right)+d_{1}\left(y^{0},\Delta\CB 0\right)
\end{align*}

Let $\xi^{0}\in\Delta\CB 0$ and $\xi^{1}:\CB 0\rightarrow\Delta\CB 1$
be

\[
\xi^{0}:=\Argmin{\xi\in\U_{0}\cap\Delta\CB 0}\Norm{y^{0}-\xi}_{1}
\]

\[
\xi^{1}(a):=\Argmin{\xi\in\U_{1}(a)\cap\Delta\CB 1}\Norm{v(a)-\xi}_{1}
\]

Define $\xi^{*}\in\U\cap\Delta\B$ by $\xi_{ab}^{*}:=\xi_{a}^{0}\xi^{1}(a)_{b}$.
Denote 

\[
S_{\min}:=\min\left(\sin\left(\U_{0}^{\flat},\Delta\CB 0\right),\min_{a}\sin\left(\U_{1}(a)^{\flat},\Delta\CB 1\right)\right)
\]

By Lemma \ref{lem:sdp-distance},

\begin{align*}
d_{1}\left(y,\U\cap\Delta\B\right) & \leq\Norm{y-\xi^{*}}\\
 & \leq\sum_{a}\Abs{y_{a}^{0}}\cdot\Norm{v(a)-\xi^{1}(a)}_{1}+\Norm{y^{0}-\xi^{0}}_{1}\\
 & =\sum_{a}\Abs{y_{a}^{0}}d_{1}\left(v(a),\U_{1}(a)\cap\Delta\CB 1\right)+d_{1}\left(y^{0},\U_{0}\cap\Delta\CB 0\right)\\
 & \leq\sum_{a}\Abs{y_{a}^{0}}\cdot\frac{d_{1}\left(v(a),\Delta\CB 1\right)}{\sin\left(\U_{1}(a)^{\flat},\Delta\CB 1\right)}+\frac{d_{1}\left(y^{0},\Delta\CB 0\right)}{\sin\left(\U_{0}^{\flat},\Delta\CB 0\right)}\\
 & \leq\sum_{a}\Abs{y_{a}^{0}}\cdot\frac{d_{1}\left(v(a),\Delta\CB 1\right)}{S_{\min}}+\frac{d_{1}\left(y^{0},\Delta\CB 0\right)}{S_{\min}}\\
 & =\frac{1}{S_{\min}}\left(\sum_{a}\Abs{y_{a}^{0}}d_{1}\left(v(a),\Delta\CB 1\right)+d_{1}\left(y^{0},\Delta\CB 0\right)\right)\\
 & =\frac{1}{S_{\min}}\cdot d_{1}\left(y,\Delta\B\right)
\end{align*}

This holds for any $y\in\U^{*}$. And, by Lemma \ref{lem:chain-generic},
such $y$ are dense in $\U^{\flat}$ and hence

\[
\sin\left(\U^{\flat},\Delta\B\right)\geq S_{\min}
\]
\end{proof}
In the following, we will use the following shorthand notation for
portions of tuples:

\[
x_{:i}:=x_{0}x_{1}\ldots x_{i-1}
\]

\[
x_{i:j}:=x_{i}x_{i+1}\ldots x_{j-1}
\]

We will require a characterization of the ``support'' of $\U$ from
Proposition \ref{prop:chain-s} (we state it only in the direction
we need but the converse also holds).
\begin{lem}
\label{lem:chain-supp}In the setting of Proposition \ref{prop:chain-s},
let $\SuppUZ$ be the minimal subset of $\B$ s.t. $\U\subseteq\R^{\B}$.
Let $a^{*}\in\B\setminus\SuppUZ$. Then, there exists an integer $i<n$
s.t.

\[
a_{i}^{*}\not\in\SuppU i{a_{:i}^{*}}
\]
\end{lem}
\begin{proof}
Assume to the contrary that for any $i<n$,

\[
a_{i}^{*}\in\SuppU i{a_{:i}^{*}}
\]

For every $i<n$, choose some $u^{i}:\PB i\rightarrow\R^{\CB i}$
s.t.
\begin{itemize}
\item For any $a\in\PB i$, $u^{i}(a)\in\U_{i}^{\flat}$.
\item $u^{i}(a_{:i}^{*})_{a_{i}^{*}}\ne0$ (possible by Lemma \ref{lem:flat-supp}).
\end{itemize}
Define $y\in\R^{\B}$ by

\[
y_{a}:=\prod_{i<n}u^{i}\left(a_{:i}\right)_{a_{i}}
\]

Let's check that $y\in\U$. 

For any $i<n$ and $a\in\PB i$, we have

\begin{align*}
y_{b}^{a} & = & \sum_{c\in\NB i} & y_{abc}\\
 & = & \sum_{c\in\NB i} & \left(\prod_{j<i}u^{j}\left(a_{:j}\right)_{a_{j}}\cdot u^{i}\left(a\right)_{b}\cdot\prod_{i<j<n}u^{j}\left(abc_{i+1:j}\right)_{c_{j}}\right)\\
 & = &  & \prod_{j<i}u^{j}\left(a_{:j}\right)_{a_{j}}\cdot u^{i}\left(a\right)_{b}\cdot\\
 &  &  & \sum_{c_{i+1}}u^{i+1}\left(ab\right)_{c_{i+1}}\sum_{c_{i+2}}u^{i+2}\left(abc_{i+1}\right)_{c_{i+2}}\ldots\sum_{c_{n-1}}u^{n-1}\left(abc_{i+1:n-1}\right)_{c_{n-1}}\\
 & = &  & \prod_{j<i}u^{j}\left(a_{:j}\right)_{a_{j}}\cdot u^{i}\left(a\right)_{b}\cdot\\
 &  &  & \sum_{c_{i+1}}u^{i+1}\left(ab\right)_{c_{i+1}}\sum_{c_{i+2}}u^{i+2}\left(abc_{i+1}\right)_{c_{i+2}}\ldots\sum_{c_{n-2}}u^{n-2}\left(abc_{i+1:n-2}\right)_{c_{n-2}}\\
 & = &  & \ldots\\
 & = &  & \prod_{j<i}u^{j}\left(a_{:j}\right)_{a_{j}}\cdot u^{i}\left(a\right)_{b}
\end{align*}

Here, we repeatedly used the fact that $\sum_{c_{j}}u^{j}(x)_{c_{j}}=1$
to eliminate the nested sums. 

The first factor on the right hand side doesn't depend on $b$, hence
$y^{a}$ is a scalar times $u^{i}(a)$ and is therefore in $\U_{i}(a)$.
Moreover, for any $b\in\CB i\setminus\SuppU ia$ and $c\in\NB i$,
we have

\begin{align*}
y_{abc} & =\sum_{c\in\NB i}\left(\prod_{j<i}u^{j}\left(a_{:j}\right)_{a_{j}}\cdot u^{i}\left(a\right)_{b}\cdot\prod_{i<j<n}u^{j}\left(abc_{i+1:j}\right)_{c_{j}}\right)\\
 & =\sum_{c\in\NB i}\left(\prod_{j<i}u^{j}\left(a_{:j}\right)_{a_{j}}\cdot0\cdot\prod_{i<j<n}u^{j}\left(abc_{i+1:j}\right)_{c_{j}}\right)\\
 & =0
\end{align*}

Now, observe that $y_{a^{*}}\ne0$ since it is a product of non-vanishing
factors. But, this implies $a^{*}\in\SuppUZ$, which is a contradiction.
\end{proof}
We are now ready to prove Proposition \ref{prop:chain-s}. The proof
works by observing that $\U$ can be described as the recursive application
of the construction in Lemma \ref{lem:chain-generic} and using Lemma
\ref{lem:short-chain-s} by induction.
\begin{proof}
[Proof of Proposition \ref{prop:chain-s}]We use induction on $n$.

For $n=1$, $y\in\U$ if and only if: $y\in\U_{0}$ and for all $b\in\CB 0\setminus\SuppUZ_{0}$,
$y_{b}=0$. But, the latter condition follows from the former, by
definition of $\SuppUZ_{0}$. Hence, $\U=\U_{0}$ and the claim is
true.

\global\long\def\CBz{\mathcal{\tilde{G}}_{0}}%

\global\long\def\CBo{\mathcal{\tilde{G}}_{1}}%

\global\long\def\Uz{\tilde{\U}_{0}}%

\global\long\def\Uo{\tilde{\U}_{1}}%

\global\long\def\SuppT{\mathcal{\tilde{E}}}%

\global\long\def\Ut{\tilde{\U}}%

Now, assume the claim for some $n\geq1$, and let's show it for $n+1$.
For any $i<n$, define

\[
\CB{(+)}^{i}:=\prod_{i<j<n}\CB j
\]

Notice that this is not the same thing as $\CB +^{i}$, since in this
context $\CB +^{i}=\prod_{i<j<n+1}\CB j$ (because we are proving
the claim for $n+1$). Define also

\[
\CBz:=\PB n
\]

\[
\CBo:=\CB n
\]

\begin{align*}
\Uz:=\{y\in\R^{\PB n}\mid\forall i<n,a\in\PB i: & \:y^{a}\in\U_{i}(a)\text{ and }\\
 & \:\forall b\in\CB i\setminus\SuppU ia,c\in\CB{(+)}^{i}:y_{abc}=0\}
\end{align*}

\[
\Uo:=\U_{n}
\]

Given $y\in\R^{\B}$, define $y^{0}\in\R^{\PB n}$ by

\[
y_{a}^{0}:=\sum_{b\in\CB n}y_{ab}
\]

Let $\SuppT\subseteq\PB n$ be the minimal set s.t. $\Uz\subseteq\R^{\SuppT}$.
Finally, define

\[
\tilde{\U}:=\SC{y\in\R^{\B}}{y^{0}\in\Uz\text{, }\forall a\in\SuppT:y^{a}\in\U_{n}(a)\text{ and }\forall a\in\PB n\setminus\SuppT:y^{a}=0}
\]

Observe that for any $y\in\R^{\B}$, $i<n$ and $a\in\PB i$, we have
$y^{a}=(y^{0})^{a}$ because

\begin{align*}
y_{b}^{a} & =\sum_{c\in\NB i}y_{abc}\\
 & =\sum_{d\in\CB{(+)}^{i}}\sum_{e\in\CB n}y_{abde}\\
 & =\sum_{d\in\CB{(+)}^{i}}y_{abd}^{0}\\
 & =\left(y^{0}\right)_{b}^{a}
\end{align*}

Let's show that $\Ut\subseteq\U$. Consider any $y\in\Ut$ and let's
check that $y\in\U$. 

For any $i<n$ and $a\in\PB i$, $y^{a}=(y^{0})^{a}$. Since $y^{0}\in\Uz$,
we get $(y^{0})^{a}\in\U_{i}(a)$ and hence $y^{a}\in\U_{i}(a)$. 

Moreover, for any $b\in\CB i\setminus\SuppU ia$, $d\in\CB{(+)}^{i}$
and $y'\in\Uz$, we have $y'_{abd}=0$. Therefore, $abd\not\in\SuppT$.
Hence, $y^{abd}=0$, meaning that for any $e\in\CB n$, $y_{abde}=0$.

For any $a\in\PB n$, if $a\in\SuppT$ then $y^{a}\in\U_{n}(a)$.
If $a\not\in\SuppT$, then $y^{a}=0\in\U_{n}(a)$. Hence, in either
case $y^{a}\in\U_{n}(a)$.

Moreover, for any $b\in\CB n\setminus\SuppU na$, $y_{ab}=y_{b}^{a}=0$,
because $y^{a}\in\U_{n}(a)$.

Now, let's show that $\U\subseteq\Ut$. Consider any $y\in\U$ and
let's check that $y\in\Ut$.

First, we need to check that $y^{0}\in\Uz$. For any $i<n$ and $a\in\PB i$,
$(y^{0})^{a}=y^{a}\in\U_{i}(a)$. For any $b\in\CB i\setminus\SuppU ia$
and $d\in\CB{(+)}^{i}$,

\begin{align*}
y_{abd}^{0} & =\sum_{e\in\CB n}y_{abde}\\
 & =\sum_{e\in\CB n}0\\
 & =0
\end{align*}

For any $a\in\SuppT$, we have $y^{a}\in\U_{n}(a)$ since $y\in\U$.
For any $a\in\PB n\setminus\SuppT$, Lemma \ref{lem:chain-supp} implies
that there is some $i<n$ s.t. $a_{i}\not\in\SuppU i{a_{:i}}$. Hence,
for any $b\in\CB n$,

\begin{align*}
y_{b}^{a} & =y_{ab}\\
 & =y_{a_{:i}a_{i}a_{i+1:n}b}\\
 & =0
\end{align*}

We got that $\U=\Ut$. Hence, Lemma \ref{lem:short-chain-s} applied
to $\CBz$, $\CBo$, $\Uz$, $\Uo$ implies that

\[
\sin\left(\U^{\flat},\Delta\B\right)\geq\min\left(\sin\left(\Uz^{\flat},\Delta\PB n\right),\min_{a\in\PB n}\sin\left(\U_{n}(a)^{\flat},\Delta\CB n\right)\right)
\]

Applying the induction hypothesis to $\Uz$, we conclude

\begin{align*}
\sin\left(\U^{\flat},\Delta\B\right) & \geq\min\left(\min_{\SUBSTACK{i<n}{a\in\PB i}}\sin\left(\U_{i}(a)^{\flat},\Delta\CB i\right),\min_{a\in\PB n}\sin\left(\U_{n}(a)^{\flat},\Delta\CB n\right)\right)\\
 & =\min_{\SUBSTACK{i<n+1}{a\in\PB i}}\sin\left(\U_{i}(a)^{\flat},\Delta\CB i\right)
\end{align*}
\end{proof}

\subsection{$R$ for Chain of Conditional Probabilities}

Now, we move towards the proof of Proposition \ref{prop:chain-r}.
We will need a technical lemma about the $\ell_{1}$ distances of
affine subspaces from the origin.
\begin{lem}
\label{lem:mu-cost}Let $\B$ be a finite set, $\U\subseteq\R^{\B}$
a linear subspace s.t. $\U\cap\Delta\B\ne\varnothing$, $u_{0}\in\R^{\B}$
and $t\in\R$. Define $\mu:\R^{\B}\rightarrow\R$ by $\mu(y)=\sum_{a}y_{a}$.
Then,

\[
\min_{y\in\left(\U+u_{0}\right)\cap\mu^{-1}\left(t\right)}\Norm y_{1}\leq2\min_{y\in\U+u_{0}}\Norm y_{1}+\Abs t
\]
\end{lem}
\begin{proof}
Define $y^{*}$ by

\[
y^{*}:=\Argmin{y\in\U+u_{0}}\Norm y_{1}
\]

Take any $\xi\in\U\cap\Delta\B$. Define $y^{!}$ by

\[
y^{!}:=y^{*}+\left(t-\mu\left(y^{*}\right)\right)\xi
\]

We have,

\begin{align*}
\mu\left(y^{!}\right) & =\mu\left(y^{*}+\left(t-\mu\left(y^{*}\right)\right)\xi\right)\\
 & =\mu\left(y^{*}\right)+\left(t-\mu\left(y^{*}\right)\right)\mu\left(\xi\right)\\
 & =\mu\left(y^{*}\right)+t-\mu\left(y^{*}\right)\\
 & =t
\end{align*}

Hence, $y^{!}\in(\U+u_{0})\cap\mu^{-1}(t)$. Moreover,

\begin{align*}
\Norm{y^{!}}_{1} & =\Norm{y^{*}+\left(t-\mu\left(y^{*}\right)\right)\xi}_{1}\\
 & \leq\Norm{y^{*}}_{1}+\Abs{t-\mu\left(y^{*}\right)}\cdot\Norm{\xi}_{1}\\
 & =\Norm{y^{*}}_{1}+\Abs{t-\mu\left(y^{*}\right)}\\
 & \leq\Norm{y^{*}}_{1}+\Abs{\mu\left(y^{*}\right)}+\Abs t\\
 & \leq2\Norm{y^{*}}_{1}+\Abs t
\end{align*}
\end{proof}
We now demonstrate a bound on the norm on $\W$ in terms of the norms
on the direct summands $\W_{a}$. It is achieved by repeatedly applying
the previous lemma to construct the needed vector.
\begin{lem}
\label{lem:chain-w}In the setting of Proposition \ref{prop:chain-r},
for any $x\in\A$ and $\theta\in\H$, $F_{x\theta}$ is onto. Moreover,
for any $w\in\W$

\[
\Norm w\leq2\sum_{\SUBSTACK{i<n}{a\in\PB i}}\Norm{w_{a}}
\]

Here, the norms on the right hand side are induced by $F_{a}$ and
$\H_{a}$ in the usual way.
\end{lem}
\begin{proof}
In the following, we will use Lemma \ref{lem:l1} implictly to justify
the appearance of $\ell_{1}$ norms. We use induction on $n$.

For $n=1$, the sum on the right hand side has only one term which
is clearly equal to $\Norm w$. Hence, the inequality holds.

Now, assume the claim for some $n\geq1$ and let's show it for $n+1$.
Consider any $x\in\A$, $\theta\in\H$ and $w\in\W$. It is enough
to show that there is some $y^{*}\in\R^{\B}$ (which may depend on
$x$ and $\theta$) s.t.

\[
F_{x\theta}y^{*}=w
\]

\[
\Norm{y^{*}}_{1}\leq2\sum_{\SUBSTACK{i<n}{a\in\PB i}}\Norm{w_{a}}
\]

Let $\Z^{n}$, $\W^{n}$, $F^{n}$ and $\H^{n}$ be defined in the
same way as $\Z$, $\W$, $F$ and $\H$ respectively, except for
$n$ instead of $n+1$. We have

\[
\W\cong\W^{n}\oplus\bigoplus_{a\in\PB n}\W_{a}
\]

Denote by $w_{:n}$ the projection of $w$ to $\W^{n}$. By definition
of the norm on $\W^{n}$, we can choose some $\tilde{y}\in\R^{\PB n}$
s.t.

\[
F_{x\theta}^{n}\tilde{y}=w_{:n}
\]

\[
\Norm{\tilde{y}}_{1}\leq\Norm{w_{:n}}
\]

For each $a\in\PB n$, apply Lemma \ref{lem:mu-cost} with 

\[
\U:=\ker F_{ax\theta_{a}}
\]

\[
u_{0}\in F_{ax\theta_{a}}^{-1}(w_{a})
\]

\[
t:=\tilde{y}_{a}
\]

Notice that

\[
\U+u_{0}=F_{ax\theta_{a}}^{-1}(w_{a})
\]

This produces $u^{a}\in\R^{\CB n}$ s.t.

\[
F_{ax\theta_{a}}u^{a}=w_{a}
\]

\[
\mu\left(u^{a}\right)=\tilde{y}_{a}
\]

\begin{align*}
\Norm{u^{a}}_{1} & \leq2\min_{y\in\R^{\CB n}:F_{ax\theta_{a}}y=w_{a}}\Norm y_{1}+\Abs{\tilde{y}_{a}}\\
 & \leq2\Norm{w_{a}}+\Abs{\tilde{y}_{a}}
\end{align*}

Define $y^{*}\in\R^{\B}$ by $y_{ab}^{*}:=u_{b}^{a}$ for any $a\in\PB n$
and $b\in\CB n$. Let's check that $F_{x\theta}y^{*}=w$.

For any $i<n$, define

\[
\CB{(+)}^{i}:=\prod_{i<j<n}\CB j
\]

For any $i<n$ and $a\in\PB i$, we have

\begin{align*}
\left(y^{*}\right)_{b}^{a} & =\sum_{c\in\NB i}y_{abc}^{*}\\
 & =\sum_{d\in\CB{(+)}^{i}}\sum_{e\in\CB n}y{}_{abde}^{*}\\
 & =\sum_{d\in\CB{(+)}^{i}}\sum_{e\in\CB n}u_{e}^{abd}\\
 & =\sum_{d\in\CB{(+)}^{i}}\mu\left(u_{e}^{abd}\right)\\
 & =\sum_{d\in\CB{(+)}^{i}}\tilde{y}_{abd}\\
 & =\tilde{y}_{b}^{a}
\end{align*}

We got $(y^{*})^{a}=\tilde{y}^{a}$. It follows that,

\begin{align*}
\left(F_{x\theta}y^{*}\right)_{a} & =F_{ax\theta_{a}}\left(y^{*}\right)^{a}\\
 & =F_{ax\theta_{a}}\tilde{y}^{a}\\
 & =\left(F_{x\theta}^{n}\tilde{y}\right)_{a}\\
 & =\left(w_{:n}\right)_{a}\\
 & =w_{a}
\end{align*}

Now, consider any $a\in\PB n$. We have,

\begin{align*}
\left(y^{*}\right)_{b}^{a} & =y_{ab}^{*}\\
 & =u_{b}^{a}
\end{align*}

Hence, $(y^{*})^{a}=u^{a}$. It follows that,

\begin{align*}
\left(F_{x\theta}y^{*}\right)_{a} & =F_{ax\theta_{a}}\left(y^{*}\right)^{a}\\
 & =F_{ax\theta_{a}}u^{a}\\
 & =w_{a}
\end{align*}

It remains only to bound the norm of $y^{*}$. We have,

\begin{align*}
\Norm{y^{*}}_{1} & =\sum_{a\in\B}\Abs{y_{a}^{*}}\\
 & =\sum_{b\in\PB b}\sum_{c\in\CB n}\Abs{y_{bc}^{*}}\\
 & =\sum_{b\in\PB b}\sum_{c\in\CB n}\Abs{u_{c}^{b}}\\
 & =\sum_{b\in\PB b}\Norm{u^{b}}_{1}\\
 & \leq\sum_{b\in\PB b}\left(2\Norm{w_{b}}+\Abs{\tilde{y}_{b}}\right)\\
 & =2\sum_{b\in\PB b}\Norm{w_{b}}+\sum_{b\in\PB b}\Abs{\tilde{y}_{b}}\\
 & =2\sum_{b\in\PB b}\Norm{w_{b}}+\Norm{\tilde{y}}_{1}\\
 & \leq2\sum_{b\in\PB b}\Norm{w_{b}}+\Norm{w_{:n}}
\end{align*}

Applying the induction hypothesis to the second term on the right
hand side, we get

\begin{align*}
\Norm{y^{*}}_{1} & \leq2\sum_{a\in\PB a}\Norm{w_{a}}+2\sum_{\SUBSTACK{i<n}{a\in\PB i}}\Norm{w_{a}}\\
 & =2\sum_{\SUBSTACK{i<n+1}{a\in\PB i}}\Norm{w_{a}}
\end{align*}
\end{proof}
Equipped with the previous lemma, the proof of Proposition \ref{prop:chain-r}
is straightforward.
\begin{proof}
[Proof of Proposition \ref{prop:chain-r}]Once again, we will use
Lemma \ref{lem:l1} implictly. First, observe that for any $y\in\R^{\B}$
and $i<n$, we have

\begin{alignat*}{1}
\sum_{a\in\PB i}\Norm{y^{a}}_{1} & =\sum_{a\in\PB i}\sum_{b\in\CB i}\Abs{y_{b}^{a}}\\
 & =\sum_{a\in\PB i}\sum_{b\in\CB i}\Abs{\sum_{c\in\NB i}y_{abc}}\\
 & \leq\sum_{a\in\PB i}\sum_{b\in\CB i}\sum_{c\in\NB i}\Abs{y_{abc}}\\
 & =\Norm y_{1}
\end{alignat*}

Now, let's analyze $R(\H,F):$

\begin{align*}
R\left(\H,F\right) & =\max_{\theta\in\H}\Norm{\theta}\\
 & =\max_{\theta\in\H}\max_{x\in\A}\Norm{F_{x\theta}}\\
 & =\max_{\SUBSTACK{\theta\in\H}{x\in\A}}\max_{y\in\R^{\B}:\Norm y=1}\Norm{F_{x\theta}y}
\end{align*}

Applying Lemma \ref{lem:chain-w} to the right hand side, we get

\begin{align*}
R\left(\H,F\right) & \leq2\max_{\SUBSTACK{\theta\in\H}{x\in\A}}\max_{y\in\R^{\B}:\Norm y=1}\sum_{\SUBSTACK{i<n}{a\in\PB i}}\Norm{\left(F_{x\theta}y\right)_{a}}\\
 & =2\max_{\SUBSTACK{\theta\in\H}{x\in\A}}\max_{y\in\R^{\B}:\Norm y=1}\sum_{\SUBSTACK{i<n}{a\in\PB i}}\Norm{F_{ax\theta}y^{a}}\\
 & \leq2\max_{y\in\R^{\B}:\Norm y=1}\max_{\SUBSTACK{\theta\in\H}{x\in\A}}\sum_{\SUBSTACK{i<n}{a\in\PB i}}\Norm{F_{ax\theta}}\cdot\Norm{y^{a}}_{1}\\
 & \leq2\max_{y\in\R^{\B}:\Norm y=1}\sum_{\SUBSTACK{i<n}{a\in\PB i}}\max_{\SUBSTACK{\theta\in\H}{x\in\A}}\Norm{F_{ax\theta}}\cdot\Norm{y^{a}}_{1}\\
 & =2\max_{y\in\R^{\B}:\Norm y=1}\sum_{\SUBSTACK{i<n}{a\in\PB i}}R\left(\H_{a},F_{a}\right)\Norm{y^{a}}_{1}\\
 & \leq2\left(\max_{\SUBSTACK{i<n}{a\in\PB i}}R\left(\H_{a},F_{a}\right)\right)\max_{y\in\R^{\B}:\Norm y=1}\sum_{\SUBSTACK{i<n}{a\in\PB i}}\Norm{y^{a}}_{1}\\
 & \leq2\left(\max_{\SUBSTACK{i<n}{a\in\PB i}}R\left(\H_{a},F_{a}\right)\right)\sum_{i<n}\max_{y\in\R^{\B}:\Norm y=1}\sum_{a\in\PB i}\Norm{y^{a}}_{1}\\
 & \leq2\left(\max_{\SUBSTACK{i<n}{a\in\PB i}}R\left(\H_{a},F_{a}\right)\right)\sum_{i<n}\max_{y\in\R^{\B}:\Norm y=1}\Norm y_{1}\\
 & =2\left(\max_{\SUBSTACK{i<n}{a\in\PB i}}R\left(\H_{a},F_{a}\right)\right)\sum_{i<n}1\\
 & =2n\max_{\SUBSTACK{i<n}{a\in\PB i}}R\left(\H_{a},F_{a}\right)
\end{align*}
\end{proof}

\end{document}